\documentclass[twoside,11pt]{article}

\usepackage{blindtext}

\usepackage{url}

\usepackage{algorithm}
\usepackage{algorithmic}

\usepackage{microtype}
\usepackage{booktabs}

\usepackage{amsmath}
\usepackage{mathtools}
\usepackage{comment}
\usepackage{dsfont}

\usepackage[disable,textsize=tiny]{todonotes}

\newcommand{\marc}[1]{\textcolor{black}{#1}}
\newcommand{\clemence}[1]{\textcolor{black}{#1}}
\newcommand{\modif}[1]{\textcolor{black}{#1}}

\newcommand{\NBATCHES}{T}
\newcommand{\DELTA}{\delta}
\newcommand{\NARMS}{K}

\newcommand{\THRESHOLD}{\theta}
\newcommand{\GUESS}[1]{\hat{a}_{#1}}

\newcommand{\KLm}{\textrm{KL}}

\newcommand{\perr}[3]{P^\text{err}_{#1,#2}(#3)}

\newcommand{\EventStickyAlgo}{\cE_{\NBATCHES}}

\newcommand{\ARMS}{\mathcal{A}}
\newcommand{\R}{\mathbb{R}}
\newcommand{\rR}{\mathbb{R}}
\newcommand{\nN}{\mathbb{N}}
\newcommand{\bE}{\mathbb{E}}
\newcommand{\bP}{\mathbb{P}}

\newcommand{\empdeltap}[2]{\widehat{\Delta}^+_{#1}(#2)}
\newcommand{\empdeltam}[2]{\widehat{\Delta}^-_{#1}(#2)}

\newcommand{\reward}[2]{X_{#1,#2}}
\newcommand{\expmean}[2]{\hat{\mu}_{#1}(#2)}
\newcommand{\nsamples}[2]{N_{#1}(#2)}

\newcommand{\Wp}[2]{W^+_{#1}(#2)}
\newcommand{\Wm}[2]{W^-_{#1}(#2)}

\newcommand{\arm}[1]{a_{#1}}

\newcommand{\mean}[1]{\mu_{#1}}
\newcommand{\set}[1]{\ARMS_{#1}}

\newcommand{\deff}{:=}

\newcommand{\argmax}{\arg\,\max}
\newcommand{\argmin}{\arg\,\min}

\newcommand{\indic}[1]{\mathds{1}\left(#1\right)}
\newcommand{\cS}{\mathcal{S}}

\newcommand{\cE}{\mathcal{E}}

\newcommand{\cG}{\mathcal{G}}
\newcommand{\cN}{\mathcal{N}}

\newcommand{\TV}{\mathrm{TV}}

\newcommand{\ie}{\textit{i}.\textit{e}.~ }
\newcommand{\eg}{\textit{e}.\textit{g}.~ }

\newcommand{\qed}{\hfill\blacksquare}

\usepackage{todonotes}
\usepackage{xcolor}

\renewcommand{\ln}{\log}
\renewcommand{\epsilon}{\varepsilon}

\usepackage{array}
\usepackage{makecell}
\usepackage{subcaption}
\usepackage{layouts}
\usepackage{mwe}
\usepackage{cancel}

\makeatletter
\newcommand{\succprec}{\mathrel{\mathpalette\succ@prec{\succ\prec}}}
\newcommand{\precsucc}{\mathrel{\mathpalette\succ@prec{\prec\succ}}}

\newcommand{\succ@prec}[2]{\succ@@prec#1#2}
\newcommand{\succ@@prec}[3]{\vcenter{\m@th\offinterlineskip
    \sbox\z@{$#1#3$}\hbox{$#1#2$}\kern-0.4\ht\z@\box\z@
  }}
\makeatother

\usepackage[abbrvbib]{jmlr2e}

\usepackage{hyperref}
\usepackage[capitalize,noabbrev]{cleveref}
\hypersetup{hidelinks}

\usepackage{lastpage}
\jmlrheading{27}{2026}{1-\pageref{LastPage}}{5/24; Revised 7/25}{1/26}{24-0680}{Marc Jourdan and Andr\'{e}e Delahaye-Duriez and Cl\'{e}mence R\'{e}da}
\ShortHeadings{An Anytime Algorithm for Good Arm Identification}{Jourdan, Delahaye-Duriez and R\'{e}da}
\firstpageno{1}

\begin{document}

\title{An Anytime Algorithm for Good Arm Identification}

\author{\name Marc Jourdan\thanks{This work was mainly done when Marc Jourdan was a PhD student in the Scool team of Inria Lille.} \email marc.jourdan@epfl.ch \\
      \addr EPFL, Lausanne, Switzerland
      \AND 
      \name Andr\'{e}e Delahaye-Duriez \email andree.delahaye@inserm.fr \\
      \addr Universit\'{e} Paris Cit\'{e}, Inserm, NeuroDiderot, UMR-1141, 75019, Paris, France; Unit\'{e} fonctionnelle de m\'{e}decine g\'{e}nomique et g\'{e}n\'{e}tique clinique, H\^{o}pital Jean Verdier, Assistance Publique des H\^{o}pitaux de Paris, 93140, Bondy, France; Universit\'{e} Sorbonne Paris Nord, 93000, Bobigny, France
      \AND
      \name Cl\'{e}mence R\'{e}da$^+$\thanks{This work was mainly done when Cl\'{e}mence R\'{e}da was a postdoctoral fellow at Universit\'{e} Paris Cit\'{e}, Inserm, NeuroDiderot, UMR-1141, 75019, Paris, France, and at the Department of Systems Biology and Bioinformatics, University of Rostock, G-18051, Rostock, Germany.} \email reda@bio.ens.psl.eu \\
    \addr BioComp, Institut de Biologie de l’ENS (IBENS UMR 8197), D\'{e}partement de biologie, \'{E}cole normale supérieure, CNRS, Inserm, Universit\'{e} PSL, 75005 Paris, France
      \AND $^+$ Corresponding author
       }

\editor{Kevin Jamieson}

\maketitle

\begin{abstract}In good arm identification (GAI), the goal is to identify one arm whose average performance exceeds a given threshold, referred to as a good arm, if it exists. 
Few works have studied GAI in the fixed-budget setting when the sampling budget is fixed beforehand, or in the anytime setting, when a recommendation can be asked at any time. 
We propose APGAI, an anytime and parameter-free sampling rule for GAI in stochastic bandits. 
APGAI can be straightforwardly used in fixed-confidence and fixed-budget settings.
First, we derive upper bounds on its probability of error at any time. 
They show that adaptive strategies can be more efficient in detecting the absence of good arms than uniform sampling in several diverse instances. 
Second, when APGAI is combined with a stopping rule, we prove upper bounds on the expected sampling complexity, holding at any confidence level. 
Finally, we show the good empirical performance of APGAI on synthetic and real-world data. 
Our work offers an extensive overview of the GAI problem in all settings.
\end{abstract}


\section{Introduction}\label{sec:introduction}

Multi-armed bandit algorithms are a family of approaches which demonstrated versatility in solving online allocation problems, where constraints exist on the possible allocations: \eg randomized clinical trials~\citep{thompson1933likelihood,berry2006bayesian}, hyperparameter optimization~\citep{li2017hyperband,shang2018adaptive}, or active learning~\citep{carpentier2011upper}. 
The agents face a black-box environment, upon which they can sequentially act through actions called \textit{arms}. 
After sampling an arm $a \in \ARMS$, they receive output from the environment through a scalar observation, which is a realization from the unknown probability distribution $\nu_{a}$ of the arm $a$ whose mean will be denoted by $\mean{a}$. 
Depending on their objectives, agents should have different sampling strategies.

In \textit{pure exploration} problems, the goal is to answer a question about the set of arms. 
It is studied in two major theoretical frameworks \citep{audibert2010best,gabillon2012best,jamieson2014best,garivier2016optimal}: the \textit{fixed-confidence} and \textit{fixed-budget} setting.
In the fixed-confidence setting, the agent aims at minimizing the number of samples used to identify a correct answer with confidence $1 - \delta$\marc{, where $\delta \in (0,1)$ is a \emph{risk} parameter}.
In the fixed-budget setting, the objective is to minimize the probability of misidentifying a correct answer with a fixed number of samples $\NBATCHES$\marc{, where $\NBATCHES \in \mathbb N$ is a \emph{budget} parameter}.\looseness=-1

While $\delta$ or $\NBATCHES$ are assumed given, choosing them is challenging for the practitioner since a ``good'' choice typically depends on unknown quantities.
Moreover, in medical applications (\eg clinical trials or outcome scoring), the maximal budget is limited but might not be fixed beforehand.
\marc{Independently of the preliminary data, medical applications are prone to reductions in funding or new sources of funding.
Therefore, an experiment might stop before the initial budget has been used, referred to as \textit{early stopping}, or continue after it has been consumed, referred to as \textit{continuation}.}
When the collected data shows sufficient evidence in favor of one answer, an experiment often stops before reaching the initial budget.
\marc{Given that this early stopping is a data-dependent random variable, it differs fundamentally from the early stopping due to funding shortfalls.}
While early stopping and continuation are common in practice, both fixed-confidence and fixed-budget settings fail to provide meaningful guarantees for them.
Recently, the \textit{anytime} setting has received increased scrutiny as it fills this gap between theory and practice. 
In the anytime setting, \marc{for any fixed deterministic time $t$ that is unknown for the learner,} the agent aims at achieving a low probability of error at time \marc{$t$}~\citep{jun2016anytime,zhao2022revisiting,jourdan2023epsilonbestarm}.
\marc{While $T$ is fixed and known in the fixed-budget setting, $t$ is fixed and unknown in the anytime setting.}
When the candidate answer has anytime guarantees, the practitioners can use \marc{data-independent} continuation and early stopping. 
\marc{When combined with a stopping rule, the early stopping can be made data-dependent.}

The most studied topic in pure exploration is the \textit{best arm (BAI) / Top-$m$ identification} problem, which aims at determining a subset of $m$ arms with the largest means~\citep{karnin2013almost,xu2018fully,tirinzoni2022elimination}. 
However, in some applications (\eg investigating treatment protocols), BAI requires too many samples for it to be useful in practice. 
To avoid wasteful queries, practitioners focus on simpler tasks, \ie identifying one ``good enough'' option.
For instance, in $\varepsilon$-BAI \citep{MannorTsi04,even2006action,GK19Epsilon,jourdan2023epsilonbestarm}, the agent is interested in an arm which is $\varepsilon$-close to the best one, \ie $\mu_{a} \ge \max_{k \in \ARMS}\mu_{k} - \varepsilon$. The larger $\varepsilon$ is, the easier the task. However, choosing a meaningful value of $\varepsilon$ can be tricky. 
In this work, we focus on good arm identification (GAI), where the agent aims to obtain a \textit{good arm}, defined as an arm whose average performance exceeds a given threshold $\theta$, \ie $\mean{a} \geq \THRESHOLD$.
GAI and variants are studied in the fixed-confidence setting~\citep{kaufmann2018sequential,kano2019good,tabata2020bad}, but algorithms for fixed-budget or anytime GAI are missing, despite their practical relevance.
We fill this gap by introducing \hyperlink{APGAI}{APGAI}, an anytime and parameter-free sampling rule for GAI. 
\hyperlink{APGAI}{APGAI} is independent of a budget $T$ or a \marc{risk} $\delta$ and is performant in the fixed-budget and fixed-confidence settings. 

\clemence{Our work is motivated by a real-life outcome scoring problem to determine the best treatment protocol for treating the encephalopathy of prematurity in newborns with stem cell injections, in collaboration with the PREMSTEM consortium (see Section~\ref{sec:experiments}). In that case, practitioners have enough information about the distributions associated with each treatment protocol to define a meaningful threshold beforehand.}

\subsection{Problem Statement}

We denote by $\mathcal D$ a set to which the distributions of the arms are known to belong. 
We suppose that all distributions in $\mathcal D$ are $\sigma$-sub-Gaussian.
A distribution $\nu_{0}$ is $\sigma$-sub-Gaussian of mean $\mu_{0}$ if it satisfies $\bE_{X \sim \nu_{0}}[e^{\lambda (X - \mu_{0})}] \le e^{\sigma^2\lambda^2/2}$ for all $\lambda \in \rR$.
By rescaling, we assume $\sigma_{a} = 1$ for all $a \in \ARMS$.
Let $\ARMS$ be the set of arms of size $\NARMS$. 
A bandit instance is defined by unknown distributions $\nu \deff (\nu_{a})_{a \in \ARMS} \in \mathcal D^{K}$ with means $\mu \deff (\mu_{a})_{a \in \ARMS} \in \rR^{K}$.
Given a threshold $\THRESHOLD \in \rR$, the set of good arms is defined as $\set{\THRESHOLD}(\mu) \deff \left\{a \in \ARMS \mid \mean{a} \geq \THRESHOLD \right\}$, which we shorten to $\set{\THRESHOLD}$ when $\mu$ is unambiguous.
In the remainder of the paper, we assume that $\mean{a}\ne\THRESHOLD$ for all $a \in \ARMS$.
Let the gap of arm $a$ compared to $\theta$ be $\Delta_a \deff |\mean{a}-\THRESHOLD| > 0$.
Let $\Delta_{\min} = \min_{a \in \ARMS} \Delta_a $ be the minimum gap over all arms.
Let 
\begin{equation} \label{eq:common_complexity}
    H_{1}(\mu) \deff \sum_{a \in \ARMS} \Delta_a^{-2} \quad \text{and}  \quad H_{\theta}(\mu) \deff \sum_{a \in \set{\THRESHOLD}(\mu)}\Delta_{a}^{-2} \: .
\end{equation}
At time $t$, the agent chooses an arm $\arm{t} \in \ARMS$ based on past observations and receives a sample $\reward{\arm{t}}{t}$, random variable with conditional distribution $\nu_{\arm{t}}$ given $\arm{t}$. 
Let $\mathcal F_t \deff \sigma(a_{1},  \reward{\arm{1}}{1} , \cdots, a_{t},  \reward{\arm{t}}{t} )$ be the $\sigma$-algebra, called \textit{history}, which encompasses all the information available to the agent after $t$ rounds.

\textit{Identification algorithm.}
In the anytime setting, an \textit{identification} algorithm defines two rules which are $\mathcal F_{t}$-measurable at time $t$: a sampling rule $a_{t+1} \in \ARMS$ and a recommendation rule $\GUESS{t} \in \ARMS \cup \{ \emptyset \}$.
In GAI, the probability of error $\perr{\nu}{\mathfrak{A}}{t} \deff \bP_{\nu}(\cE^{\text{err}}_{\mathfrak{A}}(t))$ of algorithm $\mathfrak{A}$ on instance $\mu$ at time $t$ is the probability of the error event $\cE^{\text{err}}_{\mathfrak{A}}(t) = \{ \GUESS{t} \in \{\emptyset\} \cup (\ARMS \setminus \set{\THRESHOLD} )\}$ when $\set{\THRESHOLD} \neq \emptyset$, otherwise $\cE^{\text{err}}_{\mathfrak{A}}(t) = \{\GUESS{t}\neq\emptyset\}$ when $\set{\THRESHOLD} = \emptyset$.

Those rules have a different objective depending on the considered setting. 
In fixed-budget GAI, \marc{given a fixed and known budget $T$}, the goal is to have a low $\perr{\nu}{\mathfrak{A}_{T}}{T}$\marc{, where $\mathfrak{A}_{T}$ highlights the dependency in $T$ of the algorithm}.
In anytime GAI, \marc{the objective is} to ensure that $\perr{\nu}{\mathfrak{A}}{t}$ is small at any \marc{fixed time $t$, that is unknown for $\mathfrak{A}$.}
Whereas in fixed-confidence GAI, these two rules are complemented by a stopping rule using a confidence level $1-\delta$ fixed beforehand such that \marc{the algorithm} stops sampling after $\tau_{\delta}$ rounds.
The stopping time $\tau_{\delta}$ is also known as the \marc{(verifiable)} \textit{sample complexity} of a fixed-confidence algorithm.
\marc{A fixed-confidence algorithm $\mathfrak{A}_{\delta}$ always depends on $\delta$ due to the stopping time.
When the sampling and recommendation rules are independent of $\delta$ (\ie anytime) as in \hyperlink{APGAI}{APGAI}, the notation $\mathfrak{A}$ is used.}
At stopping time $\tau_{\delta}$, the algorithm should satisfy $\delta$-correctness, which means that $\bP_{\nu}(\{\tau_{\delta} < + \infty \} \cap \cE^{\text{err}}_{\mathfrak{A}}(\tau_{\delta})) \le \delta$ for all instances $\nu$.
That requirement leads to a lower bound on the expected sample complexity for any instance. 
The following lemma is similar to other bounds derived in various settings linked to GAI~\citep{kaufmann2018sequential,tabata2020bad}.
The proof in Appendix~\ref{app:ssec_Characteristic_times} relies on the change of measure inequality in Lemma $1$ from~\citet{kaufmann2016complexity}.
\begin{lemma}\label{lem:lower_bound_GAI}
    Let $\DELTA \in (0,1)$.
    For all $\DELTA$-correct algorithm and all Gaussian instances $\nu_{a} = \mathcal N(\mu_{a},1)$ with $\mu_a \neq \THRESHOLD$, we have $\liminf_{\delta \to 0} \bE_{\nu}[\tau_{\delta}]/\log (1/\delta) \ge T^\star(\mu)$, where
    \begin{equation} \label{eq:characteristicTime}
    T^\star(\mu) \deff 2\min_{a \in \set{\THRESHOLD}(\mu)} \Delta_a^{-2} \quad \text{ if } \set{\THRESHOLD}(\mu) \ne \emptyset \quad \text{ , and } \quad 2H_{1}(\mu) \quad \text{otherwise.}
    \end{equation}
\end{lemma}
\clemence{A fixed-confidence algorithm is \textit{asymptotically optimal} if it is $\delta$-correct, and its expected sample complexity matches the lower bound, \ie $\limsup_{\delta \to 0} \bE_{\nu}[\tau_{\delta}] / \log (1/\delta) \le T^\star(\mu) $.}

\marc{Introduced in~\citet{katz2020true}, the \emph{unverifiable sample complexity} $\tau_{U,\delta}$ is the minimum number of samples after which the algorithm always outputs a correct answer with probability at least $1-\delta$, namely $\bP_{\nu}(\bigcup_{t \ge \tau_{U,\delta}} \cE^{\text{err}}_{\mathfrak{A}}(\tau_{\delta})) \le \delta$ for all instances $\nu$.
Compared to the fixed-confidence setting, the unverifiable sample complexity of a strategy is not sufficient to stop and certify a correct output with confidence $1-\delta$. }

\marc{\textit{Notation.}
For two probability distributions $\mathbb P$ and $\mathbb Q$ on the measurable space $(\Omega, \cG)$, the Total Variation (TV) distance is $\TV(\mathbb P, \mathbb Q) \deff  \sup_{A \in \cG} |\mathbb P(A) - \mathbb Q(A)|$ and the Kullback-Leibler (KL) divergence is $\KLm(\mathbb P, \mathbb Q)  \deff \int \log\left(\frac{\mathrm{d} \mathbb P}{\mathrm{d} \mathbb Q} (\omega) \right) \mathrm{d} \mathbb P (\omega)$, when $\mathbb P \ll \mathbb Q$, and $+\infty$ otherwise.
For any stopping time $\tau$, let $\bP_{\nu}^{\tau}$ be the restriction of $\bP_{\nu}$ to the $\sigma$-algebra generated by $\tau$.
For any $\tau$-measurable event $E$, we have $\bP_{\nu}^{\tau}(E) = \bP_{\nu}(E)$.}

\subsection{Contributions} 

We \marc{introduce} \hyperlink{APGAI}{APGAI} \marc{(Algorithm~\ref{algo:AnytimeAlgo} in Section~\ref{sec:anytimeid})}, an anytime and parameter-free sampling rule for GAI in stochastic bandits, which is independent of a budget $T$ or a \marc{risk} $\delta$.
 \hyperlink{APGAI}{APGAI} is the first algorithm that can be employed without modification for fixed-budget GAI (and without prior knowledge of the budget) and fixed-confidence GAI.
Furthermore, it enjoys guarantees in both settings.
As such, \hyperlink{APGAI}{APGAI} allows both continuation and early stopping.
First, we show an upper bound on \marc{the probability of error of \hyperlink{APGAI}{APGAI} at any fixed and unknown time $t$} of the order $\exp(- \mathcal O(t / H_{1}(\mu)))$ which holds for any deterministic time $t$ (Theorem~\ref{thm:upper_bound_PoE_anytimeID} \marc{in Section~\ref{sec:anytimeguarantees}}).
Adaptive strategies are more efficient in detecting the absence of good arms than uniform sampling.
Second, \marc{we obtain a deterministic upper bound on the unverifiable sample complexity of \hyperlink{APGAI}{APGAI} holding at any confidence level and scaling as $\mathcal O( H_{1}(\mu) \log (H_{1}(\mu)/\delta ))$ (Theorem~\ref{thm:unverifiable_sample_complexity} in Section~\ref{sec:unverifiable_sample_complexity}).}
\marc{Third,} when combined with a GLR stopping rule \marc{(Lemma~\ref{lem:delta_correct_threshold})}, we derive a \marc{non-asymptotic} upper bound on \marc{the expected sample complexity of \hyperlink{APGAI}{APGAI}, whose $\delta$-independent term scales as $\mathcal O( H_{1}(\mu) \log H_{1}(\mu) )$} (Theorem~\ref{thm:expected_sample_complexity_upper_bound} \marc{in Section~\ref{sec:FCGAI}}).
\marc{For GAI with Gaussian distributions,} \hyperlink{APGAI}{APGAI} is asymptotically optimal when there is no good arm\marc{, yet it is suboptimal when there are good arms.}
\marc{Forth, when there exists a unique good arm and the risk is moderate, we show that a linear dependence in $K$ on the number of samples allocated to suboptimal arms is actually unavoidable (Theorem~\ref{thm:meta_lower_bound}, Corollaries~\ref{thm:lb_unverifiable_K_BAI} and~\ref{thm:lb_K_BAI}).}
\marc{Fifth}, \hyperlink{APGAI}{APGAI} is easy to implement, computationally inexpensive, and has good empirical performance in both settings on synthetic and real-world data with an outcome scoring problem for RNA-sequencing data (see Section~\ref{sec:experiments}).
\marc{Finally, we provide extensive theoretical and empirical comparisons with other GAI algorithms in all settings, while deriving new guarantees for them as well.
For clarity, the lower bounds are summarized in Table~\ref{tab:summary_lb_GAI} and the upper bounds are compared in Tables~\ref{tab:summary_anytimeGAI},~\ref{tab:summary_unverif_GAI} and~\ref{tab:summary_FCGAI}. 
Overall, our work offers a compelling overview of the GAI problem, which has previously received little attention despite its practical relevance.
}

\begin{table}[t]
    \centering
    \begin{tabular}{l l l c c }
   \toprule
          \marc{Setting} & & \marc{Performance Metric} & \marc{$\set{\THRESHOLD}(\mu)=\emptyset$} & \marc{$\set{\THRESHOLD}(\mu)\neq\emptyset$} \\
          \midrule
          \marc{FB} & \marc{[Thm~\ref{thm:thm6_degenne2023existence}]} & \marc{$\max_{a\in [K]} \limsup_{\NBATCHES \to + \infty} \frac{\NBATCHES}{- \log \perr{\nu^{(a)}}{\mathfrak{A}_{\NBATCHES}}{\NBATCHES}} $} & \marc{$-$} & \marc{$\frac{2K}{(\Delta+\epsilon)^2}$} \\
          \marc{UC} & \marc{[Cor~\ref{thm:lb_unverifiable_K_BAI}]} & \marc{$\max_{a\in [K]} \mathbb{E}_{\nu^{(a)}}[\tau_{U,\delta} - N_{a}(\tau_{U,\delta})]$} & \marc{$-$} & \marc{$\frac{K-1}{64(\Delta+\epsilon)^2}$}  \\     
          \marc{FC} & \marc{[Cor~\ref{thm:lb_K_BAI}]} & \marc{$\max_{a\in [K]} \mathbb{E}_{\nu^{(a)}}[\tau_{\delta} - N_{a}(\tau_{\delta})]$} & \marc{$-$} & \marc{$\frac{K-1}{64(\Delta +\epsilon)^2}$}  \\
          & \marc{[Lem~\ref{lem:lower_bound_GAI}]$\dagger$} & \marc{$\liminf_{\delta \to 0}\frac{\mathbb{E}_{\nu}[\tau_{\delta}]}{\log(1/\delta)}$} & \marc{$H_1(\mu)$} & \marc{$\overline \Delta_{\max}^{-2}$}  \\
         \bottomrule
    \end{tabular}
    \caption{\marc{Lower bound on the performance of any GAI algorithm for different objectives and metrics of performance: FC (fixed confidence), FB (fixed budget) and UC (unverifiable sample complexity).
    Let $(\nu^{(a)})_{a \in [K]}$ be the Gaussian instances defined in Theorem~\ref{thm:thm6_degenne2023existence} based on $(\Delta,\epsilon) \in  (\R_{+}^{\star})^2$, namely, for all $a \in [K]$, $\set{\THRESHOLD}(\nu^{(a)}) = \{a\}$, $\Delta_{a} = \Delta$ and $\Delta_{b} = \epsilon$ for all $b \ne a$.
    ($\dagger$) Holds for any instance $\nu$.
    $H_{1}(\mu)$ as in~Eq.~\eqref{eq:common_complexity}, $\Delta_{\min} \deff \min_{a \in \ARMS} \Delta_a$ and $\overline{\Delta}_{\max} \deff \max_{a \in \set{\THRESHOLD}} \Delta_a$.} \clemence{$N_{a}(t)$ is the number of samples pulled from arm $a$ up to time $t$ included.}}
    \label{tab:summary_lb_GAI}
\end{table}

\subsection{Related Works}

GAI is not studied in a fixed-budget or anytime setting yet.
In the fixed-confidence setting, several problems are considered that are similar to GAI. 

Given two thresholds $\THRESHOLD_L < \THRESHOLD_U$,~\citet{tabata2020bad,hayashi2024gaussian} study the Bad Existence Checking problem, in which the agent should output ``negative'' if $\set{\THRESHOLD_L}(\mu)=\emptyset$ and ``positive'' if $\set{\THRESHOLD_U}(\mu) \neq \emptyset$.
In particular, \cite{tabata2020bad} proposes an elimination-based meta-algorithm called BAEC, and analyzes its expected sample complexity when combined with several index policies to define the sampling rule. \cite{hayashi2024gaussian} focus on classification bandits with margin, which is a variant of the problem where the expected rewards are sampled from a Gaussian process prior, and describe a similar phased-elimination meta-algorithm that leverages the prior assumption.

\citet{kano2019good} considers identifying the whole set of good arms $\set{\THRESHOLD}(\mu)$ with high probability, and returns $\lambda$ good arms sequentially, where $\lambda \in \{1,2,\dots,|\set{\THRESHOLD}(\mu)|\}$.
We refer to that problem as AllGAI. 
~\citet{kano2019good} introduce three index-based GAI algorithms named APT-G, HDoC, and LUCB-G, and show upper bounds on their expected sample complexity. 
In the fixed-confidence setting and for Bernoulli distributions, \citet{tsai2024lil} built upon the HDoC algorithm for AllGAI, by fine-tuning the number of uniform pulls at the start of the HDoC algorithm. 
Their contribution is targeted at instances when one of the arms has an expected reward close to the threshold $\theta$ or if two arms have similar expected rewards. 
A variant of the HDoC algorithm copes for structured versions of fixed-confidence AllGAI, \eg see~\cite{tsai2023differential} for linear bandits where the expected reward depends on the arm's feature vector.

Numerous algorithms from previously mentioned works bear a passing resemblance to the APT algorithm proposed by~\citet{locatelli2016optimal} to tackle the thresholding bandit problem in the fixed-budget setting. 
The latter should classify all arms into $\set{\THRESHOLD}$ and $\set{\THRESHOLD}^\complement$ at the end of the sampling phase. 
The resemblance to the APT algorithm lies in that those prior works rely on an arm index for sampling. 
The arm indices in BAEC~\citep{tabata2020bad}, APT-G, HDoC and LUCB-G~\citep{kano2019good} are reported in Algorithm~\ref{algo:StickyAlgo} in Appendix~\ref{sec:stickyalgo}. 
However, it should be noted that our contribution \hyperlink{APGAI}{APGAI} does not feature an elimination algorithm as those algorithms do and that those prior works hold in a fixed-confidence setting and do not convert straightforwardly to the GAI problem.
Moreover, our analyses strongly differ from those present in these prior works. 
For linear bandits, this problem has also recently received attention in the fixed-confidence setting as well~\citep{rivera2024optimal}. \clemence{Other structured versions of thresholding bandits have also been recently considered. For instance,~\citet{cheshire2021problem} considered specific shape constraints on $\mu$, such as monotonic increasing or concave series of means, in a fixed-budget setting. Leveraging these strong assumptions on the ordering of arm means, authors showed that a lower bound on the asymptotic rate on the error probability roughly scales with $\Delta_\text{min}^{-2}$, without dependency on the number of arms $K$, and that nearly-matching---up to logarithmic factors---algorithms based on binary search exist.~\citet{mason2021nearly} studied linear kernel thresholding bandits in a fixed-confidence setting, where the arm means can be approximated in a Reproducing Kernel Hilbert Space (RKHS) with a known level of misspecification and proposed a nearly-matching algorithm for the linear (kernelized) setting. However, in our paper, we make no assumption on the structure of the bandit instance.}

\clemence{More loosely related works include the all-$\varepsilon$ good arm identification problem in a fixed-confidence setting, where the goal is to identify all arms in $\{a \mid \mean{a} \geq  \max_i \mean{i}-\varepsilon \}$ with high probability $1-\delta$~\citep{mason2020finding}. 
In the moderate confidence regime,~\citet{mason2020finding} derive a lower bound scaling as $H_1(\mu)$, where the sample complexity average over several instances whose best arm is separated by at least $2\beta$ from the other arms. 
Their proof builds on a reduction to the isolated instance testing problem (see Appendix D), where the goal is to detect whether an arm has mean $\beta$ or $-\beta$, while the other means are smaller than $-\beta$.}
\marc{It is possible to adapt~\citet[Algorithm 4]{mason2020finding} to solve isolated instance testing with a GAI algorithm for $\theta = 0$, with provable guarantees only on instances with a unique good arm. 
Leveraging this reduction,~\citet[Theorem D.6]{mason2020finding} yields a lower bound scaling as $H_1(\mu)$ on at least one of these instances with a unique good arm.
~\citet[Theorem D.6]{mason2020finding} is derived by using the \emph{Simulator} argument of~\citet{simchowitz17theSimulator} that builds non-stationary bandit instances.} 
\clemence{Keeping the core idea of non-stationarity,~\citet{al2022complexity} proposed a simpler proof technique to obtain lower bounds with a linear dependency in $K$.}
\marc{~\citet{poianibest} adapted their arguments to study BAI on Unimodal instances, where the mean vector is an unimodal function of its indices.
While Lemma~\ref{lem:lower_bound_GAI} suggests that only one arm should be sampled asymptotically for GAI with good arms, at most $3$ arms are needed according to the asymptotic lower bound for Unimodal BAI.
However,~\citet[Theorem 2.3]{poianibest} shows that a linear dependence in $K$ is unavoidable.
Building on their proof technique, we derive a general lower bound for any strategy whose stopping time satisfies a lower bound constraint on the $\TV$ distance between the distributions generated by interacting with instances having different answers (Theorem~\ref{thm:meta_lower_bound}). }

\clemence{Finally,}~\citet{degenne2019pure} addressed the ``any low arm'' problem, which is a GAI problem for threshold $-\theta$ on instance $-\mu$. 
They introduce Sticky Track-and-Stop, which is asymptotically optimal in the fixed-confidence setting.
In~\citet{kaufmann2018sequential}, the ``bad arm existence'' problem aims to answer ``no'' when $\set{-\THRESHOLD}(-\mu) = \emptyset$, and ``yes'' otherwise.
They propose an adaptation of Thompson Sampling conditioning on the ``worst event'' (named Murphy Sampling).
The empirical pulling proportions converge towards the allocation that realizes $T^\star(\mu)$ in Lemma~\ref{lem:lower_bound_GAI}.
Another related framework is the identification with a high probability of $k$ arms from $\set{\THRESHOLD}(\mu)$~\citep{katz2020true}. 
\section{Anytime Parameter-free Sampling Rule}\label{sec:anytimeid}

We propose the \hyperlink{APGAI}{APGAI} (\textbf{A}nytime \textbf{P}arameter-free \textbf{GAI}) algorithm, which is independent of a budget $T$ or a \marc{risk} $\delta$ and summarized in Algorithm~\ref{algo:AnytimeAlgo}.

\textit{Notation.}
Let $\nsamples{a}{t} = \sum_{s \leq t} \indic{\arm{s} = a}$ be the number of times arm $a$ is sampled at the end of round $t$, and $\expmean{a}{t} = \frac{1}{\nsamples{a}{t}}\sum_{s \leq t} \indic{\arm{s} = a}\reward{a}{s}$ be its empirical mean.
For all $a \in \ARMS$ and all $t \ge K$, let us define
\begin{equation}  \label{eq:information_accrual_fcts}
\Wp{a}{t}=\sqrt{\nsamples{a}{t}} \Delta_a(t)_{+} \quad \text{and} \quad \Wm{a}{t}=\sqrt{\nsamples{a}{t}} (-\Delta_a(t))_{+} \: ,
\end{equation}
where $(x)_+ \deff \max(x, 0)$ and $\Delta_a(t) \deff \expmean{a}{t}-\THRESHOLD$.
If arm $a$ were a $\sigma_{a}$-sub-Gaussian distribution, the rescaling boils down to using $\Delta_a(t)/\sigma_{a}$ instead of $\Delta_a(t)$.
This empirical transportation cost $\Wp{a}{t}$ (resp. $\Wm{a}{t}$) represents the amount of information collected so far in favor of the hypothesis that $\{\mu_{a} > \theta\}$ (resp. $\{\mu_{a} < \theta\}$).
It is linked with the generalized likelihood ratio (GLR) as detailed in Appendix~\ref{app:ssec_GLR}.
As initialization, we pull each arm once. 

\textit{Recommendation rule.} 
At time $t + 1 > K $, the recommendation rule depends on whether the highest empirical mean lies below the threshold $\THRESHOLD$ or not.
When $\max_{a \in \ARMS} \expmean{a}{t} \le \THRESHOLD$,  we recommend the empty set, \ie $\GUESS{t} = \emptyset$.
Otherwise, our candidate answer is the arm which is the most likely to be a good arm given the collected evidence, \ie $\GUESS{t} \in \argmax_{a \in \ARMS} \Wp{a}{t}$. 

\textit{Sampling rule.}
The next arm to pull is based on the APT$_{P}$ indices introduced by~\citet{tabata2020bad} as a modification to the APT indices~\citep{locatelli2016optimal}.
At time $t + 1 > K$, we pull arm $a_{t+1} \in \argmax_{a \in \ARMS} \sqrt{\nsamples{a}{t}} (\expmean{a}{t} - \THRESHOLD)$.
To emphasize the link with our recommendation rule, this sampling rule can also be written as $a_{t+1} \in \argmin_{a \in \ARMS} \Wm{a}{t}$ when $\max_{a \in \ARMS} \expmean{a}{t} \le \theta$, and $a_{t+1} \in \argmax_{a \in \ARMS} \Wp{a}{t}$ otherwise.
Ties are broken arbitrarily at random, up to the constraint that $\GUESS{t} = a_{t+1}$ when $\max_{a \in \ARMS} \expmean{a}{t} > \theta$.
This formulation better highlights the dual behavior of \hyperlink{APGAI}{APGAI}, which is reminiscent of the expression of the characteristic time $T^\star(\mu)$ in Lemma~\ref{lem:lower_bound_GAI}.
When $\max_{a \in \ARMS} \expmean{a}{t} \le \THRESHOLD$, \hyperlink{APGAI}{APGAI} collects additional observations to verify that there are no good arms, hence pulling the arm which is the least likely to not be a good arm.
Otherwise, \hyperlink{APGAI}{APGAI} gathers more samples to confirm its current belief that there is at least one good arm, hence pulling the arm that is the most likely to be a good arm.
In contrast to indices solely based on the empirical means, the APT$_{P}$ indices are linked to the empirical transportation costs, which account for the empirical counts.

\begin{algorithm}[t]
	\caption{\protect\hypertarget{APGAI}{APGAI}}
        \label{algo:AnytimeAlgo}
\begin{algorithmic}[1]
      \STATE {\bfseries Input:} threshold $\THRESHOLD$, set of arms $\ARMS$
      \STATE \marc{{\bfseries Initialization:} Draw each arm once}
      \STATE {\bfseries Update:} empirical means $\hat \mu(t)$ and empirical transportation costs $W^\pm_{a}(t)$ as in~Eq.~\eqref{eq:information_accrual_fcts}
     \IF{$\max_{a \in \ARMS} \expmean{a}{t} \le \theta$}
			\STATE $\GUESS{t} \deff \emptyset$ and $a_{t+1} \in \argmin_{a \in \ARMS} W^{-}_{a}(t)$
			\ELSE
			\STATE $\GUESS{t} \deff a_{t+1} \in \argmax_{a \in \ARMS} W^{+}_{a}(t)$
      \ENDIF
      \STATE {\bfseries return} arm to pull $a_{t+1}$ and recommendation $\GUESS{t}$ 
\end{algorithmic}
\end{algorithm}

\textit{Memory and computational cost.}
\hyperlink{APGAI}{APGAI} needs to maintain in memory the values $\nsamples{a}{t}, \expmean{a}{t}, W^{\pm}_{a}(t)$ for each arm $a \in \ARMS$, hence the total memory cost is in $\mathcal O(K)$. 
The computational cost of \hyperlink{APGAI}{APGAI} is in $\mathcal O(K)$ per iteration, and its update cost is in $\mathcal O(1)$.

\textit{Differences to BAEC.}
While both \hyperlink{APGAI}{APGAI} and BAEC(APT$_{P}$) rely on the APT$_{P}$ indices~\citep{tabata2020bad}, they differ significantly and we proceed differently from~\cite{tabata2020bad} in the analysis of \hyperlink{APGAI}{APGAI}, partially due to the lack of elimination in the latter.
BAEC is an elimination-based meta-algorithm that samples active arms and discards arms whose upper confidence bounds (UCB) on the empirical means are lower than $\theta_{U}$.
The recommendation rule of BAEC is only defined at the stopping time, and it depends on lower confidence bounds (LCB) and UCB.
Since the UCB/LCB indices depend inversely on the gap $\theta_{U} - \theta_{L} > 0$ and on the confidence $\delta$, BAEC is neither anytime nor parameter-free.
More importantly, \hyperlink{APGAI}{APGAI} can be used without modification for fixed-confidence or fixed-budget GAI.
In contrast, BAEC can solely be used in the fixed-confidence setting when $\theta_{U} > \theta_{L}$, hence not for GAI itself (\ie $\theta_{U} = \theta_{L}$).
 
\section{Anytime Guarantees on the Probability of Error} \label{sec:anytimeguarantees}

To allow continuation or (deterministic) early stopping, the candidate answer of \hyperlink{APGAI}{APGAI} should be associated with anytime theoretical guarantees.
Theorem~\ref{thm:upper_bound_PoE_anytimeID} shows an upper bound of the order $\exp(- \mathcal O(t / H_{1}(\mu)))$ for $\perr{\nu}{\mathfrak{A}}{t}$ that holds for any deterministic time $t$. 
\begin{theorem} \label{thm:upper_bound_PoE_anytimeID}
    The \hyperlink{APGAI}{APGAI} algorithm $\mathfrak{A}$ satisfies that, for all $\nu \in \mathcal D^{K}$ with mean $\mu$ such that $\Delta_{\min}  > 0$, for all $t > \NARMS +  2|\set{\THRESHOLD}|$,
    \[ 
    \perr{\nu}{\mathfrak{A}}{t} \le \NARMS e \sqrt{2} \log (e^2 t) \exp\left(- p \left(\frac{t - \NARMS -  2|\set{\THRESHOLD}|}{2 \alpha_{i_{\mu}} H_{1}(\mu)} \right) \right) \quad \text{with} \quad p(x) = x - 0.5\log x  \: ,
    \]
    where $H_{1}(\mu)$ as in~Eq.~\eqref{eq:common_complexity}, $(\alpha_{1}, \alpha_{\theta}) = (9, 2)$ and $i_{\mu} = 1 + (\theta - 1)\indic{\ARMS_{\theta}(\mu) \ne \emptyset}$.
\end{theorem}

While anytime upper bounds on the probability of error exist in ($\epsilon$-)BAI~\citep{zhao2022revisiting,jourdan2023epsilonbestarm}, Theorem~\ref{thm:upper_bound_PoE_anytimeID} is the first result of its kind for GAI. 
Our result holds for any deterministic time $t >  \NARMS +  2|\set{\THRESHOLD}|$ and any $1$-sub-Gaussian instance $\nu$.
In the asymptotic regime where $t \to + \infty$, Theorem~\ref{thm:upper_bound_PoE_anytimeID} shows that $\limsup_{t \rightarrow + \infty} t \log (1/\perr{\nu}{\mathfrak{A}}{t})^{-1} \le 2 \alpha_{i_{\mu}} H_{1}(\mu)$ for \hyperlink{APGAI}{APGAI} with $(\alpha_{1}, \alpha_{\theta}) = (9,2)$. 
We defer the reader to Appendix~\ref{app:anytimealgo_error} for detailed proof.

\textit{Comparison with uniform sampling.} 
Despite the practical relevance of anytime and fixed-budget guarantees, \hyperlink{APGAI}{APGAI} is the first algorithm enjoying guarantees on the probability of error in GAI at any time $t$ (hence at a given budget $\NBATCHES$).
As a baseline, we consider the uniform round-robin algorithm, named Unif, which returns the best empirical arm at time $t$ if its empirical mean is higher than $\THRESHOLD$, and returns $\emptyset$ otherwise. 
At a time $t$ such that $t/K \in \nN$, the recommendation of Unif is equivalent to the one used in \hyperlink{APGAI}{APGAI}, \ie $\argmax_{a \in \ARMS} \Wp{a}{t} = \argmax_{a \in \ARMS} \expmean{a}{t}$ since $N_{a}(t) = t/K$.
As the two algorithms differ in their sampling rule, we can measure the benefits of adaptive sampling.
Theorem~\ref{thm:uniform_sampling_PoE_recoAnytime} in Appendix~\ref{app:sssec_unif_PoE} gives anytime upper bounds on $\perr{\nu}{\text{Unif}}{t}$, and we compare it to the ones of Theorem~\ref{thm:upper_bound_PoE_anytimeID}.
In the asymptotic regime, the upper bound for Unif has a rate in $2\NARMS \Delta_{\min}^{-2}$ when $\set{\THRESHOLD}(\mu) = \emptyset$, and $4K \min_{a \in \set{\THRESHOLD}(\mu) }\Delta_{a}^{-2}$ otherwise.
While the latter rate is better than $2H_{1}(\mu)$ when arms have dissimilar gaps, \hyperlink{APGAI}{APGAI} has better guarantees than Unif when there is no good arm.
Our experiments show that \hyperlink{APGAI}{APGAI} can outperform Unif in many instances (\eg Figures~\ref{fig:results_PREMSTEM_graphics} and~\ref{fig:PoE_no_good_arms}, and the experiments in Appendix~\ref{app:details_experiments}), and is on par with it otherwise.
\marc{In particular, the upper bound derived for \hyperlink{APGAI}{APGAI} when $\set{\THRESHOLD} \ne \emptyset$ is not aligned with its good empirical performance. 
We conjecture that \hyperlink{APGAI}{APGAI} could have a better dependency than $H_{1}(\mu)$ when there are good arms, yet our non-asymptotic analysis is not tight enough to reveal it.
Proving this conjecture is an interesting direction for future work that requires finer non-asymptotic arguments.
Even with the tightest analysis, Theorem~\ref{thm:thm6_degenne2023existence} below shows that \hyperlink{APGAI}{APGAI} can not dominate Unif in all instances.}

\subsection{\marc{Lower Bound with Dependence on the Number of Arms}}
\label{ssec:lb_anytimeguarantees}

\citet{degenne2023existence} recently studied the existence of complexity in fixed-budget pure exploration.
For the fixed-budget GAI problem,~\citet[Theorem 6]{degenne2023existence} shows that uniform sampling is asymptotically minimax optimal for the risk measure $\limsup_{\NBATCHES \to + \infty} \frac{\NBATCHES}{- T^\star(\mu)\log \perr{\nu}{\mathfrak{A}_{\NBATCHES}}{\NBATCHES}}$ with a minimax risk equals to $K$, where $T^\star(\mu)$ as in~Eq.~\eqref{eq:characteristicTime}.
While $T^\star(\mu)$ is a complexity for the fixed-confidence setting,\marc{~\citet[Theorem 6]{degenne2023existence} refutes its existence for fixed-budget GAI if the class of algorithms contains the static proportions algorithms: the asymptotic rate on the probability of error cannot be smaller than $K T^\star(\mu)$ on all Gaussian instances $\nu$.
Based on~\citet[Corollary 4]{degenne2023existence}, Theorem~\ref{thm:thm6_degenne2023existence} states the intermediate result supporting this negative result.}
\begin{theorem}[Theorem 6 in~\citet{degenne2023existence}] \label{thm:thm6_degenne2023existence}
    \marc{Let $(\theta,\Delta,\epsilon) \in \R \times (\R_{+}^{\star})^2$.
For $a \in  [K]$, let $\nu^{(a)} \deff \cN(\mu^{(a)}, I_{K})$ where $\mu^{(a)}_{a} = \theta + \Delta$ and $\mu^{(a)}_{b} = \theta - \epsilon$ if $b \ne a$. 
Let $\nu^{(\emptyset)} \deff \cN(\mu^{(\emptyset)}, I_{K})$ where $\mu^{(\emptyset)}_{a} = \theta - \epsilon$ for all $a \in [K]$.
For any sequence of fixed-budget algorithms $(\mathfrak{A}_{\NBATCHES})_{\NBATCHES}$, we have either $-\log \perr{\nu^{(\emptyset)}}{\mathfrak{A}_{\NBATCHES}}{\NBATCHES} =_{T \to +\infty} o(T)$ or }
\begin{equation} \label{eq:hardness_degenne2023existence}
    \marc{\exists a \in [K], \quad \limsup_{\NBATCHES \to + \infty} \frac{\NBATCHES}{- \log \perr{\nu^{(a)}}{\mathfrak{A}_{\NBATCHES}}{\NBATCHES}}  \ge \frac{2K}{(\Delta+\epsilon)^2} =  \frac{K T^\star(\mu^{(a)})}{\left(1 + \epsilon/\Delta\right)^2} \:.}
\end{equation}
\end{theorem}
\begin{proof}
    \marc{Obtaining~Eq.~\eqref{eq:hardness_degenne2023existence} from~\citet[Corollary 4]{degenne2023existence} is done by using the definitions therein, Lemma~\ref{lem:lower_bound_GAI} and the symmetry of the KL divergence for Gaussian distributions with known variance.}
\end{proof}
\clemence{While not being valid for all instances, Theorem~\ref{thm:thm6_degenne2023existence} holds for the class of all algorithm families, which includes the static algorithms. 
Intuitively, an initial exploration phase is necessary: any algorithm has to sample all arms before starting to recommend the unique good arm (\ie the best one).
As an arm $a$ is sampled less than the others, the algorithm is slower on $\nu^{(a)}$.
Similarly, for fixed-budget BAI with $K=2$ and Bernoulli distributions,~\citet{wang2023uniformly} showed that an adaptive algorithm that performs as well as the uniform sampling algorithm on all instances can not outperform it in some instances. 
Within a large class of consistent and stable algorithms, uniform sampling is universally optimal.
Extending their result to an arbitrary number of arms is challenging, yet SR is worse than uniform sampling in some $3$-armed instances by comparing an asymptotic lower bound for the former with an upper bound for the latter.
Based on these prior results, one has little hope for a better bound in fixed-budget GAI for an arbitrary number of arms.}

Unif achieves the rate $K T^\star(\mu)$ when $\set{\THRESHOLD} \ne \emptyset$, but suffers from worse guarantees otherwise.
Conversely, \hyperlink{APGAI}{APGAI} achieves the rate in $T^\star(\mu)$ when $\set{\THRESHOLD} = \emptyset$, but has sub-optimal guarantees otherwise.
It does not conflict with~Eq.~\eqref{eq:hardness_degenne2023existence} \eg considering $\mu$ with $\set{\THRESHOLD} \ne \emptyset$ and an arm $a \in \ARMS$ with $\Delta_{a} \le \max_{a \in \set{\THRESHOLD}} \Delta_{a}/\sqrt{\NARMS/2 -1}$.

In fixed-budget GAI, a ``good'' algorithm has highly different sampling modes depending on whether there is a good arm or not.
Since committing to one of those modes too early will incur higher error, it is challenging to find the perfect trade-off adaptively.
While uniform sampling is asymptotically minimax optimal---with a worst-case difficulty ratio equal to $K$---, it is natural to ask whether, when adaptive sampling is available, one should ever rely on a non-adaptive design.
For BAI with $K>2$,~\citet{imbens2025admissibility} showed that there exist simple adaptive designs that universally and strictly dominate non-adaptive completely randomized trials in terms of efficiency exponent, defined as $\liminf_{t \to +\infty} - t^{-1} \ln(\max_{a \in \ARMS}\mu_{a} - \mathbb E_{\nu} [\mu_{\GUESS{t}}])$. 
Extending this dominance result to GAI would require a different comparison criterion, and we leave this as an interesting direction for future work.

\marc{\textit{Trade-off between the anytime and fixed-budget setting.}
The negative result of Theorem~\ref{thm:thm6_degenne2023existence} does not explicitly leverage the fact that the sequence of fixed-budget algorithms $(\mathfrak{A}_{\NBATCHES})_{\NBATCHES}$ have prior knowledge on the budget $T$. Therefore, it trivially holds for any anytime algorithm $\mathfrak{A}$. To the best of our understanding, it is challenging to incorporate this prior knowledge into the current information-theoretic proofs.
When considering the asymptotic rate, we conjecture that the knowledge of $T$ is ``irrelevant''. For large $T$, the probability of error is exponentially small: the algorithm already ``knows'' the unknown instance's correct answer. For small budget $T$, fixed-budget algorithms might have an ``hedge'' over anytime algorithms. Intuitively, any adaptive algorithm should behave closely to uniform sampling when $T$ is small compared to the difficulty of the instance (\ie too small for identification). Any deviation from this ``naive'' choice would incur a large probability of error on at least one alternative instance whose answer is different. Since the difficulty of the encountered instance is unknown, a fixed-budget algorithm should determine whether it has enough budget to be ``smarter'' than uniform, while staying close to it in case the budget is insufficient. An anytime algorithm should also understand whether it can be ``smarter'' than uniform sampling that is minimax optimal (Theorem~\ref{thm:thm6_degenne2023existence}). Yet, it does not know when it will evaluated (\ie $t$ is fixed but unknown). However, given the knowledge of $T$, a fixed-budget algorithm might anticipate this evaluation. If the budget is close to be reached without ``knowning'' the difficulty of the instance, it could behave closer to uniform sampling to minimize the probability of error by collecting information on all the arms. Despite being intuitive, we emphasize that the above distinction has no theoretical grounding yet (to the best of our knowledge). Given our current non-asymptotic techniques, it seems almost impossible to derive theoretical guarantees that truly capture this subtlety between the behaviors of anytime and fixed-budget algorithms. }

\subsection{Benchmark: Other \marc{Fixed-budget} GAI Algorithms}
\label{sec:ssec_benchmark_FB}

To go beyond the comparison with Unif, we propose and analyze additional GAI algorithms.
A summary of the comparison with \hyperlink{APGAI}{APGAI} is shown in Table~\ref{tab:summary_anytimeGAI}.

\subsubsection{From BAI to GAI Algorithms}
\label{sec:ssec_bai_to_gai}

Since a BAI algorithm outputs the arm with the highest mean, its GAI counterpart compares the empirical mean of the returned arm to the known threshold.
We study the GAI adaptations of two fixed-budget BAI algorithms: Successive Rejects (SR)~\citep{audibert2010best} and Sequential Halving (SH)~\citep{karnin2013almost}.
SR-G and SH-G return $\hat a_{\NBATCHES} = \emptyset$ when $\expmean{a_{\NBATCHES}}{\NBATCHES} \le \THRESHOLD$ and $\hat a_{T} = a_{\NBATCHES}$ otherwise, where $a_{\NBATCHES}$ is the arm that would be recommended for the BAI problem, \ie the arm that remains.

Theorems~\ref{thm:SH_PoE_recoElim} and~\ref{thm:SR_PoE_recoElim} in Appendix~\ref{app:guarantees_other_algorithms} give an upper bound on $\perr{\nu}{\text{SR-G}}{\NBATCHES}$ and $\perr{\nu}{\text{SH-G}}{\NBATCHES}$ at the fixed budget $\NBATCHES$. 
In the asymptotic regime, their rate is in $4 \log(K) \Delta_{\min}^{-2}$ when $\set{\THRESHOLD}(\mu) = \emptyset$, otherwise $\mathcal O (\log(K)\max\{ \max_{a \in \set{\THRESHOLD}} \Delta_a^{-2}, \max_{ i > I^\star} i (\max_{a \in \ARMS} \mu_{a} - \mu_{(i)})^{-2}\} )$ with $I^\star = |\argmax_{a \in \ARMS} \mu_{a}|$ and $\mu_{(i)}$ be the $i^\text{th}$ largest mean in vector $\mu$.
\marc{We emphasize that the notation $ \widetilde{\Delta}^{-2}$ in Table~\ref{tab:summary_anytimeGAI} ``hides'' the linear dependency in $K$ of this quantity.
~\citet[Section 6.1]{audibert2010best} shows that
\begin{equation} \label{eq:hidden_linear_K_dep}
    \max_{ i > I^\star} i (\max_{a \in \ARMS} \mu_{a} - \mu_{(i)})^{-2} = \tilde \Theta \left(   I^\star (\max_{a \in \ARMS} \mu_{a} - \mu_{( I^\star+1)})^{-2} + \sum_{i > I^\star} (\max_{a \in \ARMS} \mu_{a} - \mu_{(i)})^{-2} \right)\: ,
\end{equation}
where $\tilde \Theta(\cdot)$ hides a $\overline{\log}(K)$ factor.}
Recently,~\citet{zhao2022revisiting} provides a finer analysis of SH.
Using their results yields mildly improved rates.
\clemence{When there is one good arm with a large mean and the remaining arms have means slightly smaller than $\theta$, t}hose rates are better than $2H_{1}(\mu)$.
However, \hyperlink{APGAI}{APGAI} has better guarantees than SR-G and SH-G when there is \clemence{at least another} good arm with mean slightly smaller than the largest mean \marc{as $\widetilde{\Delta}^{-2}$ can become arbitrarily large}. \clemence{See the third column in Table~\ref{tab:summary_anytimeGAI}.}

\marc{\textit{Proof Sketch.} 
When $\set{\THRESHOLD}(\mu) = \emptyset$, the error event $\{\hat a_{\NBATCHES} \ne \emptyset\}$ implies that the last active arm $a_T$ satisfies $\expmean{a_{\NBATCHES}}{\NBATCHES} > \THRESHOLD$, even though $\mu_{a_{\NBATCHES}} \le \THRESHOLD$. 
As $a_T$ is sampled linearly, this event has low probability.
When $\set{\THRESHOLD}(\mu) \ne \emptyset$, the error event $\{\hat a_{\NBATCHES} = \emptyset\} \cup \{\hat a_{\NBATCHES} \in \set{\THRESHOLD}(\mu)^\complement \}$ implies that either (1) the last active arm $a_T$ satisfies $\expmean{a_{\NBATCHES}}{\NBATCHES} \le \THRESHOLD$ and $\mu_{a_{\NBATCHES}} > \THRESHOLD$, or (2) the last active arm $a_T$ satisfies $\mu_{a_{\NBATCHES}} < \THRESHOLD$, even though $\max_{a \in \ARMS} \mu_{a} > \THRESHOLD$. The first case is unlikely with the same argument as above. The second case is unlikely since it implies that the best arm has been eliminated, \ie this fixed-budget BAI algorithm has an error. Using known upper bound on the probability of error for SR~\citep{audibert2010best} and SH~\citep{karnin2013almost} concludes the proof.
We defer the reader to Appendices~\ref{app:ssec_SH_PoE} and~\ref{app:ssec_SR_PoE} for more details.}

\textit{Doubling trick.} 
The doubling trick allows the conversion of any fixed-budget algorithm into an anytime algorithm. 
It considers a sequence of algorithms that are run with increasing budgets $(T_{k})_{k \ge 1}$ and recommends the answer returned by the last instance.~\citet{zhao2022revisiting} shows that Doubling SH obtains the same guarantees as SH in BAI. Theorem~\ref{thm:SH_PoE_recoElim} also holds for its GAI counterpart DSH-G (resp. Theorem~\ref{thm:SR_PoE_recoElim} for DSR-G) at the cost of a multiplicative factor $4$ in the rate.
Empirically, our experiments show that \hyperlink{APGAI}{APGAI} is always better than DSR-G and DSH-G (Figures~\ref{fig:results_PREMSTEM_graphics} and~\ref{fig:PoE_no_good_arms}).

\textit{Other BAI algorithms.}
While we consider SR and SH as examples, most fixed-budget BAI algorithms can be converted into GAI algorithms.
For example,~\citet{wang2024best} recently introduced and analyzed two algorithms named CR-C and CR-A, where CR-C enjoys a better asymptotic rate than SR.
However, the analysis of~\citet{wang2024best} is purely asymptotic as they leverage the Large Deviation Principle.
Therefore, it departs from our objective to provide non-asymptotic upper bounds.
For completeness, we still provide an asymptotic analysis of SR-G using their tools (see Appendix~\ref{app:sssec_SR_PoE_LDP}).

\begin{table}[t]
    \centering
    \begin{tabular}{l c c c} 
    \toprule
          Algorithm $\mathfrak{A}$ & $\set{\THRESHOLD}(\mu)=\emptyset$ & $\set{\THRESHOLD}(\mu)\neq\emptyset$ & Dominance over  \\
           &  &   &  \hyperlink{APGAI}{APGAI} if $\set{\THRESHOLD}(\mu)\neq\emptyset$\\
          \midrule
         \hyperlink{APGAI}{APGAI} [Th~\ref{thm:upper_bound_PoE_anytimeID}] & $18H_1(\mu)$ & $4H_1(\mu)$ & \clemence{-- (anytime)}\\
         Unif [Th~\ref{thm:uniform_sampling_PoE_recoAnytime}] & $2\NARMS \Delta_{\min}^{-2}$ & $4\NARMS \overline \Delta_{\max}^{-2}$ & \clemence{$\succprec$ (anytime)} \\
         DSR-G [Th~\ref{thm:SH_PoE_recoElim}] &$16 \marc{\overline{\log}}(K) \Delta_{\min}^{-2} $ & $\marc{16}  \marc{\overline{\log}}(K) \widetilde{\Delta}^{-2} $ & \clemence{$\succprec$ (anytime)} \\
         DSH-G [Th~\ref{thm:SR_PoE_recoElim}] &$16  \marc{\lceil \log_{2}(K) \rceil} \Delta_{\min}^{-2} $ & $\marc{32}  \marc{\lceil \log_{2}(K) \rceil} \widetilde{\Delta}^{-2} $  & \clemence{$\succprec$ (anytime)}\\
         \hyperlink{PKGAI}{PKGAI}($\star$) [Th~\ref{th:APTlike_error}]$\dagger$ & $2H_1(\mu)$ & $2H_1(\mu)$ & \clemence{$\succ$ (fixed-budget)} \\
         \hyperlink{PKGAI}{PKGAI}(Unif) [Th~\ref{th:stickyalgo_error}]$\dagger$ & $2H_1(\mu)$ & $2K\hat{\Delta}^{-2}$ & \clemence{$\succ$ (fixed-budget)} \\
         \bottomrule
    \end{tabular}
    \caption{Asymptotic error rate $C(\mu)$ of algorithm $\mathfrak{A}$ on $\nu$, \ie $\limsup_{t} t (\log (1/\perr{\nu}{\mathfrak{A}}{t}))^{-1} \le C(\mu)$. $(\dagger)$ Fixed-budget algorithm $\mathfrak{A}_{T,\nu}$ with prior knowledge of $H_1(\nu)$ as in~Eq.~\eqref{eq:common_complexity}, $\Delta_{\min} \deff \min_{a \in \ARMS} \Delta_a$, $\overline \Delta_{\max}  \deff  \max_{a \in \set{\THRESHOLD}} \Delta_a$, $I^\star = |\argmax_{a \in \ARMS} \mu_{a}|$, $\widetilde{\Delta}^{-2}  \deff  \max\{\max_{a \in \set{\THRESHOLD}} \Delta_a^{-2}, \max_{i > I^\star} i(\max_{a \in \ARMS} \mean{a}-\mean{(i)})^{-2}\}$ depending linearly on $K$ as shown by Eq.~\eqref{eq:hidden_linear_K_dep}, $\hat{\Delta}  \deff  \max_{a \in \set{\THRESHOLD}} \Delta_a + \min_{a \not\in \set{\THRESHOLD}} \Delta_a$, $\overline{\log}(K)  \deff  \frac{1}{2} + \sum_{i=2}^{K}\frac{1}{i}$. The dominance of a bandit strategy is defined by the comparison of their \emph{known upper bounds} (smaller means better): $\prec$ (dominated), $\succ$ (dominant) and $\succprec$ (Pareto equivalent: in some cases dominant, in others dominated).}
    \label{tab:summary_anytimeGAI}
\end{table}

\subsubsection{Prior Knowledge-based GAI Algorithms}
\label{sec:sssec_pkgai}

Several fixed-budget BAI algorithms assume that the agent has access to some prior knowledge\clemence{, for instance, of the unknown quantity $H_1(\nu)$, } to design upper/lower confidence bounds (UCB/LCB), \eg UCB-E \citep{audibert2010best} and UGapEb \citep{gabillon2012best}.
While this assumption is often not realistic, it yields better guarantees.
We investigate those approaches for fixed-budget GAI. 
We propose an elimination-based meta-algorithm for fixed-budget GAI called~\hyperlink{PKGAI}{PKGAI} (\textbf{P}rior \textbf{K}nowledge-based GAI), described in Appendix~\ref{sec:stickyalgo}. 
As for BAEC,~\hyperlink{PKGAI}{PKGAI}($\star$) takes as input an index policy $\star$ which is used to define the sampling rule. \clemence{At each sampling round $t < T$, ~\hyperlink{PKGAI}{PKGAI}($\star$) samples the arm $a_t$ which maximizes the sampling index $\star$, updates the estimated upper and lower confidence bounds on the difference $\mean{a_t}-\THRESHOLD$, and eliminates any arm $a$ such that $\mean{a_t}-\THRESHOLD < 0$ with high probability.}
The \clemence{first} main difference to BAEC lies in the definition of the UCB/LCB since they depend both on the budget $\NBATCHES$ and on knowledge of $H_{1}(\mu)$ and $H_{\theta}(\mu)$. We provide upper confidence bounds on the probability of error at time $T$ holding for any choice of indices (Theorem~\ref{th:APTlike_error} for \hyperlink{PKGAI}{PKGAI}($\star$)) and uniform round-robin sampling (Theorem~\ref{th:stickyalgo_error} for \hyperlink{PKGAI}{PKGAI}(Unif)).
The obtained upper bounds on $\perr{\nu}{\text{PKGAI}}{\NBATCHES}$ are marginally lower than the ones obtained for \hyperlink{APGAI}{APGAI}, while \hyperlink{APGAI}{APGAI} does not require the knowledge of $H_{1}(\mu)$ and $H_{\theta}(\mu)$.

The \hyperlink{PKGAI}{PKGAI}($\star$) meta-algorithm allows us to convert prior fixed-confidence algorithms for related problems~\citep{kano2019good,tabata2020bad} into fixed-budget problems. \clemence{The second main difference with fixed-confidence prior works resides in the stopping rule. In the fixed-budget setting, we should accommodate for the data-poor regime where the number of possible samples $T$ is too small (Line $14$ in Algorithm $2$). If, at the end of the sampling phase, no remaining arm seems good, then we return the empty set.} 
This additional condition penalizes fixed-confidence algorithms when the budget is too small. As such, \hyperlink{PKGAI}{PKGAI}($\star$) represents a theoretically-supported baseline for our main algorithmic contribution \hyperlink{APGAI}{APGAI}, which is otherwise missing from the literature due to the lack of work on fixed-budget and anytime settings.
 
\section{\marc{Non-asymptotic Guarantees on the Unverifiable Sample Complexity}}
\label{sec:unverifiable_sample_complexity}

The \emph{unverifiable sample complexity} was defined by~\citet{katz2020true} as the smallest stopping time $\tau_{U,\delta}$ after which an algorithm $\mathfrak{A}$ always outputs a correct answer with probability at least $1-\delta$, \ie $\bP_{\nu}(\bigcup_{t \ge  \tau_{U,\delta}} \cE^{\text{err}}_{\mathfrak{A}}(t)) \le \delta$.
Compared to the fixed-confidence setting, it does not require to certify that the candidate answer is correct.
~\citet{zhao2022revisiting} notice that anytime bounds on the error can imply an unverifiable sample complexity bound.
\marc{Therefore, anytime guarantees on the probability of error are more fine-grained}. 
Theorem~\ref{thm:unverifiable_sample_complexity} gives a deterministic upper bound $U_{\delta}(\mu)$ on the unverifiable sample complexity $ \tau_{U,\delta}$ of \hyperlink{APGAI}{APGAI} for GAI \marc{for any risk $\delta$} (see Appendix~\ref{app:ssec_unverifiable_sample_complexity} for a proof).
While upper bounds on the unverifiable sample complexity $\tau_{U,\delta}$ are known in BAI~\citep{katz2020true,zhao2022revisiting,jourdan2023epsilonbestarm}, Theorem~\ref{thm:unverifiable_sample_complexity} is the first result for GAI. 

\begin{theorem} \label{thm:unverifiable_sample_complexity}
	Let $\delta \in (0,1)$.
	The \hyperlink{APGAI}{APGAI} algorithm satisfies that, for any $1$-sub-Gaussian distribution with mean $\mu$ such that $\Delta_{\min} > 0$, we have $\bP_{\nu}(\bigcup_{t > U_{\delta}(\mu)} \cE^{\text{err}}_{\mathfrak{A}}(t)) \le \delta$ with 
	\begin{equation*} U_{\delta}(\mu) = h_{2}(\marc{\delta}, \marc{6 \alpha_{i_{\mu}}} H_{1}(\mu),  K + 2|\set{\THRESHOLD}|) \: ,
	\end{equation*} 
    \marc{where $\alpha_{i_{\mu}}$ as in Theorem~\ref{thm:upper_bound_PoE_anytimeID}, and} $h_{2}(\marc{\delta}, A, B) \deff A \overline{W}_{-1} \left( \marc{\frac{1}{3}} \log \left( \marc{\frac{K\pi^2}{6\delta}}\right) +  B/A + \log (A) \right)$ satisfies that $h_{2}(\marc{\delta}, A, B) =_{\delta \to 0} A \log(1/\delta) \marc{/3}+ \mathcal O(\log\log(1/\delta)) $.
    \marc{Moreover, $U_{\delta}(\mu) =_{\Delta_{\min} \to + \infty} \mathcal O(H_{1}(\mu) \log H_{1}(\mu))$ and 
    $\limsup_{\delta \to 0} U_{\delta}(\mu) / \log(1/\delta) \le 2\alpha_{i_{\mu}} H_{1}(\mu)$.}
\end{theorem}

\marc{Intuitively, Theorem~\ref{thm:unverifiable_sample_complexity} is an aggregated counterpart to Theorem~\ref{thm:upper_bound_PoE_anytimeID}.
Instead of stating that the probability of error is low at any fixed time, the probability that there exists any error after a large enough time should be low.
However, Theorem~\ref{thm:unverifiable_sample_complexity} is not a direct corollary Theorem~\ref{thm:upper_bound_PoE_anytimeID} obtained by applying a naive union bound.
From a technical perspective, both results are a by-product of the same lower-level statement: for large enough time $t$, the error event $\cE^{\text{err}}_{\mathfrak{A}}(t)$ implies the concentration event does not hold, \ie the empirical means deviate significantly from their means. }

\begin{table}[t]
    \centering
    \begin{tabular}{lccc}
   \toprule
          Algorithm $\mathfrak{A}$ & $\set{\THRESHOLD}(\mu)=\emptyset$ & $\set{\THRESHOLD}(\mu)\neq\emptyset$ & \clemence{Dominance over \hyperlink{APGAI}{APGAI}}\\
          & & & \marc{when $\set{\THRESHOLD}(\mu)\neq\emptyset$}\\
          \midrule
         \hyperlink{APGAI}{APGAI} [Th~\ref{thm:unverifiable_sample_complexity}] & $ 36 H_1(\mu)$ & $ 8 H_{1}(\mu)$ & \marc{-- (anytime)}\\
         \marc{Unif [Th~\ref{thm:Unif_unverifiable_sample_complexity}]} & \marc{$2\NARMS \Delta_{\min}^{-2}$} & \marc{$8\NARMS \overline \Delta_{\max}^{-2}$} & \marc{$\succprec$ (anytime)} \\
         \bottomrule
    \end{tabular}
    \caption{
    Asymptotic upper bound $C(\mu)$ on the deterministic upper bound $U_{\delta}(\mu)$ on the unverifiable sample complexity $\tau_{U,\delta}$ of algorithm $\mathfrak{A}$ on $\nu$, \ie $\limsup_{\delta \to 0} U_{\delta}(\mu)/\log(1/\delta) \le C(\mu)$.
    $H_{1}(\mu)$ as in~Eq.~\eqref{eq:common_complexity}, $\Delta_{\min} \deff \min_{a \in \ARMS} \Delta_a$, $\overline{\Delta}_{\max} \deff \max_{a \set{\THRESHOLD}} \Delta_a$. \clemence{The dominance of a bandit strategy is defined by the comparison of their \emph{known upper bounds} (smaller means better): $\prec$ (dominated), $\succ$ (dominant) and $\succprec$ (Pareto equivalent).}}
    \label{tab:summary_unverif_GAI}
\end{table}

\marc{\textit{Comparison with uniform sampling.} 
We compare Theorem~\ref{thm:unverifiable_sample_complexity} for \hyperlink{APGAI}{APGAI} with the deterministic upper bound on the unverifiable sample complexity of Unif for GAI given by Theorem~\ref{thm:Unif_unverifiable_sample_complexity} in Appendix~\ref{app:sssec_unif_unverifiable}.
Similarly as in Table~\ref{tab:summary_anytimeGAI}, in the asymptotic regime described in Table~\ref{tab:summary_unverif_GAI}, the upper bound for Unif has a rate in $\NARMS \Delta_{\min}^{-2}$ when $\set{\THRESHOLD}(\mu) = \emptyset$, and $4K \min_{a \in \set{\THRESHOLD}(\mu) }\Delta_{a}^{-2}$ otherwise.
While the latter rate is better than $2H_{1}(\mu)$ when arms have dissimilar gaps, \hyperlink{APGAI}{APGAI} has better guarantees than Unif when there is no good arm.}

\marc{\textit{Time-uniform probability of error.}
Going one step further, one might be interested in controlling the probability that there exists any error, \ie $\bP_{\nu}(\bigcup_{t \ge t_0} \cE^{\text{err}}_{\mathfrak{A}}(t))$ where $t_0$ is an initialization time.
Corollary~\ref{cor:upper_bound_PoE_anytimeID} in Appendix~\ref{app:proof_cor_upper_bound_PoE_anytimeID} gives an upper bound on the time-uniform probability of error for \hyperlink{APGAI}{APGAI}.
Its proof combines Theorems~\ref{thm:upper_bound_PoE_anytimeID} and~\ref{thm:unverifiable_sample_complexity}, by using a union bound for the time $t \le U_{\delta}(\mu)$ and taking an infimum over $\delta$.
While time-uniform guarantees are appealing, they seem to be unrealistic, at least for challenging instances.
Therefore, we conjecture our bound is vacuous for hard instances, \ie bigger than one when $H_{1}(\mu)$ is large.
An interesting direction for future work is to characterize the maximal hardness of an instance on which an algorithm can obtain time-uniform guarantees.}

\subsection{\marc{Lower Bound with Dependence on the Number of Arms}}
\label{ssec:lb_unverifiable_sample_complexity}

\marc{When there is a unique good arm, Theorem~\ref{thm:unverifiable_sample_complexity} shows that the unverifiable sample complexity of \hyperlink{APGAI}{APGAI} is upper bounded by a quantity scaling linearly with $K$, both when the risk $\delta$ is moderate or arbitrarily small.
This dependency stems from the initial exploration fostered by \hyperlink{APGAI}{APGAI}, during which it samples suboptimal arms significantly when its collected observations are ``unlucky''.
We show that a linear dependence in $K$ is actually unavoidable for moderate risk. 
Namely, for any risk $\delta$ and any GAI algorithm, we exhibit an instance on which the expected number of samples allocated to suboptimal arms scales at least linearly with $K$, see Corollary~\ref{thm:lb_unverifiable_K_BAI} below.
Similar results exist in the BAI literature, \ie~\citet{simchowitz17theSimulator,al2022complexity,poianibest}.
In particular, we adapt the techniques used in~\citet[Theorem 2]{poianibest}, inspired by~\citet{al2022complexity}, and show a more general lower bound, \ie Theorem~\ref{thm:meta_lower_bound} proven in Appendix~\ref{app:ssec_proof_meta_lower_bound}.
It holds for any strategy whose stopping time satisfies a lower bound constraint on the $\TV$ distance between the distributions generated by interacting with instances having different answers.}
\begin{theorem}\label{thm:meta_lower_bound}
    \marc{Let $(\theta,\Delta,\epsilon) \in \R \times (\R_{+}^{\star})^2$ and $(\nu^{(a)})_{a \in [K]}$ as in Theorem~\ref{thm:thm6_degenne2023existence}.
For all $\delta \in (0,1/4]$, let $\tau_{\delta}$ be any stopping time satisfying that $\min_{a \in [K], b \in [K]\setminus\{a\}}\TV(\mathbb{P}_{\nu^{(a)}}^{\tau_\delta}, \mathbb{P}_{\nu^{(b)}}^{\tau_\delta}) \ge 1 - 2\delta$. 
Then,}
\[
   \marc{\frac{1}{K} \sum_{a \in [K]} \mathbb{E}_{\nu^{(a)}}[\tau_{\delta} - N_{a}(\tau_{\delta})] \ge \frac{K-1}{64(\Delta+\epsilon)^2} \: .}
\]
\end{theorem}
\marc{Combining Theorem~\ref{thm:meta_lower_bound} with the definition of unverifiable sample complexity yields Corollary~\ref{thm:lb_unverifiable_K_BAI}.}
\begin{corollary} \label{thm:lb_unverifiable_K_BAI}
    \marc{Let $(\theta,\Delta,\epsilon) \in \R \times (\R_{+}^{\star})^2$ and $(\nu^{(a)})_{a \in [K]}$ as in Theorem~\ref{thm:thm6_degenne2023existence}.
    For any $\delta \in (0,1/4]$ and any strategy with unverifiable sample complexity $\tau_{U,\delta}$, there exists $a \in [K]$ such that $\mathbb{E}_{\nu^{(a)}}[\tau_{U,\delta} - N_{a}(\tau_{U,\delta})] \ge \frac{K-1}{64(\Delta+\epsilon)^2} $.}
\end{corollary}
\begin{proof}    
\marc{Since $\{\GUESS{\tau_{U,\delta}} \ne a\}$ is $\tau_{U,\delta}$-measurable and satisfies that
\[
    \{\GUESS{\tau_{U,\delta}} \ne a\} \subseteq \{\exists t \ge \tau_{U,\delta}, \:  \GUESS{t} \ne a\} \quad \text{and} \quad \{\forall t \ge \tau_{U,\delta}, \:  \GUESS{t} = b\} \subseteq \{\GUESS{\tau_{U,\delta}} \ne a\} \: ,
\]
we obtain that $\min_{a \in [K], b \in [K]\setminus\{a\}}\TV(\mathbb{P}_{\nu^{(a)}}^{\tau_{U,\delta}}, \mathbb{P}_{\nu^{(b)}}^{\tau_{U,\delta}}) \ge 1 - 2\delta$.
Using Theorem~\ref{thm:meta_lower_bound} concludes the proof, see Appendix~\ref{app:ssec_proof_lb_unverifiable_K_BAI} for more details.}
\end{proof}
\marc{Corollary~\ref{thm:lb_unverifiable_K_BAI} is not valid for any instance. 
Among $K$ specific instances with one good arm, any algorithm should sample the suboptimal arms at least $\frac{K-1}{64(\Delta+\epsilon)^2}$ times on at least one of those instances.
Intuitively, an initial exploration phase is necessary: the algorithm has to sample all arms before starting to recommend the unique good arm (\ie the best one).
As an arm $a$ is sampled less than the others, the algorithm is slower on $\nu^{(a)}$.}

\section{\marc{Non-asymptotic} Fixed Confidence Guarantees} \label{sec:FCGAI}

In some applications, the practitioner has a strict constraint on the confidence $\DELTA$ associated with the candidate answer.
This constraint simultaneously supersedes any limitation on the sampling budget and allows early stopping when enough evidence is collected (random since data-dependent).
In the fixed-confidence setting, an identification algorithm should define a stopping rule in addition to the sampling and recommendation rules. 

\textit{Stopping rule.} 
We couple \hyperlink{APGAI}{APGAI} with the GLR stopping rule~\citep{garivier2016optimal} for GAI (see Appendix~\ref{app:ssec_GLR}), which coincides with the Box stopping rule introduced by~\citet{kaufmann2018sequential}.
At fixed confidence $\delta$, we stop at $\tau_{\DELTA} \deff \min(\tau_{>,\DELTA}, \tau_{<,\DELTA})$ with
\begin{equation} \label{eq:stopping_rule}
 \tau_{>, \DELTA} \deff \inf \{ t  \mid \max_{a \in \ARMS} \Wp{a}{t} \ge \sqrt{2c(t, \DELTA)} \} \quad \text{and} \quad \tau_{<, \DELTA} \deff \inf \{ t  \mid \min_{a \in \ARMS} \Wm{a}{t} \ge \sqrt{2c(t, \DELTA)} \} \: ,
\end{equation}
where $c : \mathbb N \times (0,1) \to \mathbb R_{+} $ is a threshold function.
Proven in Appendix~\ref{app:proof_lem_delta_correct_threshold}, Lemma~\ref{lem:delta_correct_threshold} gives a threshold ensuring that the GLR stopping rule~Eq.~\eqref{eq:stopping_rule} is $\delta$-correct for all $\delta \in (0,1)$, independently of the sampling rule.

\begin{lemma} \label{lem:delta_correct_threshold}
    Let $\overline{W}_{-1}(x)  = - W_{-1}(-e^{-x})$ for all $x \ge 1$, where $W_{-1}$ is the negative branch of the Lambert $W$ function.
    It satisfies $\overline{W}_{-1}(x) \approx x + \log x$.
	Let $\delta \in (0,1)$.  
	Given any sampling rule, using the threshold 
        \begin{equation} \label{eq:stopping_threshold}
            2 c(t, \DELTA) = \overline{W}_{-1} (2\ln (K/\delta) +  4 \ln \ln (e^4 t ) + 1/2 )
        \end{equation}
        in the GLR stopping rule~Eq.~\eqref{eq:stopping_rule} yields a $\delta$-correct algorithm for $1$-sub-Gaussian distributions.        
\end{lemma}

\textit{Non-asymptotic upper bound.}
Theorem~\ref{thm:expected_sample_complexity_upper_bound} gives an upper bound on the expected sample complexity of the resulting algorithm holding \marc{for any risk $\delta$}. 
First, we give an implicitly defined upper bound $C_{\mu}(\delta)$ holding for any stopping threshold $c(t,\delta)$ ensuring $\delta$-correctness.
Second, thanks to approximations, we provide a closed-form upper bound $C'_{\mu}(\delta) $ on $C_{\mu}(\delta)$ for $c(t,\delta)$ defined in~Eq.~\eqref{eq:stopping_threshold}, which is free from large constants in the $\delta$-independent term.
The related proofs are given in Appendix~\ref{app:anytimealgo_sample}.
\begin{theorem} \label{thm:expected_sample_complexity_upper_bound}
    Let $\DELTA \in (0,1)$.
    Combined with GLR stopping~Eq.~\eqref{eq:stopping_rule} using threshold~Eq.~\eqref{eq:stopping_threshold}, \hyperlink{APGAI}{APGAI} is $\delta$-correct and it satisfies that, for all $\nu \in \mathcal D^{K}$ with mean $\mu$ such that $\Delta_{\min}  > 0$,
    \begin{equation*} \bE_{\nu}[\tau_{\delta}] \le C_{\mu}(\delta) + \frac{K \pi^2}{6} + 1  \text{ with }  C_{\mu}(\delta) \deff \sup \{ t \mid t \le 2H_{i_{\mu}}(\mu) ( \sqrt{c(t, \delta)} + \sqrt{3 \log t})^2 + D_{i_{\mu}}(\mu) \}
    \end{equation*}
    where $i_{\mu} \deff 1 + (\theta - 1)\indic{\ARMS_{\theta}(\mu) \ne \emptyset}$, $H_{1}(\mu)$ and $H_{\theta}(\mu)$ as in~Eq.~\eqref{eq:common_complexity}. 
    $D_{1}(\mu)$ and $D_{\theta}(\mu)$ are defined in Lemmas~\ref{lem:expected_sample_complexity_empty_is_good} and~\ref{lem:expected_sample_complexity_exist_good_arms} in Appendix~\ref{app:anytimealgo_sample}, satisfying $D_{1}(\mu) \approx_{\Delta_{\min} \to + \infty} D_{\theta}(\mu) = \mathcal O(H_{1}(\mu) \log H_{1}(\mu))$.
    \marc{In the non-asymptotic regime, the $\delta$-independent dominating dependency is $C_\mu(\delta) = \mathcal O(H_{1}(\mu) \log H_{1}(\mu))$ even when there are good arms.}
    In the asymptotic regime, we obtain  $\limsup_{\delta \to 0} \bE_{\nu}[\tau_{\delta}]/\log(1/\delta) \le 2H_{i_{\mu}}(\mu)$ since $C_{\mu}(\delta) =_{\delta \to 0} 2H_{i_{\mu}}(\mu) \log(1/\delta) + \mathcal O(\log \log(1/\delta))$. 
    We can also provide an explicit and closed-form upper-bound on the constant $C_\mu(\delta)$, namely $C_{\mu}(\delta) \le  C'_{\mu}(\delta)$ with
    \begin{equation*}C'_{\mu}(\delta) \deff  h\left(15  H_{i_{\mu}}(\mu) , 4 H_{i_{\mu}}(\mu) \left(\ln \left( K/ \delta \right) + 15/4 - 2 \ln(2 H_{i_{\mu}}(\mu))\right) + D_{i_{\mu}}(\mu) \right)  
    \end{equation*}
     where $h(x,y) \deff y + x\log(x) + x\log(y/x + \log(x)) + x/2$.
\end{theorem}
Most importantly, Theorem~\ref{thm:expected_sample_complexity_upper_bound} holds for any \marc{risk} $\delta \in (0,1)$ and any $1$-sub-Gaussian instance $\nu$. 
In the asymptotic regime where $\delta \to 0$, Theorem~\ref{thm:expected_sample_complexity_upper_bound} shows that $\limsup_{\delta \to 0} \bE_{\nu}[\tau_{\delta}]/\log(1/\delta) \le 2H_{i_{\mu}}(\mu)$.
Therefore, \hyperlink{APGAI}{APGAI} is asymptotically optimal for Gaussian distributions when $\set{\THRESHOLD} = \emptyset$.
When there are good arms, our upper bound scales as $H_{\THRESHOLD}(\mu) \log(1/\delta)$ \marc{asymptotically}, which is better than the scaling in $H_{1}(\mu) \log(1/\delta)$ obtained for the unverifiable sample complexity.
However, when $\set{\THRESHOLD} \ne \emptyset$, our upper bound is \marc{asymptotically} sub-optimal compared to $2\min_{a \in \ARMS} \Delta_{a}^{-2}$ (see Lemma~\ref{lem:lower_bound_GAI}).
This sub-optimal scaling stems from the greediness of \hyperlink{APGAI}{APGAI} when $\set{\THRESHOLD} \ne \emptyset$ since there is no mechanism to detect an arm that is easiest to verify, \ie $\argmax_{a \in \set{\THRESHOLD}} \Delta_{a}$. 
Empirically, we observe that \hyperlink{APGAI}{APGAI} can suffer from poor outliers when there are good arms with dissimilar gaps and that adding forced exploration circumvents this issue (Figure~\ref{fig:supp_dissimilar_good_arms} and Table~\ref{tab:APGAI_forced_exploration} in Appendix~\ref{app:ssec_supp_emp_stop_time}).
Intuitively, a purely asymptotic analysis of \hyperlink{APGAI}{APGAI} \marc{might} yield the dependency $ 2 \max_{a \in \set{\THRESHOLD}} \Delta_{a}^{-2}$ which is independent from $|\set{\THRESHOLD}|$. 
This intuition is supported by empirical evidence (Figure~\ref{fig:varying_good_answers}), \clemence{and we defer the reader to Appendix~\ref{app:sssec_discussion_suboptimality} for more details.}

Compared to \marc{purely} asymptotic results, our non-asymptotic upper bound holds for \marc{any} reasonable values of $\DELTA$.
\marc{It is dominated by the $\delta$-independent term $D_{i_\mu}(\mu)$ that scales as $\mathcal O(H_{1}(\mu) \log H_{1}(\mu))$, even when there are good arms.
Intuitively, we show that no error occur at time $T = \Omega(H_{1}(\mu) \log H_{1}(\mu))$, provided the empirical means do not deviate from their mean until time $T$ (Lemmas~\ref{lem:time_no_undersampled_empty_is_good'} and~\ref{lem:time_no_undersampled_exist_good_arms'}).
The dependency $H_{1}(\mu)$ is the same as previously obtained for the probability of error (Theorem~\ref{thm:upper_bound_PoE_anytimeID}) and the unverifiable sample complexity (Theorem~\ref{thm:unverifiable_sample_complexity}). 
Similarly, as observed in our previous guarantees on \hyperlink{APGAI}{APGAI} (Theorems~\ref{thm:upper_bound_PoE_anytimeID} and~\ref{thm:unverifiable_sample_complexity}), our non-asymptotic proof techniques do not allow to capture the differences in the behavior of \hyperlink{APGAI}{APGAI} when interacting with instances having good arms or not. 
However, in the asymptotic regime, our arguments are sufficient to differentiate between both behaviors, as the non-asymptotic $\delta$-independent term $\mathcal O(H_{1}(\mu) \log H_{1}(\mu))$ vanish in comparison, even though it dominates for moderate risk $\delta$.
We refer the reader to Appendix~\ref{app:anytimealgo_sample} for a detailed discussion with intuition.
Our experiments reveal that the stopping time distribution of \hyperlink{APGAI}{APGAI} is right-skewed on instances with good arms having dissimilar gaps, suggesting that the scaling in $H_{\theta}(\mu)$ or $H_{1}(\mu)$ might not be improvable.}

\begin{table}[t]
    \centering
    \begin{tabular}{lccc}
   \toprule
          Algorithm $\mathfrak{A}$ & $\set{\THRESHOLD}(\mu)=\emptyset$ & $\set{\THRESHOLD}(\mu)\neq\emptyset$ & \clemence{Dominance over}\\
          & & & \clemence{\hyperlink{APGAI}{APGAI}, $\set{\THRESHOLD}(\mu)\neq\emptyset$}\\
          \midrule
         \hyperlink{APGAI}{APGAI} [Th~\ref{thm:expected_sample_complexity_upper_bound}]\marc{$\dagger$} &$2 H_1(\mu)$ & $2 H_{\theta}(\mu)$ & \clemence{-- (anytime)}\\
         \marc{Unif [Th~\ref{thm:uniform_sampling_sample_complexity_upper_bound}]} & \marc{$2K\Delta_{\min}^{-2}$} & \marc{$2K\overline{\Delta}_{\max}^{-2}$} & \marc{$\succprec$ (anytime)} \\
         S-TaS $\mathsection$~\citep{degenne2019pure} & $2H_1(\mu)$ & $2\overline{\Delta}_{\max}^{-2}$ & \clemence{$\succ$ (fixed-confidence)}\\
         HDoC~\citep{kano2019good} & $2H_1(\mu)$ & $2\overline{\Delta}_{\max}^{-2}$ & \clemence{$\succ$ (fixed-confidence)} \\
         APT-G, LUCB-G~\citep{kano2019good} & $2H_1(\mu)$ & $-$ &  \clemence{-- (fixed-confidence)} \\
         \marc{SEE~\citep{li2025near}} & \marc{$\mathcal O(H_1(\mu))$} & \marc{$\mathcal O(\overline{\Delta}_{\max}^{-2})$} & \marc{$\succ$ (fixed-confidence)} \\
         \bottomrule
    \end{tabular}
    \caption{
    Asymptotic upper bound $C(\mu)$ on the expected sample complexity of algorithm $\mathfrak{A}$ on $\nu$, \ie $\limsup_{\delta \to 0} \bE_{\nu}[\tau_{\delta}]/\log(1/\delta) \le C(\mu)$.
    \marc{($\dagger$) The $\delta$-independent non-asymptotic bound scales as $\mathcal O(H_{1}(\mu) \log H_{1}(\mu))$ even when there are good arms.}
    ($\mathsection$) Requires an ordering on the possible answers $\ARMS \cup \{\emptyset\}$.
    $H_{1}(\mu)$ and $H_{\theta}(\mu)$ as in~Eq.~\eqref{eq:common_complexity}, $\overline{\Delta}_{\max} \deff \max_{a \set{\THRESHOLD}} \Delta_a$. \clemence{The dominance of a bandit strategy is defined by the comparison of their \emph{known upper bounds} (smaller means better): $\prec$ (dominated), $\succ$ (dominant) and $\succprec$ (Pareto equivalent).}}
    \label{tab:summary_FCGAI}
\end{table}

\marc{\textit{Comparison with uniform sampling.} 
Combined with the same GLR stopping rule~Eq.~\eqref{eq:stopping_rule} using threshold~Eq.~\eqref{eq:stopping_threshold}, we compare Theorem~\ref{thm:expected_sample_complexity_upper_bound} for \hyperlink{APGAI}{APGAI} with the non-asymptotic upper bound on the expected sample complexity of Unif for GAI given by Theorem~\ref{thm:uniform_sampling_sample_complexity_upper_bound} in Appendix~\ref{app:sssec_unif_FC}.
In contrast to \hyperlink{APGAI}{APGAI}, the non-asymptotic and asymptotic dominating terms for Unif are scaling similarly.
In both cases, the behavior is different when interacting with instances having good arms or not: $K\Delta_{\min}^{-2}$ when $\set{\THRESHOLD}(\mu) = \emptyset$, and $K\min_{a \set{\THRESHOLD}} \Delta_a^{-2}$ otherwise.
Even by accounting for the right-skewness, our experiments show that \hyperlink{APGAI}{APGAI} outperforms Unif on average in all the considered instances.}

\marc{\textit{Asymptotic dependency $H_{\theta}(\mu)$ instead of $H_{1}(\mu)$.}
While the $\delta$-independent dominating term scales as $\mathcal O(H_{1}(\mu) \log H_{1}(\mu))$ in Lemma~\ref{lem:expected_sample_complexity_exist_good_arms}, the asymptotic dependency is $2 H_{\theta}(\mu)$ when there are good arms.
    To understand this improvement over the asymptotic dependency $2 H_{1}(\mu)$ when there are no good arms (Lemma~\ref{lem:expected_sample_complexity_empty_is_good}), we provide some intuition behind the technical arguments used in the proof of Lemma~\ref{lem:expected_sample_complexity_exist_good_arms}. 
    First, as in Lemma~\ref{lem:expected_sample_complexity_empty_is_good}, the $\delta$-dependency in Lemma~\ref{lem:expected_sample_complexity_exist_good_arms} comes solely from a probabilistic statement involving the GLR stopping rule as in Eq.~\eqref{eq:stopping_rule} whose stopping threshold as in Eq.~\eqref{eq:stopping_threshold} depends on the algorithmic risk parameter $\delta$, see the definition of $D_{\mu}(\delta)$. 
    Second, using Lemma~\ref{lem:large_time_behavior_exist_good_arms} when $\set{\THRESHOLD} \ne \emptyset$, we know that there is no error at time $T$ and that the ``bad'' arms are not sampled anymore for large enough $T$ (yet independent of $\delta$), provided concentration holds.
    When $\set{\THRESHOLD} = \emptyset$, this is in stark contrast with Lemma~\ref{lem:large_time_behavior_empty_is_good} that only states that there are no errors, yet any arms can continue to be sampled.
    Third, our non-asymptotic method builds on the technique used to obtain non-asymptotic upper bounds on TTUCB in~\citet{jourdan2023NonAsymptoticAnalysis}.
    Using the piegonhole principle, for $T$ large enough (yet independent of $\delta$), there exists a good arm $a \in \set{\THRESHOLD}$ that was sampled more than $T/(\Delta_{a} H_{\theta}(\mu))$ at time $T$.
    By considering the last time where this arm was sampled, its transportation cost is simultaneously smaller than $\sqrt{2c(T, \delta)}$ (not stopped yet) and larger than $\sqrt{T / H_{\theta}(\mu)} $ (concentration result). Inverting this inequality concludes the proof, \ie $T \lessapprox 2 H_{\theta}(\mu)c(T, \delta)  $.
    When $\set{\THRESHOLD} = \emptyset$, based on Lemma~\ref{lem:large_time_behavior_empty_is_good}, the piegonhole principle only shows that there exists an arm $a \in [K]$ that was sampled more than $T/(\Delta_{a} H_{1}(\mu))$ at time $T$. 
    Unfolding the same technical argument yields $T \lessapprox 2 H_{1}(\mu)c(T, \delta)  $. 
    This explains the difference in asymptotic behavior when there are good arms. 
    The above discussion also glimpses why it is challenging to improve on $2 H_{\theta}(\mu)$ with our non-asymptotic proof technique. 
    Our proof does not control the event that all the good arms are sampled linearly, e.g., in a round-robin fashion.}

\subsection{\marc{Lower Bound with Dependence on the Number of Arms}}

\marc{When there is a unique good arm, Theorem~\ref{thm:expected_sample_complexity_upper_bound} shows that the expected sample complexity of \hyperlink{APGAI}{APGAI} is upper bounded by a quantity scaling linearly with $K$, when the risk $\delta$ is moderate.
When the risk is arbitrarily small, the lower bound in Lemma~\ref{lem:lower_bound_GAI} shows the independence in $K$ of the expected sample complexity of any asymptotically optimal algorithm.
Building on Theorem~\ref{thm:meta_lower_bound}, we show that a linear dependence in $K$ is actually unavoidable in fixed-confidence GAI (Corollary~\ref{thm:lb_K_BAI}).}
\begin{corollary} \label{thm:lb_K_BAI}
    \marc{Let $(\theta,\Delta,\epsilon) \in \R \times (\R_{+}^{\star})^2$ and $(\nu^{(a)})_{a \in [K]}$ as in Theorem~\ref{thm:thm6_degenne2023existence}.
    For any $\delta \in (0,1/4]$ and any $\delta$-correct strategy, there exists $a \in [K]$ such that $\mathbb{E}_{\nu^{(a)}}[\tau_{\delta} - N_{a}(\tau_{\delta})] \ge \frac{K-1}{64(\Delta+\epsilon)^2} $.}
\end{corollary}
\begin{proof}    
\marc{Since $\{\GUESS{\tau_{\delta}} = a\}$ is $\tau_{\delta}$-measurable and satisfies that $\{\GUESS{\tau_{\delta}} = a\} \subseteq \{\hat{a}_{\tau_{\delta}} \ne b\}$, we obtain that $\min_{a \in [K], b \in [K]\setminus\{a\}}\TV(\mathbb{P}_{\nu^{(a)}}^{\tau_{\delta}}, \mathbb{P}_{\nu^{(b)}}^{\tau_{\delta}}) \ge 1 - 2\delta$.
Using Theorem~\ref{thm:meta_lower_bound} concludes the proof, see Appendix~\ref{app:ssec_proof_lb_K_BAI} for more details.}
\end{proof}
\marc{Corollary~\ref{thm:lb_K_BAI} is similar to Corollary~\ref{thm:lb_unverifiable_K_BAI}, hence the same comments hold.
Based on~\citet[Theorem 5.6]{katz2020true},~\citet[Theorem 5.6]{li2025near} gives a lower bound on $\mathbb{E}_{\nu^{(a)}}[\tau_{\delta}]$ that resembles Corollary~\ref{thm:lb_K_BAI}.
Since it does not imply that suboptimal arms are sampled significantly, our lower bound is stronger.}

\marc{While being slightly different probabilistic properties, both the $\delta$-unverifiability and the $\delta$-correctness ensures a $1-2\delta$ lower bound on the $\TV$ distance between the distributions generated by interacting with instances having different unique good arm.
Deriving information-theoretic arguments that differentiate between both properties is an interesting direction for future research. }

\subsection{\marc{Benchmark: Other fixed-confidence GAI Algorithms}}
\label{sec:ssec_benchmark_FC}

Table~\ref{tab:summary_FCGAI} summarizes the asymptotic scaling of the upper bound on the expected sample complexity of existing GAI algorithms.
While most GAI algorithms have better asymptotic guarantees when $\set{\THRESHOLD}(\mu)\neq\emptyset$, \hyperlink{APGAI}{APGAI} is the only one of them which has anytime guarantees on the probability of error (Theorem~\ref{thm:upper_bound_PoE_anytimeID}).
However, we emphasize that \hyperlink{APGAI}{APGAI} is designed for anytime GAI and is not the best algorithm for fixed-confidence GAI.
Sticky Track-and-Stop (S-TaS) is asymptotically optimal for the ``any low arm'' problem~\citep{degenne2019pure}, hence for GAI as well.
Even though GAI is one of the few settings where S-TaS admits a computationally tractable implementation, its empirical performance heavily relies on the fixed ordering for the set of possible answers (see Table~\ref{tab:stickyTaS_ordering} in Appendix~\ref{app:ssec_impelementation_details}).
This explains the lack of non-asymptotic guarantees for S-TaS that is asymptotic by nature, while \hyperlink{APGAI}{APGAI} has non-asymptotic guarantees.
For the ``bad arm existence'' problem,~\citet{kaufmann2018sequential} prove that the empirical proportion $(\nsamples{a}{t}/t)_{a \in \ARMS}$ of Murphy Sampling converges almost surely towards the optimal allocation realizing the asymptotic lower bound of Lemma~\ref{lem:lower_bound_GAI}.
While their result implies that $\lim_{\delta \to 0} \tau_{\delta}/\log (1/\delta) = T^\star(\mu)$ almost surely, the authors provide no upper bound on the expected sample complexity of Murphy Sampling.
Finally, we consider the AllGAI algorithms introduced by~\citet{kano2019good} (HDoC, LUCB-G, and APT-G) having theoretical guarantees for some GAI instances.
When $\set{\THRESHOLD}(\mu) = \emptyset$, all three algorithms have an upper bound of the form $2 H_{1}(\mu) \log (1/\delta)  + \mathcal O(\log \log(1/\delta))$.
When $\set{\THRESHOLD}(\mu) \ne \emptyset$, only HDoC admits an upper bound on the expected number of time to return one good arm, which is of the form $2 \min_{a \in \set{\THRESHOLD}} \Delta_{a}^{-2}\log (1/\delta)  + \mathcal O(\log \log(1/\delta))$.

The indices used for the elimination and recommendation in BAEC~\citep{tabata2020bad} have a dependence in $ \mathcal O(- \ln (\theta_{U} - \theta_{L}))$, hence BAEC is not defined for GAI where $\theta_{U} = \theta_{L}$.
While it is possible to use UCB/LCB which are agnostic to the gap $\theta_{U} - \theta_{L} > 0$, these choices have not been studied by~\citet{tabata2020bad}.
Extrapolating the theoretical guarantees of BAEC when $\theta_{L} \to \theta_{U}$, one would expect an upper bound on its expected sample complexity of the form $2 H_{1}(\mu) \log (1/\delta) + \mathcal O((\log(1/\delta))^{2/3})$.
\marc{In recent concurrent work,~\citet{li2025near} propose the Sequential-Exploration-Exploitation (SEE) algorithm that proceeds in phases and alternates between exploration and exploitation subphases.
Up to the constant multiplicative factor, the upper bounds on the expected sample complexity of SEE are better than the ones obtained for \hyperlink{APGAI}{APGAI}.
~\citet[Theorem 5.3]{li2025near} shows a scaling as $ \mathcal O( H_{1}(\mu) \log (1/\delta))$ when $\set{\THRESHOLD}(\mu) = \emptyset$, and as $ \mathcal O(\min_{a \in \set{\THRESHOLD}} \Delta_{a}^{-2} \log (1/\delta))$ when $\set{\THRESHOLD}(\mu) \ne \emptyset$.
For fixed-confidence GAI, the above discussion exhibits adaptive algorithms that consistently outperform uniform sampling on all instances, \ie the ``perfect'' adaptive trade-off exist.}
 
\section{Experiments}\label{sec:experiments}

We assess the empirical performance of the \hyperlink{APGAI}{APGAI} in terms of empirical error, as well as empirical stopping time.
Overall, \hyperlink{APGAI}{APGAI} performs favorably compared to other algorithms in both settings.
\marc{While its empirical stopping time seems to align with Theorem~\ref{thm:expected_sample_complexity_upper_bound}, its (anytime) empirical error is lower than what Theorem~\ref{thm:upper_bound_PoE_anytimeID} would suggest when there are good arms.}
This \marc{partial} discrepancy between theory and practice paves the way for interesting future research.
We present a fraction of our experiments and defer the reader to Appendix~\ref{app:details_experiments} for supplementary experiments.

\textit{Outcome scoring application.}
\clemence{Our real-life motivation is outcome scoring from gene activity (transcriptomic) data (further described in Appendix~\ref{app:sssec_reallife}). This application focuses on the treatment of encephalopathy of prematurity in infants. The goal is to determine the optimal protocol for the administration of stem cells among $K = 18$ realistic possibilities. In collaboration with the PREMSTEM consortium, all treatments were tested on a rat model of encephalopathy of prematurity. Rat brain RNA-related measurement data were generated using high-throughput sequencing. Computed on $3$ technical replicates, the mean value in $[-1, 1]$ (see Table~\ref{tab:instances_arms} in Appendix~\ref{app:sssec_reallife}) corresponds to a cosine score computed between gene activity changes in treated and healthy samples. Traditional approaches use grid-search with a uniform allocation and select the best cosine score to determine the optimal protocol. Here, to model the stochasticity of the scores that would have been obtained for each protocol in a sequential approach, we applied a Bernoulli instance and considered treatment as significantly efficient when the mean score is higher than $\THRESHOLD = 0.5$. In other words, observations from arm $a$ are drawn from a Bernoulli distribution with mean $\max(\mu_a, 0)$ (which is $1/2$-sub-Gaussian) using the real cosine score of this treatment protocol as $\mu_a$.}

\begin{figure}[t]
        \centering
	\includegraphics[width=0.6\linewidth]{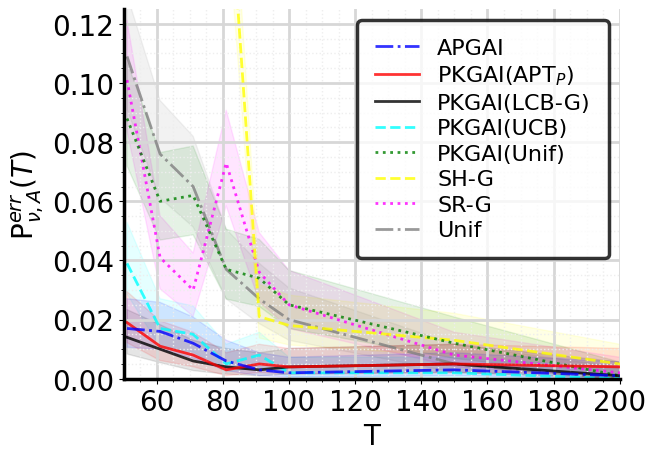} \caption{Fixed-budget empirical error for outcome scoring (see \textsc{RealL} in Table~\ref{tab:instances_arms}).}
	\label{fig:results_PREMSTEM_graphics}
\end{figure}

\textit{Fixed-budget empirical error.}
The \hyperlink{APGAI}{APGAI} algorithm is compared to fixed-budget GAI algorithms: SR-G, SH-G, \hyperlink{PKGAI}{PKGAI} and Unif.
For a fair comparison, the threshold functions in~\hyperlink{PKGAI}{PKGAI} do not use prior knowledge (see Appendix~\ref{app:sssec_FB_algos}, where theoretical thresholds are used). 
We compare several index policies for \hyperlink{PKGAI}{PKGAI}: Unif, APT$_{P}$, UCB, and LCB-G. 
At time $t$, the latter selects among the set $\mathcal{S}_t$ of active candidates $\arm{t} \gets \argmax_{a \in \mathcal{S}_t} \sqrt{\nsamples{a}{t}}\text{LCB}(a,t)$, where LCB$(a,t)$ is the lower confidence bound on $\mu_a-\THRESHOLD$ at time $t$.
For a budget of $T$ up to $200$, our results average over $1,000$ runs, with associated confidence intervals.
On our outcome scoring application, Figure~\ref{fig:results_PREMSTEM_graphics} first shows that all uniform samplings (SH-G, SR-G, Unif, and PKGAI(Unif)) are less efficient at detecting one of the good arms contrary to the adaptive strategies. 
Moreover, APGAI performs as well as the elimination-based algorithms PKGAI($\star$), while allowing early stopping. \clemence{These performances constitute a relevant advantage for outcome scoring and other medical applications such as clinical trials.}
In Appendix~\ref{app:ssec_supp_emp_error_FB}, we confirm the good performance of \hyperlink{APGAI}{APGAI} in terms of fixed-budget empirical error on other instances.

\begin{figure}[t]
    \centering
   \clemence{(a)} \includegraphics[width=0.45\linewidth]{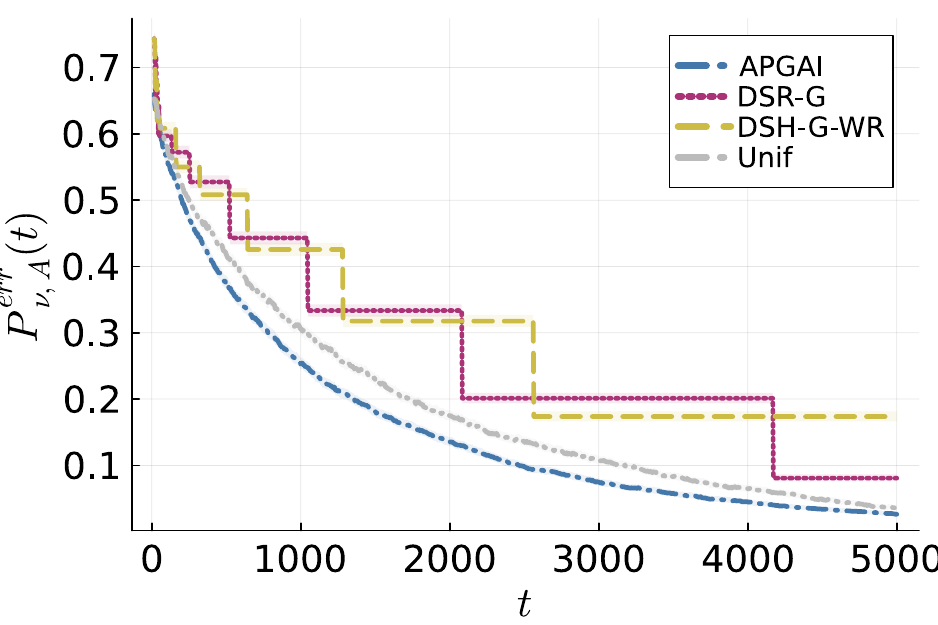}
    \clemence{(b)} \includegraphics[width=0.45\linewidth]{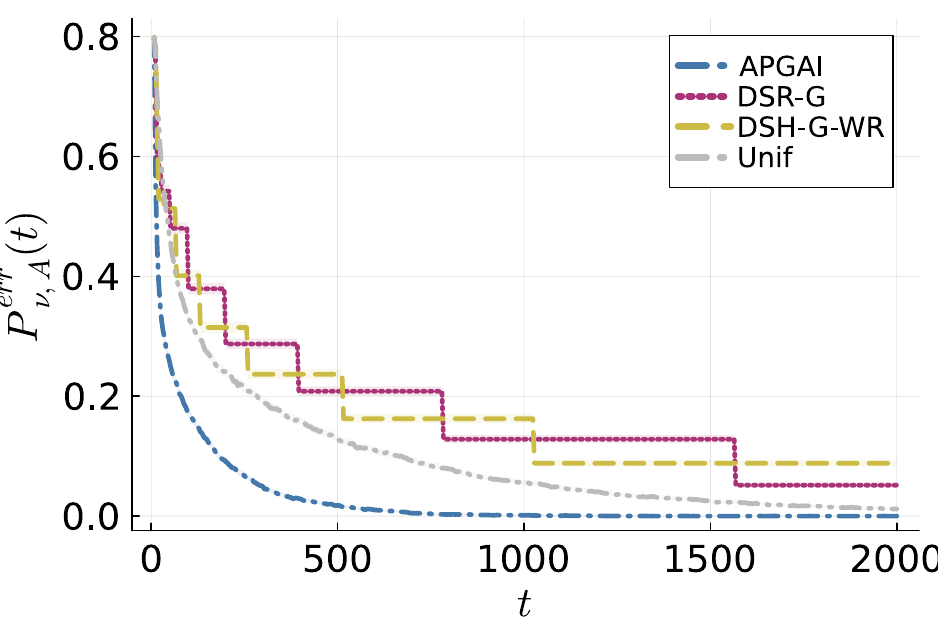}
    \caption{Anytime empirical error on Gaussian instances (a) $\mu \in \{0.55, 0.45\}^{10}$ where $|\set{\THRESHOLD}|=3$ for $\THRESHOLD = 0.5$ and (b) $\mu = - (0.1, 0.4, 0.5, 0.6)$ for $\THRESHOLD = 0$.}
    \label{fig:PoE_no_good_arms}
\end{figure}

\textit{Anytime empirical error.}
The \hyperlink{APGAI}{APGAI} algorithm is compared to anytime GAI algorithms: DSR-G, DSH-G (see Section~\ref{sec:ssec_bai_to_gai}) and Unif.
Since DSH-G has poor empirical performance (see Figure~\ref{fig:supp_PoE_doublingSH_withrefresh}), we consider the heuristic DSH-G-WR that relies on the whole history instead of discarding it.
On two Gaussian instances ($\set{\THRESHOLD}(\mu) \ne \emptyset$ and $\set{\THRESHOLD}(\mu) = \emptyset$), Figure~\ref{fig:PoE_no_good_arms} shows that \hyperlink{APGAI}{APGAI} has significantly smaller empirical error compared to Unif, which is itself better than DSR-G and DSH-G-WR.
Our results average over $10,000$ runs, with associated confidence intervals.
In Appendix~\ref{app:ssec_supp_emp_error_anytime}, we confirm the good performance of \hyperlink{APGAI}{APGAI} in terms of anytime empirical error on other instances, \eg when $\set{\THRESHOLD}(\mu) \ne \emptyset$ (Figure~\ref{fig:supp_PoE_dissimilar_good_arms}) and when $|\set{\THRESHOLD}(\mu)|$ varies (Figure~\ref{fig:supp_PoE_varying_good_answers}).
Overall, \hyperlink{APGAI}{APGAI} appears to have better empirical performance than suggested by Theorem~\ref{thm:upper_bound_PoE_anytimeID} when $\set{\THRESHOLD}(\mu) \ne \emptyset$.

\textit{Empirical stopping time.}
The \hyperlink{APGAI}{APGAI} algorithm is compared to fixed-confidence GAI algorithms using the GLR stopping rule~Eq.~\eqref{eq:stopping_rule} with threshold~Eq.~\eqref{eq:stopping_threshold} and confidence $\delta = 0.01$: Murphy Sampling (MS)~\citep{kaufmann2018sequential}, HDoC, LUCB-G~\citep{kano2019good}, Track-and-Stop for GAI (TaS)~\citep{garivier2016optimal} and Unif (see Appendix~\ref{app:sssec_FC_algos}).
\marc{While SEE is omitted from our benchmarks as concurrent work, the experiments in~\citep{li2025near} showcase that it performs on par with TaS on the considered instances.}
In Figure~\ref{fig:varying_good_answers}, we study the impact of the number of good arms by considering Gaussian instances with two groups of arms.
Our results average over $1,000$ runs, with associated standard deviations.
Figure~\ref{fig:varying_good_answers} shows that the empirical performance of \hyperlink{APGAI}{APGAI} is invariant to varying $|\set{\THRESHOLD}|$, and comparable to the one of TaS.
In comparison, the other algorithms have worse performance and suffer from increased $|\set{\THRESHOLD}|$ since an exploration bonus exists for each good arm.
In contrast, \hyperlink{APGAI}{APGAI} \marc{can be} greedy enough to only focus its allocation to one of the good arms.
\marc{Consistent with our guarantees in Theorem~\ref{thm:expected_sample_complexity_upper_bound}, \hyperlink{APGAI}{APGAI} achieves the best performance when there is no good arm.
When good arms have dissimilar means (with potentially many arms), \hyperlink{APGAI}{APGAI} seems to suffer from poor outliers (Figures~\ref{fig:supp_varying_good_answers}(b) and~\ref{fig:supp_dissimilar_good_arms} in Appendix~\ref{app:ssec_supp_emp_stop_time}).
Given that outliers greatly impact the averaged stopping time, this behavior seems to be consistent with our suboptimal upper bound on the expected sample complexity, \ie scaling as $H_{1}(\mu)$ for moderate $\delta$ and as $H_{\theta}(\mu)$ instead of $(\max_{a \in \set{\THRESHOLD}}\Delta)^{-2}$ when $|\set{\THRESHOLD}| > 1$ asymptotically (see Theorem~\ref{thm:expected_sample_complexity_upper_bound}). } 
To circumvent this problem, it is enough to add forced exploration to \hyperlink{APGAI}{APGAI} (Table~\ref{tab:APGAI_forced_exploration}).
While \hyperlink{APGAI}{APGAI} is anytime GAI algorithm, it is remarkable that it also has theoretical guarantees in fixed-confidence GAI and relatively small empirical stopping time.

\begin{figure}[t]
    \centering
\includegraphics[width=0.6\linewidth]{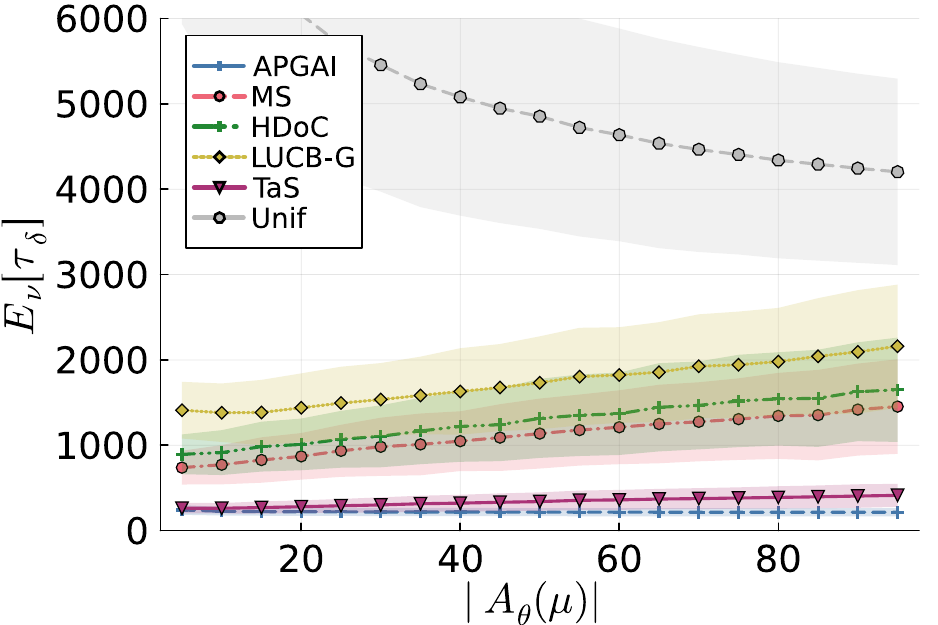} \caption{Empirical stopping time ($\delta = 0.01$) on Gaussian instances $\mu \in \{0.5, -0.5\}^{100}$ where $|\set{\THRESHOLD}|  \in \{5 k\}_{k \in [19]}$ for $\THRESHOLD = 0$.}
    \label{fig:varying_good_answers}
\end{figure}
 
\section{Perspectives}
\label{sec:discussion}

We propose \hyperlink{APGAI}{APGAI}, the first anytime and parameter-free sampling algorithm for GAI in stochastic bandits, which is independent of a budget $T$ or a confidence $\delta$.
In addition to showing its good empirical performance, we also provided guarantees on its probability of error at any deterministic time $t$ (Theorem~\ref{thm:upper_bound_PoE_anytimeID}) and on its expected sample complexity at any confidence $\DELTA$ when combined with the GLR stopping time~\eqref{eq:stopping_rule} (Theorem~\ref{thm:expected_sample_complexity_upper_bound}). 
As such, \hyperlink{APGAI}{APGAI} allows both continuation and early stopping.
We reviewed and analyzed a large number of baselines for each GAI setting for comparison.

While we considered unstructured multi-armed bandits, many applications have a known structure.
Investigating the GAI problem on \eg linear or infinitely-armed bandits would be interesting subsequent work. 
In particular, working in a structured framework when facing a possibly infinite number of arms would bring out more compelling questions about how to explore the arm space both in a tractable and meaningful way. 

\acks{Experiments presented in this paper were carried out using the Grid’5000 testbed, supported by a scientific interest group hosted by Inria and including CNRS, RENATER and several Universities as well as other organizations (see https://www.grid5000.fr). 
\clemence{This work has been partially supported by the Institut national de la sant\'{e} et de la recherche m\'{e}dicale, the Universit\'{e} Sorbonne Paris Nord, the University Paris Cit\'{e}, the French National Research Agency (the THIA ANR program “AI PhD@Lille” M.J.; ANR-21-RHUS-009, C.R., A.D-D; ANR-23-IAHU-0010, A.D-D), and Horizon 2020 Framework Program of the European Union (grant agreement no. 874721/PREMSTEM, A.D-D, C.R; grant agreement no. 101102016/RECeSS, C.R.). Laboratory experiments from which the outcome scoring application data were obtained were carried out by Cindy Bokobza in NeuroDiderot laboratory under the direction of Pierre Gressens for the PREMSTEM Consortium study. The comprehensive optimization study of human mesenchymal stem cell protocols including the new transcriptome data set is the subject of unpublished research that will be released by the PREMSTEM consortium.}}

\newpage
\appendix

\section{Outline} \label{app:outline}

The appendices are organized as follows:
\begin{itemize}
	\item The anytime guarantees of proof \hyperlink{APGAI}{APGAI} on the probability of error (Theorem~\ref{thm:upper_bound_PoE_anytimeID}) are proven in Appendix~\ref{app:anytimealgo_error}. \marc{It also contains the proof of Theorem~\ref{thm:unverifiable_sample_complexity} (Appendix~\ref{app:ssec_unverifiable_sample_complexity}) and Corollary~\ref{cor:upper_bound_PoE_anytimeID} (Appendix~\ref{app:proof_cor_upper_bound_PoE_anytimeID}).}
        \item Appendix~\ref{app:guarantees_other_algorithms} gathers error guarantees on other algorithms that are used as comparison with the anytime error guarantees of \hyperlink{APGAI}{APGAI}: Unif (Theorem~\ref{thm:uniform_sampling_PoE_recoAnytime}), SH-G (Theorem~\ref{thm:SH_PoE_recoElim}) and SR-G (Theorem~\ref{thm:SR_PoE_recoElim}).
        \marc{For Unif algorithm, we also derive a deterministic upper bound on its unverifiable sample complexity for GAI (Theorem~\ref{thm:Unif_unverifiable_sample_complexity}) and upper bound its expected sample complexity when combined with the GLR stopping~\eqref{eq:stopping_rule} using threshold~\eqref{eq:stopping_threshold} (Theorem~\ref{thm:uniform_sampling_sample_complexity_upper_bound}).}
        \item We propose the meta-algorithm \protect\hyperlink{PKGAI}{PKGAI} in Appendix~\ref{sec:stickyalgo}, and analyze its error guarantees for several choices of index policy (Theorems~\ref{th:APTlike_error} and~\ref{th:stickyalgo_error}).
        \item Appendix~\ref{app:GLR_and_Times} gives the proof of \marc{of our lower bounds:} Lemma~\ref{lem:lower_bound_GAI}\marc{,  Theorem~\ref{thm:meta_lower_bound}, Corollaries~\ref{thm:lb_unverifiable_K_BAI} and~\ref{thm:lb_K_BAI}}. We link the ATP$_{P}$ index and the GLR stopping rule~\eqref{eq:stopping_rule} with the generalized likelihood ratio for GAI. 
	\item The proof of Theorem~\ref{thm:expected_sample_complexity_upper_bound} for \hyperlink{APGAI}{APGAI} when combined with the GLR stopping~\eqref{eq:stopping_rule} using threshold~\eqref{eq:stopping_threshold} is detailed in Appendix~\ref{app:anytimealgo_sample}.
	\item Appendix~\ref{app:concentration_results} contains the proof of Lemma~\ref{lem:delta_correct_threshold}, and provides sequence of concentration events which are used for our proofs.
	\item Appendix~\ref{app:technicalities} gathers existing and new technical results which are used for our proofs.
	\item In Appendix~\ref{app:details_experiments}, we provide more details on our experimental study, as well as additional experiments.
\end{itemize}

\section{Analysis of APGAI: Proof of Theorem~\ref{thm:upper_bound_PoE_anytimeID}}\label{app:anytimealgo_error}

The \hyperlink{APGAI}{APGAI} algorithm is independent of a budget $T$ or a confidence $\delta$ which would define a stopping condition.
In the following, we consider the behavior of \hyperlink{APGAI}{APGAI} when it is sampling \textit{forever}.
Therefore, we provide guarantees at all time $T$, where $T$ can be seen as an analysis parameter.
In order to upper bound the probability of the complementary of the concentration event at time $T$, we use an analytical parameter denoted by $\delta$ which will be inverted to obtain an upper bound on the probability of error.
We emphasize that the $\delta$ used in Appendix~\ref{app:anytimealgo_error} is not the same $\delta$ than the one to calibrate the stopping thresholds used in the GLR stopping~Eq.~\eqref{eq:stopping_rule}.
We recall that each arm is pulled once as initialization.

\textit{Proof strategy.}
Let $\mu \in \rR^K$ such that $\mean{a} \ne \theta$ for all $a \in \ARMS$.
For all $T >  K$ and $\delta \in (0,1)$, let $\tilde \cE_{T,\delta}$ as in~Eq.~\eqref{eq:event_concentration_per_arm_improved} for $s = 0$, \ie
\begin{align}
	&\tilde \cE_{T,\delta} = \left\{ \forall a \in \ARMS, \forall t \le T, \: |\expmean{a}{t} - \mu_a| < \sqrt{\frac{2 	\tilde f_1(T,\delta)}{N_{a}(t)}} \right\} \: , \label{eq:event_PoE}\\
	&\text{with} \quad \tilde f_1(T, \delta) =  \frac{1}{2}\overline{W}_{-1}(2\log (1/\delta)  + 2 \log (2 + \log T ) + 2 )  \: , \nonumber
\end{align}
Recall that the error event $\cE^{\text{err}}_{\mu}(\NBATCHES)$ is defined as
\[
     \cE^{\text{err}}_{\mu}(\NBATCHES) \deff \left\{ \left( \set{\THRESHOLD} \neq \emptyset \cap (\GUESS{\NBATCHES}=\emptyset \cup \mean{\GUESS{\NBATCHES}}<\THRESHOLD) \right) \cup \left( \set{\THRESHOLD} = \emptyset \cap \GUESS{\NBATCHES}\neq\emptyset \right) \right\} \: .
\]
Using Lemma~\ref{lem:concentration_per_arm_gau_improved}, we have $\bP_{\nu}(\tilde \cE_{T,\delta}^{\complement}) \le  K \delta$.
Suppose that we have constructed a time $T_{\mu}(\delta) \ge  K$ such that $\tilde \cE_{T, \delta} \subseteq \cE^{\text{err}}_{\mu}(\NBATCHES)^{\complement}$ for $T > T_{\mu}(\delta)$.
Then, we obtain
\[
\forall T > T_{\mu}(\delta), \quad \perr{\nu}{\cdot}{\NBATCHES} = \bP_{\nu}(\cE^{\text{err}}_{\mu}(\NBATCHES)) \le K \delta \quad \text{hence} \quad \perr{\nu}{\cdot}{\NBATCHES} \le K \inf \{ \delta \mid \: T > T_{\mu}(\delta) \} \: ,
\]
where the last inequality is obtained by taking the infimum.
To prove Theorem~\ref{thm:upper_bound_PoE_anytimeID}, we will distinguish between instances $\mu$ such that $\set{\THRESHOLD} = \emptyset$ (Appendix~\ref{app:sssec_PoE_empty_is_good}) and instances $\mu$ such that $\set{\THRESHOLD} \ne \emptyset$ (Appendix~\ref{app:sssec_PoE_exist_good_arms}).

Lemma~\ref{lem:technical_result_bad_event_implies_bounded_quantity_increases} is the key technical tool on which our proofs rely on.
It assumes the existence of a sequence of ``bad'' events such that, under each ``bad'' event, the arm selected to be pulled next was not sampled a lot yet.
Then, it shows that the number of times those ``bad'' events occur is small.
\begin{lemma} \label{lem:technical_result_bad_event_implies_bounded_quantity_increases}
	Let $\delta \in (0,1]$ and $T >  K$.
	Let $(A_{t}(T, \delta))_{T \ge t \ge   K}$ be a sequence of events and $(D_{a}(T, \delta))_{a \in \ARMS}$ be positive thresholds satisfying that, for all $t \in ( K, T] \cap \mathbb N$, under the event $A_{t}(T, \delta)$, $N_{a_{t+1}}(t) \le D_{a_{t+1}}(T,\delta)$ and $N_{a_{t+1}}(t+1) = N_{a_{t+1}}(t) + 1$.
	Then, we have $\sum_{t =  K + 1}^{T} \indic{A_{t}(T, \delta)} \le \sum_{a \in \ARMS} D_{a}(T,\delta)$.
\end{lemma}
\begin{proof}
		Using the inclusion of events given by the assumption on $(A_{t}(T, \delta))_{T \ge t > K}$, we obtain
		\begin{align*}
			\sum_{t =  K + 1}^{T} \indic{A_{t}(T, \delta)} &\le \sum_{t =  K + 1}^{T} \indic{N_{a_{t+1}}(t) \le D_{a_{t+1}}(T,\delta) , \: N_{a_{t+1}}(t+1) = N_{a_{t+1}}(t) + 1} \\
			\le &\sum_{a \in \ARMS} \sum_{t =  K + 1}^{T} \indic{N_{a}(t)  \le D_{a}(T,\delta) , \: N_{a}(t+1) = N_{a}(t) + 1} \le \sum_{a \in \ARMS} D_{a}(T,\delta) \: .
		\end{align*}
		The second inequality is obtained by union bound.
		The third inequality is direct since the number of times one can increment by one a quantity that is positive and bounded by $D_{a}(T,\delta)$ is at most $D_{a}(T,\delta)$.
\end{proof}

In our proofs, we derive necessary conditions for a mistake to be made and show that having those conditions that hold is a ``bad'' event satisfying the condition of Lemma~\ref{lem:technical_result_bad_event_implies_bounded_quantity_increases}.
Theorem~\ref{thm:upper_bound_PoE_anytimeID} is obtained by combining Lemmas~\ref{lem:probability_error_empty_is_good} and~\ref{lem:probability_error_exist_good_arms}.

\subsection{Instances where \texorpdfstring{$\set{\THRESHOLD} = \emptyset$}{}}
\label{app:sssec_PoE_empty_is_good}

When $\set{\THRESHOLD} = \emptyset$, we have $ \cE^{\text{err}}_{\mu}(\NBATCHES) = \{\GUESS{T} \ne \emptyset\}$.
Lemma~\ref{lem:probability_error_empty_is_good} gives an upper bound on the probability of error based on the recommendation of the \hyperlink{APGAI}{APGAI} algorithm holding for all time $T$.
\begin{lemma} \label{lem:probability_error_empty_is_good}
    Let $p(x) = \sqrt{x}\exp(-x)$. 
    For all $\mu \in \rR^{K}$ such that $\max_{a \in \ARMS} \mean{a} < \theta$, the \hyperlink{APGAI}{APGAI} satisfies, for all $T >  K$ such that it has not stopped sampling at time $T$,
	\[
		\bP_{\nu}(\GUESS{T} \ne \emptyset) \le  K e \sqrt{2} (2 + \log T) p \left(\frac{T-   K}{18 H_{1}(\mu) } \right) \: .
	\]
\end{lemma}
\begin{proof}
In order to prove Lemma~\ref{lem:probability_error_empty_is_good}, we show key intermediate properties of the \hyperlink{APGAI}{APGAI} algorithm when $\set{\THRESHOLD} = \emptyset$.

\textit{Error due to undersampled arms.}
At a fixed $(T,\delta)$, the set of undersampled arms is
	\[
		\forall t \in ( K, T] \cap \mathbb N, \quad U_{t}(T,\delta) = \left\{a \in \ARMS \mid N_{a}(t) \le \frac{2 \tilde f_{1}(T,\delta)}{\Delta_{a}^2} \right\} \: .
	\]
We show that a necessary condition for an error to occur at time $t$, \ie $\GUESS{t} \ne \emptyset$, is that there exists undersampled arms, \ie $U_{t}(T,\delta) \ne \emptyset$ (Lemma~\ref{lem:error_implies_undersampled_empty_is_good}).
\begin{lemma} \label{lem:error_implies_undersampled_empty_is_good}
	For all $T \in \mathbb N$, under the event $\tilde \cE_{T, \delta}$ as in~Eq.~\eqref{eq:event_PoE}, for all $t \in  ( K, T] \cap \mathbb N$, we have
	\[
	\GUESS{t} \ne \emptyset \quad \implies \quad  U_{t}(T,\delta) \ne \emptyset \: .
	\]
\end{lemma}
\begin{proof}
	Not recommending $\emptyset$ only happens when the largest empirical mean exceeds $\theta$, \ie $\max_{a \in \ARMS} \expmean{a}{t} > \theta$.
    Let $\GUESS{t} = \argmax_{a \in \ARMS} \Wp{a}{t}$ which satisfies $\expmean{\GUESS{t}}{t} > \theta$.
	Under $\tilde \cE_{T, \delta}$ as in~Eq.~\eqref{eq:event_PoE}, we have $\theta < \expmean{\GUESS{t}}{t} \le \mean{\GUESS{t}} + \sqrt{2\tilde f_{1}(T,\delta)/N_{\GUESS{t}}(t)}$, hence $\GUESS{t} \in U_{t}(T,\delta)$.
\end{proof}

\textit{No remaining undersampled arms.}
We show that the events $\{U_{t}(T,\delta) \ne \emptyset\}$ satisfy the conditions of Lemma~\ref{lem:technical_result_bad_event_implies_bounded_quantity_increases}\marc{, hence applying it yields Lemma~\ref{lem:undersampled_arms_not_empty_is_bad_event_empty_is_good}}.
In other words, if there are still undersampled arms at time $t$, then $a_{t+1}$ has not been sampled too many times.
\begin{lemma} \label{lem:undersampled_arms_not_empty_is_bad_event_empty_is_good}
	Let $\delta \in (0,1)$ and $T >  K$.
	Under event $\tilde \cE_{T,\delta}$, for all $t \in ( K, T] \cap \mathbb N$ such that $U_{t}(T,\delta) \ne \emptyset$, we have $N_{a_{t+1}}(t) \le 18\tilde f_{1}(T,\delta)/\Delta_{a_{t+1}}^2$ and $N_{a_{t+1}}(t+1)  = N_{a_{t+1}}(t)  + 1$.
\end{lemma}
\begin{proof}
	We will be interested in three distinct cases since
	\begin{eqnarray*}
		\{U_{t}(T,\delta)  = \emptyset\} & = &\underbrace{\{U_{t}(T,\delta)  = \emptyset, \: \max_{a \in \ARMS} \expmean{a}{t} > \theta\}}_{\text{Case 1}} \cup \underbrace{\{U_{t}(T,\delta)  = \emptyset, \: \max_{a \in \ARMS} \expmean{a}{t} < \theta\}}_{\text{Case 2}} \\
		& & \quad \cup \underbrace{\{U_{t}(T,\delta)  = \emptyset, \: \max_{a \in \ARMS} \expmean{a}{t} = \theta\}}_{\text{Case 3}}
	\end{eqnarray*}

		{\bf Case 1.}
		Let $t \in ( K, T] \cap \mathbb N$ such that $U_{t}(T, \delta) \ne \emptyset$ and $\max_{a \in \ARMS} \expmean{a}{t} > \theta $.
		Let $c = \argmax_{a \in \ARMS} \expmean{a}{t}$. Since $W^{+}_{c}(t) > 0$ and $a_{t+1}  \in \argmax_{a \in \ARMS} W_{a}^{+}(t)$, we obtain $\expmean{a_{t+1}}{t} > \theta$.
		Then, under $\tilde \cE_{T, \delta}$ as in~Eq.~\eqref{eq:event_PoE}, we have
		\begin{align*}
			 \sqrt{N_{a_{t+1}}(t)} (\expmean{a_{t+1}}{t} - \theta)_{+} &= \sqrt{N_{a_{t+1}}(t)} (\expmean{a_{t+1}}{t} - \theta) \\
            &\le \sqrt{N_{a_{t+1}}(t)} (\mean{a_{t+1}} - \theta) + \sqrt{2\tilde f_{1}(T,\delta)} \: .
		\end{align*}
		Using that $W_{a_{t+1}}^{+}(t) > 0$, we obtain $N_{a_{t+1}}(t) \le \frac{2\tilde f_{1}(T,\delta)}{\Delta_{a_{t+1}}^2}$ and $N_{a_{t+1}}(t+1)  = N_{a_{t+1}}(t)  + 1$.

		{\bf Case 2.}
		Let $t \in ( K, T] \cap \mathbb N$ such that $U_{t}(T, \delta) \ne \emptyset$ and $\max_{a \in \ARMS} \expmean{a}{t} < \theta $.
		Let $a_{t+1}  \in \argmin_{a \in \ARMS} W_{a}^{-}(t)$ and $a \in U_{t}(T, \delta)$.
		Then, under $\tilde \cE_{T, \delta}$ as in~Eq.~\eqref{eq:event_PoE}, we have
		\begin{align*}
			&\sqrt{N_{a_{t+1} }(t)} (\theta - \mean{a_{t+1} })  - \sqrt{2\tilde f_{1}(T,\delta)} \le \sqrt{N_{a_{t+1} }(t)} (\theta - \expmean{a_{t+1} }{t}) = \sqrt{N_{a_{t+1} }(t)} (\theta - \expmean{a_{t+1} }{t})_{+}  \\
			&\sqrt{N_{a}(t)} (\theta - \expmean{a}{t})_{+} = \sqrt{N_{a}(t)} (\theta - \expmean{a}{t})  \le \sqrt{N_{a}(t)} (\theta - \mean{a}) + \sqrt{2\tilde f_{1}(T,\delta)}  \le 2\sqrt{2\tilde f_{1}(T,\delta)} 
		\end{align*}
		Using that $W_{a_{t+1}}^{-}(t) \le W_{a}^{-}(t)$, we obtain $N_{a_{t+1}}(t) \le \frac{18\tilde f_{1}(T,\delta)}{\Delta_{a_{t+1}}^2}$ and $N_{a_{t+1}}(t+1)  = N_{a_{t+1}}(t)  + 1$.

		{\bf Case 3.}
		Let $t \in ( K, T] \cap \mathbb N$ such that $U_{t}(T, \delta) \ne \emptyset$ and $\max_{a \in \ARMS} \expmean{a}{t} = \theta $.
		Then, $\argmin_{a \in \ARMS} W_{a}^{-}(t) = \{ a \in \ARMS \mid \expmean{a}{t} = \theta \}$.
		Therefore, we have $\expmean{a_{t+1}}{t} = \theta$ hence $\theta = \expmean{a_{t+1}}{t}  \le \mean{a_{t+1}} + \sqrt{\frac{2\tilde f_{1}(T,\delta)}{N_{a_{t+1}}(t)}}$.
		Therefore, we obtain $N_{a_{t+1}}(t) \le \frac{2\tilde f_{1}(T,\delta)}{\Delta_{a_{t+1}}^2}$ and $N_{a_{t+1}}(t+1)  = N_{a_{t+1}}(t)  + 1$.

		{\bf Summary.} Combing the three above cases yields the result.
\end{proof}

Lemma~\ref{lem:time_no_undersampled_empty_is_good} provides a time after which all arms are sampled enough, hence no error will be made.
\begin{lemma} \label{lem:time_no_undersampled_empty_is_good}
	Let us define $T_{\mu}(\delta) = \sup \left\{ T \mid T \le  18 H_{1}(\mu) \tilde f_{1}(T,\delta) +   K \right\}$.
	For all $T > T_{\mu}(\delta)$, under the event $\tilde \cE_{T, \delta}$ as in~Eq.~\eqref{eq:event_PoE}, we have $U_{T}(T,\delta)  = \emptyset$.
\end{lemma}
\begin{proof}
	Combining Lemmas~\ref{lem:undersampled_arms_not_empty_is_bad_event_empty_is_good} and~\ref{lem:technical_result_bad_event_implies_bounded_quantity_increases}, we obtain $\sum_{t =  K+1}^{T} \indic{U_{t}(T, \delta) \ne \emptyset} \le 18 H_{1}(\mu) \tilde f_{1}(T,\delta)$.
	For all $a \in \ARMS$, let us define $t_{a}(T, \delta)  = \max\{t \in ( K, T] \cap \mathbb N \mid a \in U_{t}(T, \delta)\}$.
	By definition, we have $a \in U_{t}(T, \delta)$ for all $t \in ( K, t_{a}(T, \delta)]$ and $a \notin U_{t}(T, \delta)$ for all $t \in (t_{a}(T, \delta), T]$.
	Therefore, for all $t\in ( K, \max_{a \in \ARMS} t_{a}(T, \delta)]$, we have $U_{t}(T, \delta) \ne \emptyset$ and $U_{t}(T, \delta) = \emptyset$ for all $t > \max_{a \in \ARMS} t_{a}(T, \delta)$, hence
	\begin{align*}
		 \max_{a \in \ARMS} (t_{a}(T, \delta)-  K) =	\sum_{t =  K+1}^{T} \indic{U_{t}(T, \delta) \ne \emptyset} \le 18 H_{1}(\mu) \tilde f_{1}(T,\delta) \: .
	\end{align*}
	Let $T_{\mu}(\delta)$ defined as in the statement of Lemma~\ref{lem:time_no_undersampled_empty_is_good} and $T > T_{\mu}(\delta)$.
	Then, we have
	\[
	T -   K > 18 H_{1}(\mu) \tilde f_{1}(T,\delta) \ge \max_{a \in \ARMS} (t_{a}(T, \delta) -   K) \: ,
	\]
	hence $T > \max_{a \in \ARMS} t_{a}(T, \delta)$. This concludes the proof that $U_{T}(T,\delta)  = \emptyset$.
\end{proof}

\textit{Conclusion} Let $T_{\mu}(\delta)$ as in Lemma~\ref{lem:time_no_undersampled_empty_is_good}.
	Combining Lemmas~\ref{lem:time_no_undersampled_empty_is_good}, ~\ref{lem:error_implies_undersampled_empty_is_good} and~\ref{lem:concentration_per_arm_gau_improved}, we obtain
	\begin{align*}
        &\forall T > T_{\mu}(\delta), \quad \{\GUESS{T} \ne \emptyset \} \cap \tilde \cE_{T,\delta} = \emptyset \quad \text{and} \quad \bP_{\nu}(\tilde \cE_{T,\delta}^{\complement}) \le K \delta \quad \text{hence} \\
	&\bP_{\nu}(\GUESS{T} \ne \emptyset) \le K \inf \{ \delta \mid \: T > T_{\mu}(\delta) \} \le K e \sqrt{2} (2 + \log T) \sqrt{ \frac{T -   K}{18 H_{1}(\mu) } }  \exp \left(-\frac{T-   K}{18 H_{1}(\mu) } \right) \: ,
\end{align*}
where the last inequality uses Lemma~\ref{lem:inversion_PoE}.	
This concludes the proof of Lemma~\ref{lem:probability_error_empty_is_good}.    
\end{proof}

\subsection{Instances where \texorpdfstring{$\set{\THRESHOLD} \ne \emptyset$}{}}
\label{app:sssec_PoE_exist_good_arms}

When $\set{\THRESHOLD} = \emptyset$, we have $ \cE^{\text{err}}_{\mu}(\NBATCHES) = \{\GUESS{\NBATCHES} = \emptyset \} \cup \{ \GUESS{\NBATCHES} \in \set{\THRESHOLD}^{\complement} \}$.
Lemma~\ref{lem:probability_error_exist_good_arms} gives an upper bound on the probability of error based on the recommendation of \hyperlink{APGAI}{APGAI} holding for all time $T$.
\begin{lemma} \label{lem:probability_error_exist_good_arms}
    Let $p(x) = \sqrt{x}\exp(-x)$. 
    For all $\mu \in \rR^{K}$ such that $\set{\THRESHOLD} \ne \emptyset$ and $\mean{a} \ne \theta$ for all $a \in \ARMS$, the \hyperlink{APGAI}{APGAI} satisfies, for all $T >  K$ such that it has not stopped sampling at time $T$,
	\begin{align*}
		\bP\left(\{\GUESS{T} = \emptyset \} \cup \{ \GUESS{\NBATCHES} \in \set{\THRESHOLD}^{\complement}\} \right) &\le  K e \sqrt{2} (2 + \log T) p \left(\frac{T-   K -  2|\set{\THRESHOLD}|}{4 H_{1}(\mu) } \right)  \: .
\end{align*}
\end{lemma}
\begin{proof}
In order to prove Lemma~\ref{lem:probability_error_exist_good_arms}, we show key intermediate properties of the \hyperlink{APGAI}{APGAI} algorithm when $\set{\THRESHOLD} \ne \emptyset$.

\textit{Error due to undersampled arms.}
At a fixed $(T,\delta)$, the set of under-sampled arms is
\begin{align*}
	\forall t \in ( K, T] \cap \mathbb N, \quad U_{t}(T, \delta) &= \left\{ a \in \ARMS \mid N_{a}(t) \le \left(\sqrt{\frac{2\tilde f_{1}(T,\delta)}{\Delta_{a}^2}} + 1 \right)^2 \right\}  \: .
\end{align*}

Lemma~\ref{lem:error_implies_undersampled_exist_good_arms} shows that a necessary condition to recommend $\emptyset$ at time $t$ is that all the good arms are undersampled arms, \ie $\set{\THRESHOLD} \subseteq U_{t}(T,\delta) $.
It also shows that a necessary condition to recommend $\GUESS{t} \in \set{\THRESHOLD}^{\complement}$ at time $t$ is that this arm is undersampled and will be sampled next, \ie $\GUESS{t} =  a_{t+1}$ \marc{and $a_{t+1}$}$  \in \set{\THRESHOLD}^{\complement} \cap U_{t}(T,\delta)$.
\begin{lemma} \label{lem:error_implies_undersampled_exist_good_arms}
	For all $T \in \mathbb N$, under the event $\tilde \cE_{T, \delta}$ as in~Eq.~\eqref{eq:event_PoE}, for all $t \in ( K, T] \cap \mathbb N$, we have
	\begin{align*}
		&\GUESS{t} = \emptyset  \quad \implies \quad \set{\THRESHOLD} \subseteq U_{t}(T,\delta)  \: , \\
		&\GUESS{t} \in \set{\THRESHOLD}^{\complement} \quad \implies \quad \GUESS{t} =  a_{t+1} \marc{\text{   and   }a_{t+1}} \in \set{\THRESHOLD}^{\complement} \cap U_{t}(T,\delta)  \: .
	\end{align*}
\end{lemma}
\begin{proof}
	{\bf Case 1.} Suppose that $\GUESS{t} = \emptyset$, hence $\max \expmean{a}{t} \le \theta$.
	Then, for all $a \in \set{\THRESHOLD}$, we have $\theta \ge  \expmean{a}{t} \ge \mean{a} - \sqrt{2 \tilde f_{1}(T, \delta)/N_{a}(t)}$, hence $\set{\THRESHOLD} \subseteq U_{t}(T,\delta)$.

	{\bf Case 2.} Suppose that $\GUESS{t} \notin \set{\THRESHOLD}$, hence $\max \expmean{a}{t} > \theta$.
	Since $a_{t+1}  \in \argmax_{a \in \ARMS} W^{+}_{a}(t)$ and $\GUESS{t} = a_{t+1}$, we have $\expmean{\GUESS{t}}{t} > \theta$.
	Then, we have $\theta <  \expmean{a_{t+1}}{t} \le \mean{a_{t+1}} + \sqrt{2 \tilde f_{1}(T, \delta)/N_{a_{t+1}}(t)}$, hence $a_{t+1} \in \set{\THRESHOLD}^{\complement} \cap U_{t}(T,\delta)$.
\end{proof}

\textit{One good arm not undersampled.}
Lemma~\ref{lem:undersampled_arms_contains_all_good_arms_is_bad_event_exist_good_arms} shows that the events $\{\set{\THRESHOLD} \subseteq  U_{t}(T,\delta)\}$ are satisfying the conditions of Lemma~\ref{lem:technical_result_bad_event_implies_bounded_quantity_increases}.
In other words, having all the good arms undersampled implies that the next arm we will pull was not sampled a lot.
\begin{lemma} \label{lem:undersampled_arms_contains_all_good_arms_is_bad_event_exist_good_arms}
	Let $\delta \in (0,1)$ and $T >  K$.
	Under event $\tilde \cE_{T,\delta}$, for all $t \in ( K, T] \cap \mathbb N$ such that $\set{\THRESHOLD} \subseteq  U_{t}(T,\delta)$, we have $N_{a_{t+1}}(t) \le D_{a_{t+1}}(T,\delta) $ and $N_{a_{t+1}}(t+1)  = N_{a_{t+1}}(t)  + 1$, where $D_{a}(T,\delta) = (\Delta_{a}^{-1}\sqrt{2\tilde f_{1}(T,\delta)} + 1 )^2$ for all $a \in \set{\THRESHOLD}$ and $D_{a}(T,\delta) = 2\tilde f_{1}(T,\delta)\Delta_{a}^{-2}$ for all $a \notin \set{\THRESHOLD}$.
\end{lemma}
\begin{proof}
	Let $t \in ( K, T] \cap \mathbb N$ such that $\set{\THRESHOLD} \subseteq  U_{t}(T,\delta)$.
When $a_{t+1} \in \set{\THRESHOLD} $, we have directly that $N_{a_{t+1}}(t) \le (\sqrt{2\tilde f_{1}(T,\delta)/\Delta_{a_{t+1}}^2} + 1 )^2$ and $N_{a_{t+1}}(t+1)  = N_{a_{t+1}}(t)  + 1$.
In the following, we consider $a_{t+1} \notin \set{\THRESHOLD}$.
	We will be interested in three cases since
 \begin{align*}
     &\{\set{\THRESHOLD} \subseteq  U_{t}(T,\delta), \: a_{t+1} \notin \set{\THRESHOLD} \}  =	 \underbrace{\{\set{\THRESHOLD} \subseteq  U_{t}(T,\delta) , \: a_{t+1} \notin \set{\THRESHOLD} , \: \max_{a \in \ARMS} \expmean{a}{t} > \theta \}}_{\text{Case 1}} \\
     &\cup 	\underbrace{\{\set{\THRESHOLD} \subseteq  U_{t}(T,\delta) , \: a_{t+1} \notin \set{\THRESHOLD} , \: \max_{a \in \ARMS} \expmean{a}{t} < \theta \}}_{\text{Case 2}} \cup 	\underbrace{\{\set{\THRESHOLD} \subseteq  U_{t}(T,\delta) , \: a_{t+1} \notin \set{\THRESHOLD}, \: \max_{a \in \ARMS} \expmean{a}{t} = \theta \}}_{\text{Case 3}} 
 \end{align*}
		{\bf Case 1.}
			Let $t \in ( K, T] \cap \mathbb N$ such that $\set{\THRESHOLD} \subseteq  U_{t}(T,\delta)$, $a_{t+1} \notin \set{\THRESHOLD}$ and $\max_{a \in \ARMS} \expmean{a}{t} > \theta $.
			Let $c = \argmax_{a \in \ARMS} \expmean{a}{t}$. Since $W^{+}_{c}(t) > 0$ and $a_{t+1}  \in \argmax_{a \in \ARMS} W_{a}^{+}(t)$, we have $\expmean{a_{t+1}}{t} > \theta$.
			Since $a_{t+1} \notin \set{\THRESHOLD}$, under $\tilde \cE_{T, \delta}$ as in~Eq.~\eqref{eq:event_PoE}, we have
			\begin{align*}
				 &\sqrt{N_{a_{t+1}}(t)} (\expmean{a_{t+1}}{t} - \theta)_{+} = \sqrt{N_{a_{t+1}}(t)} (\expmean{a_{t+1}}{t} - \theta) \le \sqrt{N_{a_{t+1}}(t)} (\mean{a_{t+1}} - \theta) + \sqrt{2\tilde f_{1}(T,\delta)} 
			\end{align*}
		Using that $W_{a_{t+1}}^{+}(t) > 0$, we obtain $N_{a_{t+1}}(t) \le 2\tilde f_{1}(T,\delta)/\Delta_{a_{t+1}}^2$ and $N_{a_{t+1}}(t+1)  = N_{a_{t+1}}(t)  + 1$.

			{\bf Case 2.}
			Let $t \in ( K, T] \cap \mathbb N$ such that $\set{\THRESHOLD} \subseteq  U_{t}(T,\delta)$, $a_{t+1} \notin \set{\THRESHOLD}$ and $\max_{a \in \ARMS} \expmean{a}{t} < \theta $.
			Let $a_{t+1}  \in \argmin_{a \in \ARMS} W_{a}^{-}(t)$.
			Since $a_{t+1} \notin \set{\THRESHOLD}$, under $\tilde \cE_{T, \delta}$ as in~Eq.~\eqref{eq:event_PoE}, for all $a \in \set{\THRESHOLD}$, we have
			\begin{align*}
				&\sqrt{N_{a_{t+1}}(t)} (\theta - \mean{a_{t+1}})  - \sqrt{2\tilde f_{1}(T,\delta)} \le \sqrt{N_{a_{t+1}}(t)} (\theta - \expmean{a_{t+1}}{t}) = \sqrt{N_{a_{t+1}}(t)} (\theta - \expmean{a_{t+1}}{t})_{+}  \\
				&\sqrt{N_{a}(t)} (\theta - \expmean{a}{t})_{+} = \sqrt{N_{a}(t)} (\theta - \expmean{a}{t})  \le \sqrt{N_{a}(t)} (\theta - \mean{a}) + \sqrt{2\tilde f_{1}(T,\delta)}  \le \sqrt{2\tilde f_{1}(T,\delta)} \: .
			\end{align*}
			Combining both inequality by using that $W_{a_{t+1}}^{-}(t) \le W_{a}^{-}(t)$ yields $\sqrt{N_{a_{t+1}}(t)} (\theta - \mean{a_{t+1}}) \le 2 \sqrt{2\tilde f_{1}(T,\delta)}$, hence $N_{a_{t+1}}(t) \le 8\tilde f_{1}(T,\delta)/\Delta_{a_{t+1}}^2 $ and $N_{a_{t+1}}(t+1)  = N_{a_{t+1}}(t)  + 1$.

			{\bf Case 3.}
			Let $t \in ( K, T] \cap \mathbb N$ such that $\set{\THRESHOLD} \subseteq  U_{t}(T,\delta)$, $a_{t+1} \notin \set{\THRESHOLD}$ and $\max_{a \in \ARMS} \expmean{a}{t} = \theta $.
			Then, $a_{t+1} \in \argmin_{a \in \ARMS} W_{a}^{-}(t) = \{ a \in \ARMS \mid \expmean{a}{t} = \theta \}$.
			Therefore, we have $\theta = \expmean{a_{t+1}}{t}  \le \mean{a_{t+1}} + \sqrt{2\tilde f_{1}(T,\delta)/N_{a_{t+1}}(t)} $.
			Since $a_{t+1} \notin \set{\THRESHOLD}$, we obtain $N_{a_{t+1}}(t) \le 2\tilde f_{1}(T,\delta)/\Delta_{a_{t+1}}^2 $ and $N_{a_{t+1}}(t+1)  = N_{a_{t+1}}(t)  + 1$.

			{\bf Summary.}  Combing the three above cases yields the result.
\end{proof}

Lemma~\ref{lem:one_good_arm_no_undersampled_implies_no_error} shows that having a good arm that is sampled enough, \ie $\set{\THRESHOLD} \cap  U_{t}(T,\delta)^{\complement} \ne \emptyset$, is a sufficient condition to recommend a good arm, \ie $\GUESS{t} \in \set{\THRESHOLD}$.
\begin{lemma} \label{lem:one_good_arm_no_undersampled_implies_no_error}
		Let $\delta \in (0,1)$ and $T >  K$.
		Under event $\tilde \cE_{T,\delta}$, for all $t \in ( K, T] \cap \mathbb N$ such that $\set{\THRESHOLD} \cap  U_{t}(T,\delta)^{\complement} \ne \emptyset$, we have $\GUESS{t} \in \set{\THRESHOLD}$.
\end{lemma}
\begin{proof}
	Let $t \in ( K, T] \cap \mathbb N$ such that $\set{\THRESHOLD} \cap  U_{t}(T,\delta)^{\complement} \ne \emptyset$.
	Let $a \in \set{\THRESHOLD} \cap  U_{t}(T, \delta)^{\complement} $, hence 
	\begin{equation} \label{eq:int_step_one_good_arm_no_undersampled_implies_no_error}
	N_{a}(t) > \left(\sqrt{\frac{2\tilde f_{1}(T,\delta)}{(\mean{a} - \theta)^2}} + 1 \right)^2 > \frac{2\tilde f_{1}(T,\delta)}{(\mean{a} - \theta)^2} \: .
	\end{equation}
	Therefore, under $\tilde \cE_{T, \delta}$ as in~Eq.~\eqref{eq:event_PoE}, we have $\max_{b \in \ARMS} \expmean{b}{t} \ge \expmean{a}{t} \ge \mean{a} - \sqrt{2\tilde f_{1}(T,\delta)/N_{a}(t)}  > \theta$, hence $\GUESS{t} =  a_{t+1} \in \argmax_{a \in \ARMS} W_{a}^{+}(t)$.

	Suppose towards contradiction that $\set{\THRESHOLD}^{\complement} \cap \argmax_{a \in \ARMS} W_{a}^{+}(t) \ne \emptyset$.
	Let $a \in \set{\THRESHOLD}^{\complement} \cap \argmax_{a \in \ARMS} W_{a}^{+}(t) \ne \emptyset$.
	It is direct to see that $\expmean{a}{t} > \theta$, otherwise there is a contradiction.
	Then, using that $a \in \set{\THRESHOLD}^{\complement}$ (\ie $\mean{a} \le \theta$), we have for all $b \in  \set{\THRESHOLD} \cap  U_{t}(T, \delta)^{\complement}$
	\begin{align*}
		&\sqrt{2\tilde f_{1}(T,\delta)} \ge \sqrt{N_{a}(t)}(\mean{a} - \theta) + \sqrt{2\tilde f_{1}(T,\delta)} \ge \sqrt{N_{a}(t)}(\expmean{a}{t} - \theta) = \sqrt{N_{a}(t)}(\expmean{a}{t} - \theta)_{+}\: , \\
		&\sqrt{N_{b}(t)}(\expmean{b}{t} - \theta)_{+}  = \sqrt{N_{b}(t)}(\expmean{b}{t} - \theta) \ge \sqrt{N_{b}(t)}(\mean{b} - \theta) - \sqrt{\frac{2\tilde f_{1}(T,\delta)}{N_{b}(t)}} \\
		&\qquad > \left(\sqrt{N_{b}(t)}-1\right)(\mean{b} - \theta) > \sqrt{2\tilde f_{1}(T,\delta)} \: ,
	\end{align*}
        where the two last inequalities are obtained by using~Eq.~\eqref{eq:int_step_one_good_arm_no_undersampled_implies_no_error} first the smaller thresholds, then the one in-between.
	Since $a \ne b$ and $W_{a}^{+}(t) \ge W_{b}^{+}(t)$, combining the above yields $\sqrt{2\tilde f_{1}(T,\delta)} > \sqrt{2\tilde f_{1}(T,\delta)}$ which is a contradiction.
	Therefore, we have proven that
	\begin{align*}
	\set{\THRESHOLD} \cap  U_{t}(T,\delta)^{\complement} \ne \emptyset \quad &\implies \quad \GUESS{t} \in \argmax_{a \in \ARMS} W_{a}^{+}(t) \: \land \: \set{\THRESHOLD}^{\complement} \cap \argmax_{a \in \ARMS} W_{a}^{+}(t) = \emptyset  
\end{align*}
which implies that $\GUESS{t} \in \set{\THRESHOLD}$.
\end{proof}

Lemma~\ref{lem:time_no_undersampled_exist_good_arms} provides a time after which there exists a good arms which is sampled enough, hence no error will be made.
\begin{lemma} \label{lem:time_no_undersampled_exist_good_arms}
	Let us define $S_{\mu}(\delta) = \sup \left\{ T \mid T  \le 4 H_{1}(\mu) \tilde f_{1}(T,\delta) +   K + 2|\set{\THRESHOLD}| \right\} $.
	For all $T > S_{\mu}(\delta)$, under the event $\tilde \cE_{T, \delta}$ as in~Eq.~\eqref{eq:event_PoE}, we have $\set{\THRESHOLD} \cap  U_{T}(T, \delta)^{\complement} \ne \emptyset$ and $\GUESS{T} \in \set{\THRESHOLD}$.
\end{lemma}
\begin{proof}
		Let $(D_{a}(T,\delta))_{a \in \ARMS}$ as in Lemma~\ref{lem:undersampled_arms_contains_all_good_arms_is_bad_event_exist_good_arms}.
		Combining Lemmas~\ref{lem:undersampled_arms_contains_all_good_arms_is_bad_event_exist_good_arms} and~\ref{lem:technical_result_bad_event_implies_bounded_quantity_increases}, we obtain $\sum_{t =  K+1}^{T} \indic{\set{\THRESHOLD} \subseteq  U_{t}(T,\delta)} \le \sum_{a \in \ARMS} D_{a}(T,\delta)$.
		For all $a \in \set{\THRESHOLD}$, let us define $t_{a}(T, \delta)  = \max\{t \in ( K, T] \cap \mathbb N \mid a \in U_{t}(T, \delta)\}$.
		By definition, we have $a \in U_{t}(T, \delta)$ for all $t \in ( K, t_{a}(T, \delta)]$ and $a \notin U_{t}(T, \delta)$ for all $t \in (t_{a}(T, \delta), T]$.
		Therefore, for all $t \in ( K, \min_{a \in \set{\THRESHOLD}} t_{a}(T, \delta)]$, we have $\set{\THRESHOLD} \subseteq U_{t}(T, \delta)$ and $\set{\THRESHOLD} \cap U_{t}(T, \delta)^{\complement} \ne \emptyset$ for all $t > \max_{a \in \ARMS} t_{a}(T, \delta)$, hence
		\begin{align*}
			 \min_{a \in \set{\THRESHOLD}} (t_{a}(T, \delta)-K) =	\sum_{t =  K+1}^{T} \indic{\set{\THRESHOLD} \subseteq  U_{t}(T,\delta)} \le \sum_{a \in \ARMS} D_{a}(T,\delta) \: .
		\end{align*}
		Let $S_{\mu}(\delta)$ defined as in the statement of Lemma~\ref{lem:time_no_undersampled_exist_good_arms} and $T > S_{\mu}(\delta)$.
		Using that $(a+1)^2 \le 2a^2 + 2$, we have $S_{\mu}(\delta) \ge \sup \{ T \mid T \le  \sum_{a \in \ARMS} D_{a}(T,\delta) +   K  \}$.
		Then, we have
		\[
		T -   K > \sum_{a \in \ARMS} D_{a}(T,\delta) \ge \min_{a \in \set{\THRESHOLD}} (t_{a}(T, \delta)-   K) \: ,
		\]
		hence $T > \min_{a \in \set{\THRESHOLD}} t_{a}(T, \delta)$.
		Therefore, we have $\set{\THRESHOLD} \cap  U_{T}(T, \delta)^{\complement} \ne \emptyset$.
		Using Lemma~\ref{lem:one_good_arm_no_undersampled_implies_no_error}, we obtain that $\GUESS{T} \in \set{\THRESHOLD}$.
		This concludes the proof.
\end{proof}

\textit{Conclusion.}
	Let $S_{\mu}(\delta)$ as in Lemma~\ref{lem:time_no_undersampled_exist_good_arms}.
	Combining Lemmas~\ref{lem:time_no_undersampled_exist_good_arms},~\ref{lem:one_good_arm_no_undersampled_implies_no_error} and~\ref{lem:concentration_per_arm_gau_improved}, we obtain
	\begin{align*}
        &\forall T > S_{\mu}(\delta), \quad \left(\{\GUESS{T} = \emptyset \} \cup \{ \GUESS{T} \in \set{\THRESHOLD}^{\complement} \} \right) \cap \tilde \cE_{T,\delta} = \emptyset \quad \text{and} \quad \bP_{\nu}(\tilde \cE_{T,\delta}^{\complement}) \le K \delta \quad \text{hence} \\
	&\bP_{\nu}(\{\GUESS{T} = \emptyset \} \cup \{ \GUESS{T} \in \set{\THRESHOLD}^{\complement} \}) \le K \inf \{ \delta \mid \: T > S_{\mu}(\delta) \} \\
	&\quad \le K e \sqrt{2} (2 + \log T) \sqrt{ \frac{T -  K -  2|\set{\THRESHOLD}|}{4 H_{1}(\mu) } }  \exp \left(-\frac{T -  K -  2|\set{\THRESHOLD}|}{4 H_{1}(\mu) } \right) \: ,
\end{align*}
where the last inequality uses Lemma~\ref{lem:inversion_PoE}.	
		This concludes the proof.
\end{proof}

\subsection{Unverifiable Sample Complexity: Proof of Theorem~\ref{thm:unverifiable_sample_complexity}}
\label{app:ssec_unverifiable_sample_complexity}

	In Appendix~\ref{app:sssec_PoE_empty_is_good} and~\ref{app:sssec_PoE_exist_good_arms}, we consider the concentration event $\tilde \cE_{\NBATCHES,\delta}$ that involved tighter concentration results with thresholds $\tilde f_{1}(\NBATCHES,\delta)$.
	Let $\NBATCHES > K$ and $\delta \in (0,1)$. 
    It is direct to see that the same argument holds for the concentration events $ \cE_{\NBATCHES,\delta}$ as in~Eq.~\eqref{eq:event_concentration_per_arm_aeps} for $s = \marc{2}$, i.e.,
	\begin{align*}
		 \cE_{\NBATCHES,\delta} &= \left\{ \forall a \in \ARMS, \forall t \le \NBATCHES, \: |\expmean{a}{t} - \mu_a| < \sqrt{\frac{2  f_1(\NBATCHES,\delta)}{N_{a}(t)}} \right\} \: , 
	\end{align*}
	where $f_1(\NBATCHES,\delta) = \log(1/\delta) + \marc{3}\log \NBATCHES + \log \marc{(K\pi^2/6)} $.
    \marc{Let $U_{\delta}(\mu) > K$ to be specified below.}
 Using Lemma~\ref{lem:concentration_per_arm_gau_aeps}, we obtain that 
 \[ 
    \marc{\bP_{\nu}\left(\bigcup_{\NBATCHES > U_{\delta}(\mu)}\cE_{\NBATCHES,\delta}^\complement\right) \le \sum_{\NBATCHES > U_{\delta}(\mu)} \bP_{\nu}\left(\cE_{\NBATCHES,\delta}^\complement\right) \le  \frac{\delta}{\zeta(2)} \sum_{\NBATCHES > U_{\delta}(\mu)} \frac{1}{T^2} \le \delta} \: .
\]
\marc{
Suppose that $U_{\delta}(\mu)$ is chosen such that $\cE^{\text{err}}_{\mathfrak{A}}(T) \cap \cE_{\NBATCHES,\delta} = \emptyset$ for all $T > U_{\delta}(\mu)$.
Then, we have}
\begin{align*}
    \bP_{\nu}\left(\bigcup_{T > U_{\delta}(\mu)} \cE^{\text{err}}_{\mathfrak{A}}(T) \right) &\le \bP_{\nu}\left(\left(\bigcup_{T > U_{\delta}(\mu)} \cE^{\text{err}}_{\mathfrak{A}}(T) \right)\cap \left(\bigcap_{\NBATCHES > U_{\delta}(\mu)}\cE_{\NBATCHES,\delta}\right) \right) + \bP_{\nu}\left( \bigcup_{\NBATCHES > U_{\delta}(\mu)}\cE_{\NBATCHES,\delta}^\complement \right) \\ 
    &\le \bP_{\nu}\left(\bigcup_{T > U_{\delta}(\mu)} \left(\cE^{\text{err}}_{\mathfrak{A}}(T) \cap \cE_{\NBATCHES,\delta}\right) \right) + \delta \le \delta \: .
\end{align*}
\marc{Therefore, we can conclude the proof by exhibiting $U_{\delta}(\mu)$ satisfying the above property.}

	{\bf Case 1: \modif{when} $\set{\THRESHOLD} = \emptyset$.} Let $T_{\mu}(\delta)$ defined similarly as in Lemma~\ref{lem:time_no_undersampled_empty_is_good}, \ie
 \[ 
 T_{\mu}(\delta) \deff  \sup \left\{ T \mid T \le  18 H_{1}(\mu) f_{1}(T,\delta) +  K \right\}  \: .
 \]
 To prove Theorem~\ref{thm:upper_bound_PoE_anytimeID} when $\set{\THRESHOLD} = \emptyset$, we obtain as an intermediary result that: for all $\NBATCHES > T_{\mu}(\delta)$, $\{ \GUESS{T} \ne \emptyset \} \subseteq \cE_{\NBATCHES,\delta}^{\complement}$. Using a proof similar to Lemma~\ref{lem:inversion_upper_bound}, applying Lemma~\ref{lem:property_W_lambert} yields that
\begin{align*}
	&T > T_{\mu}(\delta)  \:\\
        &\iff \:  T > \marc{54} H_{1}(\mu) \log T + 18 H_{1}(\mu)\log \left( \marc{\frac{K\pi^2}{6\delta}}\right) + K  \\
	&\iff \: \frac{T}{\marc{54}  H_{1}(\mu)} - \log \left( \frac{T}{\marc{54} H_{1}(\mu)} \right) > \marc{\frac{1}{3}} \log \left( \marc{\frac{K\pi^2}{6\delta}}\right) + \frac{ K}{\marc{54} H_{1}(\mu)} + \log (\marc{54}  H_{1}(\mu)) \\
	&\iff \: T > \marc{54} H_{1}(\mu) \overline{W}_{-1} \left( \marc{\frac{1}{3}} \log \left( \marc{\frac{K\pi^2}{6\delta}}\right) + \frac{ K}{\marc{54} H_{1}(\mu)} + \log (\marc{54}  H_{1}(\mu))\right)  \: ,
\end{align*} 
Let us define $U_{\delta}(\mu) \deff h_{2}(\marc{\delta}, \marc{54} H_{1}(\mu),  K) $, where
\[
 h_{2}(\marc{\delta}, A, B) \deff A \overline{W}_{-1} \left( \marc{\frac{1}{3}} \log \left( \marc{\frac{K\pi^2}{6\delta}}\right)  +  \frac{B}{A} + \log A \right)
\]
satisfies that $h_{2}(\marc{\delta}, A, B) =_{\delta \to 0} \marc{\frac{A}{3}}  \log(1/\delta) + \mathcal O(\log\log(1/\delta)) $.
Hence, we have shown that \marc{$\{ \GUESS{T} \ne \emptyset \} \cap \cE_{\NBATCHES,\delta} = \emptyset$ for all $\NBATCHES > U_{\delta}(\mu)$.
This concludes the proof when $\set{\THRESHOLD} = \emptyset$. }
 
 {\bf Case 2: \modif{when} $\set{\THRESHOLD} \ne \emptyset$.} Let $S_{\mu}(\delta)$ defined similarly as in Lemma~\ref{lem:time_no_undersampled_exist_good_arms}, \ie
 \[ 
 S_{\mu}(\delta) \deff \sup \left\{ T \mid T \le 4 H_{1}(\mu)  f_{1}(T,\delta) +  K + 2|\set{\THRESHOLD}| \right\} \: .
 \]
	To prove Theorem~\ref{thm:upper_bound_PoE_anytimeID} when $\set{\THRESHOLD} \ne \emptyset$, we obtain as an intermediary result that: for all $\NBATCHES > S_{\mu}(\delta)$, $\{ \GUESS{T} = \emptyset \} \cup \{\GUESS{T} \in \set{\THRESHOLD}^{\complement}\} \subseteq \cE_{\NBATCHES,\delta}^{\complement}$. Using a proof similar to Lemma~\ref{lem:inversion_upper_bound}, applying Lemma~\ref{lem:property_W_lambert} yields that
\begin{align*}
	&T > S_{\mu}(\delta)  \:\\ 
        &\iff \:  T > \marc{12} H_{1}(\mu) \log T + 4 H_{1}(\mu) \log \left( \marc{\frac{K\pi^2}{6\delta}}\right) +  K + 2|\set{\THRESHOLD}| \\
	&\iff \: \frac{T}{\marc{12} H_{1}(\mu)} - \log \left( \frac{T}{\marc{12} H_{1}(\mu)} \right) >  \marc{\frac{1}{3}} \log\left( \marc{\frac{K\pi^2}{6\delta}}\right) + \frac{ K + 2|\set{\THRESHOLD}|}{\marc{12}  H_{1}(\mu)} + \log (\marc{12} H_{1}(\mu)) \\
	&\iff \: T > \marc{12} H_{1}(\mu) \overline{W}_{-1} \left( \marc{\frac{1}{3}} \log\left( \marc{\frac{K\pi^2}{6\delta}}\right) + \frac{ K + 2|\set{\THRESHOLD}|}{\marc{12}  H_{1}(\mu)} + \log (\marc{12} H_{1}(\mu)) \right)  \: ,
\end{align*} 
Let us define $U_{\delta}(\mu) \deff h_{2}(\marc{\delta}, \marc{12} H_{1}(\mu),  K+ 2|\set{\THRESHOLD}|)$ where $h_{2}$ is as above. 
Then, we have shown that \marc{$\left(\{ \GUESS{T} = \emptyset \} \cup \{ \GUESS{T} \in \set{\THRESHOLD}^{\complement} \} \right) \cap \cE_{\NBATCHES,\delta} = \emptyset$ for all $\NBATCHES > U_{\delta}(\mu)$.
This concludes the proof when $\set{\THRESHOLD} \ne \emptyset$. }
$\qed$

\subsection{\marc{Time Uniform Probability of Error}}
\label{app:proof_cor_upper_bound_PoE_anytimeID}

\marc{Corollary~\ref{cor:upper_bound_PoE_anytimeID} gives an upper bound on the time-uniform probability of error of \hyperlink{APGAI}{APGAI}.}
\begin{corollary} \label{cor:upper_bound_PoE_anytimeID}
   \marc{
   Let $\alpha_{i_{\mu}} $ as in Theorem~\ref{thm:upper_bound_PoE_anytimeID}.
   The \hyperlink{APGAI}{APGAI} algorithm $\mathfrak{A}$ satisfies that, for all $\nu \in \mathcal D^{K}$ with mean $\mu$ such that $\Delta_{\min}  > 0$,
    \[ 
    \bP_{\nu}\left(\bigcup_{t>\NARMS +  2|\set{\THRESHOLD}|} \cE^{\text{err}}_{\mathfrak{A}}(t)\right) \le \inf_{\delta \in (0,1)} \{\delta + \NARMS \alpha_{i_{\mu}} H_{1}(\mu) e \sqrt{2} \gamma_{\mu}(\delta)  \} \: ,
    \]
    where $\gamma_{\mu}(\delta)$ as in~Eq.~\eqref{eq:complicated_term_for_time_unif_PoE} satisfies that $ \limsup_{\delta \to 0}  \gamma_{\mu}(\delta) < + \infty$.}
\end{corollary}
\begin{proof}
\marc{
Combining Theorems~\ref{thm:upper_bound_PoE_anytimeID} and~\ref{thm:unverifiable_sample_complexity}, one can easily upper bound the time-uniform probability or error to obtain Corollary~\ref{cor:upper_bound_PoE_anytimeID}.
Let $\delta \in (0,1)$ and $U_{\delta}(\mu)$ as in Theorem~\ref{thm:unverifiable_sample_complexity}.
Let $p(x) = x - 0.5\log x$ and $\alpha_{i_{\mu}}$ as in Theorem~\ref{thm:upper_bound_PoE_anytimeID}.
Using Theorems~\ref{thm:upper_bound_PoE_anytimeID} and~\ref{thm:unverifiable_sample_complexity}, a union bound yields
\begin{align*}
    &\bP_{\nu}\left(\bigcup_{T > \NARMS +  2|\set{\THRESHOLD}|} \cE^{\text{err}}_{\mathfrak{A}}(T) \right) \\ 
    &\le \bP_{\nu}\left(\bigcup_{\NARMS +  2|\set{\THRESHOLD}| < T \le U_{\delta}(\mu)} \cE^{\text{err}}_{\mathfrak{A}}(T) \right) + \bP_{\nu}\left(\bigcup_{T > U_{\delta}(\mu)} \cE^{\text{err}}_{\mathfrak{A}}(T) \right) \\ 
    &\le \delta + \NARMS e \sqrt{2}  \sum_{\NARMS +  2|\set{\THRESHOLD}| < T \le U_{\delta}(\mu)} \log (e^2 T) \sqrt{\frac{T - \NARMS -  2|\set{\THRESHOLD}|}{2 \alpha_{i_{\mu}} H_{1}(\mu)} } \exp\left(- \frac{T - \NARMS -  2|\set{\THRESHOLD}|}{2 \alpha_{i_{\mu}} H_{1}(\mu)} \right)  \\
    &\le \delta + \NARMS \alpha_{i_{\mu}} H_{1}(\mu) e \sqrt{2} \int_{\left(0, x_\mu(\delta) \right)} \left(2+ \ln\left(2\alpha_{i_{\mu}} H_{1}(\mu)  x  + \NARMS +  2|\set{\THRESHOLD}|\right) \right) \sqrt{x} e^{-x} \mathrm{d}x        \: ,
\end{align*} 
where $x_\mu(\delta) \deff \frac{U_{\delta}(\mu) - \NARMS -  2|\set{\THRESHOLD}|}{2\alpha_{i_{\mu}} H_{1}(\mu)}$. 
The last inequality uses that $T \le U_{\delta}(\mu)$ and bounds the summation by the integral with the change of variable $x = \frac{T - \NARMS -  2|\set{\THRESHOLD}|}{2\alpha_{i_{\mu}} H_{1}(\mu)}$.
The lower incomplete gamma function is defined $\gamma(s,x) = \int_{x \in \left(0, x\right)} t^{s-1} \exp\left(-t\right) \mathrm{d}t$.
Let $(\alpha,\beta) \in (\R_{+})^2$ such that $\beta > 1$.
Then, we define 
\[
    \tilde \gamma(s,x,\alpha,\beta) \deff \int_{x \in \left(0, x\right)}  \left(2 + \ln(\alpha t + \beta ) \right) t^{s-1} \exp\left(-t\right) \mathrm{d}t  \: ,
\]
Therefore, we have shown that
\begin{align}
    &\bP_{\nu}\left(\bigcup_{T > \NARMS +  2|\set{\THRESHOLD}|} \cE^{\text{err}}_{\mathfrak{A}}(T) \right) \le \delta + \NARMS \alpha_{i_{\mu}} H_{1}(\mu) e \sqrt{2} \gamma_{\mu}(\delta) \: , \nonumber \\ 
    &\text{where} \quad \gamma_{\mu}(\delta) \deff  \tilde\gamma \left( \frac{3}{2}, \frac{U_{\delta}(\mu) - \NARMS -  2|\set{\THRESHOLD}|}{2\alpha_{i_{\mu}} H_{1}(\mu)}, 2\alpha_{i_{\mu}} H_{1}(\mu), \NARMS +  2|\set{\THRESHOLD}|\right) \label{eq:complicated_term_for_time_unif_PoE} \: .
\end{align} 
Taking the infimum over $\delta \in (0,1)$ concludes the proof.
Up to multiplicative constant depending on $(\alpha,\beta)$, $\tilde \gamma$ behaves similarly as $\gamma$ when $x \to +\infty$, as the behavior of $t \mapsto \ln(\alpha t + \beta) t^{s-1} e^{-t}$ resembles the one of $t \mapsto t^{s-1} e^{-t}$.
Since $\lim_{x \to +\infty} \gamma(s,x) = \Gamma(s) $ where $\Gamma$ is the gamma function, we have $ \limsup_{\delta \to 0}  \gamma_{\mu}(\delta) < + \infty$ and we conjecture that 
\[
    \limsup_{\delta \to 0} \tilde\gamma \left( \frac{3}{2}, \frac{U_{\delta}(\mu) - \NARMS -  2|\set{\THRESHOLD}|}{2\alpha_{i_{\mu}} H_{1}(\mu)}, 2\alpha_{i_{\mu}} H_{1}(\mu), \NARMS +  2|\set{\THRESHOLD}|\right) = \mathcal O\left(  \log  H_{1}(\mu)\right)  \: .
\]}   
\end{proof}

\section{Analysis of Other GAI Algorithms}\label{app:guarantees_other_algorithms}

In Appendix~\ref{app:guarantees_other_algorithms}, we \marc{give extensive guarantees for} uniform sampling (Unif) in GAI (Appendix~\ref{app:ssec_uniform_sampling_PoE}) \marc{anytime guarantees (Appendix~\ref{app:sssec_unif_PoE}), unverifiable sample complexity bounds (Appendix~\ref{app:sssec_unif_unverifiable}) and fixed confidence guarantees (Appendix~\ref{app:sssec_unif_FC}).}
\marc{We also provide }fixed-budget guarantees of Sequential Halving and Successive Reject when modified to tackle GAI (SH-G in Appendix~\ref{app:ssec_SH_PoE} and SR-G in Appendix~\ref{app:ssec_SR_PoE}).

\subsection{Uniform Sampling (Unif)}
\label{app:ssec_uniform_sampling_PoE}

Uniform sampling (Unif) combines a uniform round-robin sampling rule with the recommendation rule used by \hyperlink{APGAI}{APGAI}, namely
\begin{equation}    \label{eq:reco_rule_Anytime}
    \GUESS{\NBATCHES} = \emptyset \quad \text{if } \max_{a \in \ARMS} \expmean{a}{\NBATCHES} \le  \THRESHOLD \quad \text{else} \quad \GUESS{\NBATCHES} \in \argmax_{a \in \ARMS} \Wp{a}{\NBATCHES} \: .
\end{equation}
At time $t$ such that $t/K \in \nN$, the recommendation of Unif is equivalent to outputing the arm with the largest empirical mean when $\max_{a \in \ARMS} \expmean{a}{\NBATCHES} >  \THRESHOLD$ since $\argmax_{a \in \ARMS} \Wp{a}{t} = \argmax_{a \in \ARMS} \expmean{a}{t}$ and $N_{a}(t) = t/K$ for all $a \in \ARMS$.
The goal is to compare the rate obtained in the exponential decrease of the probability of error with the one in Theorem~\ref{thm:upper_bound_PoE_anytimeID}.
Since they have the same recommendation rule, this would allow us to measure the benefit of adaptive sampling.

\subsubsection{\marc{Anytime Guarantees on the Probability of Error}}
\label{app:sssec_unif_PoE}

Theorem~\ref{thm:uniform_sampling_PoE_recoAnytime} shows that the exponential decrease of the probability of error of Unif is linear as a function of time.
\begin{theorem} \label{thm:uniform_sampling_PoE_recoAnytime}
    Let $\mathfrak{A}$ be Unif with recommendation rule~Eq.~\eqref{eq:reco_rule_Anytime}.
    Then, for any $1$-sub-Gaussian distribution $\nu \in \mathcal D^{K}$ with mean $\mu$ such that $\Delta_{\min} > 0$, and for all $t > \NARMS$ such that $t/K \in \nN$,
    \begin{align*}
        \text{if }\set{\THRESHOLD} = \emptyset, \quad &\perr{\nu}{\mathfrak{A}}{t} \le \NARMS \exp \left( - \frac{t \min_{a \in \ARMS} \Delta_{a}^2}{2 \NARMS } \right) \: , \\
	\text{if }\set{\THRESHOLD} \ne \emptyset, \quad &\perr{\nu}{\mathfrak{A}}{t} \le  (|\set{\THRESHOLD}^{\complement}| + 1)\exp \left( - \frac{T \max_{a \in \set{\THRESHOLD}} \Delta_{a}^2}{4K} \right)   \: .
    \end{align*}
\end{theorem}
\begin{proof}
    For the sake of simplicity, we consider only times $t$ that are multiples of $K$.
    Therefore, at time $\NBATCHES$, we have $N_{a}(\NBATCHES) = \NBATCHES/K$ for all arms $a \in \ARMS$.
    We distinguish between the cases (1) $\set{\THRESHOLD} = \emptyset$ and (2) $\set{\THRESHOLD} \ne \emptyset$.

    {\bf Case 1: $\set{\THRESHOLD} = \emptyset$.} 
    When $\set{\THRESHOLD} = \emptyset$, we have $\cE^{\text{err}}_{\mu}(\NBATCHES) = \{\GUESS{\NBATCHES} \ne \emptyset \}  = \{ \max_{a \in \ARMS} \expmean{a}{\NBATCHES} > \THRESHOLD \} = \bigcup_{a \in \ARMS} \{ \expmean{a}{\NBATCHES} >  \THRESHOLD\}$.
    Since the empirical are deterministic and the observations comes from a $1$-sub-Gaussian with mean $\mu_{a} <  \THRESHOLD$, we obtain that for all $a \in \ARMS$
    \begin{align*}
        \bP_{\nu}(\expmean{a}{\NBATCHES} >  \THRESHOLD) = \bP\left( \frac{K}{\NBATCHES} \sum_{s = 1}^{\NBATCHES/K} X_{s} >  \Delta_{a} \right) \le \exp \left( - \frac{T\Delta_{a}^2}{2K} \right) \: .
    \end{align*}
    Using that $H_{6}(\mu) = 1/\min_{a \in \ARMS} \Delta_{a}^2$, a direct union bound yields that
    \[
        \perr{\nu}{\mathfrak{A}}{\NBATCHES} \le \sum_{a \in [K]} \exp \left( - \frac{T\Delta_{a}^2}{2K} \right) \le K\exp \left( - \frac{T \min_{a \in \ARMS} \Delta_{a}^2}{2 K } \right) \: .
    \]    
    {\bf Case 2: $\set{\THRESHOLD} \ne \emptyset$.} 
    When $\set{\THRESHOLD} = \emptyset$, we have $\cE^{\text{err}}_{\mu}(\NBATCHES) = \{\GUESS{\NBATCHES} = \emptyset \} \cup \{ \GUESS{\NBATCHES} \in \set{\THRESHOLD}^{\complement} \}$, hence
    \begin{align*}
    \cE^{\text{err}}_{\mu}(\NBATCHES) = \{ \max_{a \in \ARMS} \expmean{a}{\NBATCHES} \le \THRESHOLD \} \cup \{ \max_{a \in \ARMS} \expmean{a}{\NBATCHES} > \THRESHOLD, \: \argmax_{a \in \ARMS} \Wp{a}{\NBATCHES} \cap \set{\THRESHOLD}^{\complement} \ne \emptyset \} \: .
    \end{align*}
    Let $a^\star \in \argmax_{a \in \ARMS} \mu_{a}$. By inclusion, we have $\{ \max_{a \in \ARMS} \expmean{a}{\NBATCHES} \le \THRESHOLD \} \subset \{ \expmean{a^\star}{\NBATCHES} \le \THRESHOLD\}$.
    Therefore, since $N_{a^\star}(\NBATCHES) = \NBATCHES/K$ using similar argument as above yields that
    \[
    \bP_{\nu}(\expmean{a^\star}{\NBATCHES} \le  \THRESHOLD) \le \exp \left( - \frac{T \max_{a \in \ARMS} \Delta_{a}^2}{2K} \right) \: .
    \]
    Since $N_{a}(\NBATCHES) = \NBATCHES/K$ for all $a \in \ARMS$, we have $\argmax_{a \in \ARMS} \Wp{a}{\NBATCHES} = \argmax_{a \in \ARMS} \expmean{a}{\NBATCHES}$.
    Therefore, we have 
    \[
        \{ \max_{a \in \ARMS} \expmean{a}{\NBATCHES} > \THRESHOLD, \: \argmax_{a \in \ARMS} \Wp{a}{\NBATCHES} \cap \set{\THRESHOLD}^{\complement} \ne \emptyset \} \subseteq \bigcup_{ b \notin \set{\THRESHOLD}}  \{\expmean{b}{\NBATCHES} \ge \expmean{a^\star}{\NBATCHES}\} \: .
    \]
    Likewise, we obtain that
    \begin{align*}
        \bP_{\nu}(\expmean{b}{\NBATCHES} \ge \expmean{a^\star}{\NBATCHES}) = \bP\left(\frac{K}{\NBATCHES} \sum_{s = 1}^{\NBATCHES/K} (X_{s} -  Y_{s}) \ge \mu_{a^\star} - \mu_{b}\right) \le \exp \left( - \frac{T(\mu_{a^\star} - \mu_{b})^2}{4K} \right) \: .
    \end{align*}
    Therefore, we obtain
    \begin{align*}
    \perr{\nu}{\mathfrak{A}}{\NBATCHES} &\le \exp \left( - \frac{T \max_{a \in \ARMS} \Delta_{a}^2}{2K} \right)  + \sum_{a \notin \set{\THRESHOLD}} \exp \left( - \frac{T(\mu_{a^\star} - \mu_{b})^2}{4K} \right) \\
    &\le \exp \left( - \frac{T \max_{a \in \set{\THRESHOLD}} \Delta_{a}^2}{2K} \right)  + |\set{\THRESHOLD}^{\complement}| \exp \left( - \frac{T (\max_{a \in \set{\THRESHOLD}} \Delta_{a} + \min_{b \notin \set{\THRESHOLD}} \Delta_{b})^2}{4 K} \right) \\
    &\le (|\set{\THRESHOLD}^{\complement}| + 1)\exp \left( - \frac{T \max_{a \in \set{\THRESHOLD}} \Delta_{a}^2}{4K} \right)   \: .
    \end{align*}
\end{proof}

\subsubsection{\marc{Unverifiable Sample Complexity}}
\label{app:sssec_unif_unverifiable}

\marc{
Theorem~\ref{thm:Unif_unverifiable_sample_complexity} gives a deterministic upper bound $U_{\delta}(\mu)$ on the unverifiable sample complexity $ \tau_{U,\delta}$ of Unif for GAI.
Its proof is similar to the one of Theorem~\ref{thm:unverifiable_sample_complexity} by adapting the arguments used in Theorem~\ref{thm:uniform_sampling_PoE_recoAnytime}.}
\begin{theorem} \label{thm:Unif_unverifiable_sample_complexity}
	\marc{Let $\delta \in (0,1)$.
	The Unif algorithm satisfies that, for any $1$-sub-Gaussian distribution with mean $\mu$ such that $\Delta_{\min} > 0$, we have $\bP_{\nu}(\bigcup_{t \ge U_{\delta}(\mu)} \cE^{\text{err}}_{\mathfrak{A}}(t)) \le \delta$ where 
	\begin{equation*} U_{\delta}(\mu) =\begin{cases}
	      h_{2}\left(\delta, \frac{6K}{\min_{a \in \ARMS }\Delta_{a}^2},  K \right) & \text{if }\set{\THRESHOLD} = \emptyset \\ 
            h_{2}\left(\delta, \frac{24K}{\max_{a \in  \set{\THRESHOLD}}\Delta_{a}^2},  K \right)  & \text{if } \set{\THRESHOLD} \ne \emptyset 
	   \end{cases}\: ,
	\end{equation*} 
    where $h_{2}$ is defined in Theorem~\ref{thm:unverifiable_sample_complexity}.}
\end{theorem}
\begin{proof}
\marc{
    Let $\NBATCHES > K$ and $\delta \in (0,1)$. 
    Let $ \cE_{\NBATCHES,\delta}$ as in~Eq.~\eqref{eq:event_concentration_per_arm_aeps} for $s =2$, i.e.,
	\begin{align*}
		 \cE_{\NBATCHES,\delta} &= \left\{ \forall a \in \ARMS, \forall t \le \NBATCHES, \: |\expmean{a}{t} - \mu_a| < \sqrt{\frac{2  f_1(\NBATCHES,\delta)}{N_{a}(t)}} \right\} \: , 
	\end{align*}
	where $f_1(\NBATCHES,\delta) = \log(1/\delta) + 3\log \NBATCHES + \log (K\pi^2/6) $.
    Let $U_{\delta}(\mu) > K$ to be specified below.
Suppose that $U_{\delta}(\mu)$ is chosen such that $\cE^{\text{err}}_{\mathfrak{A}}(T) \cap \cE_{\NBATCHES,\delta} = \emptyset$ for all $T > U_{\delta}(\mu)$.
Using the same arguments as in Theorem~\ref{thm:unverifiable_sample_complexity}, we can conclude the proof by exhibiting $U_{\delta}(\mu)$ satisfying the above property since
\begin{align*}
    \bP_{\nu}\left(\bigcup_{T > U_{\delta}(\mu)} \cE^{\text{err}}_{\mathfrak{A}}(T) \right) &\le \bP_{\nu}\left(\bigcup_{T > U_{\delta}(\mu)} \left(\cE^{\text{err}}_{\mathfrak{A}}(T) \cap \cE_{\NBATCHES,\delta}\right) \right) + \delta \le \delta \: .
\end{align*}
By definition of Unif, we have $N_{a}(\NBATCHES) \ge \lfloor T/K \rfloor \ge T/K - 1$.}

\marc{{\bf Case 1: when $\set{\THRESHOLD} = \emptyset$.} Using the same arguments as in Theorem~\ref{thm:uniform_sampling_PoE_recoAnytime}, one can show that
    \[
     \{\GUESS{\NBATCHES} \ne \emptyset \} \cap \cE_{\NBATCHES,\delta} \subseteq \bigcup_{a \in \ARMS} \left\{ \mu_a + \sqrt{\frac{2  f_1(\NBATCHES,\delta)}{N_{a}(\NBATCHES)}}  >  \THRESHOLD \right\} \subseteq  \left\{ \frac{2K  f_1(\NBATCHES,\delta)}{\min_{a \in \ARMS }\Delta_{a}^2}  + K >  T \right\} \: .
    \]
    Let $T_{\mu}(\delta) \deff \sup \{T \mid T \le  \frac{2K  f_1(\NBATCHES,\delta)}{\min_{a \in \ARMS }\Delta_{a}^2}  + K \}$.
    Then, we have $ \{\GUESS{\NBATCHES} \ne \emptyset \} \cap \cE_{\NBATCHES,\delta} = \emptyset$ for all $T > T_{\mu}(\delta)$.
    Let $h_{2}$ as in Theorem~\ref{thm:unverifiable_sample_complexity} and $U_{\delta}(\mu) \deff h_{2}\left(\delta, \frac{6K}{\min_{a \in \ARMS }\Delta_{a}^2},  K \right) $
    Applying Lemma~\ref{lem:property_W_lambert} as in Theorem~\ref{thm:unverifiable_sample_complexity}, we obtain $T > T_{\mu}(\delta)$ if and only if $ T > U_{\delta}(\mu)$.
This concludes the proof when $\set{\THRESHOLD} = \emptyset$. }
 
 \marc{{\bf Case 2: when $\set{\THRESHOLD} \ne \emptyset$.} 
 Let $a^\star \in \argmax_{a \in \ARMS} \mu_{a}$.
 Then, $\max_{a \in \set{\THRESHOLD}}\Delta_{a}^2 = \Delta_{a^\star}$ and $\min_{b \notin \set{\THRESHOLD} }( \mu_{a^\star}  - \mu_{b} ) \ge \Delta_{a^\star}$.
  Using the same arguments as in Theorem~\ref{thm:uniform_sampling_PoE_recoAnytime} and the same manipulation as above, one can show that
  \begin{align*}
      &\left(\{ \GUESS{T} = \emptyset \} \cup \{\GUESS{T} \in \set{\THRESHOLD}^{\complement}\}\right) \cap  \cE_{\NBATCHES,\delta} \\ 
      &\subseteq \left\{  \mu_{a^\star} - \sqrt{\frac{2  f_1(\NBATCHES,\delta)}{N_{a^\star}(\NBATCHES)}}  <  \THRESHOLD  \right\} \cup \bigcup_{b \notin \set{\THRESHOLD} } \left\{ \mu_{b} + \sqrt{\frac{2  f_1(\NBATCHES,\delta)}{N_{b}(\NBATCHES)}} >  \mu_{a^\star} - \sqrt{\frac{2  f_1(\NBATCHES,\delta)}{N_{a^\star}(\NBATCHES)}}    \right\} \\ 
      &\subseteq \left\{  \frac{8K  f_1(\NBATCHES,\delta)}{\max_{a \in \set{\THRESHOLD}}\Delta_{a}^2}  + K >  T \right\} \: .
  \end{align*}
 Taking $U_{\delta}(\mu) \deff h_{2}\left(\delta, \frac{24K}{\max_{a \in  \set{\THRESHOLD}}\Delta_{a}^2},  K \right) $ concludes the proof for the case $\set{\THRESHOLD} \ne \emptyset$, similarly as above.}
\end{proof}

\subsubsection{\marc{Fixed Confidence Guarantees}}
\label{app:sssec_unif_FC}

\marc{Theorem~\ref{thm:uniform_sampling_sample_complexity_upper_bound} gives an upper bound on the expected sample complexity of the Unif algorithm coupled with the GLR stopping rule~Eq.~\eqref{eq:stopping_rule} with threshold~Eq.~\eqref{eq:stopping_threshold} holding for any confidence $\DELTA$.
Its proof resembles the one of Theorem~\ref{thm:expected_sample_complexity_upper_bound}.
Using similar manipulation as in Appendix~\ref{app:sssec_explicit_ub}, one could obtain more explicit upper bound $C_{\mu}(\delta)$.
While we omit those statements for simplicity, they would show that the $\delta$-independent scaling of the upper bound is $\mathcal O\left(\frac{K}{\min_{a \in \ARMS} \Delta_{a}^2} \log \left(\frac{K}{\min_{a \in \ARMS} \Delta_{a}^2}\right)\right)$ when $\set{\THRESHOLD} = \emptyset$, and $\mathcal O\left(\frac{K}{\max_{a \in  \set{\THRESHOLD} } \Delta_{a}^2}  \log \left(\frac{K}{\max_{a \in  \set{\THRESHOLD} } \Delta_{a}^2} \right)\right)$ otherwise.}

\begin{theorem} \label{thm:uniform_sampling_sample_complexity_upper_bound}
    \marc{Let $\DELTA \in (0,1)$.
    Combined with GLR stopping~Eq.~\eqref{eq:stopping_rule} using threshold~Eq.~\eqref{eq:stopping_threshold}, Unif is $\delta$-correct and it satisfies that, for all $\nu \in \mathcal D^{K}$ with mean $\mu$ such that $\Delta_{\min}  > 0$,
    \begin{align*} 
        &\bE_{\nu}[\tau_{\delta}] \le C_{\mu}(\delta) + \frac{K \pi^2}{6} + 1  \quad \text{ where } \\
        &C_{\mu}(\delta) \deff \begin{cases}
            \sup \left\{ t \mid t \le \frac{2K}{\min_{a \in \ARMS} \Delta_{a}^2} ( \sqrt{c(t, \delta)} + \sqrt{3 \log t})^2 + K \right\}  &\text{if } \set{\THRESHOLD} = \emptyset \\ 
            \sup \left\{ t \mid t \le \frac{2K}{\max_{a \in  \set{\THRESHOLD} } \Delta_{a}^2} ( \sqrt{c(t, \delta)} + \sqrt{3 \log t})^2 + K \right\}  &\text{if } \set{\THRESHOLD} \ne \emptyset 
        \end{cases} \: , \\
        &\text{and} \quad \limsup_{\delta \to 0} \frac{\bE_{\nu}[\tau_{\delta}]}{\log(1/\delta)} \le \begin{cases}
             \frac{2K}{\min_{a \in \ARMS} \Delta_{a}^2}  &\text{if } \set{\THRESHOLD} = \emptyset \\ 
             \frac{2K}{\max_{a \in  \set{\THRESHOLD} } \Delta_{a}^2} &\text{if } \set{\THRESHOLD} \ne \emptyset 
        \end{cases} \: .
    \end{align*}}
\end{theorem}
\begin{proof}
\marc{The $\delta$-correctness property is a direct consequence of Lemma~\ref{lem:delta_correct_threshold}.}

\marc{For all $T >  K$, let $\cE_{T} = \cE_{T,1}$ where $\cE_{T,\delta}$ as in~Eq.~\eqref{eq:event_concentration_per_arm_aeps} with $s=2$, \ie 
\[
\cE_{T} = \left\{ \forall a \in \ARMS, \forall t \le T, \: |\expmean{a}{t} - \mu_a| < \sqrt{2 f_1(T)/N_{a}(t)} \right\} \: ,  
\]
with $ f_1(T) =  3 \log T$.
Using Lemma~\ref{lem:concentration_per_arm_gau_aeps}, we have $\sum_{T > K} \bP_{\nu}(\cE_{T}^{\complement}) \le  K \pi^2/6$.
Suppose that we have constructed a time $T_{\mu}(\delta) >  K$ such that $ \cE_{T} \subseteq \{\tau_{\delta} \le T\}$ for $T \ge T_{\mu}(\delta)$.
Then, using Lemma~\ref{lem:lemma_1_Degenne19BAI}, we obtain $\bE_{\nu}[\tau_{\delta}] \le T_{\mu}(\delta) +   K \pi^2/6$.
Therefore, one can conclude the proof by exhibiting such $ T_{\mu}(\delta) $.
By definition of Unif, we have $N_{a}(\NBATCHES) \ge \lfloor T/K \rfloor \ge T/K - 1$.}

\marc{{\bf Case 1: when $\set{\THRESHOLD} = \emptyset$.}
By definition of $\tau_{\delta}$, we have $\tau_{\delta} \le \tau_{<,\delta}$ almost surely.
Under $\cE_{T}$, we obtain, for all $a\in \ARMS$,
\[
     \sqrt{\nsamples{a}{T}} (\THRESHOLD-\expmean{a}{T}) \ge \sqrt{\nsamples{a}{T}}\left(\THRESHOLD - \mu_{a}\right) - \sqrt{2 f_1(T)} \ge \sqrt{T/K - 1}\min_{a\in \ARMS} \Delta_{a} - \sqrt{6 \log(T)} \: .
\]
Then, under $\cE_{T} \cap \{\tau_{\delta} > T\}$, we obtain
\begin{align*}
    \sqrt{2c(T, \DELTA)} \ge  \min_{a\in \ARMS} \sqrt{\nsamples{a}{T}} (\THRESHOLD-\expmean{a}{T})_{+} \ge \left( \sqrt{T/K - 1}\min_{a\in \ARMS} \Delta_{a}  - \sqrt{6 \log(T)} \right)_{+} 
\end{align*}
Let us define
\[
 C_{\mu}(\delta) \deff \sup \left\{ t \mid t \le \frac{2K}{\min_{a \in \ARMS} \Delta_{a}^2} ( \sqrt{c(t, \delta)} + \sqrt{3 \log t})^2 + K \right\} \: .
\]
By re-ordering the above equation, we obtain $\cE_{T} \cap \{\tau_{\delta} > T\} = \emptyset$ for all $T > C_{\mu}(\delta)$.
Therefore, taking $T_{\mu}(\delta) = C_{\mu}(\delta) + 1$ concludes the proof when $\set{\THRESHOLD} = \emptyset$.}

\marc{{\bf Case 2: when $\set{\THRESHOLD} \ne \emptyset$.} 
By definition of $\tau_{\delta}$, we have $\tau_{\delta} \le \tau_{>,\delta}$ almost surely.
Let $a^\star \in \argmax_{a \in \ARMS} \mu_{a}$.
Then, we have $\Delta_{a^\star} = \max_{a \in \set{\THRESHOLD}}\Delta_{a}$.
Under $\cE_{T}$, we obtain, 
\[
     \sqrt{\nsamples{a^\star}{T}} (\expmean{a^\star}{T} - \THRESHOLD) \ge \sqrt{\nsamples{a^\star}{T}}\left(\mu_{a^\star} - \THRESHOLD \right) - \sqrt{2 f_1(T)} \ge \sqrt{T/K - 1}\max_{a \in \set{\THRESHOLD}}\Delta_{a}- \sqrt{6 \log(T)} \: .
\]
Then, under $\cE_{T} \cap \{\tau_{\delta} > T\}$, we obtain
\begin{align*}
    \sqrt{2c(T, \DELTA)} \ge  \max_{a\in \ARMS} \sqrt{\nsamples{a}{T}} (\expmean{a}{T} - \THRESHOLD)_{+} \ge  \left( \sqrt{T/K - 1}\max_{a \in \set{\THRESHOLD}}\Delta_{a}- \sqrt{6 \log(T)}  \right)_{+} 
\end{align*}
Let us define
\[
 C_{\mu}(\delta) \deff \sup \left\{ t \mid t \le \frac{2K}{\max_{a \in \ARMS} \Delta_{a}^2} ( \sqrt{c(t, \delta)} + \sqrt{3 \log t})^2 + K \right\} \: .
\]
By re-ordering the above equation, we obtain $\cE_{T} \cap \{\tau_{\delta} > T\} = \emptyset$ for all $T > C_{\mu}(\delta)$.
Therefore, taking $T_{\mu}(\delta) = C_{\mu}(\delta) + 1$ concludes the proof when $\set{\THRESHOLD} \ne \emptyset$.}

\marc{The asymptotic upper bounds are a direct consequence of Lemma~\ref{lem:asymptotic_inversion_result}.}
\end{proof}

\subsection{Sequential Halving for GAI (SH-G)}
\label{app:ssec_SH_PoE}

In Appendix~\ref{app:ssec_SH_PoE}, we study the SH~\citep{karnin2013almost} algorithm where instead of recommending the last active arm $a_{\NBATCHES}$, we recommend 
\begin{equation}  \label{eq:reco_rule_elimination_algos}
    \GUESS{\NBATCHES} = \emptyset \quad \text{if } \expmean{a_{\NBATCHES}}{\NBATCHES} \le  \THRESHOLD \quad \text{else} \quad \GUESS{\NBATCHES} = a_{\NBATCHES} \: .
\end{equation}
We refer to this modified SH algorithm as SH-G.
In SH, there are two arms $(a_{1},a_{2})$ at the last of the $\lceil \log_{2}(K)\rceil$ phases.
Then, both arms are pulled $N_{T} = \left\lfloor \frac{T}{2 \lceil \log_{2}(K)\rceil} \right\rfloor $ times.
Since SH drops the sampled collected in the previous phase, the last active arm $a_{\NBATCHES}$ is based on the comparison of the empirical mean of each arm after $N_{T}$ samples.

Theorem~\ref{thm:SH_PoE_recoElim} shows that the exponential decrease of the probability of error of SH-G is linear as a function of time.
The notation $\tilde{\Theta}(\cdot)$ hides logarithmic factors which were not made explicit in Theorems $1$ and $5$ from~\cite{zhao2022revisiting}.
Since one component of our proof uses their result, we suffer from this lack of explicit constant in that case.
\begin{theorem} \label{thm:SH_PoE_recoElim}
    Let $\NBATCHES > K$. Let $\mathfrak{A}_{\NBATCHES}$ be the SH-G algorithm with recommendation rule as in~Eq.~\eqref{eq:reco_rule_elimination_algos}.
    Then, for any $1$-sub-Gaussian distribution $\nu \in \mathcal D^{K}$ with mean $\mu$ such that $\Delta_{\min} > 0$,
    \begin{align*}
        &\text{if }\set{\THRESHOLD} = \emptyset, \quad \perr{\nu}{\mathfrak{A}_{\NBATCHES}}{\NBATCHES} \le K \exp \left( -\frac{\NBATCHES \min_{a \in \ARMS} \Delta_{a}^2}{4 \lceil \log_{2}(K) \marc{\rceil}} + \min_{a \in \ARMS} \Delta_{a}^2/2 \right) \: , \\
	&\text{if }\set{\THRESHOLD} \ne \emptyset, \quad \perr{\nu}{\mathfrak{A}_{\NBATCHES}}{\NBATCHES} \le |\set{\THRESHOLD}| \exp \left( -  \frac{\NBATCHES \min_{a \in \set{\THRESHOLD}} \Delta_{a}^2}{4 \lceil \log_{2}(K)\rceil } + \min_{a \in \set{\THRESHOLD}} \Delta_{a}^2/2\right)+ \\
            & \min \left\{ 3 \log_2(K)\exp \left( - \frac{T}{8 \log_2 (K) \max_{i > I^\star} i (\max_{a \in \ARMS} \mu_{a} - \mu_{(i)})^{-2}}\right),\: \exp \left( - \tilde{\Theta}  \left(  \frac{T}{G_{1}(\mu)} \right) \right) \right\}
    \end{align*}
    where $I^\star = |\argmax_{a \in \ARMS} \mu_{a}|$ and $G_{1}(\mu)$ is defined in~Eq.~\eqref{eq:improved_SHG}.
\end{theorem}

\begin{proof}    
    We distinguish between the cases (1) $\set{\THRESHOLD} = \emptyset$ and (2) $\set{\THRESHOLD} \ne \emptyset$.

    {\bf Case 1: $\set{\THRESHOLD} = \emptyset$.} 
    When $\set{\THRESHOLD} = \emptyset$, we have
    \[
    \cE^{\text{err}}_{\mu}(\NBATCHES) = \{\GUESS{\NBATCHES} \ne \emptyset \}   = \{\GUESS{\NBATCHES} \ne \emptyset , \expmean{a_{\NBATCHES}}{\NBATCHES} >  \THRESHOLD \}  \subseteq  \{ \expmean{a_{\NBATCHES}}{\NBATCHES}  > \THRESHOLD \} = \bigcup_{a \in \ARMS} \{ a_{\NBATCHES} = a,\: \expmean{a}{\NBATCHES} >  \THRESHOLD\} \: .
    \]
    Therefore, using $N_{a_{\NBATCHES}}(\NBATCHES) = N_{\NBATCHES} \ge \frac{\NBATCHES}{2 \lceil \log_{2}(K)\rceil} -1$ (drop observations from past phases) and similar argument as in the proof of Theorem~\ref{thm:uniform_sampling_PoE_recoAnytime}, we obtain
    \begin{align*}
        \perr{\nu}{\mathfrak{A}_{\NBATCHES}}{\NBATCHES} \le   \sum_{a \in \ARMS} \exp \left( -\frac{N_{T}}{2}\Delta_{a}^{2} \right) \le   K e^{\min_{a \in \ARMS} \Delta_{a}^2/2} \exp \left( -\frac{\NBATCHES \min_{a \in \ARMS} \Delta_{a}^2}{4 \lceil \log_{2}(K) \rceil}  \right) \: .
    \end{align*}
    
    {\bf Case 2: $\set{\THRESHOLD} \ne \emptyset$.} 
    When $\set{\THRESHOLD} \ne \emptyset$, we have $\cE^{\text{err}}_{\mu}(\NBATCHES) = \{\GUESS{\NBATCHES} = \emptyset \} \cup \{ \GUESS{\NBATCHES} \in \set{\THRESHOLD}^{\complement} \}$. \marc{By definition of the recommendation rule of SH-G in Eq.~\eqref{eq:reco_rule_elimination_algos}, we obtain}
    \begin{align*}
    &\marc{\{\GUESS{\NBATCHES} = \emptyset \} = \{\GUESS{\NBATCHES} = \emptyset, \: a_{\NBATCHES} \in \set{\THRESHOLD}, \: \expmean{a_{\NBATCHES}}{\NBATCHES} \le \THRESHOLD  \} \cup \{\GUESS{\NBATCHES} = \emptyset, \: a_{\NBATCHES} \in \set{\THRESHOLD}^{\complement} , \: \expmean{a_{\NBATCHES}}{\NBATCHES} \le \THRESHOLD\} } \\ 
    &\qquad \qquad \: \marc{\subseteq \{ a_{\NBATCHES} \in \set{\THRESHOLD}, \: \expmean{a_{\NBATCHES}}{\NBATCHES} \le \THRESHOLD  \} \cup \{ a_{\NBATCHES} \in \set{\THRESHOLD}^{\complement} \} \: ,}  \\ 
    &\marc{ \{ \GUESS{\NBATCHES} \in \set{\THRESHOLD}^{\complement} \} = \{ \GUESS{\NBATCHES} \in \set{\THRESHOLD}^{\complement}, \: a_{\NBATCHES} \in \set{\THRESHOLD}^{\complement} , \: \expmean{a_{\NBATCHES}}{\NBATCHES} > \THRESHOLD \}  \subseteq \{ a_{\NBATCHES} \in \set{\THRESHOLD}^{\complement} \} } \: .
    \end{align*}
    \marc{The dichotomy on whether the last active arm $a_T$ is a good arm or not is crucial when $\GUESS{\NBATCHES} = \emptyset$. When $a_{\NBATCHES} \in \set{\THRESHOLD}$, having $\GUESS{\NBATCHES} = \emptyset$ implies that this arm was not sampled enough to ensure that $\expmean{a_{\NBATCHES}}{\NBATCHES} > \THRESHOLD $, even though it satisfies $\mu_{a_T} > \theta$. Since it is sampled linearly, it means that the budget $T$ is not large enough compared to the difficulty $1/\min_{a \in \set{\THRESHOLD}} \Delta_{a}^2$.
    When $a_{\NBATCHES} \notin \set{\THRESHOLD}$, having $\GUESS{\NBATCHES} = \emptyset$ implies that all the good arms have been eliminated in previous phases.
    Therefore, SH has eliminated the best arm in previous phases, namely we have
    \[
    \{ a_{\NBATCHES} \in \set{\THRESHOLD}^{\complement} \} \subseteq \{ a_{\NBATCHES} \notin a^\star(\mu) \} \quad \text{where} \quad a^\star(\mu) \deff \argmax_{a \in [K]} \mu_{a} \subseteq \set{\THRESHOLD} \: .
    \]
    Using existing analysis of SH, $\{ a_{\NBATCHES} \notin a^\star(\mu) \}$ is known to have a low probability of occuring. 
    Putting everything together, we have shown that
    \[
        \cE^{\text{err}}_{\mu}(\NBATCHES) \subseteq \{ a_{\NBATCHES} \in \set{\THRESHOLD}, \: \expmean{a_{\NBATCHES}}{\NBATCHES} \le \THRESHOLD  \}  \cup  \{  a_{\NBATCHES} \in \set{\THRESHOLD}^{\complement} \} \subseteq \{ a_{\NBATCHES} \in \set{\THRESHOLD}, \: \expmean{a_{\NBATCHES}}{\NBATCHES} \le \THRESHOLD  \}  \cup  \{  a_{\NBATCHES} \notin a^\star(\mu) \} \: .
    \]}
    Since $N_{a_{\NBATCHES}}(\NBATCHES) = N_{\NBATCHES} \ge \frac{\NBATCHES}{2 \lceil \log_{2}(K)\rceil} -1$, using similar argument as above yields that 
    \begin{align*}
    \bP_{\nu}(a_{\NBATCHES} \in \set{\THRESHOLD}, \: \expmean{a_{\NBATCHES}}{\NBATCHES} \le \THRESHOLD) &\le \sum_{a \in  \set{\THRESHOLD}} \exp \left( -  \frac{N_{T}}{2}\Delta_{a}^{2}  \right) \\
    &\le |\set{\THRESHOLD}| e^{\min_{a \in \set{\THRESHOLD}} \Delta_{a}^2/2}\exp \left( -  \frac{\NBATCHES \min_{a \in \set{\THRESHOLD}} \Delta_{a}^2}{4 \lceil \log_{2}(K)\rceil } \right)\: .
    \end{align*}
    Using Theorem $4.1$ from~\cite{karnin2013almost} for SH yields
    \[
    \marc{\bP_{\nu}(a_{\NBATCHES} \notin a^\star(\mu))} \le 3 \log_2(K) \exp \left( - \frac{T}{8 \log_2 (K) \max_{i > I^\star} i (\max_{a \in \ARMS} \mu_{a} - \mu_{(i)})^{-2}}\right)
    \]
    where $I^\star = |\argmax_{a \in \ARMS} \mu_{a}|$. 
    
    {\bf Improved case 2.} 
    Instead of simply using Theorem $4.1$~\cite{karnin2013almost}, we can use recent results from~\cite{zhao2022revisiting} by noting that 
    \[
    \{ \marc{a_{\NBATCHES}} \in \set{\THRESHOLD}^{\complement} \}  = \bigcup_{\varepsilon \in (\max_{a \in \set{\THRESHOLD}}\Delta_{a} + \min_{b \notin \set{\THRESHOLD} } \Delta_{b}, \max_{a \in \set{\THRESHOLD}}\mu_{a} - \min_{b \in \set{\THRESHOLD}}\mu_{b})}\{ \mu_{ \marc{a_{\NBATCHES}}} < \mu_{a^\star} - \varepsilon \} \: .
    \]
    Then, using Theorem $1$ from~\cite{zhao2022revisiting} and taking the infimum over $\varepsilon$ yields that $\bP_{\nu}(\marc{a_{\NBATCHES}}\in \set{\THRESHOLD}^{\complement} )   \le  \exp \left( - \tilde{\Theta}  \left( \frac{T}{G_{1}(\mu)} \right) \right)$ with 
    \begin{equation} \label{eq:improved_SHG}
        G_{1}(\mu) = \min_{\varepsilon \in (\max_{a \in \set{\THRESHOLD}}\Delta_{a} + \min_{b \notin \set{\THRESHOLD} } \Delta_{b}, \max_{a \in \set{\THRESHOLD}}\mu_{a} - \min_{b \in \set{\THRESHOLD}}\mu_{b}) } \max_{i \ge g(\varepsilon) + 1} \frac{i}{g(\varepsilon/2)(\mu_{a^\star} - \mu_{(i)})^2} \: ,
    \end{equation}
    where $g(\varepsilon) = |\{a \in \ARMS \mid \mu_{a} \ge \mu_{a^\star} - \epsilon \}|$.
\end{proof}

\textit{Doubling SH.}
It is possible to convert the fixed-budget SH-G algorithm into an anytime algorithm by using the doubling trick.
It considers a sequences of algorithms that are run with increasing budgets $(T_{k})_{k \ge 1}$, with $T_{k+1} = 2 T_{k}$ and $T_{1} = 2 K \lceil \log_{2} K \rceil$, and recommend the answer outputted by the last instance that has finished to run.
Theorem $5$ from~\cite{zhao2022revisiting} shows that Doubling SH achieves the same guarantees than SH for any time $t$, where the ``cost'' of doubling is hidden by the $\tilde{\Theta}(\cdot)$ notation.
It is well know that the ``cost'' of doubling is to have a multiplicative factor $4$ in front of the hardness constant.
The first two-factor is due to the fact that we forget half the observations.
The second two-factor is due to the fact that we use the recommendation from the last instance of SH that has finished.
Therefore, Theorem~\ref{thm:SH_PoE_recoElim} can be modified for DSH-G by simply adding this multiplicative factor $4$.

While it might look to be a mild cost, this intervenes inside the exponential hence we need four times as many samples to achieves the same error.
For application where sampling is limited, this price is to high to be paid in practice.
Moreover, since past observations are dropped when reached budget $T_{k}$, doubling-based algorithms are known to have empirical performances that decreases by steps.

\subsection{Successive Reject for GAI (SR-G)}
\label{app:ssec_SR_PoE}

In Appendix~\ref{app:ssec_SR_PoE}, we study the SR~\citep{audibert2010best} algorithm where instead of recommending the last active arm $a_{\NBATCHES}$, we 
use the recommendation~Eq.~\eqref{eq:reco_rule_elimination_algos}. 
We refer to this modified SR algorithm as SR-G.
In SR, there is only one arm $a_{\NBATCHES}$ at time $\NBATCHES$ since we eliminated all but one arm after $K-1$ phases.
Let us denote by $n_{k} = \left\lceil \frac{\NBATCHES-K}{\overline{\log}(K) (K+1-k)} \right\rceil $ and $ u_{\NBATCHES} = \sum_{k=1}^{K-1} n_{k}$, where $\overline{\log}(K) = \frac{1}{2} + \sum_{i=2}^{K}\frac{1}{i}$.
Therefore, we have $N_{a_{\NBATCHES}}(\NBATCHES) = \NBATCHES - u_{\NBATCHES}$.

Theorem~\ref{thm:SR_PoE_recoElim} shows that the exponential decrease of the probability of error of SR-G is linear as a function of time.
\begin{theorem} \label{thm:SR_PoE_recoElim}
    Let $\NBATCHES > K$. Let $\mathfrak{A}_{\NBATCHES}$ be the SR-G algorithm with recommendation rule as in~Eq.~\eqref{eq:reco_rule_elimination_algos}.
    Then, for any $1$-sub-Gaussian distribution $\nu \in \mathcal D^{K}$ with mean $\mu$ such that $\Delta_{\min} > 0$,
    \begin{align*}
        &\text{if }\set{\THRESHOLD} = \emptyset, \quad \perr{\nu}{\mathfrak{A}_{\NBATCHES}}{\NBATCHES} \le K \exp \left( -\frac{\NBATCHES-K}{4\overline{\log}(K) } \min_{a \in \ARMS} \Delta_{a}^2 \right) \: , \\
	&\text{if }\set{\THRESHOLD} \ne \emptyset, \quad \perr{\nu}{\mathfrak{A}_{\NBATCHES}}{\NBATCHES} \le |\set{\THRESHOLD}| \exp \left( -  \frac{\NBATCHES - K}{4 \overline{\log}(K) } \min_{a \in \set{\THRESHOLD}} \Delta_{a}^2 \right) + \\
 & \qquad \qquad \qquad\qquad \qquad \qquad K^2 \exp \left( - \frac{\NBATCHES - K}{\overline{\log} (K) \max_{i > I^\star} i (\max_{a \in \ARMS} \mu_{a} - \mu_{(i)})^{-2}}\right)  \: ,
    \end{align*}
    where $I^\star = |\argmax_{a \in \ARMS} \mu_{a}|$. 
\end{theorem}
\begin{proof}    
    We distinguish between the cases (1) $\set{\THRESHOLD} = \emptyset$ and (2) $\set{\THRESHOLD} \ne \emptyset$.

    {\bf Case 1: $\set{\THRESHOLD} = \emptyset$.} 
    When $\set{\THRESHOLD} = \emptyset$, we have
    \[
    \cE^{\text{err}}_{\mu}(\NBATCHES) = \{\GUESS{\NBATCHES} \ne \emptyset \}   = \{\GUESS{\NBATCHES} \ne \emptyset , \expmean{a_{\NBATCHES}}{\NBATCHES} >  \THRESHOLD \}  \subseteq  \{ \expmean{a_{\NBATCHES}}{\NBATCHES}  > \THRESHOLD \} = \bigcup_{a \in \ARMS} \{ a_{\NBATCHES} = a,\: \expmean{a}{\NBATCHES} >  \THRESHOLD\} \: .
    \] 
    Therefore, using $N_{a_{\NBATCHES}}(\NBATCHES) = \NBATCHES - u_{\NBATCHES}$ and similar argument as in the proof of Theorem~\ref{thm:uniform_sampling_PoE_recoAnytime}, we obtain
    \begin{align*}
        \perr{\nu}{\mathfrak{A}_{\NBATCHES}}{\NBATCHES} \le   \sum_{a \in \ARMS} \exp \left( -\frac{T - u_{T}}{2}\Delta_{a}^{2} \right) \le   K \exp \left( -\frac{\NBATCHES-K}{4\overline{\log}(K) } \min_{a \in \ARMS} \Delta_{a}^2 \right) \: ,
    \end{align*}
    where the last inequality uses that $\NBATCHES - u_{\NBATCHES} \ge \frac{\NBATCHES-K}{2\overline{\log}(K)} $.
        
    {\bf Case 2: $\set{\THRESHOLD} \ne \emptyset$.} 
    When $\set{\THRESHOLD} \ne \emptyset$, we have $\cE^{\text{err}}_{\mu}(\NBATCHES) = \{\GUESS{\NBATCHES} = \emptyset \} \cup \{ \GUESS{\NBATCHES} \in \set{\THRESHOLD}^{\complement} \}$. \marc{By definition of the recommendation rule of SR-G in Eq.~\eqref{eq:reco_rule_elimination_algos}, we obtain}
    \begin{align*}
    &\marc{\{\GUESS{\NBATCHES} = \emptyset \} = \{\GUESS{\NBATCHES} = \emptyset, \: a_{\NBATCHES} \in \set{\THRESHOLD}, \: \expmean{a_{\NBATCHES}}{\NBATCHES} \le \THRESHOLD  \} \cup \{\GUESS{\NBATCHES} = \emptyset, \: a_{\NBATCHES} \in \set{\THRESHOLD}^{\complement} , \: \expmean{a_{\NBATCHES}}{\NBATCHES} \le \THRESHOLD\} } \\ 
    &\qquad \qquad \: \marc{\subseteq \{ a_{\NBATCHES} \in \set{\THRESHOLD}, \: \expmean{a_{\NBATCHES}}{\NBATCHES} \le \THRESHOLD  \} \cup \{ a_{\NBATCHES} \in \set{\THRESHOLD}^{\complement} \} \: ,}  \\ 
    &\marc{ \{ \GUESS{\NBATCHES} \in \set{\THRESHOLD}^{\complement} \} = \{ \GUESS{\NBATCHES} \in \set{\THRESHOLD}^{\complement}, \: a_{\NBATCHES} \in \set{\THRESHOLD}^{\complement} , \: \expmean{a_{\NBATCHES}}{\NBATCHES} > \THRESHOLD \}  \subseteq \{ a_{\NBATCHES} \in \set{\THRESHOLD}^{\complement} \} } \: .
    \end{align*}
    \marc{The dichotomy on whether the last active arm $a_T$ is a good arm or not is crucial when $\GUESS{\NBATCHES} = \emptyset$. When $a_{\NBATCHES} \in \set{\THRESHOLD}$, having $\GUESS{\NBATCHES} = \emptyset$ implies that this arm was not sampled enough to ensure that $\expmean{a_{\NBATCHES}}{\NBATCHES} > \THRESHOLD $, even though it satisfies $\mu_{a_T} > \theta$. Since it is sampled linearly, it means that the budget $T$ is not large enough compared to the difficulty $1/\min_{a \in \set{\THRESHOLD}} \Delta_{a}^2$.
    When $a_{\NBATCHES} \notin \set{\THRESHOLD}$, having $\GUESS{\NBATCHES} = \emptyset$ implies that all the good arms have been eliminated in previous phases.
    Therefore, SR has eliminated the best arm in previous phases, namely we have
    \[
    \{ a_{\NBATCHES} \in \set{\THRESHOLD}^{\complement} \} \subseteq \{ a_{\NBATCHES} \notin a^\star(\mu) \} \quad \text{where} \quad a^\star(\mu) \deff \argmax_{a \in [K]} \mu_{a} \subseteq \set{\THRESHOLD} \: .
    \]
    Using existing analysis of SR, $\{ a_{\NBATCHES} \notin a^\star(\mu) \}$ is known to have a low probability of occuring. 
    Putting everything together, we have shown that
    \[
        \cE^{\text{err}}_{\mu}(\NBATCHES) \subseteq \{ a_{\NBATCHES} \in \set{\THRESHOLD}, \: \expmean{a_{\NBATCHES}}{\NBATCHES} \le \THRESHOLD  \}  \cup  \{  a_{\NBATCHES} \in \set{\THRESHOLD}^{\complement} \} \subseteq \{ a_{\NBATCHES} \in \set{\THRESHOLD}, \: \expmean{a_{\NBATCHES}}{\NBATCHES} \le \THRESHOLD  \}  \cup  \{  a_{\NBATCHES} \notin a^\star(\mu) \} \: .
    \]}
    Since $N_{a_{\NBATCHES}}(\NBATCHES) = \NBATCHES - u_{\NBATCHES} \ge \frac{\NBATCHES-K}{2\overline{\log}(K)}$, using similar argument as above yields that 
    \begin{align*}
    \bP_{\nu}(a_{\NBATCHES} \in \set{\THRESHOLD}, \: \expmean{a_{\NBATCHES}}{\NBATCHES} \le \THRESHOLD) &\le \sum_{a \in  \set{\THRESHOLD}} \exp \left( -  \frac{(\NBATCHES-K)\Delta_{a}^2}{4\overline{\log}(K)} \right) \le |\set{\THRESHOLD}| \exp \left( -  \frac{\NBATCHES - K}{4 \overline{\log}(K)}  \min_{a \in \set{\THRESHOLD}} \Delta_{a}^2 \right) \: .
    \end{align*}
    Using Theorem 2 from~\cite{audibert2010best} for SR yields
    \[
    \bP_{\nu}(\marc{a_{\NBATCHES} \notin a^\star(\mu)}) \le \frac{K(K-1)}{2} \exp \left( - \frac{\NBATCHES - K}{\overline \log (K) \max_{i > I^\star} i (\max_{a \in \ARMS} \mu_{a} - \mu_{(i)})^{-2}}\right) \: .
    \]
    where $I^\star = |\argmax_{a \in \ARMS} \mu_{a}|$. 
    
    {\bf Improved case 2: $\set{\THRESHOLD} \ne \emptyset$.} 
As in the proof of Theorem~\ref{thm:SH_PoE_recoElim}, using $\{\GUESS{\NBATCHES} \in \set{\THRESHOLD}^{\complement}\} \subset \{ \GUESS{\NBATCHES} \ne a^\star\}$ can lead to highly sub-optimal rate on some instances.
Inspired by the recent analysis of SH conducted in~\cite{zhao2022revisiting}, we believe that improved guarantees can also be achieved for SR. 
Namely, it should be able to control $\bP_{\nu}(\mu_{a_{\NBATCHES}} < \max_{a \in \ARMS}\mu_{a} - \varepsilon)$ for any $\epsilon > 0$.
Proving such improved guarantees on SR is beyond the scope of this paper, hence we let this question as open problem.
However, it is possible to get some intuition on the dependency we would get for GAI.

The core argument of the analysis of SR is to say that if we make a mistake at time $T$, then there exists a phase $k$ such that the best arm was eliminated at the end of phase $k$.
This argument can be adapted to GAI.
A necessary condition for the event $\{\GUESS{\NBATCHES} \in \set{\THRESHOLD}^{\complement}\}$ to occurs is that all arms $a \in \ARMS_{\theta}$ are eliminated.
By definition, all arms are eliminated if and only if there exists a set of phases $\{k_{a}\}_{a \in \ARMS_{\theta}}$ such that, any arm $a \in \ARMS_{\theta}$ is eliminated at the end of phase $k_{a}$.
Let $\{k_{a}\}_{a \in \ARMS_{\theta}}$ be a given set of phases and $a \in \ARMS_{\theta}$.
A necessary condition for an arm $a$ to be eliminated at the end of phase $k_{a}$ is that $\expmean{a}{n_{k_a}} \le \max_{b \notin \ARMS_{\theta}} \expmean{b}{n_{k_a}}$.
Since both arms have been sampled $n_{k_a}$ times, using similar arguments as the one in the proof of Theorem~\ref{thm:uniform_sampling_PoE_recoAnytime}, we obtain that
\[
    \bP_{\nu}(\expmean{a}{n_{k_a}} \le \max_{b \notin \ARMS_{\theta}} \expmean{b}{n_{k_a}}) \le \exp \left( -\frac{n_{k_a}}{4}(\Delta_{a} + \min_{b \notin \ARMS_{\theta}} \Delta_{b})^2\right) \: .
\]
Therefore, by union bound and inclusion of event, we have shown that
\[
    \bP_{\nu}(\GUESS{\NBATCHES} \in \set{\THRESHOLD}^{\complement}) \le |\ARMS_{\theta}^{\complement}| \sum_{\{k_{a}\}_{a \in \ARMS_{\theta}}}  \exp \left( - \frac{\NBATCHES-K}{4 \overline{\log}(K) } \max_{a \in \ARMS_{\theta}}\frac{(\Delta_{a} + \min_{b \notin \ARMS_{\theta}} \Delta_{b})^2}{K+1-k_{a}}\right) \: .
\]
where we used that $n_{k} \ge \frac{\NBATCHES-K}{\overline{\log}(K) (K+1-k)}$ and $\bP_{\nu}(\bigcap_{i} A_{i}) \le \min_{i} \bP_{\nu}(A_{i})$.
A simple combinatorial argument yields that there are $\binom{K - 1}{|\ARMS_{\theta}|}$ possibilities to define a set of $|\ARMS_{\theta}|$ phases within the $K-1$ total phases where an arm can be eliminated.
Accounting for the $ |\ARMS_{\theta}| !$ possible re-ordering, we have $|\ARMS_{\theta}| ! \binom{K - 1}{|\ARMS_{\theta}|} =\frac{(K - 1) !}{(K-1 - |\ARMS_{\theta}|) !}$ possible set of phases $\{k_{a}\}_{a \in \ARMS_{\theta}}$ that eliminate all arms in $\ARMS_{\theta}$.
By upper bounding all the above probability by their smallest term, we obtain that
\[
    \bP_{\nu}(\GUESS{\NBATCHES} \in \set{\THRESHOLD}^{\complement}) \le  \frac{(K - 1) !}{(|\ARMS_{\theta}^{\complement}|-1) !} |\ARMS_{\theta}^{\complement}| \exp \left( - \frac{T-K}{4 \overline{\log} (K) G_{2}(\mu)}  \right) 
\]
where $G_{2}(\mu) = \max_{\substack{p :\ARMS_{\theta} \to [K-1] \\ p \text{ injective}}} \min_{a \in \ARMS_{\theta}} \frac{K+1-p(a)}{(\Delta_{a} + \min_{b \notin \ARMS_{\theta}} \Delta_{b})^2}$.
\end{proof}

\textit{Doubling SR.}
Likewise, it is possible to convert the fixed-budget SR-G algorithm into an anytime algorithm by using the doubling trick.
Therefore, Theorem~\ref{thm:SR_PoE_recoElim} can be modified for DSR-G by simply adding the multiplicative factor $4$ in front of each hardness constant.

\subsubsection{Large Deviation Analysis}
\label{app:sssec_SR_PoE_LDP}

A key benefit of Theorem~\ref{thm:SR_PoE_recoElim} is that it holds for any moderate budget $T$.
When one is only interested by the asymptotic error rate $C(\mu)$ of SR-G, as reported in Table~\ref{tab:summary_anytimeGAI}, one can leverage asymptotic results such as the Large Deviation Principle (LDP).
We build on the recent analysis proposed by~\cite{wang2024best} to provide improved asymptotic error rate for SR-G and DSR-G.
Namely, we combine the arguments presented in the proof of their Theorem 2 in Section 3.4 with the proof of Theorem~\ref{thm:SR_PoE_recoElim}.
In both cases, we recover exactly the asymptotic upper bound obtained in Theorem~\ref{thm:SR_PoE_recoElim}.

\begin{theorem} \label{thm:asymptotic_SR_PoE_recoElim}
    Let $\NBATCHES > K$. Let $\mathfrak{A}_{\NBATCHES}$ be the SR-G algorithm with recommendation rule as in~Eq.~\eqref{eq:reco_rule_elimination_algos}.
    Then, for any $1$-sub-Gaussian distribution $\nu \in \mathcal D^{K}$ with mean $\mu$ such that $\Delta_{\min} > 0$,
    \begin{align*}
        &\text{if }\set{\THRESHOLD} = \emptyset, \quad \liminf_{T \to + \infty} \frac{1}{T} \log \frac{1}{\perr{\nu}{\mathfrak{A}_{\NBATCHES}}{\NBATCHES}} \ge \frac{\Delta_{\min}^2}{4\overline{\log}(K)} \: , \\
	&\text{if }\set{\THRESHOLD} \ne \emptyset, \quad \liminf_{T \to + \infty} \frac{1}{T} \log \frac{1}{\perr{\nu}{\mathfrak{A}_{\NBATCHES}}{\NBATCHES}} \ge \frac{1}{4\overline{\log}(K) G_{2}(\mu)} \: .
    \end{align*}
   where 
\begin{equation}  \label{eq:complexity_SRG_asymptotic}
  G_{2}(\mu) = \max_{\substack{p :\ARMS_{\theta} \to [K-1] \\ p \text{ injective}}} \min_{a \in \ARMS_{\theta}} \frac{K+1-p(a)}{(\Delta_{a} + \min_{b \notin \ARMS_{\theta}} \Delta_{b})^2} \: .
\end{equation}
\end{theorem}
\begin{proof}    
    {\bf Case 1: $\set{\THRESHOLD} = \emptyset$.} 
    Since the lower order terms disappear asymptotically, we have
    \begin{align*}
       \liminf_{T \to + \infty} \frac{1}{T} \log \frac{1}{\bP_{\nu} (\GUESS{\NBATCHES} \ne \emptyset )} \ge \min_{a \in \ARMS} \liminf_{T \to + \infty} \frac{1}{T} \log \frac{1}{\bP_{\nu} (a_{T} = a, \expmean{a}{\NBATCHES}  > \THRESHOLD )}  \: .
    \end{align*}
    Recall that $N_{a_{\NBATCHES}}(\NBATCHES) = \NBATCHES - u_{\NBATCHES} \ge \frac{\NBATCHES-K}{2\overline{\log}(K)} $.
    Let $\epsilon > 0$ and $T_{\epsilon}$ such that $1 - u_{\NBATCHES} / \NBATCHES \ge \frac{1-\epsilon}{2\overline{\log}(K)}$ for all $T \ge T_{\epsilon}$.
    Let $T \ge T_{\epsilon}$. 
    The event $\{a_{T} = a, \expmean{a}{\NBATCHES}  > \THRESHOLD \}$ implies that $\{\hat \mu(\NBATCHES) \in \cS_{a}, N(\NBATCHES)/\NBATCHES \in \mathcal W_{a} \}$ where $\cS_{a} = \{\lambda \in \R^{K} \mid \lambda_{a} > \THRESHOLD \}$ and $\mathcal W_{a} = \{w \in \triangle_{K} \mid w_{a} \ge  \frac{1-\epsilon}{2\overline{\log}(K)} \}$.
    Applying the useful corollary (c) of Theorem 1 in~\cite{wang2024best} yields that
    \begin{align*}
        \liminf_{T \to + \infty} \frac{1}{T} \log \frac{1}{\bP_{\nu} (\hat \mu(\NBATCHES) \in \cS_{a}, N(\NBATCHES)/\NBATCHES \in \mathcal W_{a})} \ge \inf_{w \in \mathcal W_{a}} \inf_{\lambda \in \text{cl}(\cS_{a})}\Psi(\lambda, w) = \frac{1-\epsilon}{4\overline{\log}(K)} \Delta_{a}^2 \: ,
    \end{align*}
    where the last equality is obtained by direct computation since $\Psi(\lambda, w) = \sum_{a \in \ARMS} w_{a} (\mu_a-\lambda_a)^2/2$ and $\mu_{a} < \THRESHOLD$.
    Combining the above inequalities and taking the limit when $\epsilon \to 0$, we conclude that $\liminf_{T \to + \infty} \frac{1}{T} \log \frac{1}{\bP_{\nu} (\GUESS{\NBATCHES} \ne \emptyset )} \ge \frac{\Delta_{\min}^2}{4\overline{\log}(K)}$.
    
    {\bf Case 2: $\set{\THRESHOLD} \ne \emptyset$.} 
    We re-use the arguments from the ``Improved case 2'' paragraph of the proof of Theorem~\ref{thm:SR_PoE_recoElim}.
    Let $\mathcal C_{j}$ be the set active arms at phase $j$ and $\ell_{j}$ be the empirical worst arm at the end of phase $j$, \ie $\mathcal C_{j+1} = \mathcal C_{j} \setminus \{j\}$.
    Similarly, we obtain that
    \begin{align*}
       \liminf_{T \to + \infty} \frac{1}{T} \log \frac{1}{\bP_{\nu} (\GUESS{\NBATCHES} \notin \set{\THRESHOLD} )} \ge \min_{\substack{p :\ARMS_{\theta} \to [K-1] \\ p \text{ injective}}} \liminf_{T \to + \infty} \frac{1}{T} \log \frac{1}{\bP_{\nu} (\forall a \in \ARMS_{\theta}, \: \ell_{p(a)} = a , \mathcal C_{p(a)} \setminus \ARMS_{\theta} \ne \emptyset)}  
    \end{align*}
    where we used that there is a finite number of such injective mapping to swap the limit and the sum.
    Moreover, we have $\bP_{\nu} (\forall a \in \ARMS_{\theta}, \: \ell_{p(a)} = a , \mathcal C_{p(a)} \setminus \ARMS_{\theta} \ne \emptyset) \le \bP_{\nu} ( \ell_{p(a)} = a , \mathcal C_{p(a)} \setminus \ARMS_{\theta} \ne \emptyset)$ for all $a \in \ARMS_{\theta}$.
    Let $\mathcal J_{a} = \{(G,B) \subseteq \ARMS_{\theta} \times \ARMS_{\theta}^{\complement}  \mid a \in G, B \ne \emptyset, |G \cup B| = K-p(a)+1\}$.
    By union bound, we obtain that
    \begin{align*}
        &\liminf_{T \to + \infty} \frac{1}{T} \log \frac{1}{\bP_{\nu} (\ell_{p(a)} = a , \mathcal C_{p(a)} \setminus \ARMS_{\theta} \ne \emptyset)} \\
        &\quad \ge \min_{(G,B) \in \mathcal J_a}\liminf_{T \to + \infty} \frac{1}{T} \log \frac{1}{\bP_{\nu} (\ell_{p(a)} = a , \mathcal C_{p(a)} =  G \cup B)} \: ,
    \end{align*}
    where we used that $|\mathcal J_{a}| < + \infty$ to swap the limit and the sum.
    Recall that the active arms have been sampled $n_{k}$ times at the end of phase $k$, where $n_{k} \ge \frac{\NBATCHES-K}{\overline{\log}(K) (K+1-k)}$.
    For all $k \in [K-1]$, let $\alpha_{k} > 0$ such that the end of phase $k$ corresponds to a time $\alpha_{k} T$ (assumed to be integer for simplicity).
    Let $\epsilon > 0$ and $T_{\epsilon}$ such that $n_{k} / (\alpha_{k} \NBATCHES) \ge \frac{1-\epsilon}{\alpha_{k} \overline{\log}(K) (K+1-k)}$ for all $T \ge T_{\epsilon}$ and all $k \in [K-1]$.
    Let $T \ge T_{\epsilon}$. 
    The event $\{\ell_{p(a)} = a , \mathcal C_{p(a)} \setminus \ARMS_{\theta} \ne \emptyset \}$ implies that $\{\hat \mu(\alpha_{p(a)} \NBATCHES) \in \cS_{a}, N(\alpha_{p(a)}\NBATCHES)/(\alpha_{p(a)} \NBATCHES) \in \mathcal W_{a} \}$ where $\cS_{a} = \{\lambda \in \R^{K} \mid \lambda_{a} \le \min_{b \in B} \lambda_{b}\}$ and $\mathcal W_{a} = \{w \in \triangle_{K} \mid \forall b \in B \cup \{a\}, w_{b}  \ge  \frac{1-\epsilon}{\alpha_{p(a)}\overline{\log}(K) (K+1-p(a))} \}$.    
    Applying the useful corollary (c) of Theorem 1 in~\cite{wang2024best} yields that
    \begin{align*}
        &\liminf_{T \to + \infty} \frac{1}{\alpha_{p(a)} T} \log \frac{1}{\bP_{\nu} (\hat \mu(\alpha_{p(a)} T) \in \cS_{a}, N(\alpha_{p(a)} T)/(\alpha_{p(a)} T) \in \mathcal W_{a})} \ge \inf_{w \in \mathcal W_{a}} \inf_{\lambda \in \text{cl}(\cS_{a})}\Psi(\lambda, w) \\
        &\qquad \qquad= \frac{1-\epsilon}{2\alpha_{p(a)}\overline{\log}(K) (K+1-p(a))} \inf \left\{ \sum_{b \in B \cup \{a\}}(\mu_{b} - \lambda_{b})^2 \mid \lambda \in \cS_{a} \right\} \\
        &\qquad\qquad \ge \frac{1-\epsilon}{4\alpha_{p(a)}\overline{\log}(K) (K+1-p(a))} \min_{b \in B} (\mu_{a} - \mu_{b} )^2 \: ,
    \end{align*}
    where we solved explicitly the infimum after using that $\sum_{c \in B \cup \{a\}}(\mu_{c} - \lambda_{c})^2 \ge \sum_{c \in \{a,b\}}(\mu_{c} - \lambda_{c})^2 $ for all $b \in B$.
    Combining the above inequalities and taking the limit when $\epsilon \to 0$, we conclude that
    \begin{align*}
        \liminf_{T \to + \infty} \frac{1}{T} \log \frac{1}{\bP_{\nu} (\GUESS{\NBATCHES} \notin \set{\THRESHOLD} )} \ge \min_{\substack{p :\ARMS_{\theta} \to [K-1] \\ p \text{ injective}}} \max_{a \in \set{\THRESHOLD} } \min_{b \notin \set{\THRESHOLD} }  \frac{(\mu_{a} - \mu_{b} )^2}{4\overline{\log}(K) (K+1-p(a))}  = \frac{1}{4\overline{\log}(K) G_{2}(\mu)} 
    \end{align*}
\end{proof}

\section{Prior Knowledge-based GAI Algorithm (PKGAI)}
\label{sec:stickyalgo}

In this section, we describe a meta-algorithm for fixed-budget GAI called~\hyperlink{PKGAI}{PKGAI} (\textbf{P}rior \textbf{K}nowledge-based GAI, shown in Algorithm~\ref{algo:StickyAlgo}). This meta-algorithm can be used to convert fixed-confidence GAI algorithms from prior works. As previously mentioned, the sampling rule in this algorithm depends on an index policy $(i_a(t))_{a \in \ARMS, t \leq \NBATCHES}$. We provide guarantees on the error probability for both the partially specified algorithm (without a specific index policy, Theorem~\ref{th:APTlike_error}) and the uniform round-robin version (Theorem~\ref{th:stickyalgo_error}).

\subsection{A Meta-algorithm for Fixed-budget GAI}

\begin{algorithm}[t]
    \centering    
	\caption{\protect\hypertarget{PKGAI}{PKGAI} (\textbf{P}rior \textbf{K}nowledge-based GAI)}
        \label{algo:StickyAlgo}
\begin{algorithmic}[1]
   \STATE {\bfseries Input:} budget $\NBATCHES \geq \NARMS$, threshold $\THRESHOLD$\STATE {\bfseries Define:} for all $a \in \ARMS$, confidence intervals $([\widehat{\Delta}^-_a(t),\widehat{\Delta}^+_a(t)])_{t \le \NBATCHES}$ on $\mu_a-\THRESHOLD$\STATE {\bfseries Define:} for all $a \in \ARMS$ and $t \leq \NBATCHES$, sampling index $i_a(t) : \ARMS \times \nN \rightarrow \rR$.\\ Possible index policies:
   \begin{eqnarray*}
       & \text{PKGAI(APT$_P$) : } & i_a(t) \deff \sqrt{\nsamples{a}{t}}(\expmean{a}{t}-\THRESHOLD)\:,\\
       & \text{PKGAI(UCB) : } & i_a(t) \deff \widehat{\Delta}^+_a(t)\:,\\
       & \text{PKGAI(Unif) : } & i_a(t) \deff -\nsamples{a}{t}\:, \\
       & \text{PKGAI(LCB-G) : } & i_a(t) \deff \sqrt{\nsamples{a}{t}}\widehat{\Delta}^-_a(t)\:.
   \end{eqnarray*}\STATE Sample each arm $a \in \ARMS$ once
   \STATE Set $t \gets K$, $\cS_t \gets \ARMS$, $\nsamples{a}{t} \gets 1$ and initialize $\widehat{\Delta}^-_a(t), \widehat{\Delta}^+_a(t)$ for $a \in \ARMS$
\WHILE{$t<\NBATCHES$ \textbf{ and } $|\cS_t|>0$}
        \STATE $\arm{t+1} \in \argmax_{a \in \cS_{t}} i_a(t)$ 
        \STATE Draw arm $\arm{t+1}$ and observe $\reward{\arm{t+1}}{t+1}$
        \STATE Update $\empdeltam{a}{t+1}$, $\empdeltap{a}{t+1}$ for all $a \in \ARMS$
\STATE $\cS_{t+1} \gets \cS_{t} \setminus \{ a \in \cS_{t} \mid \empdeltap{a}{t+1} < 0 \}$ 
        \STATE $t \gets t+1$
   \ENDWHILE
   \STATE \textbf{end}
   \IF{$|\cS_t|=0$ \textbf{ or } $\max_{a \in \cS_{\NBATCHES}} \empdeltam{a}{\NBATCHES}+\empdeltap{a}{\NBATCHES} \leq 0$}
        \STATE {\bfseries return} $\GUESS{\NBATCHES} \deff \emptyset$ \ELSE
        \STATE {\bfseries return} $\GUESS{\NBATCHES}\in \argmax_{a \in \cS_{\NBATCHES}} \empdeltam{a}{\NBATCHES}$\ENDIF
\end{algorithmic}
\end{algorithm}

The meta-algorithm~\hyperlink{PKGAI}{PKGAI}---where the sampling index is unspecified---is shown in Algorithm~\ref{algo:StickyAlgo}. Similarly to fixed-confidence GAI algorithms proposed in the literature~\citep{kano2019good,tabata2020bad}, it relies on confidence bounds $([\empdeltam{a}{t},\empdeltap{a}{t}])_{t \leq \NBATCHES}$ on gap $\mean{a}-\THRESHOLD$ for any arm $a$ and phased elimination (Line L.$11$) on the corresponding $\sigma$-sub-Gaussian distribution (in our paper, $\sigma=1$) $\left[\empdeltam{a}{t},\empdeltap{a}{t}\right] \deff \left\{\expmean{a}{t}-\THRESHOLD \pm \sigma\sqrt{\beta(t)/\nsamples{a}{t}} \right\}$, where $\beta$ is a well-chosen threshold function, which is increasing in its argument.

Intuitively, $\empdeltam{a}{t}$ (resp. $\empdeltap{a}{t}$) represents an lower (resp. upper) bound on the amount of information towards decision $\left\{ a \in \set{\THRESHOLD} \right\}$. In the elimination step, all unsuitable candidates are removed at the end of the sampling round; that is, arms which corresponding upper confidence bound is below $0$. We assume in the remainder of the section that the sampling budget $\NBATCHES$ is at least equal to $\NARMS$.

\textit{Recommendation rule.} 
This algorithm enables early stopping, as if there is no suitable candidate left (\ie $\mathcal{S}_t = \emptyset$), then~\hyperlink{PKGAI}{PKGAI} returns the empty set (Line L.$13$). If there is no suitable candidate $a$ such that $\empdeltam{a}{\NBATCHES}+\empdeltap{a}{\NBATCHES} > 0$, it also returns the empty set---when considering symmetrical confidence intervals, it is equivalent to testing whether $\expmean{a}{t}>\THRESHOLD$ (L.$13$). Otherwise, it returns one of the arms maximizing the lower confidence bound (L.$16$).

\textit{Sampling rule.}
As initialization, each arm $a \in \ARMS$ is pulled once.
\hyperlink{PKGAI}{PKGAI} combines upper/lower confidence bounds-based sampling~\citep{kano2019good,kaufmann2018sequential}, and exploitation-oriented approaches~\citep{locatelli2016optimal,tabata2020bad}. Several sampling rules, some inspired by prior fixed-confidence algorithms, are described in Algorithm~\ref{algo:StickyAlgo}. We also propose another exploration algorithm, named LCB-G, which targets the lower confidence bound. We denote \hyperlink{PKGAI}{PKGAI}(*) the meta-algorithm where the sampling rule remains undefined.

\textit{Comparison with prior works.} 
Note that, contrary to~\hyperlink{APGAI}{APGAI}, this algorithm requires the knowledge of instance-dependent quantities to define the confidence bounds, and of $\NBATCHES$, thus not permitting continuation. This meta-algorithm is related to algorithms proposed in fixed-confidence variants of the GAI problem (\eg BAEC~\citep{tabata2020bad} for PKGAI(APT$_P$), HDoC and LUCB-G~\citep{kano2019good} for PKGAI(UCB)), albeit not entirely similar. To adapt to the fixed-budget constraint, Lines L.$14$ and L.$16$ are introduced, corresponding to cases where the allocated budget is probably too small to assess with certainty whether $\set{\THRESHOLD}=\emptyset$.

\subsection{Fixed-budget Guarantees for PKGAI}

Theorem~\ref{th:APTlike_error} shows that for any sampling index (at Line L.$7$) and if we have access to $H_1(\mu)$ and $H_{\theta}(\mu)$---which is quite a strong assumption in practice---using the structure as in~\hyperlink{PKGAI}{PKGAI} ensures that the error probability is upper bounded by roughly $\exp(-\NBATCHES/H_1(\mu))$ in all cases, which matches optimality when $\set{\THRESHOLD}=\emptyset$. 

\begin{theorem}[Proof in Section~\ref{app:APTlike_error}]\label{th:APTlike_error}
Let $T> K$ and consider any $1$-sub-Gaussian distribution with mean $\mu \in \rR^\NARMS$ such that $\mu_{a} \ne \THRESHOLD$ for all $a \in \ARMS$. If confidence intervals $[\empdeltam{a}{t},\empdeltap{a}{t}]$ for all arm $a \in \ARMS$ and $t \leq \NBATCHES$ are such that 
\begin{align}\label{eq:StickyAlgo_bestCI}
\bP_{\nu}(\bigcup_{a \in \ARMS,t \leq \NBATCHES} \{ |\expmean{a}{t}-\mean{a}| \leq \sqrt{\beta(t)\nsamples{a}{t}} \} ) \in (0,1) \;, \text{ with } \: \beta(\NBATCHES) \leq \frac{\NBATCHES-\NARMS}{4H_1(\mu)}\;.
\end{align}
Then, we have $\perr{\nu}{\text{PKGAI(*)}}{\NBATCHES} \leq 2\NARMS \NBATCHES e^{-2\beta(\NBATCHES)}$. This is minimized when Inequality~\eqref{eq:StickyAlgo_bestCI} is an equality, hence
\[
    \perr{\nu}{\text{PKGAI(*)}}{\NBATCHES} \leq 2\NARMS \NBATCHES \exp\left(-\frac{\NBATCHES-\NARMS}{2H_1(\mu)}\right) \: .
\]
\end{theorem}

Furthermore, when considering an uniform round-robin sampling, \ie PKGAI(Unif) (in Line L.$7$, Algorithm~\ref{algo:StickyAlgo}) $i_a(t) \deff -\nsamples{a}{t}$ for all $a \in \ARMS$ and $t \le \NBATCHES$, the error probability is upper bounded by a term of order $ \exp(-\NBATCHES/H_1(\mu))$ when $\set{\THRESHOLD}=\emptyset$ or $\GUESS{\NBATCHES}=\emptyset$, and of order $\exp(-\NBATCHES/(\NARMS \hat{\Delta}^{-2}))$ otherwise, where $\hat{\Delta} \deff \max_{a \in \set{\THRESHOLD}} \Delta_a + \min_{a \not\in \set{\THRESHOLD}} \Delta_a$ (Theorem~\ref{th:stickyalgo_error}).

\begin{theorem}[Proof in Section~\ref{app:stickyalgo_error}]\label{th:stickyalgo_error}
Let $\NBATCHES > \NARMS$ and consider any $1$-sub-Gaussian distribution with mean $\mu \in \rR^\NARMS$ such that $\mean{a} \neq \THRESHOLD$ for all $a\in\ARMS$. Let $\beta(\NBATCHES)$ satisfying \begin{align}\label{eq:UniformAlgo_bestCI}
\beta(\NBATCHES) & \leq \begin{cases}(\NBATCHES-\NARMS)/(4\NARMS \hat{\Delta}^{-2}) & \text{ if } \set{\THRESHOLD}(\mu) \neq \emptyset\\ (\NBATCHES-\NARMS)/(4H_1(\mu))  & \text{otherwise} \end{cases}\;.
\end{align}
where $\hat{\Delta} \deff \max_{a \in \set{\THRESHOLD}} \Delta_a + \min_{a \not\in \set{\THRESHOLD}} \Delta_a\;.$
Then 
$\perr{\nu}{\text{PKGAI[Unif]}}{\NBATCHES} \leq 2\NARMS \NBATCHES e^{-2\beta(\NBATCHES)}$. This is minimized when Inequality~\eqref{eq:UniformAlgo_bestCI} is an equality, hence
\[\perr{\nu}{\text{PKGAI[Unif]}}{\NBATCHES} \leq \begin{cases} 2 \NARMS \NBATCHES \exp\left(-\frac{\NBATCHES-\NARMS}{2\NARMS \hat{\Delta}^{-2}}\right) & \text{ if }\set{\THRESHOLD}\neq \emptyset\;,\\ 2 \NARMS \NBATCHES \exp\left(-\frac{\NBATCHES-\NARMS}{2H_1(\mu)}\right) & \text{ otherwise}\;.\end{cases}\] 
\end{theorem}

This theorem yields a strictly better bound than~\hyperlink{APGAI}{APGAI} and Theorem~\ref{th:APTlike_error} for instances such that $\set{\THRESHOLD}\neq\emptyset$ and 
\[ \NARMS \hat{\Delta}^{-2} = \NARMS \left(\max_{a \in \set{\THRESHOLD}} \Delta_a + \min_{b \not\in \set{\THRESHOLD}} \Delta_b\right)^{-2} < H_1(\mu) := \sum_{a \in \ARMS} \Delta_a^{-2}\;,\]
\eg in all but one instances among those we have considered (see Table~\ref{tab:instances_constants_value}).

\subsection{Proof Sketch}
\label{app:analysis_StickyAlgo}

The idea behind the proofs of Theorems~\ref{th:APTlike_error} and~\ref{th:stickyalgo_error} is to consider each recommendation case, and to determine a value of $\beta(\NBATCHES)$ which prevents an error in  \protect\hyperlink{PKGAI}{PKGAI} when confidence intervals hold. As a consequence, 
\[ \perr{\nu}{\text{PKGAI}(*)}{\NBATCHES} \leq \bP_{\nu}(\EventStickyAlgo^\complement) \text{ where } \EventStickyAlgo \deff \bigcap_{\substack{a \in \ARMS\\ t \leq \NBATCHES}} \left\{|\expmean{a}{t}-\mean{a}| \leq \sqrt{\frac{\beta(t)}{\nsamples{a}{t}}} \right\}\;.\]
Let us denote the last round in~\hyperlink{PKGAI}{PKGAI}, for any sampling index $\tau := \NBATCHES \land \inf_{t\leq \NBATCHES} \left\{ |\cS_t|=0  \right\}\;,$
\ie the number of samples after which the recommendation rule is applied. 
The probability of error of any algorithm $\mathfrak{A}$ with the same structure as ~\hyperlink{PKGAI}{PKGAI} can be decomposed as follows by union bound\begin{align*}
\perr{\nu}{\mathfrak{A}}{\NBATCHES} & \leq \bP\left[ \left( \set{\THRESHOLD} \neq \emptyset \cap \left( \GUESS{\tau} \in \{\emptyset\} \cup \ARMS \setminus \set{\THRESHOLD} \right) \cap \EventStickyAlgo \right) \cup  \left(\set{\THRESHOLD} = \emptyset \cap \GUESS{\tau}\neq\emptyset \cap \EventStickyAlgo \right) \right]+\bP_\nu( \EventStickyAlgo^\complement )\;, \nonumber \\
 & \leq \underbrace{\bP\left[\set{\THRESHOLD} \neq \emptyset \cap \left( \GUESS{\tau} \in \{\emptyset\} \cup \ARMS \setminus \set{\THRESHOLD} \right) \cap \EventStickyAlgo \right]}_\text{Case $1$}+\underbrace{\bP\left[\set{\THRESHOLD} = \emptyset \cap \GUESS{\tau}\neq\emptyset \cap \EventStickyAlgo \right]}_\text{Case $2$}+\bP_\nu( \EventStickyAlgo^\complement )\;.  \label{eq:probError_Explore}
\end{align*}
For both Theorems~\ref{th:APTlike_error} and~\ref{th:stickyalgo_error}, we will then proceed by considering two cases, $\set{\THRESHOLD}=\emptyset$ and $\set{\THRESHOLD}\neq \emptyset$, assuming that $\EventStickyAlgo$ holds. In both cases, the goal is to determine the form of appropriate confidence intervals which prevent an error in~\hyperlink{PKGAI}{PKGAI} when $\EventStickyAlgo$ holds (by proving a contradiction), such that ultimately, $\perr{\nu}{\text{PKGAI}}{\NBATCHES} \leq \bP_\nu( \EventStickyAlgo^\complement )$.

\subsection{Proof of Theorem~\ref{th:APTlike_error}}\label{app:APTlike_error}

\subsubsection{Case \texorpdfstring{$\set{\THRESHOLD}(\mu)=\set{\THRESHOLD}=\emptyset$}{}}

\begin{proof}
\modif{Let $\nu \in \mathcal{D}^\NARMS$ be any instance of mean vector $\mu$ such that $\set{\THRESHOLD}(\mu)=\emptyset$. }
Let us denote $\EventStickyAlgo^\text{Case $1$} := \{\EventStickyAlgo \cap \set{\THRESHOLD} = \emptyset\}$. 
The error probability $\bP\left[\EventStickyAlgo^\text{Case $1$} \cap \GUESS{\tau}\neq\emptyset\right]$ is lesser than
\[ \bP\left[ \EventStickyAlgo^\text{Case $1$} \cap \exists a \in \ARMS, \  \empdeltap{a}{\tau}+\empdeltam{a}{\tau} \geq 0 \right] \text{ (Line L.$13$)}\;. \]

Since $\cS_\tau \neq \emptyset$ (otherwise, $\GUESS{\NBATCHES}=\emptyset$), then necessarily $\tau = \NBATCHES$. Here, the contradiction will involve the number of samples drawn from each arm during the sampling phase. 
For any arm $b \in \cS_{\NBATCHES} \subseteq \set{\THRESHOLD}^\complement$, on $\EventStickyAlgo$
\begin{align}\label{eq:StickyNotElim_BadArms}
\empdeltap{b}{\NBATCHES} \geq 0 \implies -\Delta_b+2\sqrt{\frac{\beta(\NBATCHES)}{\nsamples{b}{\NBATCHES}}} \geq 0 \implies \nsamples{b}{\NBATCHES} \leq \frac{4\beta(\NBATCHES)}{\Delta_b^2} < \frac{4\beta(\NBATCHES)}{\Delta_b^2}+1\;.
\end{align}
Moreover, for any arm $c \in \cS_\NBATCHES^\complement \subseteq \set{\THRESHOLD}^\complement$, it means that $c$ has been eliminated after exactly $\NARMS+1 \leq t_c \leq \NBATCHES$ rounds, and is no longer sampled after round $t_c$ (\ie $\nsamples{c}{\NBATCHES}=\nsamples{c}{t_c}$). By a reasoning similar to the one that led to Inequality~\eqref{eq:StickyNotElim_BadArms} on round $t_c-1$, 
\begin{align}\label{eq:StickyElim_BadArms}
\empdeltap{c}{t_c-1} \geq 0 > \empdeltap{c}{t_c} &\implies \nsamples{c}{\NBATCHES}-1 = \nsamples{c}{t_c-1} \leq \frac{4\beta(t_c-1)}{\Delta_c^2} \leq \frac{4\beta(\NBATCHES)}{\Delta_c^2} \nonumber \\
&\implies \nsamples{c}{\NBATCHES} \leq \frac{4\beta(\NBATCHES)}{\Delta_c^2}+1\;.
\end{align}
~\ref{eq:StickyNotElim_BadArms} and~\ref{eq:StickyElim_BadArms}, since $\cS_\NBATCHES \neq \emptyset$, 
$\NBATCHES = \sum_{k \in \ARMS} \nsamples{k}{\NBATCHES} < \sum_{a \in \ARMS} \left( \frac{4\beta(\NBATCHES)}{\Delta_a^2} +1\right) \leq 4H_1(\mu)\beta(\NBATCHES)+\NARMS\;.$

That is, any choice of $\beta$ such that $\beta(\NBATCHES) \leq (\NBATCHES-\NARMS)/(4H_1(\mu))$ automatically yields a contradiction. 
Then $\bP\left[ \EventStickyAlgo^\text{Case $1$} \cap \exists a \in \ARMS, \  \empdeltap{a}{\tau}+\empdeltam{a}{\tau} \geq 0 \right]=0$.
\end{proof}

\subsubsection{Case \texorpdfstring{$\set{\THRESHOLD}(\mu)=\set{\THRESHOLD}\neq\emptyset$}{}}

\begin{proof}
Now, we consider any instance $\nu \in \mathcal{D}^\NARMS$ of mean vector $\mu$ such that $\set{\THRESHOLD}(\mu)=\emptyset$. 
Let us denote $\EventStickyAlgo^\text{Case $2$} := \EventStickyAlgo \cap (\set{\THRESHOLD} \neq \emptyset)$. The error probability of~\hyperlink{PKGAI}{PKGAI} when $\set{\THRESHOLD}\neq\emptyset$ can be decomposed as follows
\[ \bP\left[\EventStickyAlgo^\text{Case $2$} \cap (\GUESS{\tau} \in \{\emptyset\} \cup \ARMS \setminus \set{\THRESHOLD})\right] = \underbrace{\bP\left[\EventStickyAlgo^\text{Case $2$} \cap \GUESS{\tau} = \emptyset\right]}_\text{Case $2.1$ (L.$14$ in Algorithm~\ref{algo:StickyAlgo})} + \underbrace{\bP\left[\EventStickyAlgo^\text{Case $2$} \cap \GUESS{\tau} \in \ARMS \setminus \set{\THRESHOLD}\right]}_\text{Case $2.2$ (L.$16$)}\;.\]

\textbf{Case $2.1$.} Necessarily, either $\cS_\tau = \emptyset$ or $\max_{a \in \cS_\tau} \empdeltam{a}{\tau}+\empdeltap{a}{\tau} \leq 0$ (L.$13$). 

$\bullet$ If $\cS_\tau = \emptyset$, then it means in particular that for any good arm $a \in \set{\THRESHOLD}$, if $\EventStickyAlgo$ holds, then 
\[ \exists t_a < \tau, \ \empdeltap{a}{t_a} < 0 \implies (\mean{a}-\THRESHOLD) = \Delta_a < 0\;,\]
which contradicts $a \in \set{\THRESHOLD}$. Then, good arms cannot be eliminated at any round on event $\EventStickyAlgo$, that is, $\bP\left[\EventStickyAlgo^\text{Case $2$} \cap \GUESS{\NBATCHES}=\emptyset \cap \cS_\tau \neq \emptyset \right] = 0$.

$\bullet$ In that case, $\tau=\NBATCHES$. If $\max_{a \in \cS_\NBATCHES} \empdeltam{a}{\NBATCHES}+\empdeltap{a}{\NBATCHES} \leq 0$ on event $\EventStickyAlgo$, then since $\mathcal{E}_\NBATCHES$ holds, for all $a \in \set{\THRESHOLD} \subseteq \cS_\NBATCHES$
\begin{equation}\label{eq:good_arm_ineq}
 2\left(\Delta_a - \sqrt{\frac{\beta(\NBATCHES)}{\nsamples{a}{\NBATCHES}}}\right) \leq \empdeltam{a}{\NBATCHES}+\empdeltap{a}{\NBATCHES} \leq 0 \implies \nsamples{a}{\NBATCHES} \leq \frac{\beta(\NBATCHES)}{\Delta_a^2} < \frac{\beta(\NBATCHES)}{\Delta_a^2}+1\:.
\end{equation}
Furthermore, as a direct consequence of Inequalities~\ref{eq:StickyNotElim_BadArms} and~\ref{eq:StickyElim_BadArms}, for any $b \not\in \set{\THRESHOLD}$, $\nsamples{b}{\NBATCHES} \leq \frac{4\beta(\NBATCHES)}{\Delta_b^2}+1$. From these upper bounds on the number of samples drawn from each arm, we can again build a contradiction\[\NBATCHES = \sum_{a \in \ARMS} \nsamples{a}{\NBATCHES} < \beta(\NBATCHES)\left(H_{\theta}(\mu) + 4(H_1(\mu)-H_{\theta}(\mu))\right)+\NARMS  = \beta(\NBATCHES)\left(4H_1(\mu)-3H_{\theta}(\mu)\right)+\NARMS\:.\]

That is, any choice of $\beta$ such that $\beta(\NBATCHES) \leq \frac{1}{4}(\NBATCHES-\NARMS)/(H_1(\mu)-\frac{3}{4}H_{\theta}(\mu))$
automatically yields a contradiction. 
In that case, $\bP\left[\EventStickyAlgo^\text{Case $2$} \cap \GUESS{\tau}=\emptyset\right] =0$.

\textbf{Case $2.2$.} 
The only remaining case is when $\tau=\NBATCHES$ (Line L.$16$). On event $\EventStickyAlgo$, since $\set{\THRESHOLD} \subseteq \cS_{\NBATCHES}$, for all $a \in \set{\THRESHOLD}$,
    \[ 0 >_{\GUESS{\NBATCHES} \not\in \set{\THRESHOLD}} -\Delta_{\GUESS{\NBATCHES}} \geq \empdeltam{\GUESS{\NBATCHES}}{\NBATCHES} \geq_\text{Line L.$16$} \empdeltam{a}{\NBATCHES} \geq \Delta_a - 2\sqrt{\frac{\beta(\NBATCHES)}{\nsamples{a}{\NBATCHES}}} \implies \nsamples{a}{\NBATCHES} < 4\beta(\NBATCHES)\Delta_a^{-2}\;.\]
    Furthermore, as Inequalities~\ref{eq:StickyNotElim_BadArms} and~\ref{eq:StickyElim_BadArms} hold, for any $b \not\in \set{\THRESHOLD}$, $\nsamples{b}{\NBATCHES} \leq 4\beta(\NBATCHES)\Delta_b^{-2}+1$. All in all, $\NBATCHES < 4H_1(\mu)\beta(\NBATCHES) + \NARMS - |\set{\THRESHOLD}|$.
That is, any choice of $\beta$ such that $\beta(\NBATCHES) \leq \frac{T-\NARMS+|\set{\THRESHOLD}|}{4H_1(\mu)}$
    automatically yields a contradiction.
In that case, \modif{$\bP\left[ \EventStickyAlgo^\text{Case $2$} \cap \GUESS{\NBATCHES} \in \ARMS \setminus \set{\THRESHOLD} \right]=0$.}
\end{proof}

\subsubsection{Final Step}

Combining all previous cases, it suffices to consider $\beta$ such that $\beta(\NBATCHES) \leq \frac{\NBATCHES-\NARMS}{4H_1(\mu)}\;,$ to obtain the following upper bound on the error probability from Inequality~\eqref{eq:good_arm_ineq}, using successively the Hoeffding concentration bounds and union bounds over $\ARMS$ of size $\NARMS$ and over $\{1,2,\dots,\NBATCHES\}$, $\perr{\nu}{\text{PKGAI(*)}}{\NBATCHES} \leq 2\NARMS \NBATCHES \exp\left(-2\beta(\NBATCHES)\right)$.
In particular, the right-hand term is minimized for $\beta(\NBATCHES) = \frac{\NBATCHES-\NARMS}{4H_1(\mu)}$, and in that case $\perr{\nu}{\text{PKGAI(*)}}{\NBATCHES} \leq 2\NARMS \NBATCHES \exp\left(-\frac{\NBATCHES-\NARMS}{2H_1(\mu)}\right)$.
$\qed$

\subsection{Proof of Theorem~\ref{th:stickyalgo_error}}\label{app:stickyalgo_error}

\subsubsection{Case \texorpdfstring{$\set{\THRESHOLD}(\mu)=\set{\THRESHOLD}=\emptyset$}{}}

\begin{proof}
    Since PKGAI(Unif) belongs to the family of PKGAI algorithms, then Theorem~\ref{th:APTlike_error} applies, and conditioned on the fact that $\beta(\NBATCHES) \leq (\NBATCHES-\NARMS)/(4 H_1(\mu))\;,$
    the upper bound on the error probability for any instance $\nu \in \mathcal{D}^\NARMS$ in that case is $\perr{\nu}{\text{PKGAI(Unif)}}{\NBATCHES} \leq 2 \NARMS \NBATCHES \exp(-2\beta(\NBATCHES))$, and is minimized when the previous inequality on $\beta(\NBATCHES)$ is an equality.
\end{proof}

\subsubsection{Case \texorpdfstring{$\set{\THRESHOLD}(\mu) \neq \emptyset$}{} and \texorpdfstring{$\GUESS{\NBATCHES}=\emptyset$}{}}

\begin{proof}
However, when $\set{\THRESHOLD}\neq\emptyset$, we will take into account the sampling rule in order to find a tighter upper bound on the probability $\bP\left[ \EventStickyAlgo^\text{Case $2$} \cap \GUESS{\NBATCHES} = \emptyset \right]$. Then, necessarily, according to Case $2.1$ in the proof of Theorem~\ref{th:APTlike_error}
\[\bP\left[ \EventStickyAlgo^\text{Case $2$} \cap \GUESS{\NBATCHES} = \emptyset \right] = \bP\left[ \EventStickyAlgo^\text{Case $2$} \cap \max_{a \in \cS_\NBATCHES} \empdeltam{a}{\NBATCHES}+\empdeltap{a}{\NBATCHES} \leq 0\right]\;.\]

$\set{\THRESHOLD} \subseteq \cS_\NBATCHES$ (otherwise, we end up with a contradiction with event $\EventStickyAlgo$). Moreover, if $\max_{a \in \cS_\NBATCHES} \empdeltam{a}{\NBATCHES}+\empdeltap{a}{\NBATCHES} \leq 0$, then Inequality~\eqref{eq:good_arm_ineq} applies. Finally, since PKGAI(Unif) uses a uniform sampling, $\nsamples{a}{\NBATCHES} \geq \left\lfloor \frac{\NBATCHES}{\NARMS} \right\rfloor \geq \frac{\NBATCHES}{\NARMS}-1$ for any arm $a$. Combining all of this yields the following inequalities
\[ \forall a \in \set{\THRESHOLD}, \ \frac{\NBATCHES}{\NARMS}-1 \leq \nsamples{a}{\NBATCHES} < \frac{\beta(\NBATCHES)}{\Delta_a^2}+1 \implies \beta(\NBATCHES) > \frac{\NBATCHES-2\NARMS}{\NARMS}\max_{a \in \set{\THRESHOLD}} \Delta_a^{2} = \frac{\NBATCHES-2\NARMS}{\NARMS\left( \max_{a \in \set{\THRESHOLD}} \Delta_a \right)^{-2}}\:.\]

Then any choice of $\beta$ such that $\beta(\NBATCHES) \leq (\NBATCHES-2\NARMS)/(\NARMS (\max_{a \in \set{\THRESHOLD}} \Delta_a)^{-2})$ would lead to a contradiction. 

\end{proof}

\subsubsection{Case \texorpdfstring{$\set{\THRESHOLD}\neq \emptyset$}{} and \texorpdfstring{$\GUESS{\NBATCHES}\neq \emptyset$}{}}

\begin{proof}
Let us find a tighter upper bound on the error probability $\bP\left[ \EventStickyAlgo^\text{Case $2$} \cap \GUESS{\NBATCHES} \in \ARMS \setminus \set{\THRESHOLD} \right]$. This necessarily implies that the recommendation rule at Line L.$16$ is fired ($\tau=\NBATCHES$) and that the algorithm makes a mistake ($\GUESS{\NBATCHES} \not\in \set{\THRESHOLD}$). On event $\EventStickyAlgo$, there exists $b \not\in \set{\THRESHOLD}$, for all $a \in \set{\THRESHOLD}$,
\begin{align*}
     -\Delta_b \geq_{b \not\in \set{\THRESHOLD}} \empdeltam{b}{\NBATCHES} & \geq \empdeltam{a}{\NBATCHES} \geq_{a \in \set{\THRESHOLD}} \Delta_a -2\sqrt{\frac{\beta(\NBATCHES)}{\nsamples{a}{\NBATCHES}}} \geq \Delta_a -2\sqrt{\frac{\beta(\NBATCHES)}{\min_{c \in \set{\THRESHOLD}}\nsamples{c}{\NBATCHES}}}\\
    \implies \max_{b \not\in \set{\THRESHOLD}} (-\Delta_b) & \geq \max_{a \in \set{\THRESHOLD}}
    \Delta_a -2\sqrt{\frac{\beta(\NBATCHES)}{\min_{c \in \set{\THRESHOLD}}\nsamples{c}{\NBATCHES}}}\\
\end{align*}
Reordering terms and since PKGAI(Unif) uses a uniform sampling
\[ 2\sqrt{\frac{\beta(\NBATCHES)}{\NBATCHES/\NARMS-1}} \geq \hat{\Delta} \deff \max_{a \in \set{\THRESHOLD}} \Delta_a+\min_{b \not\in \set{\THRESHOLD}} \Delta_b \implies \beta(\NBATCHES) \geq \frac{\NBATCHES-\NARMS}{4\NARMS \hat{\Delta}^{-2}}\;.\]

Then any choice of $\beta$ such that $\beta(\NBATCHES) < \frac{\NBATCHES-\NARMS}{4\NARMS \hat{\Delta}^{-2}}$ would lead to a contradiction. 

\subsubsection{Final Step} 

All in all, similarly to the proof of Theorem~\ref{th:APTlike_error}, if the following inequality is satisfied for $\nu \in \mathcal{D}^\NARMS$ of mean vector $\mu$
\[ \beta(\NBATCHES) \leq W_\mu(\NBATCHES) \deff \begin{cases} (\NBATCHES-\NARMS)/(4H_1(\mu)) & \text{ if } \set{\THRESHOLD}(\mu)=\emptyset\\ (\NBATCHES-\NARMS)/(4\NARMS \hat{\Delta}^{-2}) & \text{ otherwise}\end{cases}\;,\]
where $\hat{\Delta} \deff \max_{a \in \set{\THRESHOLD}} \Delta_a + \min_{b \not\in \set{\THRESHOLD}} \Delta_b$, then we end with the following upper bound on the error probability
$\perr{\nu}{\text{PKGAI(Unif)}}{\NBATCHES} \leq 2 \NARMS \NBATCHES \exp(-2\beta(\NBATCHES))$, which is minimized when the inequalities on $\beta(\NBATCHES)$ are equalities.
\end{proof}

\section{Lower Bounds for GAI and Generalized Likelihood Ratio}
\label{app:GLR_and_Times}

In Appendix~\ref{app:ssec_Characteristic_times}, we prove Lemma~\ref{lem:lower_bound_GAI} which is an asymptotic lower bound on the expected sample complexity of a fixed-confidence GAI algorithm.
In Appendix~\ref{app:ssec_GLR}, we present the generalized likelihood ratios for GAI, which relate to the APT$_{P}$ index policy and the GLR stopping rule~Eq.~\eqref{eq:stopping_rule}.  
\marc{In Appendix~\ref{app:ssec_lb_dep_K}, we prove lower bounds showing that a linear dependence in $K$ is actually unavoidable, even when there is a unique good arm: Theorem~\ref{thm:meta_lower_bound} (Appendix~\ref{app:ssec_proof_meta_lower_bound}), Corollary~\ref{thm:lb_unverifiable_K_BAI} (Appendix~\ref{app:ssec_proof_lb_unverifiable_K_BAI}) and Corollary~\ref{thm:lb_K_BAI} (Appendix~\ref{app:ssec_proof_lb_K_BAI}).}

\subsection{Asymptotic Lower Bound for GAI in Fixed Confidence Setting}
\label{app:ssec_Characteristic_times}

Lemma~\ref{lem:lower_bound_GAI} gives an asymptotic lower bound on the expected sample complexity in fixed-confidence GAI, and relies on the well-known change of measure inequality (Lemma $1$ from~\cite{kaufmann2016complexity}).
\begin{lemma}[Lemma~\ref{lem:lower_bound_GAI}]
Let $\DELTA \in (0,1)$.
    For all $\DELTA$-correct algorithm and all Gaussian instances $\nu_{a} = \mathcal N(\mu_{a},1)$, with $\mu_a \neq \THRESHOLD$, $\liminf_{\delta \to 0} \bE_{\nu}[\tau_{\delta}]/\log (1/\delta) \ge T^\star(\mu)$, where
    \begin{equation*} 
    T^\star(\mu) \deff \begin{cases}
        2\min_{a \in \set{\THRESHOLD}(\mu)} \Delta_a^{-2} & \text{if }\set{\THRESHOLD}(\mu) \ne \emptyset\;,\\
        2H_{1}(\mu) & \text{otherwise}\;.
        \end{cases}
    \end{equation*}
\end{lemma}
\begin{proof}
Let $\DELTA \in (0,1)$.
Let us consider any Gaussian instance $\nu_a = \mathcal{N}(\mu_a,1)$, where $\mu_a \neq \THRESHOLD$. We define the following sets of alternative instances, depending on $\set{\THRESHOLD}(\mu)$
\[ 
	\text{Alt}(\mu) \deff \begin{cases}
	\{\lambda \in \mathbb{R}^\NARMS \mid \exists a \in \ARMS, \lambda_a \geq \THRESHOLD\} = \bigcup_{a \in \ARMS} \{\lambda \in \mathbb{R}^\NARMS \mid \lambda_a \geq \THRESHOLD\} & \text{ if } \set{\THRESHOLD}(\mu)=\emptyset\;,\\ 
	\bigcap_{a \in \set{\THRESHOLD}(\mu)} \{\lambda \in \mathbb{R}^\NARMS \mid \lambda_a < \THRESHOLD \} & \text{ otherwise}\;.
	\end{cases}
\]

Let us call kl the binary relative entropy. Let us consider any $\DELTA$-correct algorithm. Combining Lemma $1$ from ~\cite{kaufmann2016complexity} with the $\DELTA$-correctness of the algorithm and the monotonicity of function kl, for any $1$-Gaussian distribution $\kappa$ of mean $\lambda \in \text{Alt}(\mu)$
\[ 
\frac{1}{2}\sum_{a \in \ARMS} \bE_{\nu}[\nsamples{a}{\tau_\DELTA}](\mu_a-\lambda_a)^2 \geq \text{kl}(\perr{\nu}{\mathfrak{A}}{\tau_\DELTA},\perr{\kappa}{\mathfrak{A}}{\tau_\DELTA} \geq \text{kl}(1-\DELTA,\DELTA) \geq \log(1/(2.4\DELTA))\;.
\]
As it holds for any alternative instance $\kappa$, if $\triangle_\NARMS \deff \{p \in [0,1]^\NARMS \mid \sum_{i} p_i = 1\}$, it yields that
\[ \bE_{\nu}[\tau_\DELTA] = \sum_{a \in \ARMS} \bE_{\nu}[\nsamples{a}{\tau_\DELTA}] \geq \underbrace{2\left( \sup_{\omega \in \triangle_\NARMS} \inf_{\lambda \in \text{Alt}(\mu)} \sum_{a \in \ARMS} \omega_a(\mu_a-\lambda_a)^2\right)^{-1}}_{=T^\star(\mu)} \log(1/(2.4\DELTA))\;.\]

If $\set{\THRESHOLD}(\mu) = \emptyset$, then using the definition of Alt($\mu$) in that case and since $\Delta_a \deff |\mu_a-\THRESHOLD|$
\[ \sup_{\omega \in \triangle_\NARMS} \inf_{\lambda \in \text{Alt}(\mu)} \sum_{a \in \ARMS} \omega_a(\mu_a-\lambda_a)^2 = \sup_{\omega \in \triangle_\NARMS} \min_{a \in \ARMS} \omega_a(\mu_a - \THRESHOLD)^2 = \sup_{\omega \in \triangle_\NARMS} \min_{a \in \ARMS} \omega_a\Delta_a^2 = \left( \sum_{a \in \ARMS} \Delta_a^{-2} \right)^{-1} \; ,\]
and $\omega_a \deff \frac{\Delta_a^{-2}}{\sum_{b \in \ARMS} \Delta_b^{-2}}$.
Otherwise, $\set{\THRESHOLD}(\mu) \neq \emptyset$, and then
\[ \sup_{\omega \in \triangle_\NARMS} \inf_{\lambda \in \text{Alt}(\mu)} \sum_{a \in \ARMS} \omega_a(\mu_a-\lambda_a)^2 = \sup_{\omega \in \triangle_\NARMS} \sum_{a \in \set{\THRESHOLD}(\mu)} \omega_a(\mu_a - \THRESHOLD)^2 = \max_{a \in \set{\THRESHOLD}} \Delta_a^2 \: , 
\]
and $\omega_a \deff \mathds{1}(a=\argmax_{a \in \set{\THRESHOLD}} \mu_a)$.
This concludes the proof for $T^\star(\mu)$ as in~Eq.~\eqref{eq:characteristicTime}.
\end{proof}

\subsection{Generalized Likelihood Ratio (GLR)}
\label{app:ssec_GLR}

While we consider $1$-sub-Gaussian distributions $\nu \in \mathcal D^{K}$ with mean $\mu$ in all generality, the ATP$_{P}$ index and the GLR stopping rule stem from generalized likelihood ratios for Gaussian distributions with unit variance.
In the following, we consider Gaussian distributions $\nu_{a} = \mathcal N (\mu_{a},1)$ which are uniquely characterized by their mean parameter $\mu_{a}$.

The generalized log-likelihood ratio between the whole model space $\mathcal{M}$ and a subset $\Lambda\subseteq \mathcal{M}$ is $\mbox{GLR}_t^\mathcal{M}(\Lambda) = \log \frac{\sup_{\tilde{\mu} \in \mathcal{M}} \mathcal{L}_{\tilde{\mu}}(X_1,\ldots,X_t)}{\sup_{\lambda \in \Lambda} \mathcal{L}_{\lambda}(X_1,\ldots,X_t)}$.
In the case of independent Gaussian distributions with unit variance, the likelihood ratio for two models with mean vectors $\xi, \lambda \in \mathcal{M}$,
\begin{align*}
    \log \frac{\mathcal{L}_\xi(X_1,\ldots,X_t)}{\mathcal{L}_\lambda(X_1,\ldots,X_t)}
    = \frac{1}{2} \sum_{a \in \ARMS} \nsamples{a}{t} \left( (\expmean{a}{t} - \lambda_{a})^2 - (\expmean{a}{t} - \xi_{a})^2 \right) \: .
\end{align*}
When $\hat \mu (t) \in \mathcal{M}$, the maximum likelihood estimator $\tilde{\mu} (t)$ coincide with the empirical mean, otherwise it is $\tilde{\mu} (t) = \argmin_{\lambda \in\mathcal{M}} \sum_{a \in \ARMS} \nsamples{a}{t} (\expmean{a}{t} - \lambda_{a})^2$.
In the following, we consider the case where $\hat \mu (t) \in \mathcal{M}$.
The GLR for set $\Lambda$ is $\mbox{GLR}_t^\mathcal{M}(\Lambda) = \frac{1}{2}  \min_{\lambda \in \Lambda} \sum_{a \in \ARMS} \nsamples{a}{t} (\expmean{a}{t} - \lambda_{a})^2$. 

When $\max_{a \in \ARMS} \expmean{a}{t} \le \THRESHOLD$, the recommendation is $\GUESS{t} = \emptyset$.
Therefore, the set of alternative parameters (\ie admitting a different recommendation) is Alt$(\hat \mu(t)) = \bigcup_{a \in \ARMS} \{ \lambda \in \mathbb R^{K} \mid \lambda_{a} > \THRESHOLD \}$.
By direct manipulations similar to the ones in Appendix~\ref{app:ssec_Characteristic_times}, the corresponding GLR can be written as
\begin{align*}
    2 \mbox{GLR}_t^\mathcal{M}(\text{Alt}(\hat \mu(t))) = \min_{a \in \ARMS}  \nsamples{a}{t} (\THRESHOLD - \expmean{a}{t})^2 = (\min_{a \in \ARMS} \Wm{a}{t})^2 \: .
\end{align*}

When $\max_{a \in \ARMS} \expmean{a}{t} > \THRESHOLD$, the recommendation is $\GUESS{t} \in \set{\THRESHOLD}(\hat \mu(t))$.
For each possible answer $a \in \set{\THRESHOLD}(\hat \mu(t))$, the set of alternative parameters (\ie admitting a different recommendation) is Alt$(\hat \mu(t), a) = \{ \lambda \in \mathbb R^{K} \mid \lambda_{a} \le \THRESHOLD \}$.
By direct manipulations similar to the ones in Appendix~\ref{app:ssec_Characteristic_times}, the corresponding GLR can be written as
\begin{align*}
    \forall a \in \set{\THRESHOLD}(\hat \mu(t)), \quad 2 \mbox{GLR}_t^\mathcal{M}(\text{Alt}(\hat \mu(t), a)) = \nsamples{a}{t} (\expmean{a}{t} - \THRESHOLD)^2 = \Wp{a}{t}^2 \: .
\end{align*}

\subsection{\marc{Lower Bounds with Dependence on the Number of Arms}}
\label{app:ssec_lb_dep_K}

\subsubsection{\marc{Proof of Theorem~\ref{thm:meta_lower_bound}}}
\label{app:ssec_proof_meta_lower_bound}

\marc{All arms are Gaussian with variance $1$. These are instances such that $\mathcal A_{\theta}(\nu^{(a)}) = \{a\}$.
Let $\bP_{\nu}^{\tau}$ be the restriction of $\bP_{\nu}$ to the $\sigma$-algebra generated by $\tau$. 
For any $\tau$-measurable event $E$ (e.g., $\{N_{b}(\tau) > n\}$), we have $\bP_{\nu}^\tau(E) = \bP_{\nu}(E)$.}

\marc{A bandit model by specifying the law of each successive reward from each arm: the first rewards queried from arm $a$ will have a given distribution, then the second reward will have a (possibly different) distribution, etc. 
The sequence of distributions is an array of reward laws.
In true bandit models, the distribution is stationary, \ie it does not change. 
However, for the construction of the lower bound, we will use arrays where the distribution changes after some number $n$, \ie non-stationnary distribution.
For all $n \in \mathbb N$ and $(a,b) \in [K]^2$ with $a \ne b$, we write $\eta_{a \to b}^n$ for the following array of reward laws:
\begin{itemize}
    \item For $k \notin \{a,b\}$, $\eta_{a \to b,k}^n$ is constant equal to $\mathcal N(\theta - \epsilon,1)$.
    \item $\eta_{a \to b,a}^n$ is constant equal to $\cN(\theta + \Delta,1)$.
    \item For the first $n$ rewards, $\eta_{a \to b,b}^n$ is $\cN(\theta - \epsilon,1)$. For the next rewards, $\eta_{a \to b,b}^n$ is $\cN(\theta + \Delta, 1)$.
\end{itemize}
Since $\TV$ is symmetric and satisfies the triangle inequality, we have
\begin{align*}
\TV(\mathbb{P}_{\nu^{(a)}}^\tau, \mathbb{P}_{\nu^{(b)}}^\tau)
\le \TV(\mathbb{P}_{\nu^{(a)}}^\tau, \mathbb{P}_{\eta_{a \to b}^n}^\tau)
    + \TV(\mathbb{P}_{\nu^{(b)}}^\tau, \mathbb{P}_{\eta_{b \to a}^n}^\tau)
    + \TV(\mathbb{P}_{\eta_{a \to b}^n}^\tau, \mathbb{P}_{\eta_{b \to a}^n}^\tau) \: .
\end{align*}
Using Pinsker's inequality and the data-processing inequality, we obtain
\[
    \TV(\mathbb{P}_{\eta_{a \to b}^n}^\tau, \mathbb{P}_{\eta_{b \to a}^n}^\tau) \le \sqrt{\KLm(\mathbb{P}_{\eta_{a \to b}^n}^\tau, \mathbb{P}_{\eta_{b \to a}^n}^\tau)/2} \le \sqrt{\KLm(\mathbb{P}_{\eta_{a \to b}^n}, \mathbb{P}_{\eta_{b \to a}^n})/2} = \sqrt{n (\Delta + \epsilon)^2/2} \: .
\]
An application of the general property that conditioning increases $f$-divergences yields Lemma~\ref{lem:TV_eta_le_prob}, proved in Lemma C.4 in~\citet{poianibest}.}
\begin{lemma}[Lemma C.4 in~\citet{poianibest}]\label{lem:TV_eta_le_prob}
\marc{
\begin{align*}
    \forall n \in \mathbb N, \forall a \in [K], \forall b \in [K]\setminus\{a\} , \quad\TV(\mathbb{P}_{\nu^{(a)}}^\tau, \mathbb{P}_{\eta_{a \to b}^n}^\tau)
&\le \mathbb{P}_{\nu^{(a)}}(N_{b}(\tau) > n)
\: .
\end{align*}}
\end{lemma}
\marc{Combining the above inequalities with Lemma~\ref{lem:TV_eta_le_prob} yields
\[
    \TV(\mathbb{P}_{\nu^{(a)}}^\tau, \mathbb{P}_{\nu^{(b)}}^\tau) \le \mathbb{P}_{\nu^{(a)}}(N_{b}(\tau) > n) + \mathbb{P}_{\nu^{(b)}}(N_{a}(\tau) > n) + \sqrt{n (\Delta+\epsilon)^2/2} \: ,
\]
which is exactly Lemma C.6 in~\citet{poianibest}.
Summing these inequalities over $a \in [K]$ and $b \in [K] \setminus \{a\}$, we obtain
\begin{align*}
    &\sum_{a \in [K], b \ne a} \TV(\mathbb{P}_{\nu^{(a)}}^\tau, \mathbb{P}_{\nu^{(b)}}^\tau) -  K(K-1)\sqrt{n (\Delta+\epsilon)^2/2} \\ 
    &\le \sum_{a \in [K], b \ne a}\left( \mathbb{P}_{\nu^{(a)}}(N_{b}(\tau) > n) + \mathbb{P}_{\nu^{(b)}}(N_{a}(\tau) > n) \right) \le \frac{2}{n}\sum_{a \in [K]} \mathbb{E}_{\nu^{(a)}}[\tau - N_{a}(\tau)]    \: .
\end{align*}
where the second inequality uses $\mathbb{E}_{\nu^{(a)}}[\tau - N_{a}(\tau)] = \sum_{b \ne a}\mathbb{E}_{\nu^{(a)}}[N_{b}(\tau)]$ and Markov's inequality, i.e., $\mathbb{P}_{\nu^{(a)}}(N_{b}(\tau) > n) \le \mathbb{E}_{\nu^{(a)}}[N_{b}(\tau)]/n $ for all $a \ne b$.
Summing the inequalities obtained by assumption on the stopping time $\tau_{\delta}$ and re-ordering, we obtain
\[
    \frac{1}{K} \sum_{a \in [K]} \mathbb{E}_{\nu^{(a)}}[\tau_{\delta} - N_{a}(\tau_{\delta})] \ge \frac{n(K-1)}{2} \left( 1 - 2\delta -  \sqrt{n (\Delta+\epsilon)^2/2} \right) \: .
\]
Taking $n = \frac{2}{(\Delta+\epsilon)^2} \left(\frac{1 - 2 \delta}{2}\right)^2$ concludes the proof since
\begin{align*}
    \frac{1}{K} \sum_{a \in [K]} \mathbb{E}_{\nu^{(a)}}[\tau_{\delta} - N_{a}(\tau_{\delta})] \ge \frac{K-1}{(\Delta+\epsilon)^2}  \left( \frac{1}{2} - \delta \right)^3 \ge \frac{K-1}{64(\Delta+\epsilon)^2}  \left( \frac{1}{2} - \delta_0 \right)^3 \: ,
\end{align*}
where the last inequality uses that $\delta \to \left( \frac{1}{2} - \delta \right)^3$ is decreasing on $(0,1/4]$ and $\delta \le 1/4$.}

\subsubsection{\marc{Proof of Corollary~\ref{thm:lb_unverifiable_K_BAI}}}
\label{app:ssec_proof_lb_unverifiable_K_BAI}

\marc{Let $(\theta,\Delta,\epsilon) \in \R \times (\R_{+}^{\star})^2$ and $(\nu^{(a)})_{a \in [K]}$ as in Theorem~\ref{thm:meta_lower_bound}.
All arms are Gaussian with variance $1$. 
These are instances such that $\mathcal A_{\theta}(\nu^{(a)}) = \{a\}$.
Let $\delta \in (0,1/4]$.
Let $\tau_{U,\delta}$ be the unverifiable sample complexity of a given strategy.
For all $a \in [K]$ and all $b \in [K]\setminus \{a\}$,
\[
     \bP_{\nu^{(a)}}(\exists t \ge \tau_{U,\delta}, \:  \GUESS{t} \ne a) \le \delta \quad \text{and} \quad \bP_{\nu^{(b)}}(\forall t \ge \tau_{U,\delta}, \:  \GUESS{t} = b) \ge 1 - \delta \: .
\]
For any $\tau_{U,\delta}$-measurable event $E$, we have $\bP_{\nu}^{\tau_{U,\delta}}(E) = \bP_{\nu}(E)$.
Since $\{\GUESS{\tau_{U,\delta}} \ne a\}$ is $\tau_{U,\delta}$-measurable and satisfies that
\[
    \{\GUESS{\tau_{U,\delta}} \ne a\} \subseteq \{\exists t \ge \tau_{U,\delta}, \:  \GUESS{t} \ne a\} \quad \text{and} \quad \{\forall t \ge \tau_{U,\delta}, \:  \GUESS{t} = b\} \subseteq \{\GUESS{\tau_{U,\delta}} \ne a\} \: ,
\]
we obtain 
\[
    \TV(\mathbb{P}_{\nu^{(a)}}^{\tau_{U,\delta}}, \mathbb{P}_{\nu^{(b)}}^{\tau_{U,\delta}}) \ge \mathbb{P}_{\nu^{(b)}}(\GUESS{\tau_{U,\delta}} \ne a) - \mathbb{P}_{\nu^{(a)}}(\GUESS{\tau_{U,\delta}} \ne a)  \ge 1 -2\delta \: .
\]
Applying Theorem~\ref{thm:meta_lower_bound} concludes the proof since
\[
     \max_{a \in [K]} \mathbb{E}_{\nu^{(a)}}[\tau_{U,\delta} - N_{a}(\tau_{U,\delta})] \ge \frac{1}{K} \sum_{a \in [K]} \mathbb{E}_{\nu^{(a)}}[\tau_{U,\delta} - N_{a}(\tau_{U,\delta})] \ge \frac{K-1}{64(\Delta+\epsilon)^2}   \: .
\]}

\subsubsection{\marc{Proof of Corollary~\ref{thm:lb_K_BAI}}}
\label{app:ssec_proof_lb_K_BAI}

\marc{Let $(\theta,\Delta,\epsilon) \in \R \times (\R_{+}^{\star})^2$ and $(\nu^{(a)})_{a \in [K]}$ as in Theorem~\ref{thm:meta_lower_bound}.
All arms are Gaussian with variance $1$. 
These are instances such that $\mathcal A_{\theta}(\nu^{(a)}) = \{a\}$.
Let $\delta \in (0,1/4]$.
Let $\tau_{\delta}$ be the sample complexity of a $\delta$-correct strategy.
For all $a \in [K]$ and all $b \in [K]\setminus \{a\}$, 
\[
     \bP_{\nu^{(a)}}(\hat{a}_{\tau_{\delta}} = a) \ge 1 - \delta \quad \text{and} \quad \bP_{\nu^{(b)}}(\hat{a}_{\tau_{\delta}} \ne b) \le \delta \: .
\]
For any $\tau_{\delta}$-measurable event $E$, we have $\bP_{\nu}^{\tau_{\delta}}(E) = \bP_{\nu}(E)$.
Since $\{\GUESS{\tau_{\delta}} = a\}$ is $\tau_{\delta}$-measurable and satisfies that $\{\GUESS{\tau_{\delta}} = a\} \subseteq \{\hat{a}_{\tau_{\delta}} \ne b\}$, we obtain 
\[
    \TV(\mathbb{P}_{\nu^{(a)}}^{\tau_{\delta}}, \mathbb{P}_{\nu^{(b)}}^{\tau_{\delta}}) \ge \bP_{\nu^{(a)}}(\hat{a}_{\tau_{\delta}} = a) - \bP_{\nu^{(b)}}(\hat{a}_{\tau_{\delta}} = a)  \ge 1 -2\delta \: .
\]
Applying Theorem~\ref{thm:meta_lower_bound} concludes the proof since
\[
     \max_{a \in [K]} \mathbb{E}_{\nu^{(a)}}[\tau_{\delta} - N_{a}(\tau_{\delta})] \ge \frac{1}{K} \sum_{a \in [K]} \mathbb{E}_{\nu^{(a)}}[\tau_{\delta} - N_{a}(\tau_{\delta})] \ge \frac{K-1}{64(\Delta+\epsilon)^2}   \: .
\]}

\section{Analysis of APGAI: Proof of Theorem~\ref{thm:expected_sample_complexity_upper_bound}}\label{app:anytimealgo_sample}

When combined with the GLR stopping~Eq.~\eqref{eq:stopping_rule} using threshold~Eq.~\eqref{eq:stopping_threshold}, \hyperlink{APGAI}{APGAI} becomes dependent of a \marc{risk} $\delta \in (0,1)$.

\begin{remark}[\marc{Risk $\delta$: algorithmic (Appendix~\ref{app:anytimealgo_sample}) or analysis (Appendix~\ref{app:anytimealgo_error})}]    
\marc{The \\risk parameter $\delta$ is only present in the probabilistic statements that involves the GLR stopping rule Eq.~\eqref{eq:stopping_rule} due to the stopping threshold $c(T,\delta)$ as in Eq.~\eqref{eq:stopping_threshold} that depends on the risk $\delta$.
The risk $\delta$ is a parameter of the algorithm ensuring the $\delta$-correctness of the resulting algorithm by Lemma~\ref{lem:delta_correct_threshold}.
We highlight the difference with the analysis of the probability of error for \hyperlink{APGAI}{APGAI} detailed in Appendix~\ref{app:anytimealgo_error}.
The parameter $\delta$ is only used for the analysis to define a similar sequence of concentration events $(\tilde \cE_{T,\delta})_{T > K}$.
While the non-asymptotic analysis of the expected sample complexity only requires coarse upper bound on $\sum_{T > K} \bP_{\nu}(\cE_{T}^{\complement}) $ by Lemma~\ref{lem:lemma_1_Degenne19BAI}, the non-asymptotic analysis of the probability of error requires a small upper bound on each $\bP_{\nu}(\tilde \cE_{T,\delta}^{\complement})$. 
Therefore, it is not necessary to introduce a similar analysis parameter $\tilde \delta$ here, and we simply take $\tilde \delta \deff 1$.
The purpose of the analysis parameter $\delta$ in Appendix~\ref{app:anytimealgo_error} is to quantify how small $\bP_{\nu}(\tilde \cE_{T,\delta}^{\complement})$ is. As we show that the error event is included in $\tilde \cE_{T,\delta}^{\complement}$ for $T$ large enough (as a function of $\delta$), we can invert the upper bound based on Lemma~\ref{lem:inversion_PoE}.}
\end{remark}

\textit{Proof strategy.}
Let $\mu \in \rR^K$ such that $\mean{a} \ne \theta$ for all $a \in \ARMS$.
Let $s > 1$.
For all $T >  K$ and $\cE_{T} = \cE_{T,1}$ where $\cE_{T,\delta}$ as in~Eq.~\eqref{eq:event_concentration_per_arm_aeps}, \ie 
\begin{equation} \label{eq:event_sample}
	\cE_{T} = \left\{ \forall a \in \ARMS, \forall t \le T, \: |\expmean{a}{t} - \mu_a| < \sqrt{2 f_1(T)/N_{a}(t)} \right\} \: ,    
\end{equation}
with $ f_1(T) =  (1+s) \log T$.
\marc{The sequence of concentration events $(\cE_{T})_{T > K}$ will be used to derive probabilistic statements on the \hyperlink{APGAI}{APGAI} sampling and recommendation rules, holding provided concentration holds. 
Crucially, while these events are independent of the risk $\delta$, the probability that $\bigcup_{T > K} \cE_{T}^{\complement}$ can still be upper bounded. 
Namely, combining a direct union bound with} Lemma~\ref{lem:concentration_per_arm_gau_aeps}, we have $\sum_{T > K} \bP_{\nu}(\cE_{T}^{\complement}) \le  K \zeta(s)$ where $\zeta$ is the Riemann $\zeta$ function.

Suppose that we have constructed a time $T_{\mu}(\delta) >  K$ such that $ \cE_{T} \subseteq \{\tau_{\delta} \le T\}$ for $T \ge T_{\mu}(\delta)$.
Then, using Lemma~\ref{lem:lemma_1_Degenne19BAI}, we obtain $\bE_{\nu}[\tau_{\delta}] \le T_{\mu}(\delta) +  K \zeta(s)$.
To prove Theorem~\ref{thm:expected_sample_complexity_upper_bound}, we will distinguish between instances $\mu$ such that $\set{\THRESHOLD} = \emptyset$ (Appendix~\ref{app:sssec_sample_complexity_empty_is_good}) and instances $\mu$ such that $\set{\THRESHOLD} \ne \emptyset$ (Appendix~\ref{app:sssec_sample_complexity_exist_good_arms}).

As for the proof of Theorem~\ref{thm:upper_bound_PoE_anytimeID}, our main technical tool is Lemma~\ref{lem:technical_result_bad_event_implies_bounded_quantity_increases}.
It is direct to see that Lemmas~\ref{lem:time_no_undersampled_empty_is_good} and~\ref{lem:time_no_undersampled_exist_good_arms} can be adapted to hold for $\cE_{T}$ and $f_{1}(T) = (1+s) \log T$.
Combined with Lemma~\ref{lem:inversion_upper_bound}, we state those results in a more explicit form, and omit the \marc{details of the} proof.
\marc{Since the concentration event $\cE_{T}$ is independent of the risk $\delta$, the time $T_{\mu}$ and $S_{\mu}$ in Lemmas~\ref{lem:time_no_undersampled_empty_is_good'} and~\ref{lem:time_no_undersampled_exist_good_arms'} are independent of $\delta$.
Since both $T_{\mu}$ and $S_{\mu}$ scale as $\mathcal O(H_{1}(\mu) \log H_{1}(\mu))$, the $\delta$-independent non-asymptotic bound for \hyperlink{APGAI}{APGAI} will scale as $\mathcal O(H_{1}(\mu) \log H_{1}(\mu))$ even when there are good arms.
The independence in $\delta$ is crucial to differentiate the asymptotic behavior of \hyperlink{APGAI}{APGAI} when there are good arms. 
If $T_{\mu}(\tilde \delta)$ and $S_{\mu}(\tilde \delta)$ were used, we would obtain a dependency in $\mathcal O(H_{1}(\mu) \log ( H_{1}(\mu)/\tilde{\delta}))$, which is undesirable when the analysis parameter $\tilde \delta$ is chosen as the algorithmic parameter $\delta$.
Taking $\tilde \delta = 1$ circumvents this issue.}

\begin{lemma}[\marc{Lemma~\ref{lem:time_no_undersampled_empty_is_good}: concentration event $\cE_{T}$ instead of $\tilde \cE_{T,\delta}$}] \label{lem:time_no_undersampled_empty_is_good'}
	Let $\mu \in \rR^{K}$ such that $\set{\THRESHOLD} = \emptyset$ and $\mean{a} \ne \theta$ for all $a \in \ARMS$.
	Let $s > 1$.
	Let $T_{\mu} = h_{1}(18 (1+s) H_{1}(\mu),  K)$ where $h_{1}$ is defined in Lemma~\ref{lem:inversion_upper_bound}.
	For all $T > T_{\mu}$, under the event $\cE_{T}$ as in~Eq.~\eqref{eq:event_sample}, we have $N_{a}(T) > 2 (1+s)  \Delta_{a}^{-2} \log ( T) $ for all $a \in \ARMS$.
\end{lemma}
\begin{proof}
	Let us define $\tilde T_{\mu} = \sup \left\{ T \mid T \le  18 (1+s) H_{1}(\mu) \log ( T) +  K \right\}$. Using Lemma~\ref{lem:inversion_upper_bound}, we obtain $\tilde T_{\mu} \le T_{\mu}$ where $T_{\mu} = h_{1}(18 (1+s) H_{1}(\mu),  K)$.
	Combined with the proof of Lemma~\ref{lem:time_no_undersampled_empty_is_good}, this concludes the proof.
\end{proof}

\begin{lemma}[\marc{Lemma~\ref{lem:time_no_undersampled_exist_good_arms}: concentration event $\cE_{T}$ instead of $\tilde \cE_{T,\delta}$}] \label{lem:time_no_undersampled_exist_good_arms'}
	Let $\mu \in \rR^{K}$ such that $\set{\THRESHOLD} \ne \emptyset$ and $\mean{a} \ne \theta$ for all $a \in \ARMS$.
	Let $s > 1$.
	Let $S_{\mu} = h_{1}(4 (1+s) H_{1}(\mu),  K + 2|\set{\THRESHOLD}|)$ where $h_{1}$ is defined in Lemma~\ref{lem:inversion_upper_bound}.
	For all $T > S_{\mu}$, under the event $\cE_{T}$ as in~Eq.~\eqref{eq:event_sample}, we have $\GUESS{\marc{T}} \in \set{\THRESHOLD}$ and there exists $a \in \marc{\set{\THRESHOLD} } $ such that $N_{a}(T) > (\Delta_{a}^{-1}\sqrt{2(1+s) \log(T)} + 1 )^2$.
\end{lemma}
\begin{proof}
		Let us define $\tilde 	S_{\mu} = \sup \left\{ T \mid T \le 4 (1+s) H_{1}(\mu) \log ( T) +  K + 2|\set{\THRESHOLD}| \right\}$.
		Lemma~\ref{lem:inversion_upper_bound} yields $\tilde 	S_{\mu} \le S_{\mu}$ where $S_{\mu} = h_{1}(4 (1+s) H_{1}(\mu),  K + 2|\set{\THRESHOLD}|)$.
		Combined with the proof of Lemma~\ref{lem:time_no_undersampled_exist_good_arms}, this concludes the proof.
\end{proof}

Theorem~\ref{thm:expected_sample_complexity_upper_bound} is obtained by combining Lemmas~\ref{lem:expected_sample_complexity_empty_is_good} and~\ref{lem:expected_sample_complexity_exist_good_arms}.
 
\subsection{Instances where \texorpdfstring{$\set{\THRESHOLD} = \emptyset$}{}}
\label{app:sssec_sample_complexity_empty_is_good}

When $\set{\THRESHOLD} = \emptyset$\marc{, provided concentration event $\cE_{T}$ holds}, we have $\GUESS{T} = \emptyset$ and $a_{T+1} \in \argmin_{a \in \ARMS} W_{a}^{-}(T)$ \marc{for $T > T_{\mu}$.
As detailed above, we have $T_{\mu} = \mathcal O(H_{1}(\mu) \log H_{1}(\mu))$, yet is independent of the risk $\delta$.}
Lemma~\ref{lem:large_time_behavior_empty_is_good} formalizes this intuition.
\begin{lemma}\label{lem:large_time_behavior_empty_is_good}
	Let $s > 1$.
	Let $T_{\mu} = h_{1}(18 (1+s) H_{1}(\mu),  K)$ where $h_{1}$ is defined in Lemma~\ref{lem:inversion_upper_bound}.
	For all $T > T_{\mu}$, \marc{under $\cE_{T}$ as in~Eq.~\eqref{eq:event_sample},} $a_{T+1} \in \argmin_{a \in \ARMS} W_{a}^{-}(T)$ and $\GUESS{T} = \emptyset$.
\end{lemma}
\begin{proof}
	Let $T_{\mu}$ as in Lemma~\ref{lem:time_no_undersampled_empty_is_good}.
	Let $T > T_{\mu}$.
	Using Lemma~\ref{lem:time_no_undersampled_empty_is_good'}, under $\cE_{T}$ as in~Eq.~\eqref{eq:event_sample}, we obtain that $N_{a}(T) > \frac{2f_{1}(T)}{(\theta - \mean{a})^2}$ for all $a \in \ARMS$.
	Then $\expmean{a}{\marc{T}} \le \mean{a} + \sqrt{2 f_{1}(T)/N_{a}(T)} < \theta$ for all $a \in \ARMS$, hence $\max_{a \in \ARMS} \expmean{a}{\marc{T}} < \theta$.
	Using the definition of the sampling rule when $\max_{a \in \ARMS} \expmean{a}{\marc{T}} < \theta$, for all $T > T_{\mu}$, we have $a_{T+1} \in \argmin_{a \in \ARMS} W_{a}^{-}(T)$ and $\GUESS{T} = \emptyset$.
\end{proof}

When coupled with the GLR stopping~Eq.~\eqref{eq:stopping_rule} using threshold~Eq.~\eqref{eq:stopping_threshold}, Lemma~\ref{lem:expected_sample_complexity_empty_is_good} gives an upper bound on the expected sample complexity of \hyperlink{APGAI}{APGAI} when $\set{\THRESHOLD} = \emptyset$.
\marc{Since it involves the stopping threshold Eq.~\eqref{eq:stopping_threshold}, the upper bound $C_{\mu}(\delta)$ depends on the risk $\delta$.
It satisfies $\limsup_{\delta \to 0} C_{\mu}(\delta) / \log(1/\delta) \le 2 H_{1}(\mu) $ and its $\delta$-independent dominating dependency scales as $\mathcal O(H_{1}(\mu) \log H_{1}(\mu))$.}
\begin{lemma} \label{lem:expected_sample_complexity_empty_is_good}
			Let $\delta \in (0,1)$.
    Combined with GLR stopping~Eq.~\eqref{eq:stopping_rule} using threshold~Eq.~\eqref{eq:stopping_threshold}, the \hyperlink{APGAI}{APGAI} algorithm is $\delta$-correct and it satisfies that, for all $\nu \in \mathcal D^{K}$ with mean $\mu$ such that $\set{\THRESHOLD}(\mu) = \emptyset$ and $\Delta_{\min} > 0$, $\bE_{\nu}[\tau_{\delta}] \le C_{\mu}(\delta) + K \pi^2/6 + 1$, with $H_{1}(\mu)$ as in~Eq.~\eqref{eq:common_complexity} and $T_{\mu} = h_{1}(54 H_{1}(\mu),  K)$ with $h_{1}$ is defined in Lemma~\ref{lem:inversion_upper_bound} and
				\begin{align*}
    C_{\mu}(\delta) &= \sup \left\{  T \mid \frac{T - T_{\mu}}{2H_{1}(\mu)} \le \left( \sqrt{c(T, \delta)} + \sqrt{3 \log T}\right)^2  +  \left(\theta - \min_{a \in \ARMS} \mean{a}\right)^2 - 3 \log T_{\mu}  \right\}  \\
                    &= \sup \{ t \mid t \le 2H_{1}(\mu) ( \sqrt{c(t, \delta)} + \sqrt{3 \log t})^2 + D_{1}(\mu) \} \: ,
				\end{align*}
    where $D_{1}(\mu) = T_{\mu} + 2H_{1}(\mu) \left(\theta - \min_{a \in \ARMS} \mean{a}\right)^2 - 6  H_{1}(\mu)\log T_{\mu}  $.
				In particular, it satisfies $\limsup_{\delta \to 0} \bE_{\nu}[\tau_{\delta}] / \log(1/\delta) \le 2 H_{1}(\mu) $.
\end{lemma}

\begin{proof}
	Let $T_{\mu}$ as in Lemma~\ref{lem:large_time_behavior_empty_is_good}.
	Let $T > T_{\mu}$ such that $\cE_{T} \cap \{\tau_{\delta} > T\}$ holds true.
Let $w \in \triangle_{K}$ such that $w_{a} = (\theta - \mean{a})^{-2} H_{1}(\mu)^{-1}$ for all $a \in \ARMS$.
Using the pigeonhole principle, at time $T$ there exists $a_{1} \in \ARMS$ such that $N_{a_{1}}(T) - N_{a_{1}}(T_{\mu}) \ge (T - T_{\mu}) w_{a_{1}}$.
Let $T \ge T_{\mu} + (\min_{a \in \ARMS}  w_{a})^{-1}$, hence we have $N_{a_{1}}(T) - N_{a_{1}}(T_{\mu}) \ge  w_{a_{1}}/\min_{a \in \ARMS}  w_{a} \ge 1$.
Therefore, arm $a_{1}$ has been sampled at least once in $(T_{\mu}, T)$.
Let $t_{a_{1}} \in (T_{\mu}, T)$ be the last time at which arm $a_1$ was selected to be pulled next, \ie $a_{t_{a_{1}} + 1} = a_{1}$ and $N_{a_{1}}(T) = N_{a_{1}}(t_{a_{1}} + 1) = N_{a_{1}}(t_{a_{1}}) + 1$.
Since $t_{a_{1}} > T_{\mu}$, Lemma~\ref{lem:large_time_behavior_empty_is_good} yields that $a_{1} = a_{t_{a_{1}} + 1} \in \argmin_{a \in \ARMS} \Wm{a}{t_{a_{1}}}$.
Moreover, we have
\[
	N_{a_{1}}(t_{a_{1}}) = N_{a_{1}}(T) - 1 \ge (T - T_{\mu}) w_{a_{1}} + N_{a_{1}}(T_{\mu}) - 1 \ge  T w_{a_{1}} + \frac{2f_{1}(T_{\mu}) - T_{\mu} H_{1}(\mu)^{-1}}{(\theta - \mu_{a_{1}})^2}  - 2\: ,
\]
where we used that $N_{a_{1}}(T_{\mu}) \ge N_{a_{1}}(T_{\mu}+1) - 1 > 2f_{1}(T_{\mu}+1) \Delta_{a_{1}}^{-2}$ and $f_{1}$ is increasing.
Under $\cE_{T}$ as in~Eq.~\eqref{eq:event_sample}, using that $a_{1} = a_{t_{a_{1}} + 1} \in \argmin_{a \in \ARMS} \Wm{a}{t_{a_{1}}}$, we obtain
\begin{align*}
	W_{a_{1}}^{-}(t_{a_{1}}) &= \sqrt{N_{a_{1}}(t_{a_{1}})}(\theta - \expmean{a}{t_{a_{1}}})_{+} = \sqrt{N_{a_{1}}(t_{a_{1}})}(\theta - \expmean{a}{t_{a_{1}}}) \\
	&\ge \sqrt{N_{a_{1}}(t_{a_{1}})}(\theta - \mu_{a_{1}}) - \sqrt{2f_{1}(T)} \\
	&\ge \sqrt{\left(  T w_{a_{1}}(\theta - \mu_{a_{1}})^2 + 2f_{1}(T_{\mu}) - T_{\mu} H_{1}(\mu)^{-1}  - 2(\theta - \mu_{a_{1}})^2 \right)} - \sqrt{2f_{1}(T)} \\
	&= \sqrt{ (T- T_{\mu}) H_{1}(\mu)^{-1} + 2f_{1}(T_{\mu})  - 2(\theta - \mu_{a_{1}})^2 } - \sqrt{2f_{1}(T)} \: .
\end{align*}
Since $a_{1} = a_{t_{a_{1}} + 1} \in \argmin_{a \in \ARMS} \Wm{a}{t_{a_{1}}}$, using that the condition of the stopping rule is not met at time $t_{a_1}$ yields
\begin{align*}
	\sqrt{2c(T, \delta)} &\ge \sqrt{2c(\delta, t_{a_{1}})} \ge \min_{b \in \ARMS} W_{b}^{-}(t_{a_1}) =   W_{a_1}^{-}(t_{a_1}) \quad \text{hence} \\
	\sqrt{2c(T, \delta)} &\ge \sqrt{ (T- T_{\mu}) H_{1}(\mu)^{-1} + 2f_{1}(T_{\mu})  - 2(\theta - \mu_{a_{1}})^2 } - \sqrt{2f_{1}(T)} \: .
\end{align*}
Using $\mu_{a_{1}} \ge \min_{a \in \ARMS} \mean{a}$, the above inequality can be rewritten as
\[
	T - T_{\mu} \le 2\left(\sqrt{c(T, \delta)} + \sqrt{f_{1}(T)}\right)^2 H_{1}(\mu)   + 2 H_{1}(\mu) \left( (\theta - \min_{a \in \ARMS} \mean{a})^2 - f_{1}(T_{\mu}) \right)   \: .
\]
Let us define
\begin{align*}
	C_{\mu}(\delta) &= \sup \left\{  T \mid \frac{T - T_{\mu}}{2H_{1}(\mu)} \le \left( \sqrt{c(T, \delta)} + \sqrt{f_{1}(T)}\right)^2  +   (\theta - \min_{a \in \ARMS} \mean{a})^2 - f_{1}(T_{\mu})   \right\} \: .
\end{align*}
It is direct to notice that $T_{\mu} + (\min_{a \in \ARMS}  w_{a})^{-1} = T_{\mu} +  (\theta - \min_{a \in \ARMS}\mean{a})^2 H_{1}(\mu) \le C_{\mu}(\delta)$.
Therefore, we have shown that for $T \ge C_{\mu}(\delta) + 1$, we have $\cE_{T} \subset \{\tau_{<,\delta} \le T \} \marc{\subseteq} \{\tau_{\delta} \le T \}$ \marc{since $\tau_{\delta} \deff \min\{\tau_{>,\delta}, \tau_{<,\delta}\}  \le \tau_{<,\delta}$ almost surely by definition}.
Using Lemma~\ref{lem:lemma_1_Degenne19BAI}, we obtain $\bE_{\nu}[\tau_{\delta}] \le C_{\mu}(\delta) + K \zeta(s) + 1 $.
Taking $s = 2$, using that $\zeta(2) = \pi^2/6$ and $f_{1}(T) = 3 \log T$ yields the second part of the result.
Using Lemma~\ref{lem:asymptotic_inversion_result}, direct manipulations show that $\limsup_{\delta \to 0} \frac{\bE_{\nu}[\tau_{\delta}]}{\log(1 / \delta)} \le	\limsup_{\delta \to 0} \frac{C_{\mu}(\delta)}{\log(1 / \delta)} \le 2 H_{1}(\mu)$.
According to Lemma~\ref{lem:lower_bound_GAI}, we have proven asymptotic optimality. 
Lemma~\ref{lem:delta_correct_threshold} gives the $\delta$-correctness of the \hyperlink{APGAI}{APGAI} algorithm due to our recommendation rule.
\end{proof}

\subsection{Instances where \texorpdfstring{$\set{\THRESHOLD} \ne \emptyset$}{}}
\label{app:sssec_sample_complexity_exist_good_arms}

When $\set{\THRESHOLD} \ne \emptyset$\marc{, provided concentration event $\cE_{T}$ holds}, we have $\GUESS{T} = a_{T+1} $ and $ a_{T+1}  \in \argmax_{a \in \set{\THRESHOLD}} W_{a}^{+}(T)$ \marc{for $T > S_{\mu}$.
As detailed above, we have $S_{\mu} = \mathcal O(H_{1}(\mu) \log H_{1}(\mu))$, yet it is independent of the risk $\delta$.}
Lemma~\ref{lem:large_time_behavior_exist_good_arms} formalizes this intuition.
\begin{lemma}\label{lem:large_time_behavior_exist_good_arms}
	Let $s > 1$.
	Let $S_{\mu} = h_{1}(4 (1+s) H_{1}(\mu),  K + 2|\set{\THRESHOLD}| )$ where $h_{1}$ is defined in Lemma~\ref{lem:inversion_upper_bound}.
	For all $T > S_{\mu}$, \marc{under $\cE_{T}$ as in~Eq.~\eqref{eq:event_sample},} $\GUESS{T} = a_{T+1} $ and $ a_{T+1}  \in \argmax_{a \in \set{\THRESHOLD}} W_{a}^{+}(T)$.
\end{lemma}
\begin{proof}
	Let $S_{\mu}$ as in Lemma~\ref{lem:time_no_undersampled_exist_good_arms'}
	Let $T > S_{\mu}$.
	Using Lemma~\ref{lem:time_no_undersampled_exist_good_arms'}, under $\cE_{T}$ as in~Eq.~\eqref{eq:event_sample}, we have $\GUESS{T} \in \set{\THRESHOLD}$ and there exists $a \in \set{\THRESHOLD}$ such that $N_{a}(T) > \frac{2 f_{1}(T)}{(\mean{a} - \theta)^2}$.
	Then, we have $\expmean{a}{\marc{T}} \ge \mean{a} - \sqrt{2 f_{1}(T)/N_{a}(T) } > \theta$,	hence $\max_{a \in  \set{\THRESHOLD}} \expmean{a}{\marc{T}} > \theta$.
	Using Lemma~\ref{lem:one_good_arm_no_undersampled_implies_no_error} and the definition of the recommendation rule when $\max_{a \in \ARMS} \expmean{a}{\marc{T}} > \theta$, we obtain that $\GUESS{T} = a_{T+1}$, hence $a_{T+1} \in \set{\THRESHOLD}$.
	This concludes the proof.
\end{proof}

When coupled with the GLR stopping~Eq.~\eqref{eq:stopping_rule} using threshold~Eq.~\eqref{eq:stopping_threshold}, Lemma~\ref{lem:expected_sample_complexity_exist_good_arms} gives an upper bound on the expected sample complexity of \hyperlink{APGAI}{APGAI} when $\set{\THRESHOLD} \ne \emptyset$.
\marc{Since it involves the stopping threshold Eq.~\eqref{eq:stopping_threshold}, the upper bound $C_{\mu}(\delta)$ depends on the risk $\delta$.
It satisfies $\limsup_{\delta \to 0} C_{\mu}(\delta) / \log(1/\delta) \le 2 H_{\theta}(\mu) $.
However, its $\delta$-independent dominating dependency scales as $\mathcal O(H_{1}(\mu) \log H_{1}(\mu))$, \ie the same dependency as when there are no good arms.}
\begin{lemma} \label{lem:expected_sample_complexity_exist_good_arms}
			Let $\delta \in (0,1)$.
    Combined with GLR stopping~Eq.~\eqref{eq:stopping_rule} using threshold~Eq.~\eqref{eq:stopping_threshold}, the \hyperlink{APGAI}{APGAI} algorithm is $\delta$-correct and it satisfies that, for all $\nu \in \mathcal D^{K}$ with mean $\mu$ such that $\set{\THRESHOLD} \ne \emptyset$ and $\Delta_{\min} > 0$, $\bE_{\nu}[\tau_{\delta}] \le C_{\mu}(\delta) + K \pi^2/6 + 1$, where $H_{\theta}(\mu)$ as in~Eq.~\eqref{eq:common_complexity} and $S_{\mu} = h_{1}(12 H_{1}(\mu),  K+2|\set{\THRESHOLD}| )$ with $h_{1}$ is defined in Lemma~\ref{lem:inversion_upper_bound} and
	\begin{align*}
			C_{\mu}(\delta) &= \sup \left\{  T \mid \frac{T - S_{\mu} - 1}{2 H_{\theta}(\mu)}\le \left(\sqrt{c(T, \delta)} + \sqrt{3 \log T}\right)^2 - \frac{3 \log S_{\mu}}{H_{\theta}(\mu) \max_{a \in \set{\THRESHOLD}}\Delta_a^2} \right\} \\
  &= \sup \{ t \mid t \le 2H_{\theta}(\mu) ( \sqrt{c(t, \delta)} + \sqrt{3 \log t})^2 + D_{\theta}(\mu) \} \: ,
		\end{align*}
  where $D_{\theta}(\mu) = S_{\mu} + 1 - \frac{6 \log S_{\mu}}{\max_{a \in \set{\THRESHOLD}}\Delta_a^2}$.
		It satisfies $\limsup_{\delta \to 0} \bE_{\nu}[\tau_{\delta}] / \log(1/\delta) \le 2 H_{\theta}(\mu)$.
\end{lemma}
\begin{proof}
		Let $S_{\mu}$ as in Lemma~\ref{lem:large_time_behavior_exist_good_arms}.
		Let $T > S_{\mu}$ such that $\cE_{T} \cap \{\tau_{\delta} > T\}$ holds true.
		Using Lemma~\ref{lem:large_time_behavior_exist_good_arms}, we know that $a_{t+1} \in \set{\THRESHOLD}$ for all $t \in (S_{\mu}, T]$.
		Direct summation yields that
		\[
		T - S_{\mu} = \sum_{a \in \set{\THRESHOLD}} \left( N_{a}(T) - N_{a}(S_{\mu}) \right) + \sum_{t \in (S_{\mu}, T]} \indic{a_{t+1} \notin  \set{\THRESHOLD} } =  \sum_{a \in \set{\THRESHOLD}}  (N_{a}(T) -   N_{a}(S_{\mu})) \: .
		\]
		At time $S_{\mu} + 1$, let $a_{1} \in \set{\THRESHOLD}$ as in Lemma~\ref{lem:large_time_behavior_exist_good_arms}, \ie such that $N_{a_{1}}(S_{\mu}+1) > \frac{2 f_{1}(S_{\mu}+1)}{(\mu_{a_{1}} - \theta)^2}$.
		Using that $f_{1}$ is increasing, we obtain
		\[
		\sum_{b \in \set{\THRESHOLD}}  N_{b}(S_{\mu}) \ge N_{a_{1}}(S_{\mu}+1) - 1 > \frac{2 f_{1}(S_{\mu} + 1)}{(\mu_{a_{1}} - \theta)^2} -1 \ge \frac{2 f_{1}(S_{\mu})}{\max_{a \in \set{\THRESHOLD}}(\mean{a} - \theta)^2} -1 \: .
		\]
  Let $g(S_{\mu}) = S_{\mu} - 2 f_{1}(S_{\mu})/\max_{a \in \set{\THRESHOLD}}\Delta_a^2 + 1$.
		Therefore, we have shown that $\sum_{a \in \set{\THRESHOLD}}  N_{a}(T) \ge 	T - g(S_{\mu})$. 
		Let $A_{\theta} = |\set{\THRESHOLD}| $ and $w \in \triangle_{A_{\theta}}$ such that $w_{a} = (\mean{a} - \theta)^{-2}H_{\theta}(\mu)^{-1}$ with $H_{\theta}(\mu)$ as in~Eq.~\eqref{eq:common_complexity}.
Using the pigeonhole principle, there exists $a_{0} \in \set{\THRESHOLD}$ such that $N_{a_{0}}(T) \ge w_{a_{0}}(T - g(S_{\mu})) = \Delta_{a_0}^{-2} H_{\theta}(\mu)^{-1} (T - g(S_{\mu}))$.
Let $E_{\mu} = \sup \left\{ T \mid T \le g(S_{\mu}) + 2H_{\theta}(\mu)f_{1}(T)  \right\} $.
Let $T > E_{\mu} $.
Then, we have $N_{a_{0}}(T) \ge \Delta_{a_0}^{-2} H_{\theta}(\mu)^{-1} (T - g(S_{\mu})) >  2f_{1}(T) \Delta_{a_0}^{-2}$, hence $\mu_{a_0}(T) > \theta$. 
Using that the condition of the stopping rule is not met at time $T$, we obtain
\begin{align*}
	\sqrt{2c(T, \delta)} \ge \max_{a \in \ARMS} W_{a}^{+}(T) \ge W_{a_0}^{+}(T) &= \sqrt{N_{a_0}(T)}(\expmean{a_0}{T} - \theta)_{+}  = \sqrt{N_{a_0}(T)}(\expmean{a_0}{T} - \theta) \: .
\end{align*}
Then, we obtain
\begin{align*}
	\sqrt{2c(T, \delta)} \ge \sqrt{N_{a_0}(T)}(\mu_{a_0} - \theta) - \sqrt{2f_{1}(T)}
	&\ge \sqrt{T - g(S_{\mu})} \sqrt{w_{a_{0}}(\mu_{a_0} - \theta)^2} - \sqrt{2f_{1}(T)} \\
	&= \sqrt{T - g(S_{\mu})} H_{\theta}(\mu)^{-1/2} - \sqrt{2f_{1}(T)} \: .
\end{align*}
The above can be rewritten as $T \le 2\left(\sqrt{c(T, \delta)} + \sqrt{f_{1}(T)}\right)^2 H_{\theta}(\mu) +  g(S_{\mu})$.
Using that $g(S_{\mu}) = S_{\mu} - \frac{2 f_{1}(S_{\mu})}{\max_{a \in \set{\THRESHOLD}}\Delta_a^2} + 1$, let us define
\[
	D_{\mu}(\delta) = \sup \left\{  T \mid \frac{T - S_{\mu} - 1}{2 H_{\theta}(\mu)}\le \left(\sqrt{c(T, \delta)} + \sqrt{f_{1}(T)}\right)^2 - \frac{ f_{1}(S_{\mu})}{H_{\theta}(\mu) \max_{a \in \set{\THRESHOLD}}\Delta_a^2} \right\} \: .
\]
It is direct to see that $D_{\mu}(\delta) \ge E_{\mu}  \ge  S_{\mu}$.
Therefore, we have shown that for $T \ge D_{\mu}(\delta) + 1$, we have $\cE_{T} \subset \{\tau_{>,\delta} \le T \}  \marc{\subseteq} \{\tau_{\delta} \le T \}$ \marc{since $\tau_{\delta} \deff \min\{\tau_{>,\delta}, \tau_{<,\delta}\} \le \tau_{>,\delta}$ almost surely by definition}.
Using Lemma~\ref{lem:lemma_1_Degenne19BAI}, we obtain $\bE_{\nu}[\tau_{\delta}] \le D_{\mu}(\delta) + K \zeta(s) + 1$.
Taking $s = 2$, using that $\zeta(2) = \pi^2/6$ and $f_{1}(T) = 3 \log T$ yields the second part of the result.
Using Lemma~\ref{lem:asymptotic_inversion_result}, direct manipulations show that $\limsup_{\delta \to 0} \frac{\bE_{\nu}[\tau_{\delta}]}{\log(1 / \delta)} \le	\limsup_{\delta \to 0} \frac{D_{\mu}(\delta)}{\log(1 / \delta)} \le 2 H_{\theta}(\mu)$.
According to Lemma~\ref{lem:lower_bound_GAI}, our result is weaker than asymptotic optimality when $|\set{\THRESHOLD}| \ge 2$. Lemma~\ref{lem:delta_correct_threshold} gives the $\delta$-correctness of the \hyperlink{APGAI}{APGAI} algorithm, since the recommendation rule of matches the one of Lemma~\ref{lem:delta_correct_threshold}.
\end{proof}

\subsection{Explicit Non Asymptotic Upper Bound}\label{app:sssec_explicit_ub} 

In the above, we have shown the following implicit upper bound on the sample complexity $\tau_\delta$ of the \hyperlink{APGAI}{APGAI} algorithm, namely $\bE_{\nu}[\tau_{\delta}] \le C_{\mu}(\delta) + K \pi^2/6 + 1$ with
\[ 
     C_{\mu}(\delta) \deff \sup \{ t \mid t \le 2H_{i_{\mu}}(\mu) ( \sqrt{c(t, \delta)} + \sqrt{3 \log t})^2 + D_{i_{\mu}}(\mu) \} \; ,
\]
where $i_{\mu} \deff 1 + (\theta - 1)\indic{\ARMS_{\theta}(\mu) \ne \emptyset}$.
Since $C_\mu(\delta)$ is defined implicitly, we provide an explicit upper bound by leveraging some (loose) approximations. 
Using that $(x+y)^2 \le 2(x^2 + y^2)$ and $\overline W_{-1}(y) \le x$ if and only if $y \le x - \log (x)$ (see Lemma~\ref{lem:property_W_lambert}), we obtain $C_{\mu}(\delta) \le$
\begin{align*}
    &\sup \{t \mid t \le 2 H_{i_{\mu}}(\mu) \overline{W}_{-1} \left(2\ln (K/\delta) +  4 \ln \ln (e^4 t ) + 1/2 \right) + 12 H_{i_{\mu}}(\mu) \log (t) + D_{i_{\mu}}(\mu) \} \\
    &\le \sup \left\{t \mid \frac{t - 12 H_{i_{\mu}}(\mu) \log (t) - D_{i_{\mu}}(\mu)}{2 H_{i_{\mu}}(\mu)} \le  2\ln \left(\frac{Ke^{1/4}}{\delta} \right) +\right.\\
    &\quad \left. \log \left( (4 + \log t)^4\frac{t - 12 H_{i_{\mu}}(\mu) \log (t) - D_{i_{\mu}}(\mu)}{2 H_{i_{\mu}}(\mu)}\right) \right\} \: .
\end{align*}
Numerically, we observe that $\frac{3}{2} \ln(x) + 7 \ge \ln\left( x (4 + \log x)^4 \right)$ for all $x \ge 0.0015$.
Since $C_{\mu}(\delta)  \ge 1$ and $\log \left( (4 + \log t)^4\frac{t - 12 H_{i_{\mu}}(\mu) \log (t) - D_{i_{\mu}}(\mu)}{2 H_{i_{\mu}}(\mu)}\right) \le \log \left( (4 + \log t)^4 t \right) - \ln(2 H_{i_{\mu}}(\mu))$, we obtain that
\begin{align*}
    &C_{\mu}(\delta) \le  \sup \left\{t \mid t   \le  4 H_{i_{\mu}}(\mu)\ln \left(\frac{Ke^{15/4}}{2 H_{i_{\mu}}(\mu)\delta} \right) + 15  H_{i_{\mu}}(\mu) \ln(t) + D_{i_{\mu}}(\mu) \right\} \\
    &\le  h_{1}\left(15  H_{i_{\mu}}(\mu) , 4 H_{i_{\mu}}(\mu)\ln \left(\frac{Ke^{15/4}}{2 H_{i_{\mu}}(\mu)\delta} \right) + D_{i_{\mu}}(\mu) \right) \: ,
\end{align*}
where the last inequality uses Lemma~\ref{lem:inversion_upper_bound} with $h_{1}$ is defined therein as $h_1(x,y) \deff x \overline W_{-1}(y/x+ \log(x))$. 
This upper bound is fully explicit since the function $h_{1}$ depends on $\bar W_{-1}$. 
Finally, we can use the approximation $\overline W_{-1}(x) \le x + \log(x) + \min (1/\sqrt{x},1/2)$ (see Lemma~\ref{lem:property_W_lambert}), hence
\[
    C_{\mu}(\delta) \le h\left(15  H_{i_{\mu}}(\mu) , 4 H_{i_{\mu}}(\mu) \left(\ln \left( K/ \delta \right) + 15/4 - 2 \ln(2 H_{i_{\mu}}(\mu))\right) + D_{i_{\mu}}(\mu) \right) 
\]
where $h(x,y) \deff y + x\log(x) + x\log(y/x + \log(x)) + x/2$.
$\qed$

\subsubsection{Discussion on Sub-optimal Upper Bound}
\label{app:sssec_discussion_suboptimality}

As discussed in Section~\ref{sec:FCGAI}, Theorem~\ref{thm:expected_sample_complexity_upper_bound} has a sub-optimal scaling when $\set{\THRESHOLD}(\mu) \ne \emptyset$.
Instead of $2\min_{a \in \set{\THRESHOLD}(\mu)} \Delta_{a}^{-2}$, our asymptotic upper bound on the expected sample complexity scales only as $2 H_{\theta}(\mu)$.
It is quite natural to wonder whether we could improve on this dependency, and whether $2\min_{a \in \set{\THRESHOLD}(\mu)} \Delta_{a}^{-2}$ is achievable by \hyperlink{APGAI}{APGAI}.
In the following, we provide intuition on why we could improve up to $2\max_{a \in \set{\THRESHOLD}(\mu)} \Delta_{a}^{-2}$, but not till $2\min_{a \in \set{\THRESHOLD}(\mu)} \Delta_{a}^{-2}$.

\textit{On the impossibility to achieve $2\min_{a \in \set{\THRESHOLD}(\mu)} \Delta_{a}^{-2}$.}
We argue that whenever $\set{\THRESHOLD}(\mu) \ne \argmax_{a \in \set{\THRESHOLD}(\mu)} \Delta_{a}$, there is no mechanism to avoid that the sampling rule of \hyperlink{APGAI}{APGAI} focuses all its samples on an arm $a \in \set{\THRESHOLD}(\mu) \setminus \argmax_{a \in \set{\THRESHOLD}(\mu)} \Delta_{a}$.
Therefore, it is not possible to achieve $2\min_{a \in \set{\THRESHOLD}(\mu)} \Delta_{a}^{-2}$.

For the sake of presentation, we consider the most simple case where this impossibility result occur.
Let $\nu$ be a two-arms instance with mean $\mu$ such that $\mu_{1} > \mu_{2} > \THRESHOLD = 0$.
Let $(X_{s})_{s \ge 1}$ and $(Y_{s})_{s \ge 1}$ be i.i.d. observations from $\nu_{1}$ and $\nu_{2}$.
\hyperlink{APGAI}{APGAI} initializes by sampling each arm once.
Let $\epsilon \in (0, \mu_{2} )$ and $T \in \mathbb N$ such that
\[
    T > n_{\epsilon}(T) \deff \sup \{t \mid \sqrt{t-1} \mu_{2} - 2\sqrt{\log T} \le \mu_{2} - \epsilon \} \: .
\]
By conditional independence, the event $\mathcal G_{\epsilon, T} = \{ X_{1} < \mu_{2} - \epsilon \le \min_{1\le s \le n_{\epsilon}(T) } Y_{s} \}$ has probability $\bP_{\nu}(\mathcal G_{\epsilon, T}) = \bP_{X \sim \nu_{1}}(X < \mu_{2} - \epsilon) (1 - \bP_{Y \sim \nu_{2}}(Y < \mu_{2} - \epsilon ) )^{n_{\epsilon}(T)}$.
Under $\mathcal G_{\epsilon, T}$, we have $a_{t + 1} = 2$ for all $2 \le t \le n_{\epsilon}(T)$, hence $N_{2}(t) = t - 1$ and $N_{1}(t) = 1$.
Let $\cE_{T}$ as in~Eq.~\eqref{eq:event_concentration_per_arm_aeps} for $s = 1$ and $\delta = 1$, \ie
\[
    \cE_{T} = \{ \forall a \in \{1,2\}, \forall t \le T, \: |\expmean{a}{t} - \mu_a| < \sqrt{4 \ln (T)/N_{a}(t)} \} \: .
\]
It satisfies $\bP_{\nu}(\cE_{T}^{\complement}) \le 2/T$.
We will show that by induction that $N_{2}(t) = t - 1$ and $N_{1}(t) = 1$ under $\cE_{T} \cap \mathcal G_{\epsilon, T}$.
Under $\mathcal G_{\epsilon, T}$, we know that the property holds for all $2 \le t \le n_{\epsilon}(T)$.
Suppose it is true at time $T > t > n_{\epsilon}(T)$, we will show that $a_{t+1} = 2$ hence it is true at time $t+1$.
Under $\cE_{T} \cap \mathcal G_{\epsilon, T}$, we have $\Wp{2}{t}  = \sqrt{N_{2}(t)}\expmean{2}{t}  > \sqrt{N_{2}(t)} \mu_{2} - 2\sqrt{\ln T} = \sqrt{t-1} \mu_{2} - 2\sqrt{\ln T} >  \mu_{2} - \epsilon \ge \Wp{1}{2} = \Wp{1}{t}$.
Therefore, we have $a_{t+1} = 2$. 
This concludes the proof by induction that, under $\cE_{T} \cap \mathcal G_{\epsilon, T}$, for all $t \le T$, $N_{2}(t) = t - 1 $ and $N_{1}(t) = 1$.
Since $\cE_{T}$ and $\mathcal G_{\epsilon, T}$ are both likely events, it is reasonable to expect $\cE_{T} \cap \mathcal G_{\epsilon, T}$ to be likely as well.
Under this likely event, we see that \hyperlink{APGAI}{APGAI} focuses its sampling allocation to the arm $2$ instead of the arm $1$.
The greediness of \hyperlink{APGAI}{APGAI} prevents it to switch the arm that is easiest to verify.

While the above argument considers only two arms and is not formally proven, it gives some intuition as regards what prevents \hyperlink{APGAI}{APGAI} from reaching $2\min_{a \in \set{\THRESHOLD}(\mu)} \Delta_{a}^{-2}$. 
It is not possible to recover from one unlucky first draw for the best arm if a sub-optimal arm has no unlucky first draws.
Formally proving such a negative result is an interesting direction for future work.

\textit{Towards reaching $2\max_{a \in \set{\THRESHOLD}(\mu)} \Delta_{a}^{-2}$ asymptotically.}
We argue that \hyperlink{APGAI}{APGAI} focuses its sampling allocation to only one of the good arm $a \in \set{\THRESHOLD}(\mu)$, after a long enough time.
Therefore, it should be possible to achieve $2\max_{a \in \set{\THRESHOLD}(\mu)} \Delta_{a}^{-2}$.

Suppose towards contradiction that
\[
\exists (a_{1},a_{2}) \in \set{\THRESHOLD}(\mu)^2, \quad \min_{a \in \{a_{1}, a_{2}\}}N_{a}(T) \to_{T \to + \infty} + \infty \: .
\]
Let $S_{\mu}$ as in Lemma~\ref{lem:large_time_behavior_exist_good_arms}.
Let $T > S_{\mu}$ such that $\cE_{T} \cap \{\tau_{\delta} > T\}$ holds true.
In the proof of Lemma~\ref{lem:expected_sample_complexity_exist_good_arms}, we have shown that
\[
    \max_{a \in \ARMS} \Wp{a}{T} \ge \sqrt{T - g(S_{\mu})} H_{\theta}(\mu)^{-1/2} - \sqrt{2f_{1}(T)} \: .
\]
At time $S_{\mu}+1$, we have $\max_{a \in \ARMS} \Wp{a}{S_{\mu}+1} \ge \Wp{a_{1}}{S_{\mu}+1} $.
Since the transportation costs are independent to the other arms, we will show that sampling two arms an infinite number of times implies that the transportation costs are bounded.
Given that we have shown they are growing towards $+ \infty$, this is a contradiction.
Using our assumption that $\min_{a \in \{a_{1}, a_{2}\}}N_{a}(T) \to_{T \to + \infty} + \infty$, we have that there exists an infinite number of intervals $(t^{L}_{i}, t^{U}_{i})_{i \in \mathbb N}$ such that $a_{t+1} = a_{1}$ for all $t \in \bigcup_{i \in \mathbb N}[t^{L}_{i}, t^{U}_{i})$, otherwise $a_{t+1} \ne a_{1}$.
Let $i \in \mathbb N$.
Using that $a_{1}$ is the only arm that is sampled in $[t^{L}_{i}, t^{U}_{i})$ and that is not sampled at $t^{U}_{i}$, we obtain that 
\[
    \Wp{a_{1}}{t^{L}_{i}} \ge \max_{a \in \ARMS \setminus \{a_{1}\} } \Wp{a}{t^{L}_{i}} = \max_{a \in \ARMS \setminus \{a_{1}\} } \Wp{a}{t^{U}_{i}} \ge \Wp{a_{1}}{t^{U}_{i}} \: .
\]
Since it is not sampled until $t^{L}_{i+1}$, we obtain that $\Wp{a_{1}}{t^{U}_{i}} = \Wp{a_{1}}{t^{L}_{i+1}}$.
By induction is is direct to see that
\[
    \Wp{a_{1}}{S_{\mu}+1} \ge \max_{i, t^{L}_{i} \ge S_{\mu}+1} \Wp{a_{1}}{t^{L}_{i}} \ge \sqrt{t^{L}_{i} - g(S_{\mu})} H_{\theta}(\mu)^{-1/2} - \sqrt{2f_{1}(t^{L}_{i})} \: .
\]
Since the right-hand side converges towards infinity, there is a contradiction.
Therefore, there exists a unique arm $a \in \set{\THRESHOLD}(\mu)$ such that $N_{a}(T) \to_{T \to + \infty} + \infty$.

While the above argument is not formally proven, it gives some intuition as regards why \hyperlink{APGAI}{APGAI} can reach $2\max_{a \in \set{\THRESHOLD}(\mu)} \Delta_{a}^{-2}$. 
It is not possible to sample two good arms an infinite number of times since it would imply that the transportation costs are simultaneously bounded and converge towards infinity.

\section{Concentration Results}
\label{app:concentration_results}

In Appendix~\ref{app:proof_lem_delta_correct_threshold}, we prove the $\delta$-correctness of the GLR stopping rule~Eq.~\eqref{eq:stopping_rule} with threshold~Eq.~\eqref{eq:stopping_threshold} (Lemma~\ref{lem:delta_correct_threshold}).
Appendix~\ref{app:sequence_concentration_events} gathers sequence of concentration events which are used for our proofs. 

\subsection{Analysis of the GLR Stopping Rule: Proof of Lemma~\ref{lem:delta_correct_threshold}}
\label{app:proof_lem_delta_correct_threshold}

Proving $\delta$-correctness of a GLR stopping rule is done by leveraging concentration results.
In particular, we build upon Lemma $28$~\cite{jourdan_2022_DealingUnknownVariance}.
Lemma~\ref{lem:uniform_upper_lower_tails_concentration_mean} is obtained as a Corollary of Lemma $28$ from~\cite{jourdan_2022_DealingUnknownVariance} by using a union bound over arms $a \in \ARMS$.
While it was only proven for Gaussian distributions, the concentration results also holds for sub-Gaussian distributions with variance $\sigma^2 = 1$ since we have $\bE_{X}[\exp(sX)] \le \exp (\lambda^2/2)$ for all $\lambda \in \rR$.
\begin{lemma}[Lemma 28 in \cite{jourdan_2022_DealingUnknownVariance}] \label{lem:uniform_upper_lower_tails_concentration_mean}
	Let $s> 1$ and $\delta \in (0,1)$.
 Let $\overline{W}_{-1}(x)  = - W_{-1}(-e^{-x})$ for all $x \ge 1$ (see Lemma~\ref{lem:property_W_lambert}), where $W_{-1}$ is the negative branch of the Lambert $W$ function.
	Let
	\[
	c(T,\delta) = \frac{1}{2}\overline{W}_{-1} \left(2\ln \left( K/\delta\right) +  2s \ln(2s + \ln T ) + 2g(s)  \right) \: ,
	\]
	with $g(s) = \ln(\zeta(s)) + s(1 -  \ln(2s)) + 1/2$ and $\zeta$ be the Riemann $\zeta$ function.
	Then,
	\begin{align*}
		\bP \left( \exists T \in \nN, \: \exists a \in \ARMS, \: \sqrt{N_{a}(T)}|\mu_a(T) - \mean{a}| > \sqrt{2 c(T, \delta)} \right) \le \delta \: .
	\end{align*}
\end{lemma}

We distinguish between the two cases $\set{\THRESHOLD} = \emptyset$ and $\set{\THRESHOLD} \ne \emptyset$.
For the sake of simplicity, we use Lemma~\ref{lem:uniform_upper_lower_tails_concentration_mean} for $s = 2$ and use that $2g(2) = 2\ln(\pi^2/6) + 5 - 4 \ln(4) \le 1/2 $, which can be easily checked numerically.

{\bf Case 1.}
When $\set{\THRESHOLD} = \emptyset$, we have to show $\bP_{\nu}(\tau_{\delta} < + \infty, \: \hat a_{\tau_{\delta}} \ne \emptyset) \le \delta$.
We recommend $\GUESS{T} \ne \emptyset$ only when $\max_{a \in \ARMS} \expmean{a}{t} > \theta$.
In that case, we have $\GUESS{T} \in \argmax_{a \in \ARMS} W^{+}_{a}(T)$ where $W^{+}_{a}(T) = \sqrt{N_{a}(T)}(\mu_a(T) - \theta)_{+}$.
Therefore, direct manipulations yield that
\begin{align*}
	& \bP_{\nu}(\tau_{\delta} < + \infty, \: \hat a_{\tau_{\delta}} \ne \emptyset) \\
	& \le \bP \left( \exists T \in \nN, \: \exists a \in \ARMS, \: \expmean{a}{t} > \theta, \: \sqrt{N_{a}(T)}(\expmean{a}{T} - \theta)_{+} \ge \sqrt{2c(T,\delta)} \right)\\
	& \le \bP \left( \exists T \in \nN, \: \exists a \in \ARMS, \: \sqrt{N_{a}(T)}(\expmean{a}{T} - \mean{a}) + \sqrt{N_{a}(T)}(\mu_a - \theta) \ge \sqrt{2c(T,\delta)} \right)\\
	& \le \bP \left( \exists T \in \nN, \: \exists a \in \ARMS, \: \sqrt{N_{a}(T)}(\expmean{a}{T} - \mean{a}) \ge \sqrt{2c(T,\delta)} \right) \le \delta/2 \: .
\end{align*}
The second inequality uses that $\expmean{a}{t} > \theta$ before dropping this condition.
The third inequality uses that $\mu_a - \theta \le 0$ since $\set{\THRESHOLD} = \emptyset$.
The last inequality uses Lemma~\ref{lem:uniform_upper_lower_tails_concentration_mean}.

{\bf Case 2.}
When $\set{\THRESHOLD} \ne \emptyset$, we have to show $\bP_{\nu}(\{\tau_{\delta} < + \infty\} \cap (\{\hat a_{\tau_{\delta}} = \emptyset \} \cup \{\hat a_{\tau_{\delta}} \notin \set{\THRESHOLD} \}) ) \le \delta$.
As above, when we recommend $\GUESS{T} \notin \set{\THRESHOLD}$, direct manipulations yield that
\begin{align*}
	& \bP_{\nu}(\tau_{\delta} < + \infty, \: \hat a_{\tau_{\delta}} \notin \set{\THRESHOLD}) \\
	& \le \bP \left( \exists T \in \nN, \: \exists a \notin \set{\THRESHOLD}, \: \expmean{a}{t} > \theta, \: \sqrt{N_{a}(T)}(\expmean{a}{T} - \theta)_{+} \ge \sqrt{2c(T,\delta)} \right)\\
	& \le \bP \left( \exists T \in \nN, \: \exists a \notin \set{\THRESHOLD}, \: \sqrt{N_{a}(T)}(\expmean{a}{T} - \mean{a}) + \sqrt{N_{a}(T)}(\mu_a - \theta) \ge \sqrt{2c(T,\delta)} \right)\\
	& \le \bP \left( \exists T \in \nN, \: \exists a  \notin \set{\THRESHOLD}, \: \sqrt{N_{a}(T)}(\expmean{a}{T} - \mean{a}) \ge \sqrt{2c(T,\delta)} \right) \le \delta/2 \: .
\end{align*}
The third inequality uses that $\mu_a - \theta \le 0$ since $a \notin \set{\THRESHOLD}$.

Similarly, we recommend $\GUESS{T} = \emptyset$ only when $\max_{a \in \ARMS} \expmean{a}{t} \le \theta$.
In that case, we consider $W^{-}_{a}(T) = \sqrt{N_{a}(T)}(\theta - \mu_a(T) )_{+}$.
Therefore, direct manipulations yield that
\begin{align*}
	& \bP_{\nu}(\tau_{\delta} < + \infty, \: \hat a_{\tau_{\delta}} = \emptyset) \\
	& \le \bP \left( \exists T \in \nN, \: \forall a \in \ARMS, \: \expmean{a}{t} \le \theta, \: \sqrt{N_{a}(T)}(\theta - \expmean{a}{T})_{+} \ge \sqrt{2c(T,\delta)} \right)\\
	& \le \bP \left( \exists T \in \nN, \: \forall a \in \set{\THRESHOLD}, \: \sqrt{N_{a}(T)}(\theta - \mu_a) + \sqrt{N_{a}(T)}(\mean{a} - \expmean{a}{T}) \ge \sqrt{2c(T,\delta)} \right)\\
	& \le \bP \left( \exists T \in \nN, \: \forall a \in \set{\THRESHOLD}, \: \sqrt{N_{a}(T)}(\mu_a - \expmean{a}{T} ) \ge \sqrt{2c(T,\delta)} \right) \le \delta/2 \: .
\end{align*}
The second inequality uses that $\expmean{a}{t} \le \theta$ before dropping this condition, and restrict to $a \in \set{\THRESHOLD}$.
The third inequality uses that $\mu_a - \theta > 0$ since $a \in \set{\THRESHOLD}$.
The last inequality uses Lemma~\ref{lem:uniform_upper_lower_tails_concentration_mean}.
$\qed$

\subsection{Sequence of Concentration Events}
\label{app:sequence_concentration_events}

Appendix~\ref{app:sequence_concentration_events} provides sequence of concentration events which are used for our proofs. 
Lemma~\ref{lem:sub_Gaussian_concentration} is a standard concentration result for sub-Gaussian distribution, hence we omit the proof.
\begin{lemma} \label{lem:sub_Gaussian_concentration}
	Let $X$ be an observation from a sub-Gaussian distribution with mean $0$ and variance $\sigma^2 = 1$.
	Then, for all $\delta \in (0,1]$, $\mathbb{P}_{X}\left( \left| X \right| \ge \sqrt{2 \log(1/\delta)}\right) \le \delta $.
\end{lemma}

Lemma~\ref{lem:concentration_per_arm_gau_aeps} gives a sequence of concentration events under which the empirical means are close to their true values.
\begin{lemma} \label{lem:concentration_per_arm_gau_aeps}
	Let $\delta \in (0,1]$ and $s \ge 0$. For all $T > K$, let us defined
	\begin{equation} \label{eq:event_concentration_per_arm_aeps}
		 \cE_{T,\delta} = \{ \forall a \in \ARMS, \forall t \le T, \: |\expmean{a}{t} - \mu_a| < \sqrt{2 f_1(T,\delta)/N_{a}(t)} \} \: .
	\end{equation}
 with $f_1(T,\delta) =  \log(1/\delta) + (1+s) \log T $.
	Then, for all $T > K$, $\bP_{\nu}((\cE_{T,\delta})^{\complement}) \le \frac{K\delta}{T^{s}}$.
\end{lemma}
\begin{proof}
	Let $(X_s)_{s \in [T]}$ be i.i.d. observations from one sub-Gaussian distribution with mean $0$ and variance $\sigma^2 = 1$.
	Then, $\frac{1}{m}\sum_{i=1}^{m} X_i$ is sub-Gaussian with mean $0$ and variance $\sigma^2 = 1/m$.
	By union bound over $\ARMS$ and over $m \in [T]$, we obtain
	\begin{align*}
		& \mathbb{P}_{\mu}\left( \exists a \in \ARMS, \exists t \le T, \: |\expmean{a}{t} - \mu_a| < \sqrt{\frac{2 f_1(T,\delta)}{N_{a}(t)}} \right) \\
		&\le \sum_{a \in \ARMS} \sum_{m \in [T]} \mathbb{P}\left( \left| \frac{1}{m}\sum_{s \in [m]} X_s \right| \ge \sqrt{\frac{2f_{1} (T, \delta)}{m}}\right) \le \delta \sum_{a \in \ARMS} \sum_{m \in [T]} T^{-(1+s)} = K \delta T^{-s} \: ,
	\end{align*}
	where we used that $\expmean{a}{t} - \mu_a = \frac{1}{N_{a}(t)}\sum_{s=1}^{t} \indic{a_s = a} X_{s,a}$ and concentration results for sub-Gaussian observations (Lemma~\ref{lem:sub_Gaussian_concentration}).
\end{proof}

Lemma~\ref{lem:concentration_per_arm_gau_improved} provides concentration results on the empirical means, which are tighter than the one obtained in Lemma~\ref{lem:concentration_per_arm_gau_aeps}.
\begin{lemma} \label{lem:concentration_per_arm_gau_improved}
	Let $\delta \in (0,1]$ and $s \ge 0$.
 Let $\overline{W}_{-1}(x)  = - W_{-1}(-e^{-x})$ for all $x \ge 1$ (see Lemma~\ref{lem:property_W_lambert}), where $W_{-1}$ is the negative branch of the Lambert $W$ function.
	For all $T > K$, 
	\begin{equation} \label{eq:concentration_threshold_per_arm_improved}
	\tilde f_1(T, \delta) =  \frac{1}{2}\overline{W}_{-1}(2 \log(1/\delta) + 2s\log T  + 2 \log (2 + \log T ) + 2 ) \: , 
\end{equation}
	\begin{equation} \label{eq:event_concentration_per_arm_improved}
		 \text{and} \quad \tilde \cE_{T, \delta} = \{ \forall a \in \ARMS, \forall t \le T, \: |\expmean{a}{t} - \mu_a| < \sqrt{2 \tilde f_1(T,\delta)/N_{a}(t)} \} \: .
	\end{equation}
	Then, for all $T > K$, $\bP_{\nu}((\tilde \cE_{T, \delta})^{\complement}) \le \frac{K\delta}{T^{s}}$.
\end{lemma}
\begin{proof}
	Let $(X_s)_{s \in [T]}$ be i.i.d. observations from one sub-Gaussian distribution with mean $0$ and variance $\sigma^2 = 1$.
	Let $S_t = \sum_{s\in [t]} X_{s}$.
To derive the concentration result, we use peeling.

Let $\eta > 0$, $\gamma > 0$ and $D = \lceil \frac{\log(T)}{\log(1+\eta)}\rceil$.
For all $i \in [D]$, let $N_{i} = (1+\eta)^{i-1}$.
For all $i \in [D]$, we define the family of priors $f_{N_i, \gamma}(x) = \sqrt{\frac{\gamma  N_i}{2 \pi }} \exp\left( - \frac{x^2\gamma N_i}{2}\right)$ with weights $w_{i} = \frac{1}{D}$ and process $\overline M(t) = \sum_{i \in [D]} w_i \int f_{N_i, \gamma}(x) \exp \left( x S_{t}- \frac{1}{2} x^2 t \right) \,dx $, which satisfies $\overline M(0) = 1$.
		It is direct to see that $M(t) = \exp \left( x S_t- \frac{1}{2} x^2  t \right)$ is a non-negative supermartingale since sub-Gaussian distributions with mean $0$ and variance $\sigma^2 = 1$ satisfy $\bE_{X}[\exp(sX)] \le \exp (\lambda^2/2)$ for all $\lambda \in \rR$.
		By Tonelli's theorem, then $\overline M(t)$ is also a non-negative supermartingale of unit initial value.

		Let $i \in [D]$ and consider $t\in [N_{i}, N_{i+1})$. For all $x$, $f_{N_i, \gamma}(x) \geq \sqrt{\frac{N_i}{t}} f_{t, \gamma}(x) \geq \frac{1}{\sqrt{1+\eta}} f_{t, \gamma}(x)$.
		Direct computations shows that
		\begin{align*}
			\int f_{t, \gamma}(x) \exp \left( x S_{t}- \frac{1}{2} x^2 t \right)  \,dx = \frac{1}{\sqrt{1+\gamma^{-1}}} \exp \left( \frac{S_{t}^2}{2(1+\gamma)t}\right) \: .
		\end{align*}
		Minoring $\overline M(t)$ by one of the positive term of its sum, we obtain
		\begin{align*}
			\overline M(t) \geq \frac{1}{D} \frac{1}{\sqrt{(1+\gamma^{-1})(1+\eta)}} \exp \left( \frac{S_{t}^2}{2(1+\gamma)t}\right) \: ,
		\end{align*}

		Using Ville's maximal inequality for non-negative supermartingale, we have that with probability greater than $1-\delta$, $\ln\overline M(t) \leq \ln \left(  1/\delta\right)$. Therefore, with probability greater than $1-\delta$, for all $i \in [D]$ and $t \in [N_{i}, N_{i+1})$,
		\begin{align*}
			S_{t}^2/t &\leq (1+\gamma) \left( 2\ln \left(  1/\delta\right) +  2\ln D + \ln(1+\gamma^{-1}) + \ln (1+\eta) \right) \: .
		\end{align*}
		Since this upper bound is independent of $t$, we can optimize it and choose $\gamma$ as in Lemma~\ref{lem:lemma_A_3_of_Remy}.

\begin{lemma}[Lemma A.3 in \cite{degenne_2019_ImpactStructureDesign}] \label{lem:lemma_A_3_of_Remy}
	For $a,b\geq 1$, the minimum of $f(\eta)=(1+\eta)(a+\ln(b+\frac{1}{\eta}))$ is attained at $\eta^\star$ such that $f(\eta^\star) \leq 1-b+\overline{W}_{-1}(a+b)$.	If $b=1$, then there is equality.
\end{lemma}

		Therefore, with probability greater than $1-\delta$, for all $i \in [D]$ and $t \in [N_{i}, N_{i+1})$,
		\begin{align*}
			\frac{S_{t}^2}{t} &\leq \overline{W}_{-1}\left( 1 + 2\ln \left(  1/\delta\right) + 2 \ln D + \ln (1+\eta)\right) \\
			&\leq \overline{W}_{-1}\left( 1 + 2\ln \left(  1/\delta\right) + 2 \ln\left(\ln(1+\eta) + \ln T \right)   - 2 \ln \ln (1+\eta)+ \ln (1+\eta)\right) \\
			&= \overline{W}_{-1}\left( 2\ln \left( 1/\delta\right) + 2 \ln\left(2 + \ln T\right)   + 3 -2 \ln 2\right)
		\end{align*}
		The second inequality is obtained since $D \leq 1+ \frac{\ln T}{\ln(1+\eta)}$.
		The last equality is obtained for the choice $\eta^\star = e^{2} - 1$, which minimizes $\eta \mapsto \ln (1+\eta) - 2 \ln(\ln(1+\eta))$.
		Since $[T] \subseteq \bigcup_{i\in [D]}[N_{i}, N_{i+1})$ and $N_{a}(t) (\expmean{a}{t} - \mu_a) = \sum_{s \in [N_{a}(t)]} X_{s,a}$ (unit-variance), this yields
\begin{align*}
\mathbb{P}\left( \exists m \le T, \left|\frac{1}{m}\sum_{s=1}^m X_s \right|\ge
\sqrt{\frac{1}{m}\overline{W}_{-1} \left(  2\log(1/\delta) + 2 \log(2+\log(T)) + 3 - 2\log 2\right) } \right) \le \delta \: .
\end{align*}
Since $3 - 2\log 2 \le 2$ and $\overline{W}_{-1}$ is increasing, taking $\delta T^{-s}$ instead of $\delta$ yields
\begin{align*}
\mathbb{P}_{\mu}\left( \exists t \le T, \sqrt{N_{a}(t)}\left|\expmean{a}{t} - \mu_a \right| \ge
\sqrt{2 \tilde f_1(T,\delta)} \right) \le \delta T^{-s} \: .
\end{align*}
Doing a union bound over arms yields the result.
\end{proof}

\section{\clemence{Inversion Lemmas and Other Technical Results}} \label{app:technicalities}

Appendix~\ref{app:technicalities} gathers existing and new technical results which are used for our proofs.

\textit{Methodology.}
Lemma~\ref{lem:lemma_1_Degenne19BAI} is a standard result to upper bound the expected sample complexity of an algorithm, \eg see Lemma 1 in \cite{Degenne19GameBAI}.
This is a key method extensively used in the literature.
\begin{lemma} \label{lem:lemma_1_Degenne19BAI}
	Let $(\cE_{t})_{t > K}$ be a sequence of events and $T_{\mu}(\delta) > K $ be such that for $T \ge T_{\mu}(\delta)$, $\cE_{T} \subseteq \{\tau_{\delta} \le T\}$. Then, $\bE_{\nu}[\tau_{\delta}] \le T_{\mu}(\delta)  + \sum_{T > K} \bP_{\nu}(\cE_{T}^{\complement})$.
\end{lemma}
\begin{proof}
	Since the random variable $\tau_{\delta}$ is positive and $\{\tau_{\delta} > T\} \subseteq \cE_T^{\complement}$ for all $T \ge T_{\mu}(\delta)$, we have $\bE_{\nu}[\tau_{\delta}] = \sum_{T \ge 0} \bP_{\nu}(\tau_{\delta} > T) \le T_{\mu}(\delta)  + \sum_{T \ge T_{\mu}(\delta)} \bP_{\nu}(\cE_{T}^{\complement})$, which concludes the proof.
\end{proof}

\textit{Inversion results.}
Lemma~\ref{lem:property_W_lambert} gathers properties on the function $\overline{W}_{-1}$, which is used in the literature to obtain concentration results.
\begin{lemma}[\cite{jourdan_2022_DealingUnknownVariance}] \label{lem:property_W_lambert}
	Let $\overline{W}_{-1}(x)  = - W_{-1}(-e^{-x})$ for all $x \ge 1$, where $W_{-1}$ is the negative branch of the Lambert $W$ function.
	The function $\overline{W}_{-1}$ is increasing on $(1, +\infty)$ and strictly concave on $(1, + \infty)$.
	In particular, $\overline{W}_{-1}'(x) = \left(1-\frac{1}{\overline{W}_{-1}(x)} \right)^{-1}$ for all $x > 1$.
	Then, for all $y \ge 1$ and $x \ge 1$, $\overline{W}_{-1}(y) \le x$ if and only if $ y \le x - \ln(x)$.
	Moreover, for all $x > 1$, $x + \log(x) \le \overline{W}_{-1}(x) \le x + \log(x) + \min \left\{ \frac{1}{2}, \frac{1}{\sqrt{x}} \right\}$.
\end{lemma}

Lemma~\ref{lem:inversion_PoE} is an inversion result to upper bound a probability which is implicitly defined based on times that are implicitly defined.
\begin{lemma}\label{lem:inversion_PoE}
		Let $\overline{W}_{-1}$ defined in Lemma~\ref{lem:property_W_lambert}.
		Let $A,B,C,E,\alpha,\beta > 0$ and $D_{A, B, C, E, \alpha, \beta}(\delta) = \sup \left\{ x \mid \: x \le \frac{A}{\alpha} \overline{W}_{-1}\left(\alpha \left(\log (1/\delta) + C \log (\beta + \log x) + E \right)\right) + B \right\}$.
		Then,
		\begin{align*}
		&\inf \{ \delta \mid x >  D_{A, B, C, E, \alpha, \beta}(\delta)\} \le e^{E} \left( \alpha\frac{x - B}{A} \right)^{1/\alpha} (\beta + \log x)^{C} \exp \left(-\frac{x-B}{A} \right)  \:  .
		\end{align*}
\end{lemma}
\begin{proof}
	Using Lemma~\ref{lem:property_W_lambert}, direct manipulations yield that
	\begin{align*}
		x > D_{A, B, C, E, \alpha, \beta}(\delta) \:  &\iff \:  \alpha\frac{x - B}{A} > \overline{W}_{-1}\left(\alpha \left(\log (1/\delta) + C \log (\beta + \log x) + E \right)\right) \\
		&\iff \:  \frac{x - B}{A} - \frac{1}{\alpha}\log \left( \alpha\frac{x - B}{A} \right) > \log (1/\delta) + C \log (\beta + \log x) + E \\
		&\iff \:  \delta < e^{E} \left( \alpha\frac{x - B}{A} \right)^{1/\alpha} (\beta + \log x)^{C} \exp \left(-\frac{x-B}{A} \right) \: .
	\end{align*}
\end{proof}

Lemma~\ref{lem:inversion_upper_bound} is an inversion result to upper bound an implicitly defined time.
\begin{lemma} \label{lem:inversion_upper_bound}
	Let $\overline{W}_{-1}$ defined in Lemma~\ref{lem:property_W_lambert}.
	Let $A > 0$, $B > 0$ such that $B/A + \log A >  1$ and $C(A, B) = \sup \left\{ x \mid \: x < A \log x + B \right\}$.
	Then, $C(A,B) < h_{1}(A,B)$ with $h_{1}(z,y) = z \overline{W}_{-1} \left(y/z  + \log z\right)$
\end{lemma}
\begin{proof}
	Since $B/A + \log A >  1$, we have $C(A, B) \ge A$, hence
	\[
	C(A, B) = \sup \left\{ x \mid \: x < A \log (x) + B \right\} = \sup \left\{ x \ge A \mid \: x < A \log (x) + B \right\} \: .
	\]
	Using Lemma~\ref{lem:property_W_lambert} yields that
	\begin{align*}
		x \ge A \log x + B  \: \iff \: \frac{x}{A} - \log \left( \frac{x}{A} \right) \ge \frac{B}{A} + \log A \: \iff \: x \ge A \overline{W}_{-1} \left( \frac{B}{A} + \log A \right) \: .
	\end{align*}
\end{proof}

Lemma~\ref{lem:asymptotic_inversion_result} is an inversion result to asymptotically upper bound an implicit time.
\begin{lemma} \label{lem:asymptotic_inversion_result}
	Let $B > 0$ and $A > 0$
	\[
		D(\delta) = \sup \left\{  T \mid \frac{T - B}{A} \le \left( \sqrt{\frac{1}{2}\overline{W}_{-1} \left(2\ln \left( 2K/\delta\right) +  4 \ln(4 + \ln T ) + 1 \right)} + \sqrt{3 \log T}\right)^2     \right\}
	\]
	Then, we have $\limsup_{\delta \to 0} D(\delta)/\log(1 / \delta) \le A$.
\end{lemma}
\begin{proof}
	Direct manipulations yields that
	\begin{align*}
		&\frac{T - B}{A} > \left( \sqrt{\frac{1}{2}\overline{W}_{-1} \left(2\ln \left( 2K/\delta\right) +  4 \ln(4 + \ln T ) + 1 \right)} + \sqrt{3 \log T}\right)^2 \\
		\iff \quad & 2\left(\sqrt{\frac{T - B}{A}} - \sqrt{3 \log T} \right)^2 >  \overline{W}_{-1} \left(2\ln \left( 2K/\delta\right) +  4 \ln(4 + \ln T ) + 1 \right)\\
		\iff \quad & \ln (1/\delta) < \frac{T - B}{A} - 6 \log T\sqrt{\frac{T - B}{A}} + 3 \log T   -  \ln \left(\sqrt{\frac{T - B}{A}} - \sqrt{3 \log T} \right) \\
		& \qquad \qquad -  2 \ln(4 + \ln T )  - \frac{1+3\ln 2}{2} -  \ln K  \: .
	\end{align*}
	Let $\gamma > 0$. There exists $T_{\gamma}$, which depends on $(B,A)$, such that
	\begin{align*}
		&\frac{T - B}{A} - 6 \log T\sqrt{\frac{T - B}{A}} + 3 \log T   -  \ln \left(\sqrt{\frac{T - B}{A}} - \sqrt{3 \log T} \right) \\
		&  \qquad -  2 \ln(4 + \ln T )  - \frac{1+3\ln 2}{2} -  \ln K \ge \frac{T}{A(1+\gamma)} \: .
	\end{align*}
	Therefore, we have $D(\delta) \le T_{\gamma} + C(\delta)$ where $C(\delta) = \sup \left\{  T \mid \frac{T}{A(1+\gamma)} \le \log (1/\delta)\right\}$.
	Then, we have
	\[
	\limsup_{\delta \to 0} \frac{C(\delta)}{\log(1 / \delta)} \le A(1+\gamma) \quad \text{hence} \quad 	\limsup_{\delta \to 0} \frac{D(\delta)}{\log(1 / \delta)} \le A(1+\gamma) \: .
	\]
	Letting $\gamma$ goes to $0$ yields the result.
\end{proof}

\section{Details on the Experimental Study}
\label{app:details_experiments}

In this appendix, we detail the benchmark instances in Appendix~\ref{app:ssec_considered_instances} and the implementation details in Appendix~\ref{app:ssec_impelementation_details}.
Then, we provide supplementary experiments to assess the performance of the \hyperlink{APGAI}{APGAI} algorithm on the empirical error both for fixed-budget (Appendix~\ref{app:ssec_supp_emp_error_FB}) and anytime algorithms (Appendix~\ref{app:ssec_supp_emp_error_anytime}), as well as on the empirical stopping time (Appendix~\ref{app:ssec_supp_emp_stop_time}).

\subsection{Benchmark Instances}
\label{app:ssec_considered_instances}

 We detail our real-life instance based on an outcome scoring application in Appendix~\ref{app:sssec_reallife} (\textsc{RealL} in Tables~\ref{tab:instances_arms} and~\ref{tab:instances_constants_value}), as well as synthetic instances in Appendix~\ref{app:sssec_synthetic}.
 For all the experiments considered below, the parameters and the mean vectors are respectively displayed in Table~\ref{tab:instances_constants} and Table~\ref{tab:instances_arms}. 
 The numerical values for the difficulties are reported in Table~\ref{tab:instances_constants_value}. 

\begin{table}
     \centering
     \scalebox{0.9}{\begin{tabular}{l r r r r r r r r r r}
    \toprule
    Name   &  \textsc{Thr1} & \textsc{Thr2}  & \textsc{Thr3} & \textsc{Med1}  & \textsc{Med2} &  \textsc{IsA1} & \textsc{NoA1} & \textsc{IsA2} & \textsc{NoA2} &   \textsc{RealL}\\
\midrule
$K$ & $10$ & $6$ & $10$ & $5$ & $7$ & $10$ & $5$ & $7$ & $4$ & $18$\\
 $\theta$ & $0.5$ & $0.35$ & $0.5$ & $0.5$ & $1.2$ & $0$ & $0$ & $0$ & $0$ & $0.5$  \\
 $|\set{\THRESHOLD}|$ & $5$ & $3$ & $3$ & $1$ & $2$ & $5$ & $0$ & $3$ & $0$ & $6$\\
         \bottomrule
     \end{tabular}}
     \caption{Parameters in synthetic and real-life instances. } 
     \label{tab:instances_constants}
 \end{table}

 \begin{table}
     \centering     
     \scalebox{0.9}{\begin{tabular}{l r r r r r r r r r r r r r}
    \toprule
&  &  & Arms & & &  & & & & \\
\midrule
& $1$ & $2$ & $3$ & $4$ & $5$ & $6$ & $7$ & $8$ & $9$ & $10$ \\
\midrule
         \textsc{Thr1} & $0.9$ & $0.9$ & $0.9$ & $0.65$ & $0.55$ & $0.45$ & $0.35$ & $0.1$ & $0.1$ & $0.1$\\
         \textsc{Thr2} & $0.6$ & $0.5$ & $0.4$ & $0.3$ & $0.2$ & $0.1$ & $-$ & $-$ & $-$ & $-$\\
         \textsc{Thr3} & $0.55$ & $0.55$ & $0.55$ & $0.45$ & $0.45$ & $0.45$ & $0.45$ & $0.45$ & $0.45$ & $0.45$ \\
         \textsc{Med1} & $0.537$ & $0.469$ & $0.465$ & $0.36$ & $0.34$ & $-$ & $-$ & $-$ & $-$ & $-$\\
         \textsc{Med2} & $1.8$ & $1.6$ & $1.1$ & $1$ & $0.7$ & $0.6$ & $0.5$ & $-$ & $-$ & $-$\\
         \textsc{IsA1} & $0.5$ & $0.39$ & $0.28$ & $0.17$ & $0.06$ & $-0.06$ & $-0.17$ & $-0.28$ & $-0.39$ & $-0.50$ \\
         \textsc{NoA1} & $-0.5$ & $-0.62$ & $-0.75$ & $-0.88$ & $-1$ & $-$ & $-$ & $-$ & $-$ & $-$ \\
         \textsc{IsA2} & $1.0$ & $0.5$ & $0.1$ & $-0.1$ & $-0.4$ & $-0.5$ & $-0.6$ & $-$ & $-$ & $-$\\
         \textsc{NoA2} & $-0.1$ & $-0.4$ & $-0.5$ & $-0.6$ & $-$ & $-$ & $-$ & $-$ & $-$ & $-$ \\
         \textsc{RealL} & $0.800$ & $0.791$ & $0.676$ & $0.545$ & $0.538$ & $0.506$ & $0.360$ & $0.329$ & $0.306$ & $0.274$ \\
\midrule
        & $11$ & $12$ & $13$ & $14$ & $15$ & $16$ & $17$ & $18$ & & \\
        \midrule
&  $0.241$ & $0.203$ & $0.112$ & $0.084$ & $0.081$ & $0.007$ & $-0.018$ & $-0.120$ & & \\
         \bottomrule
     \end{tabular}}
     \caption{Synthetic and real-life mean vector instances (scores for the real-life instance are rounded up to the $3^\text{rd}$ decimal place). }
     \label{tab:instances_arms}
 \end{table}

  \begin{table}
     \centering
     \begin{tabular}{l r r r r r}\toprule
         Name & $H_1(\mu)$ & $H_{\theta}(\mu)$ & $\min_{a \in \set{\theta}(\mu)} \Delta_{a}^{-2}$ & $\max_{a \in \set{\theta}(\mu)} \Delta_{a}^{-2}$ & $\NARMS \hat{\Delta}^{-2}$ \\
         \midrule
         \textsc{Thr1} & $926$  & $463$ & $6$ & $400$ & $49$\\
         \textsc{Thr2} & $921$ & $460$ & $16$ & $400$ & $67$ \\
         \textsc{Thr3} & $4000$  & $1200$ & $400$ & $400$ & $1000$\\
         \textsc{Med1} & $2677$ & $730$ & $730$ & $730$ & $1081$ \\
         \textsc{Med2} & $143$ & $9$ & $3$ & $6$ & $14$\\
         \textsc{IsA1} & $533$  & $266$ & $3$ &  $225$ & $23$\\
         \textsc{NoA1} & $30$  & $-$  & $-$ & $-$ & $55$\\
         \textsc{IsA2} & $218$ & $104$ & $1$ & $100$ & $5$ \\ 
         \textsc{NoA2} & $113$ & $-$ & $-$ & $-$ & $399$  \\ 
         \textsc{RealL} & $29206$ & $29019$ & $11$ & $27778$  & $93$ \\
         \textsc{TwoG} & $4K$ & $K$ & $4$ & $4$ & $4|\set{\THRESHOLD}|$ \\
          \bottomrule
     \end{tabular}
     \caption{Numerical values of difficulty constants. $H_1(\mu)$ and $H_{\theta}(\mu)$ as in~Eq.~\eqref{eq:common_complexity}, $\hat{\Delta} \deff \max_{a \in \set{\THRESHOLD}} \Delta_a + \min_{b \not\in \set{\THRESHOLD}} \Delta_b$.\\
     }
     \label{tab:instances_constants_value}
 \end{table}

 \subsubsection{Real-life Data Set (\textsc{RealL}): Outcome Scoring Application} 
 \label{app:sssec_reallife}

\clemence{ Premature birth is known to induce moderate to severe neuronal dysfunction in newborns. Human mesenchymal stem cells might help repair and protect neurons from the injury induced by the inflammation. The goal is to determine whether one among possible therapeutic protocols exerts a strong enough positive effect on patients. }

\clemence{In order to answer this question, in collaboration with the PREMSTEM consortium, we have considered a rat model of perinatal neuroinflammation, which mimics brain injuries due to premature birth. Here, the set of arms are considered protocols for the injection of human mesenchymal stem cells (HuMSCs) in rats. Briefly, rat pups received intraperitoneal IL-1$\beta$ injections ($20$ $\mu$g/kg) twice daily from post-natal day (P)1-P4 and once at P5 to model preterm brain injury, and controls received only PBS. Human umbilical cord-derived MSCs (HuMSCs, Chiesi Pharmaceuticals/Lonza) were administered using $18$ different protocols testing three doses ($20$, $50$, $125$ M cells/kg), three time points (P5, P10, P20), and two delivery routes (intranasal vs intravenous).} 

\clemence{Animals were sacrificed $48$ hours post-treatment, and microglia were isolated from brain tissue using anti-CD11b/c magnetic beads (Miltenyi Biotec). RNA was extracted using NucleoSpin RNA XS Plus kit, with quality assessed by fragment analyzer ($>7$ cutoff). Libraries were prepared using TruSeq Stranded mRNA kit and sequenced on NextSeq 500 (75 bp single reads, $\sim$27M reads/sample). Reads were aligned to rnor6 genome using STAR, processed with samtools and HTSeq-count. Treatment efficacy was evaluated by comparing gene expression signatures between injured-to-treated groups versus injured-to-control groups using characteristic direction differential expression analysis~\cite{clark2014characteristic} and cosine similarity scoring ($N=3$ per protocol). Those score quantifies the effect of each protocol using a cosine score on gene activity measurement profiles between model animals injected with HuMSCs and control animals, which have not been exposed to the inflammation. The cosine score is between -1 and 1. The closer this score is to 1, the more similar the gene activity changes of the treated group are to those of control group. We considered a threshold of $\THRESHOLD=0.5$ for treatment efficiency. }

\clemence{Traditional approaches use grid-search with a uniform allocation and select the best cosine score to determine the optimal protocol. Here, to model the stochasticity of the scores that would have been obtained for each protocol in a sequential approach, we applied a Bernoulli instance. In this application observations from arm $a$ for one treatment are drawn from a Bernoulli distribution with mean $\max(\mu_a, 0)$ using the real cosine score of this treatment protocol as $\mu_a$. Bernoulli distributions are here more realistic with respect to our real-life application, while our algorithms can still be applied to this instance, as a Bernoulli distribution is $1/2$-sub-Gaussian. One must nevertheless note that in real life, the data generation were carried out sequentially into several batches, with each treatment protocol tested in triplicate, but only once in the same batch. The real stochasticity of such data is unknown and would require costly and heavier laboratory experiments and sequencing.}

  \subsubsection{Synthetic Data Set: Gaussian Instances}
 \label{app:sssec_synthetic}
  
 Along with the above real-life application described above and in Section~\ref{sec:experiments}, we have also considered several Gaussian instances with unit variance.
 
 Mimicking the experiments conducted in~\cite{kano2019good}, we consider their three synthetic instances, referred to as \textsc{Thr1} (three group setting), \textsc{Thr2} (arithmetically progressive setting) and \textsc{Thr3} (close-to-threshold setting), as well as their two medical instances, referred to as \textsc{Med1} (dose-finding of secukinumab for rheumatoid arthritis with satisfactory effect) and \textsc{Med2} (dose-finding of GSK654321 for rheumatoid arthritis with satisfactory effect).
 While some instances were studied in~\cite{kano2019good} for Bernoulli distributions, here we only consider Gaussian instances.
 For \textsc{Med2}, the Gaussian instances have variance $\sigma^2 = 1.44$.

 Mimicking the experiments conducted in~\cite{kaufmann2018sequential}, we consider instances whose means are linearly spaced with and without good arms.
 \textsc{IsA1} is linearly space between $0.5$ and $-0.5$ with $K=10$, and \textsc{NoA1} between $-0.5$ and $-1$ with $K=5$.
 In addition, we complement those synthetic experiments with two instances with and without good arms, named \textsc{IsA2} and \textsc{NoA2}.

 Finally, as done in~\cite{kaufmann2018sequential}, we study the impact of the number of good arms $|\set{\THRESHOLD}|$ among $K=100$ arms on the performance.
 We will consider $|\set{\THRESHOLD}| \in \{5 k\}_{k \in [19]}$, with $\THRESHOLD = 0$.
 In the \textsc{TwoG} instances, we have $\mu_{a} = 0.5$ for all $a \in \set{\THRESHOLD}$, otherwise $\mu_{a} = - 0.5$. 
 In the \textsc{LinG} instances, we have $\mu_{a} = -0.5$ for all $a \notin \set{\THRESHOLD}$, and the $|\set{\THRESHOLD}|$ good arms have a strictly positive mean which is linearly spaced up to $\max_{a \in \ARMS} \mu_{a} = 0.5$.

\subsection{Implementation Details}
\label{app:ssec_impelementation_details}

We provide details about the implementation of the considered algorithms for the anytime setting (Appendix~\ref{app:sssec_any_algos}), fixed-budget setting (Appendix~\ref{app:sssec_FB_algos}) and the fixed-confidence setting (Appendix~\ref{app:sssec_FC_algos}).
The reproducibility of our experiments is addressed in Appendix~\ref{app:sssec_reproducibility}.

\subsubsection{Anytime Algorithms}
\label{app:sssec_any_algos}

As described in Section~\ref{sec:ssec_bai_to_gai}, we modify Successive Reject (SR)~\citep{audibert2010best} and Sequential Halving (SH)~\citep{karnin2013almost} to tackle GAI.
We derived upper bound on the probability of errors of those modified algorithms (Theorems~\ref{thm:SH_PoE_recoElim} and~\ref{thm:SR_PoE_recoElim} in Appendix~\ref{app:guarantees_other_algorithms}).
As a reminder, SR eliminates one arm with the worst empirical mean at the end of each phase, and SH eliminated half of them but drops past observations between each phase.
Within each phase, both algorithms use a round-robin uniform sampling rule on the remaining active arms.
SR-G and SH-G return $\hat a_{\NBATCHES} = \emptyset$ when $\expmean{a_{\NBATCHES}}{\NBATCHES} \le \THRESHOLD$ and $\hat a_{T} = a_{\NBATCHES}$ otherwise, where $a_{\NBATCHES}$ is the arm that would be recommended for the BAI problem, \ie the last arm that was not eliminated.
Then, we convert the fixed-budget SH-G and SR-G algorithms into anytime algorithms by using the doubling trick.
It considers a sequences of algorithms that are run with increasing budgets $(T_{k})_{k \ge 1}$, with $T_{k+1} = 2 T_{k}$ and $T_{1} = 2 K \lceil \log_{2} K \rceil$, and recommend the answer outputted by the last instance that has finished to run.
It is well know that the ``cost'' of doubling is to have a multiplicative factor $4$ in front of the hardness constant.
The first two-factor is due to the fact that we forget half the observations.
The second two-factor is due to the fact that we use the recommendation from the last instance of SH that has finished.
The doubling version of SR-G and SH-G are named Doubling SR-G (DSR-G) and Doubling SH (DSH-G).

Compared to SR, the empirical performance of SH suffers from the fact that it drops observation between phases.
While the impact of this forgetting step is relatively mild for BAI where all the arms are sampled linearly, it is larger for GAI since arms are not sampled linearly.
In order to assess the impact of this forgetting step, we implement the DSH-G-WR (``without refresh'') algorithm in which each SH-G instance keeps all the observations at the end of each phase.
To the best of our knowledge, there is no theoretical analysis of this version of SH, even in the recent analysis of~\cite{zhao2022revisiting}.
Figure~\ref{fig:supp_PoE_doublingSH_withrefresh} highlights the dramatic increase of the empirical error incurred by dropping past observations.
This phenomenon occurs in almost all of our experiments, both when $\set{\THRESHOLD}(\mu) =\emptyset $ and when $\set{\THRESHOLD}(\mu) \ne \emptyset $.

\begin{figure}
    \centering
    \clemence{(a)} \includegraphics[width=0.45\linewidth]{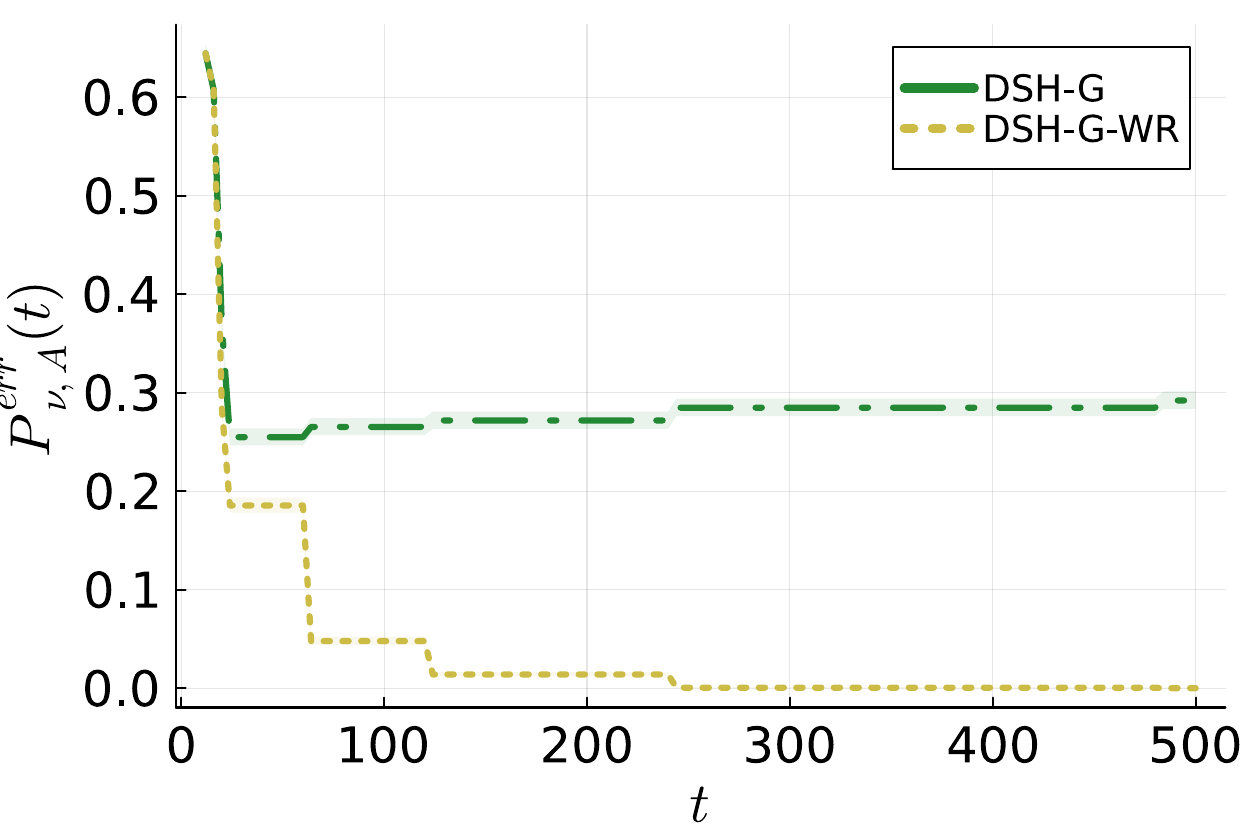}
    \clemence{(b)} \includegraphics[width=0.45\linewidth]{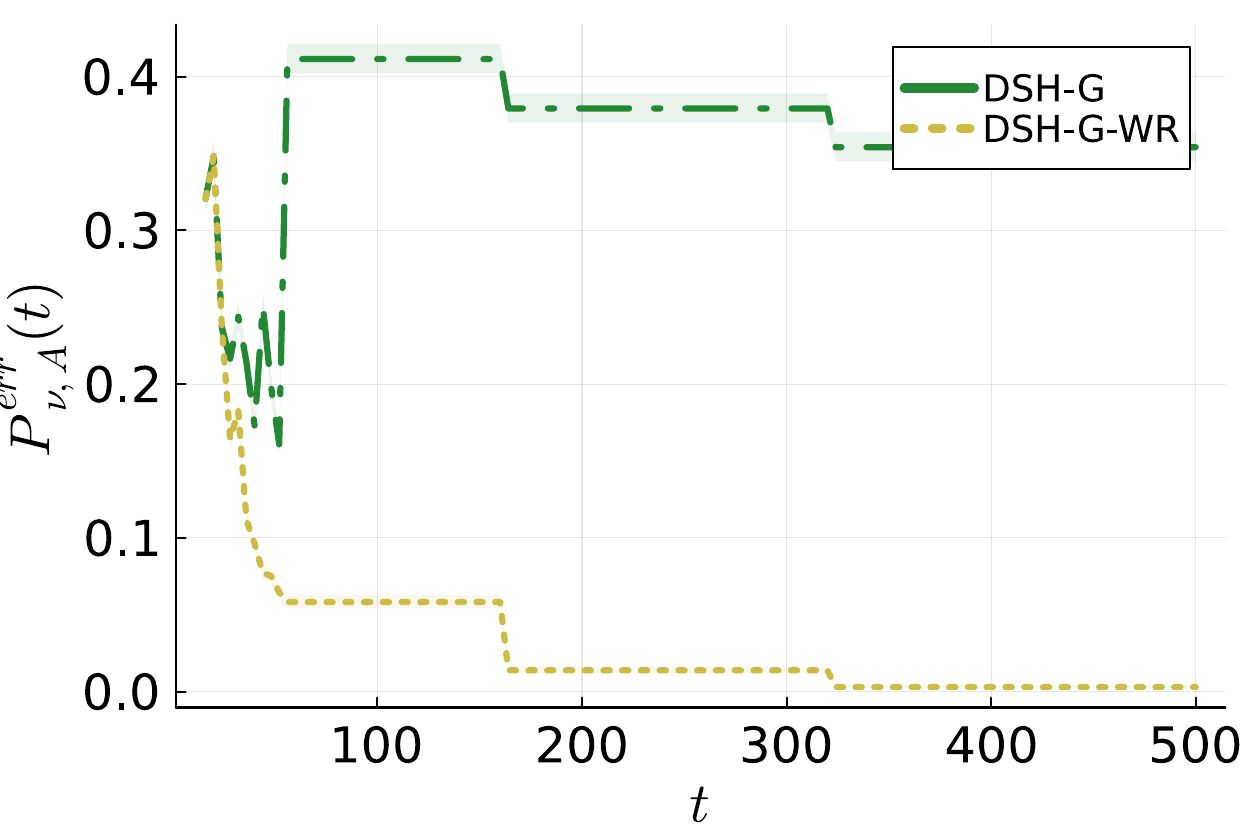}\\
     \clemence{(c)} \includegraphics[width=0.45\linewidth]{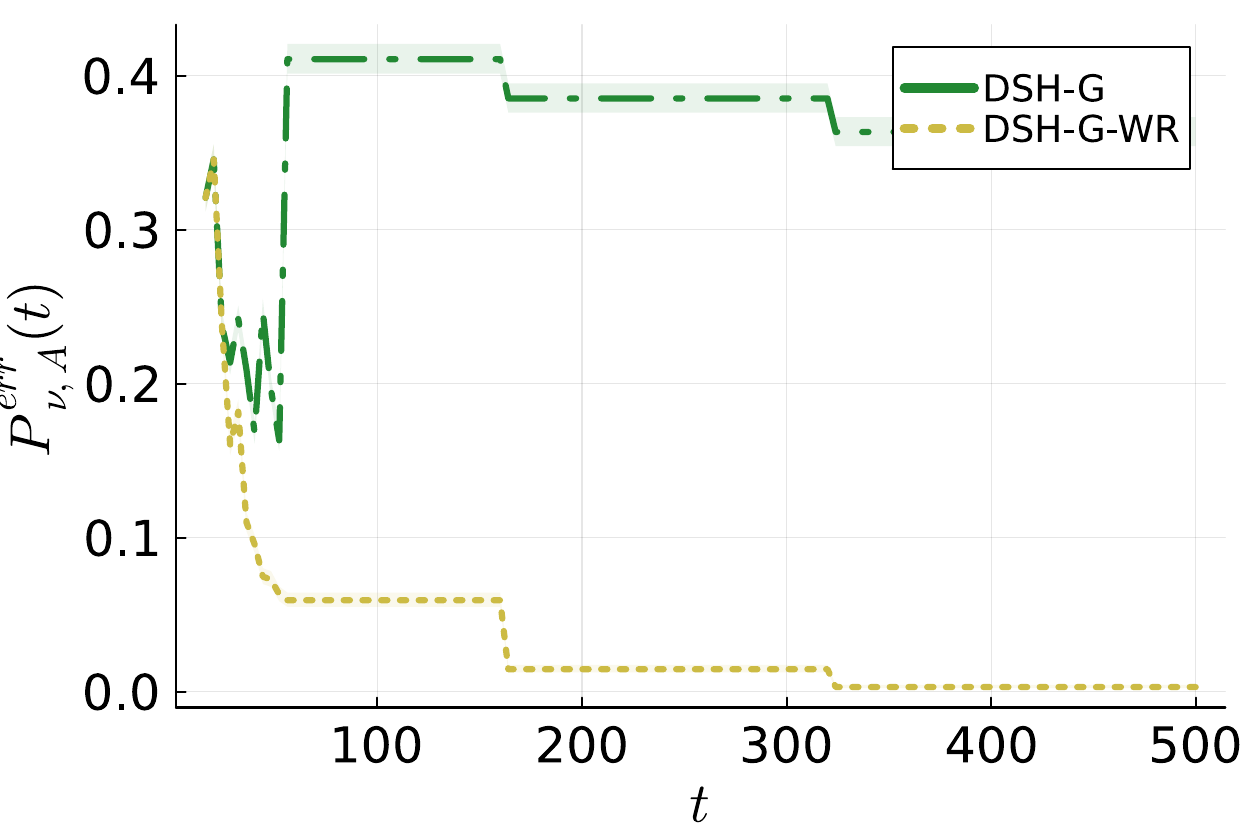}
     \clemence{(d)} \includegraphics[width=0.45\linewidth]{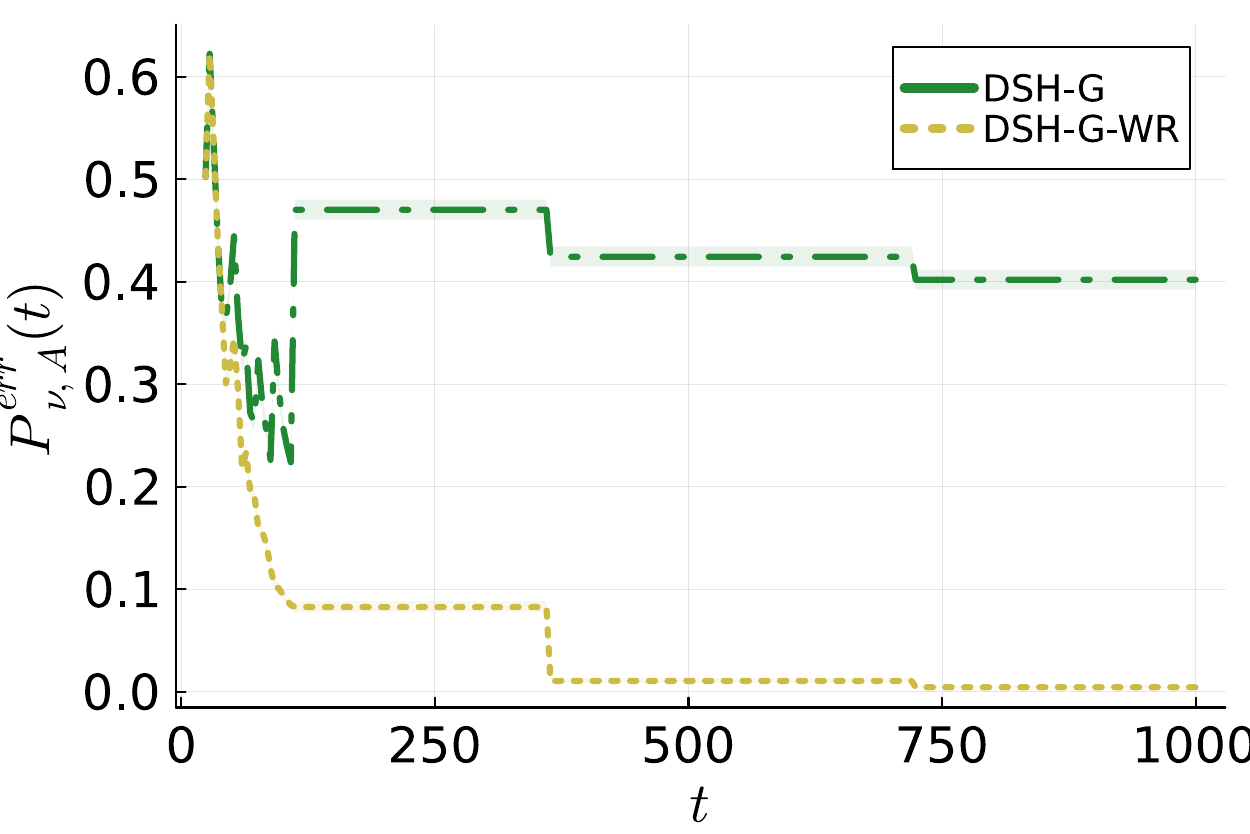}\\
    \clemence{(e)}  \includegraphics[width=0.45\linewidth]{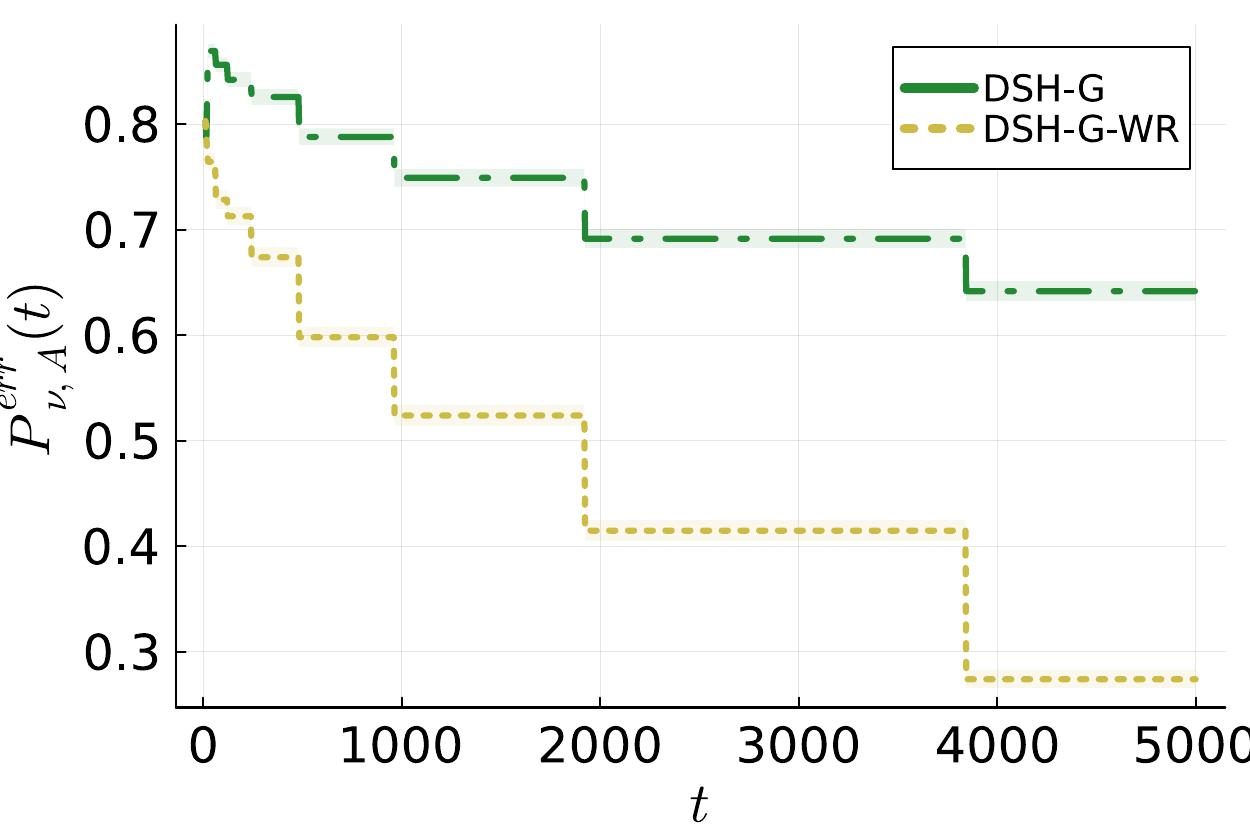}
     \clemence{(f)} \includegraphics[width=0.45\linewidth]{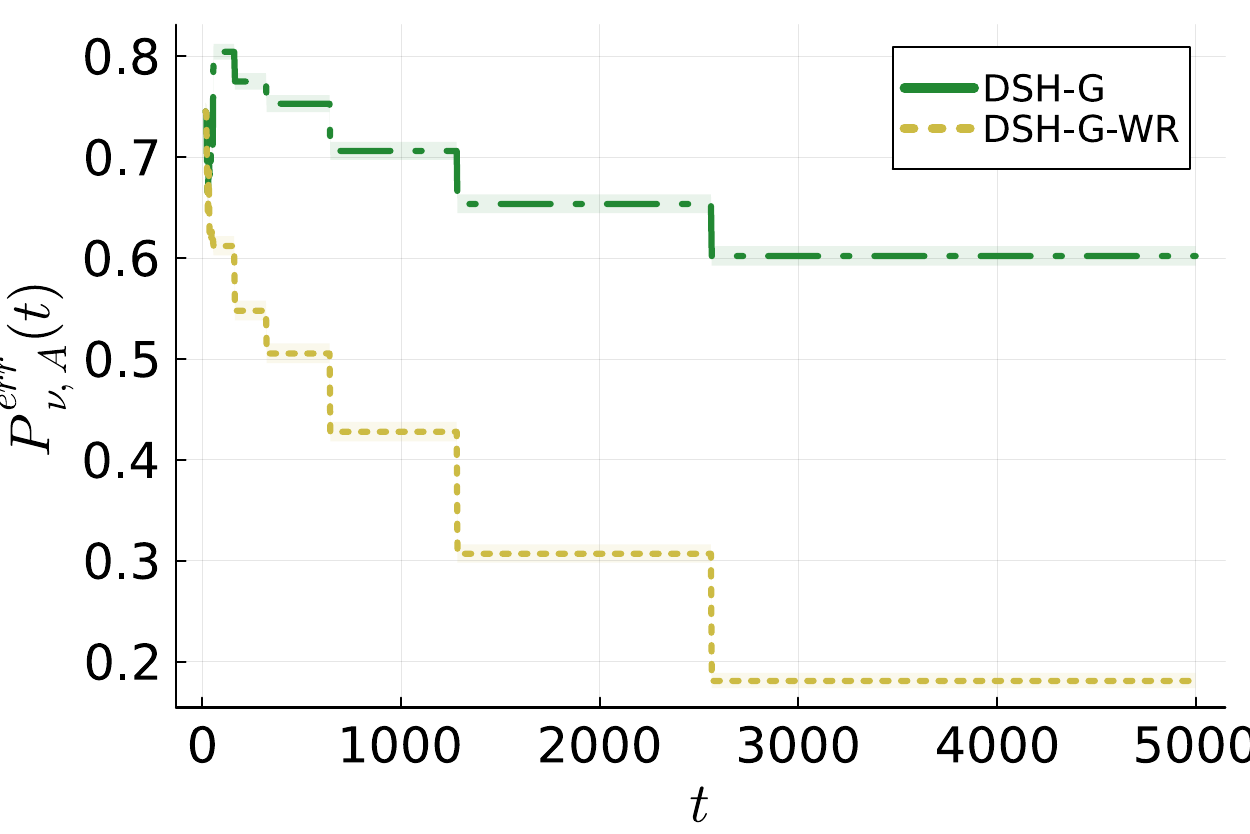}
    \caption{Empirical error on instances (a) \textsc{NoA1}, (b) \textsc{IsA1}, (c) \textsc{Thr1}, (d) \textsc{RealL}, (e) \textsc{Med1} and (f) \textsc{Thr3}. ``-WR'' means that each SH instance keeps all its history instead of discarding it.}
    \label{fig:supp_PoE_doublingSH_withrefresh}
\end{figure}

\subsubsection{Fixed-budget Algorithms}
\label{app:sssec_FB_algos}

We compare the fixed-budget performances of~\hyperlink{APGAI}{APGAI} with the GAI versions SH-G and SR-G of SH and SR as described in Subsection~\ref{app:sssec_any_algos}, the uniform round-robin algorithm Unif, and different index policies in the prior knowledge-based meta algorithm~\hyperlink{PKGAI}{PKGAI}. Those index policies are defined in Section~\ref{sec:stickyalgo} and recalled below
   \begin{eqnarray*}
        \text{PKGAI(APT$_P$) : } & i_a(t) \deff \sqrt{\nsamples{a}{t}}(\expmean{a}{t}-\THRESHOLD)\:, \\
        \text{PKGAI(UCB) : } & i_a(t) \deff \expmean{a}{t}-\THRESHOLD+\sqrt{\frac{\beta(t)}{\nsamples{a}{t}}}\:,\\
        \text{PKGAI(Unif) : } & i_a(t) \deff -\nsamples{a}{t}\:, \\
        \text{PKGAI(LCB-G) : } & i_a(t) \deff \sqrt{\nsamples{a}{t}}(\expmean{a}{t}-\THRESHOLD)+\sqrt{\beta(t)}\:.
   \end{eqnarray*}
Note that, contrary to~\hyperlink{APGAI}{APGAI} and Unif, the other algorithms require the definition of the sampling budget $\NBATCHES$. For the sake of fairness, we do not use the theoretical value for $\beta$ as in Theorems~\ref{th:APTlike_error} and~\ref{th:stickyalgo_error}. We implement the following confidence width, which is theoretically backed by Lemma~\ref{lem:concentration_per_arm_gau_improved} in Appendix~\ref{app:sequence_concentration_events} (for $s=0$),
\begin{equation}\label{eq:thresholds_fb}
\beta(t) = \sigma\sqrt{z(\NBATCHES,\delta)/\nsamples{a}{t}}, \ \text{ where } z(\NBATCHES,\delta) \deff \overline{W}_{-1}(2 \log(\NARMS/\delta)  + 2 \log (2 + \log \NBATCHES ) + 2 ) \: ,
\end{equation}
using $\delta = 0.01$.

We also consider for algorithms of the~\hyperlink{PKGAI}{PKGAI} family the theoretical threshold functions featured in Theorems~\ref{th:APTlike_error} and~\ref{th:stickyalgo_error}, \ie relying on problem quantities in practice unavailable at runtime
\begin{equation}\label{eq:thresholds_pk_fb}
\beta(t) = \sigma\sqrt{q(\NBATCHES,\delta)/\nsamples{a}{t}}, \ \text{ where } q(\NBATCHES,\delta) \deff \begin{cases} (\NBATCHES-\NARMS)/(4H_1(\mu)) & \text{ if } \set{\THRESHOLD}(\mu) = \emptyset\\ (\NBATCHES-\NARMS)/(4\NARMS \hat{\Delta}^{-2}) & \text{ otherwise }\end{cases} \;,
\end{equation}
where $\hat{\Delta} \deff \max_{a \in \set{\THRESHOLD}(\mu)} \Delta_a + \min_{b \not\in \set{\THRESHOLD}(\mu)} \Delta_b$.

\subsubsection{Fixed-confidence Algorithms}
\label{app:sssec_FC_algos}

\textit{Link between GLR stopping and UCB/LCB stopping.}
In~\cite{kano2019good}, all the algorithms (HDoC, LUCB-G and APT-G) use a stopping rule which is based on UCB/LCB indices.
Namely, they return an arm $a$ as soon as its associated LCB exceeds the threshold $\THRESHOLD$.
Since we consider GAI instead of AllGAI, this condition becomes a stopping rule.
The second stopping condition is to return $\emptyset$ as soon as all the arms are eliminated, and an arm is eliminated when its UCB is lower than the threshold $\THRESHOLD$.
Direct manipulations show that the GLR stopping~Eq.~\eqref{eq:stopping_rule} is equivalent to their stopping provided that the UCB and LCB are using the same stopping threshold for the bonuses, \ie
\begin{align*}
    &\max_{a \in \ARMS} \Wp{a}{t} \ge \sqrt{2c(t, \DELTA)} \qquad \iff \qquad \exists a \in \ARMS, \quad \expmean{a}{t} - \sqrt{\frac{2c(t, \DELTA)}{\nsamples{a}{t}}} \ge \THRESHOLD \: , \\
    &\min_{a \in \ARMS} \Wm{a}{t} \ge \sqrt{2c(t, \DELTA)} \qquad \iff \qquad \forall a \in \ARMS, \quad \expmean{a}{t} + \sqrt{\frac{2c(t, \DELTA)}{\nsamples{a}{t}}} \le \THRESHOLD \: .
\end{align*}
In~\cite{kano2019good}, they consider bonuses that only depend on the pulling count $N_{a}(t)$ instead of depending on the global time $t$.
This ensures that the UCB remains constant once the arm has been eliminated.
In contrast, using a UCB which depends on the global time $t$ (such as our stopping threshold in~Eq.~\eqref{eq:stopping_threshold}) implies that this elimination step does not ensure that the condition on this arm still hold at stopping time.
Mathematically, they use the following UCB/LCB, $\expmean{a}{t} \pm \sqrt{2\Lambda_a(t, \DELTA)/\nsamples{a}{t}}$ where $\Lambda_a(t, \DELTA) = \ln(4\NARMS/\delta) + 2 \ln \nsamples{a}{t}$.
Since~\cite{kano2019good} consider Bernoulli distributions which are $1/2$-sub-Gaussian, we modified the bonuses to match the ones for $1$-sub-Gaussian (by using that the proper scaling is in $\sqrt{2\sigma^2}$).

While both stopping threshold $c$ and $(\Lambda_{a})_{a \in \ARMS}$ have the same dominating $\DELTA$-dependency in $\log(1/\DELTA)$, it is worth noting that the time dependency of $c$ is significantly better since $c(t, \DELTA) \sim_{t \to + \infty} 2 \log \log t$.
Ignoring the $\DELTA$-dependent terms and the constant, we have a lower bonus as long as $\nsamples{a}{t} \gtrapprox \log t$.
For a fair comparison, we will use the stopping threshold in~Eq.~\eqref{eq:stopping_threshold} for the UCB/LCB used by HDoC and LUCB-G (both in the sampling and stopping rule) instead of the larger bonuses $(\Lambda_{a})_{a \in \ARMS}$ considered in~\cite{kano2019good}.

\textit{Limits of existing algorithms.}
The APT-G algorithm introduced in~\cite{kano2019good} samples $a_{t+1} = \argmin_{a \in \ARMS_{t}} \sqrt{\nsamples{a}{t}}| \expmean{a}{t}  - \THRESHOLD|$, where $\ARMS_{t}$ is the set of active arms.
This index policy is tailored for the Thresholding setting, where one needs to classify all the arms as above or below the threshold $\THRESHOLD$.
Intuitively, a good algorithm for Thresholding will perform poorly on the GAI setting since it must pay $H_{1}(\mu)$ even when $\set{\THRESHOLD}$.
This is confirmed by the experiments in~\cite{kano2019good}, as well as our own experiments.
Since it is not competitive, we omitted its empirical performance from our experiments.

The Sticky Track-and-Stop (S-TaS) algorithm introduced in~\cite{degenne2019pure} admits a computationally tractable implementation for GAI.
To the best of our knowledge, this is one of the few setting where this holds, \eg it is not tractable for $\varepsilon$-BAI.
The major limitation of S-TaS lies in its dependency on an ordering $\mathcal O$ on the set of candidate answers $\ARMS \cup \{ \emptyset \}$.
Informally, S-TaS computes a set of admissible answer based on a confidence region on the true mean, and sticks to the answer with the lowest ranking in the ordering $\mathcal O$.
Then, S-TaS samples according to the optimal allocation for this specific answer.
Depending on the choice of this ordering, the empirical performance can change drastically, especially for instances such that $\set{\THRESHOLD}(\mu) \ne \emptyset$.
We consider two orderings to illustrate this.
The \textsc{Asc} considers the ordering $ \mathcal O$ such that $o_{a} = a $ for all $a \in \ARMS$, and $a_{\NARMS + 1} = \emptyset$.
The \textsc{Desc} considers the ordering $ \mathcal O$ such that $o_{a} = K - a + 1$ for all $a \in \ARMS$, and $a_{\NARMS + 1} = \emptyset$.
In Table~\ref{tab:stickyTaS_ordering}, we can see that S-TaS performs considerably better for \textsc{Asc} compared to \textsc{Desc}.
This can be explained by the fact that in all our instances the means are ordered, so that lower indices correspond to higher mean.
Since higher means are easier to verify, this explains the improved performance for \textsc{Asc}.

 \begin{table}[t]
     \centering
     \begin{tabular}{l r r r r r r r r}
    \toprule
         Ordering & \textsc{Thr1}  & \textsc{Thr2} & \textsc{Thr3}  & \textsc{Med1} & \textsc{Med2} & \textsc{IsA1} &  \textsc{IsA2} & \textsc{RealL}  \\
         \midrule
         \textsc{Asc} & $183$ & $435$ & $11787$ & $20488$ & $114$ & $120$ & $33$ & $341$ \\
                      & $\pm 68$ & $\pm 163$ & $\pm 4539$ & $\pm 7972$ & $\pm 41$ & $\pm 41$ & $\pm 10$ & $\pm 122$ \\
         \textsc{Desc} & $20574$ & $19960$ & $71057$ & $60275$ & $3087$ & $16469$ & $4539$ & $-$ \\
                       & $\pm 5835$ & $\pm 5885$ & $\pm 11684$ & $\pm 16112$ & $\pm 1293$ & $\pm 4680$ & $\pm 1434$ & $-$ \\
          \bottomrule
     \end{tabular}
     \caption{Empirical stopping time ($\pm$ standard deviation) of Sticky Track-and-Stop depending on the ordering on the set of candidate answers $\ARMS \cup \{ \emptyset \}$.
     ``-'' means that the algorithm didn't stop after $10^{5}$ steps.}
     \label{tab:stickyTaS_ordering}
 \end{table}
 
The Murphy Sampling (MS) algorithm introduced in~\cite{kaufmann2018sequential} uses a rejection step on top of a Thompson Sampling procedure.
For Gaussian instances, the posterior distribution $\Pi_{t,a}$ of the arm $a \in \ARMS$ for the improper prior $\Pi_{0,a} = \mathcal N(0, + \infty)$ is $\Pi_{t,a} = \mathcal N(\expmean{a}{t}, 1/\sqrt{\nsamples{a}{t}})$.
Let $\Pi_{t} = (\Pi_{t,a})_{a \in \ARMS}$.
Then, MS samples $\lambda \sim \Pi_{t}$ until $\max_{a \in \ARMS} \lambda_{a} > \THRESHOLD$, and samples arm $\argmax_{a \in \ARMS} \lambda_{a}$ for this realization.
This rejection steps is equivalent to conditioning on the fact that $\set{\THRESHOLD}(\mu) \ne \emptyset$.
As noted in~\cite{kaufmann2018sequential}, this rejection step can be computationally costly when $\set{\THRESHOLD}(\mu) = \emptyset$.
Intuitively, we need to draw many vectors before observing $\lambda$ such that $\set{\THRESHOLD}(\lambda) \ne \emptyset$ once the posterior $\Pi_{t}$ has converged close to the Dirac distribution on $\mu$ when $\set{\THRESHOLD}(\mu) = \emptyset$.
Empirically, we observed this phenomenon on the \textsc{NoA2} instance.
While all the other algorithms has a CPU running time of the order of $10$ milliseconds, MS reached a CPU running time of $10^5$ milliseconds.

We consider the Track-and-Stop (TaS) algorithm for GAI.
It is direct to adapt the ideas of the original Track-and-Stop introduced in~\cite{garivier2016optimal} for BAI.
When $\max_{a \ARMS} \expmean{a}{t} \ge \THRESHOLD$, the optimal allocation $w^\star(\hat \mu (t))$ to be tracked is a Dirac in $\argmax_{a \ARMS} \expmean{a}{t}$.
Otherwise, using the proof of Lemma~\ref{lem:lower_bound_GAI}, the optimal allocation is $w^\star(\hat \mu (t))$, which is defined as $w^\star(\hat \mu (t))_{a} \propto (\expmean{a}{t} - \THRESHOLD)^{-2}$.
On top of the C-Tracking procedure used to target the average optimal allocation, Track-and-Stop relies on a forced exploration procedure which samples under-sampled arms, \ie arms in $\{a \in \ARMS \mid \nsamples{a}{t} \le \sqrt{t} - K/2 \}$.
Without the forced exploration, TaS would have worse empirical performance since it would be too greedy.

As mentioned in Sections~\ref{sec:anytimeid} and~\ref{sec:FCGAI}, the BAEC meta-algorithm is only defined for asymmetric threshold $\theta_{U} > \theta_{L}$.
Mathematically, it uses the following UCB/LCB indices
\begin{align*}
    &\expmean{a}{t} + \sqrt{\frac{2\Lambda^+_a(t, \DELTA)}{\nsamples{a}{t}}} \quad \text{where} \quad \Lambda^+_a(t, \DELTA) = \ln\left(N(\DELTA)/\delta\right) \quad \text{and}\\
     & \qquad \qquad N(\DELTA) := \left\lceil \frac{2e}{(e-1)(\theta_{U} - \theta_{L})^2}\ln\left(\frac{2\sqrt{\NARMS}}{(\theta_{U} - \theta_{L})^2\delta}\right) \right\rceil \: , \\
    &\expmean{a}{t} - \sqrt{\frac{2\Lambda^-_a(t, \DELTA)}{\nsamples{a}{t}}} \quad \text{where} \quad \Lambda^-_a(t, \DELTA) = \ln\left(\sqrt{\NARMS}N(\DELTA)/\delta\right) \: . \\    
\end{align*}
In the GAI setting, those indices will infinite, hence BAEC is not defined properly.
Instead of using asymmetric threshold, one could simply use symmetric ones which are independent of $(\theta_{U} - \theta_{L})^{-2}$.
In that case, BAEC coincide with the HDoC and LUCB-G algorithms introduced in~\cite{kano2019good}.

\subsubsection{Reproducibility}
\label{app:sssec_reproducibility}

\textit{Experiments on fixed-budget empirical error.}
The benchmark was implemented in  \texttt{Python 3.9}, and run on a personal computer (configuration: processor Intel Core i$7-8750$H, $12$ cores @$2.20$GHz, RAM $16$GB). 
The code, along with assets for the real-life instance---where the exact treatment protocols have been replaced with placeholder names---are available in a \texttt{.zip} file under MIT (code) and Creative Commons Zero (assets) licenses. 
Commands which have generated plots and tables in this paper can be found in the Bash file named \texttt{experiments.sh}.

\textit{Experiments on anytime empirical error and empirical stopping time.}
Our code is implemented in \texttt{Julia 1.9.0}, and the plots are generated with the \texttt{StatsPlots.jl} package.
Other dependencies are listed in the \texttt{Readme.md}.
The \texttt{Readme.md} file also provides detailed julia instructions to reproduce our experiments, as well as a \texttt{script.sh} to run them all at once.
The general structure of the code (and some functions) is taken from the \href{https://bitbucket.org/wmkoolen/tidnabbil}{tidnabbil} library.
This library was created by \cite{Degenne19GameBAI}, see \url{https://bitbucket.org/wmkoolen/tidnabbil}.
No license were available on the repository, but we obtained the authorization from the authors.
Our experiments are conducted on an institutional cluster with $4$ Intel Xeon Gold 5218R CPU with $20$ cores per CPU and an x86\_64 architecture.

\subsection{Supplementary Results on Fixed-budget Empirical Error}
\label{app:ssec_supp_emp_error_FB}

Recall that we use here the prior-knowledge-agnostic threshold functions defined in Equation~Eq.~\eqref{eq:thresholds_fb}. We report in Figures~\ref{fig:results_PREMSTEM_appendix},~\ref{fig:results_IsA1_appendix},~\ref{fig:results_IsA2_appendix},~\ref{fig:results_NoA1_appendix} and~\ref{fig:results_NoA2_appendix} the empirical error curves for all algorithms described in Subsection~\ref{app:sssec_FB_algos} on real-life instance \textsc{RealL}, along with two synthetic instances \textsc{IsA1} and \textsc{IsA2} where $\set{\THRESHOLD}\neq \emptyset$, and two other instances where $\set{\THRESHOLD} =\emptyset$ (\textsc{NoA1} and \textsc{NoA2}). 
Results are averaged over $1,000$ runs. 
In plots, we display the mean empirical error and shaded area corresponds to Wilson confidence intervals~\citep{wilson1927probable} with confidence $95\%$. 
Those Wilson confidence intervals are also reported on the corresponding tables.

In the real-life instance along with the instances with no good arms, uniform samplings (SH-G, SR-G, Unif and PKGAI(Unif)) are noticeably less efficient at detecting the presence or absence of good arms, contrary to the adaptive strategies. 
Moreover, except for instance \textsc{IsA2}, APGAI actually performs as well as more complex, elimination-based algorithms PKGAI($\star$), while allowing early stopping as well. 
Perhaps unsurprisingly, the performance of APGAI are closely related to those of PKGAI(APT$_P$), as both algorithms share the same sampling rule. 
In all three instances, although~\hyperlink{PKGAI}{PKGAI} has unrealistic assumptions in its theoretical guarantees (Theorems~\ref{th:APTlike_error} and~\ref{th:stickyalgo_error}), its performance actually turns out to be the best of all algorithms. 
In particular, using the UCB sampling rule seems to be the most efficient. 
This shows that adaptive strategies can fare better than uniform samplings, which are more present in prior works in fixed-budget.

\begin{remark}
    Our experiments below highlight that an algorithm which only aims at allocating most of the budget to the best arm (\eg based on UCB indices) would be efficient on instances with a good arm with large gap.
    However, it would be heavily penalized in instances where there are no good arms, or in instances where the gap between the good and the bad arms is small. 
\end{remark}

\textit{Performance on the real-life application.} 
We report empirical errors at $\NBATCHES=200$ in Table~\ref{tab:results_PREMSTEM_appendix}, at which budget empirical errors for all algorithms seem to converge (see Figure~\ref{fig:results_PREMSTEM_appendix}).

\textit{Performance on synthetic data sets ($\set{\THRESHOLD} \neq \emptyset$).} 
We report empirical errors at $\NBATCHES=700$ in Tables~\ref{tab:results_IsA1_appendix} and~\ref{tab:results_IsA2_appendix}, at which budget empirical errors for all algorithms seem to converge (see Figures~\ref{fig:results_IsA1_appendix} and~\ref{fig:results_IsA2_appendix}). 
In the figures, the curves of PKGAI(APT$_P$) and PKGAI(LCB-G) overlap.

\textit{Performance on synthetic data sets ($\set{\THRESHOLD} = \emptyset$).} 
We report empirical errors at $\NBATCHES=150$ in Table~\ref{tab:results_NoA1_appendix} and $\NBATCHES=700$ in Table~\ref{tab:results_NoA2_appendix}, at which budget empirical errors for all algorithms seem to converge (see Figures~\ref{fig:results_NoA1_appendix} and~\ref{fig:results_NoA2_appendix}).
In the figures, the curves of PKGAI(APT$_P$) and PKGAI(LCB-G) overlap.

  \begin{minipage}{0.5\linewidth}
\begin{table}[H]
    \centering
\begin{tabular}{lrl}
    \toprule
          Algorithm & Error & Conf. intervals \\
          \midrule
         \hyperlink{APGAI}{APGAI}  & 0.001 & $2.10^{-4} \quad 6.10^{-3}$\\
         \hyperlink{PKGAI}{PKGAI}(APT$_P$)   & 0.004 & $2.10^{-3} \quad  0.01$  \\
         \hyperlink{PKGAI}{PKGAI}(LCB-G)  & 0.001 & $2.10^{-4} \quad 6.10^{-3}$  \\
         \hyperlink{PKGAI}{PKGAI}(UCB)  & 0.000 & $0.00 \quad \quad 4.10^{-3}$  \\
         \hyperlink{PKGAI}{PKGAI}(Unif)  & 0.001 & $2.10^{-4} \quad 6.10^{-3}$  \\
         SH-G  & 0.005 & $2.10^{-3} \quad  1.10^{-2}$  \\
         SR-G   & 0.002 & $5.10^{-4} \quad  7.10^{-3}$  \\
         Unif   & 0.000 & $0.00 \quad \quad 4.10^{-3}$  \\
        \bottomrule
    \end{tabular}
    \caption{Error across $1,000$ runs at $\NBATCHES=200$. }
    \label{tab:results_PREMSTEM_appendix}
\end{table}
\begin{table}[H]
    \centering
\begin{tabular}{lrl}
            \toprule
          Algorithm & Error & Conf. intervals \\
          \midrule
         \hyperlink{APGAI}{APGAI}  & 0.003 &$1.10^{-3} \quad 9.10^{-3}$\\
         \hyperlink{PKGAI}{PKGAI}(APT$_P$)   & 0.004 &$2.10^{-3} \quad 0.01$  \\
         \hyperlink{PKGAI}{PKGAI}(LCB-G)  & 0.004 &$2.10^{-3} \quad 0.01$  \\
         \hyperlink{PKGAI}{PKGAI}(UCB)  & 0.000 & $0.00 \quad \quad 4.10^{-3}$  \\
         \hyperlink{PKGAI}{PKGAI}(Unif)  & 0.000 & $0.00 \quad \quad 4.10^{-3}$  \\
         SH-G & 0.000 & $0.00 \quad \quad 4.10^{-3}$   \\
         SR-G  & 0.000 & $0.00 \quad \quad 4.10^{-3}$  \\
         Unif   & 0.000 & $0.00 \quad \quad 4.10^{-3}$   \\
                 \bottomrule
    \end{tabular}
    \caption{Error across $1,000$ runs at $\NBATCHES=700$. }
    \label{tab:results_IsA1_appendix}
\end{table}
\begin{table}[H]
    \centering
\begin{tabular}{lrl}
            \toprule
          Algorithm & Error & Conf. intervals \\
          \midrule
         \hyperlink{APGAI}{APGAI}  & 0.000 & $0.00 \quad 4.10^{-3}$  \\
         \hyperlink{PKGAI}{PKGAI}(APT$_P$)   & 0.000 & $0.00 \quad 4.10^{-3}$  \\
         \hyperlink{PKGAI}{PKGAI}(LCB-G)  & 0.000 & $0.00 \quad 4.10^{-3}$   \\
         \hyperlink{PKGAI}{PKGAI}(UCB) & 0.000 & $0.00 \quad 4.10^{-3}$   \\
         \hyperlink{PKGAI}{PKGAI}(Unif)& 0.000 & $0.00 \quad 4.10^{-3}$   \\
         SH-G & 0.000 & $0.00 \quad 4.10^{-3}$   \\
         SR-G   & 0.000 & $0.00 \quad 4.10^{-3}$   \\
         Unif  & 0.000 & $0.00 \quad 4.10^{-3}$    \\
                 \bottomrule
    \end{tabular}
    \caption{Error across $1,000$ runs at $\NBATCHES=700$. }
    \label{tab:results_IsA2_appendix}
\end{table}
\end{minipage}
\begin{minipage}{0.45\linewidth}
    \begin{figure}[H]
        \centering
    \includegraphics[width=0.9\textwidth]{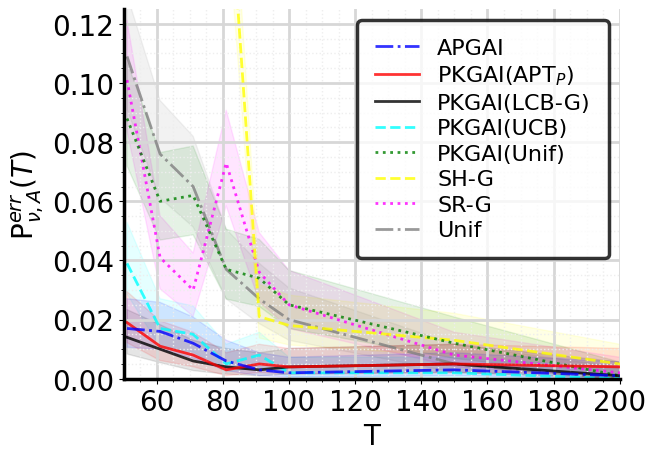}
    \caption{Empirical error (\textsc{RealL}).}
\label{fig:results_PREMSTEM_appendix}
    \end{figure}
    \begin{figure}[H]
        \centering
    \includegraphics[width=0.93\textwidth]{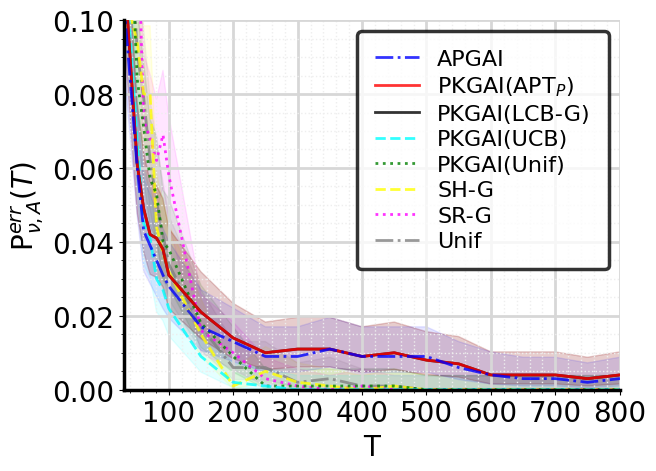}
    \caption{Empirical error (\textsc{IsA1}).}
\label{fig:results_IsA1_appendix}
    \end{figure}
    \begin{figure}[H]
        \centering
    \includegraphics[width=0.93\textwidth]{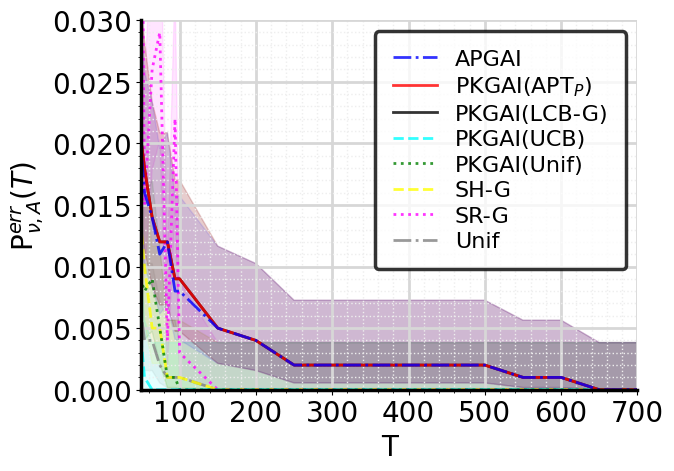}
    \caption{Empirical error (\textsc{IsA2}).}
\label{fig:results_IsA2_appendix}
    \end{figure}
\end{minipage}

  \begin{minipage}{0.5\linewidth}
\begin{table}[H]
    \centering
\begin{tabular}{lrl}
            \toprule
          Algorithm & Error & Conf. intervals \\
          \midrule
         \hyperlink{APGAI}{APGAI}  & 0.000 & $0.00 \quad \quad  4.10^{-3}$  \\
         \hyperlink{PKGAI}{PKGAI}(APT$_P$)   & 0.000 & $0.00 \quad \quad  4.10^{-3}$    \\
         \hyperlink{PKGAI}{PKGAI}(LCB-G)  & 0.000 & $0.00 \quad \quad  4.10^{-3}$    \\
         \hyperlink{PKGAI}{PKGAI}(UCB)  & 0.000 & $0.00 \quad \quad  4.10^{-3}$    \\
         \hyperlink{PKGAI}{PKGAI}(Unif) & 0.002 & $5.10^{-4} \quad 7.10^{-3}$   \\
         SH-G  & 0.000 & $0.00 \quad \quad  4.10^{-3}$  \\
         SR-G   & 0.007 & $3.10^{-3} \quad   0.01$  \\
         Unif   & 0.005 & $2.10^{-3} \quad 0.01$  \\
                 \bottomrule
    \end{tabular}
    \caption{Error across $1,000$ runs at $\NBATCHES=150$. }
    \label{tab:results_NoA1_appendix}
\end{table}
\begin{table}[H]
    \centering
\begin{tabular}{lrl}
            \toprule
          Algorithm & Error & Conf. intervals \\
          \midrule
         \hyperlink{APGAI}{APGAI}  & 0.002 &$5.10^{-4} \quad 7.10^{-3}$\\
         \hyperlink{PKGAI}{PKGAI}(APT$_P$)   & 0.002 & $5.10^{-4} \quad  7.10^{-3}$  \\
         \hyperlink{PKGAI}{PKGAI}(LCB-G)  & 0.002 & $5.10^{-4} \quad 7.10^{-3}$  \\
         \hyperlink{PKGAI}{PKGAI}(UCB)  & 0.007 & $3.10^{-3} \quad 0.01$  \\
         \hyperlink{PKGAI}{PKGAI}(Unif)  & 0.021 & $0.01 \quad \quad  0.03$  \\
         SH-G  & 0.018 & $0.01 \quad \quad  0.03$  \\
         SR-G   & 0.127 & $0.11 \quad \quad  0.15$  \\
         Unif   & 0.084 & $0.07  \quad \quad  0.10$  \\
                 \bottomrule
    \end{tabular}
    \caption{Error across $1,000$ runs at $\NBATCHES=700$. }
    \label{tab:results_NoA2_appendix}
\end{table}
\end{minipage}
\begin{minipage}{0.45\linewidth}
    \begin{figure}[H]
        \centering
    \includegraphics[width=0.93\textwidth]{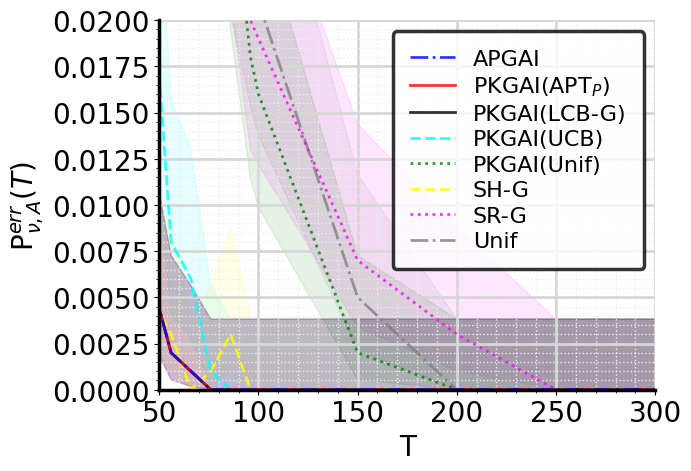}
\caption{Empirical error (\textsc{NoA1}).}
    \label{fig:results_NoA1_appendix}
    \end{figure}
    \begin{figure}[H]
        \centering
    \includegraphics[width=0.93\textwidth]{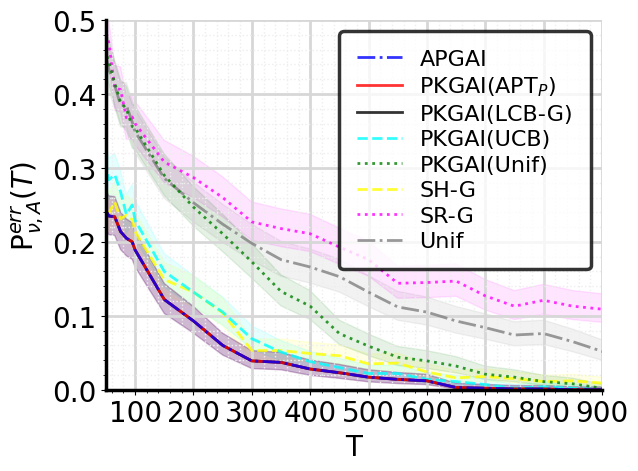}
    \caption{Empirical error (\textsc{NoA2}).}
\label{fig:results_NoA2_appendix}
    \end{figure}
\end{minipage}

    \begin{figure}[p]
        \centering
    \includegraphics[width=0.45\textwidth]{images/FBexpe/Figure_1.png}
    \includegraphics[width=0.45\textwidth]{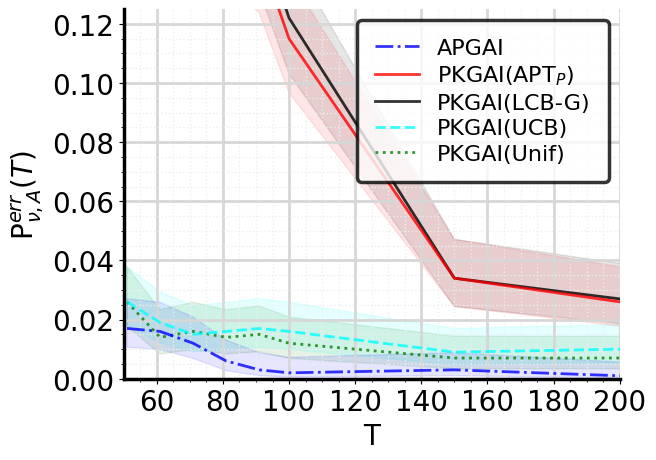}
    \caption{Empirical error on instance \textsc{RealL}. \textbf{Left}: with threshold functions from Equation~Eq.~\eqref{eq:thresholds_fb}. \textbf{Right}: with prior knowledge thresholds in Equation~Eq.~\eqref{eq:thresholds_pk_fb}.}
\label{fig:results_RealL_compare_thresholds}
    \end{figure}

        \begin{figure}[p]
        \centering
    \includegraphics[width=0.45\textwidth]{images/FBexpe/Figure_2.png}
    \includegraphics[width=0.45\textwidth]{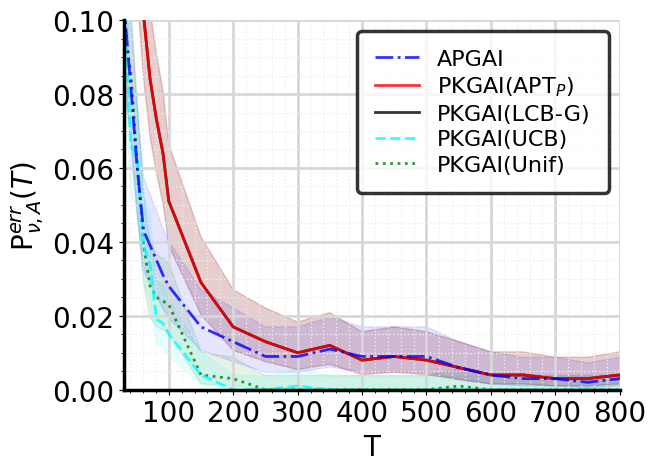}
    \caption{Empirical error on instance \textsc{IsA1}. \textbf{Left}: with threshold functions from Equation~Eq.~\eqref{eq:thresholds_fb}. \textbf{Right}: with prior knowledge thresholds in Equation~Eq.~\eqref{eq:thresholds_pk_fb}.}
\label{fig:results_IsA1_compare_thresholds}
    \end{figure}

        \begin{figure}[p]
        \centering
    \includegraphics[width=0.45\textwidth]{images/FBexpe/Figure_3.png}
    \includegraphics[width=0.45\textwidth]{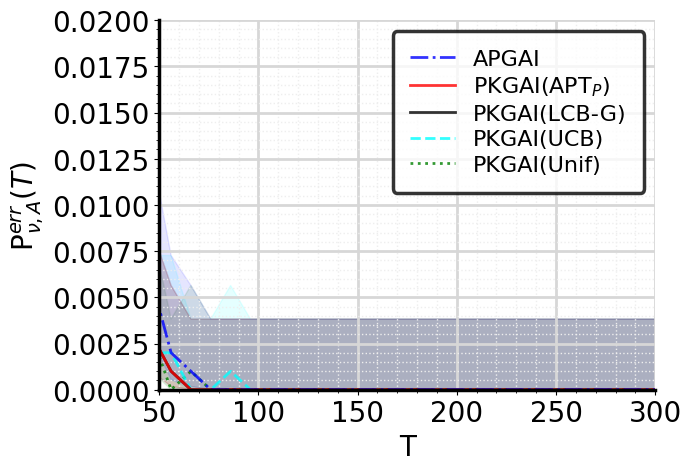}
    \caption{Empirical error on instance \textsc{NoA1}. \textbf{Left}: with threshold functions from Equation~Eq.~\eqref{eq:thresholds_fb}. \textbf{Right}: with prior knowledge thresholds in Equation~Eq.~\eqref{eq:thresholds_pk_fb}.}
\label{fig:results_NoA1_compare_thresholds}
    \end{figure}

        \begin{figure}[p]
        \centering
    \includegraphics[width=0.45\textwidth]{images/FBexpe/Figure_4.png}
    \includegraphics[width=0.45\textwidth]{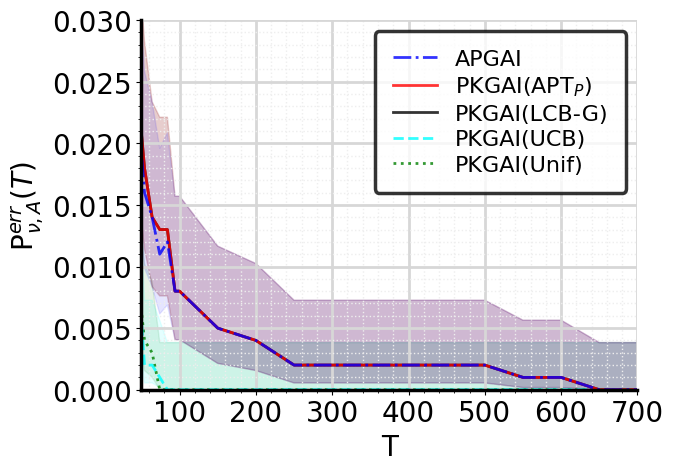}
    \caption{Empirical error on instance \textsc{IsA2}. \textbf{Left}: with threshold functions from Equation~Eq.~\eqref{eq:thresholds_fb}. \textbf{Right}: with prior knowledge thresholds in Equation~Eq.~\eqref{eq:thresholds_pk_fb}.}
\label{fig:results_IsA2_compare_thresholds}
    \end{figure}

        \begin{figure}[p]
        \centering
    \includegraphics[width=0.45\textwidth]{images/FBexpe/Figure_5.png}
    \includegraphics[width=0.45\textwidth]{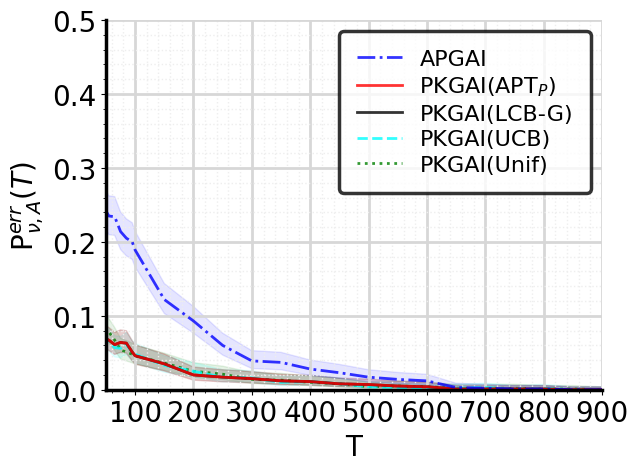}
    \caption{Empirical error on instance \textsc{NoA2}. \textbf{Left}: with threshold functions from Equation~Eq.~\eqref{eq:thresholds_fb}. \textbf{Right}: with prior knowledge thresholds in Equation~Eq.~\eqref{eq:thresholds_pk_fb}.}
\label{fig:results_NoA2_compare_thresholds}
    \end{figure}

\textit{On prior-knowledge based threshold functions.} 
For the sake of completeness, we have also iterated those experiments using the prior-knowledge threshold functions (in practice, they are unavailable) in algorithms belonging to the~\hyperlink{PKGAI}{PKGAI} family.

In those figures, when plotting the empirical curves for~\hyperlink{PKGAI}{PKGAI}-like algorithms, we also report on the same plot the corresponding curve for our contribution~\hyperlink{APGAI}{APGAI} (which is not expected to be different from the one on the left-hand plot, as the change in thresholds only affects~\hyperlink{PKGAI}{PKGAI}-like algorithms). 
As expected, the use of the prior-knowledge-based thresholds considerably improves the performance of~\hyperlink{PKGAI}{PKGAI} algorithms across most of the considered instances (except for \textsc{RealL} in Figure~\ref{fig:results_RealL_compare_thresholds} where the performance of index policies APT$_P$ and LUCB-G is severely impacted). 
However, more specifically in instances \textsc{IsA2} (Figure~\ref{fig:results_IsA2_compare_thresholds}), \textsc{NoA1} (Figure~\ref{fig:results_NoA1_compare_thresholds}), \textsc{IsA1} (Figure~\ref{fig:results_IsA1_compare_thresholds}) and \textsc{RealL} (Figure~\ref{fig:results_RealL_compare_thresholds}), we can notice that the gap in performance between~\hyperlink{APGAI}{APGAI} and algorithms from the~\hyperlink{PKGAI}{PKGAI} (and more surprisingly,~\hyperlink{PKGAI}{PKGAI}(Unif)) is not very large. 
This means that the theoretical gap in Table~\ref{tab:summary_anytimeGAI} does not necessarily translate into practice and highlights the need for more refined tools for the analysis of these algorithms.

\subsection{Supplementary Results on Anytime Empirical Error}
\label{app:ssec_supp_emp_error_anytime}

Since we are interested in the empirical error holding for any time, we conly consider the anytime algorithms: \hyperlink{APGAI}{APGAI}, Unif, DSR-G and DSH-G.
As mentioned in Appendix~\ref{app:ssec_impelementation_details}, we consider the implementation DSH-G-WR (``without refresh`') which keeps all the history within each SH instance. 
We repeat our experiments over $10000$ runs.
We display the mean empirical error and shaded area corresponds to Wilson confidence intervals~\citep{wilson1927probable} with confidence $95\%$.

In summary, our experiments show that \hyperlink{APGAI}{APGAI} significantly outperforms all the other anytime algorithms when $\set{\THRESHOLD}(\mu) = \emptyset$.
When $\set{\THRESHOLD}(\mu) \ne \emptyset$, \hyperlink{APGAI}{APGAI} has always better performance than DSR-G and DSH-G, and it performs on par with Unif.
Our empirical results suggest that \hyperlink{APGAI}{APGAI} enjoys better empirical performance than suggested by the theoretical guarantees summarized in Table~\ref{tab:summary_anytimeGAI}.

\begin{figure}[p]
    \centering
    \clemence{(a)} \includegraphics[width=0.45\linewidth]{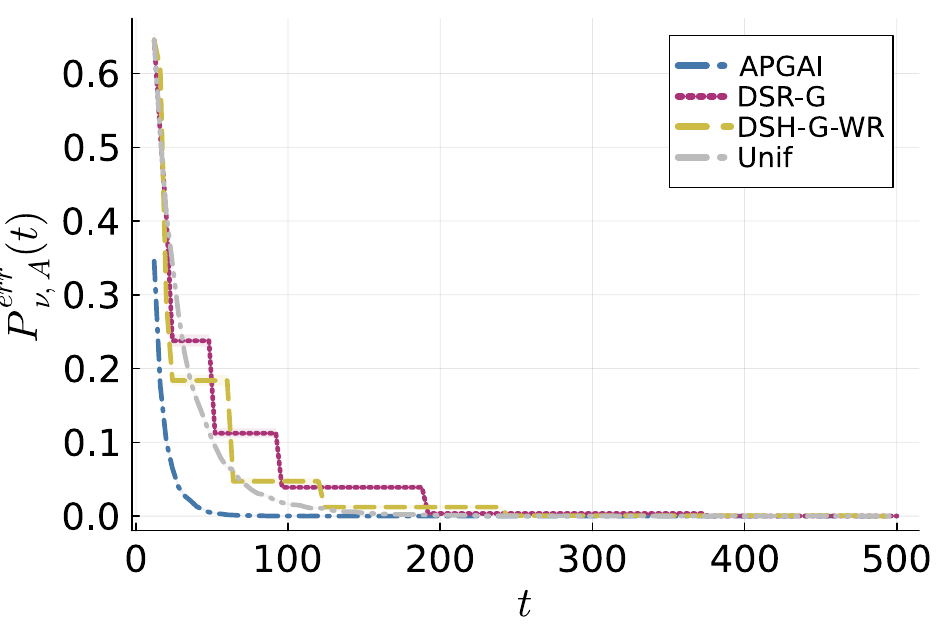}
    \clemence{(b)} \includegraphics[width=0.45\linewidth]{images/Anyexpe/plot_allAny_PoE_exp_NoArm_K4_PoE_init1_partial4_N10000_delta01.pdf}
    \caption{Empirical error on instances (a) \textsc{NoA1} and (b) \textsc{NoA2}. ``-WR'' means that each SH instance keeps all its history instead of discarding it.}
    \label{fig:supp_PoE_no_good_arms}
\end{figure}

\textit{No good arms.}
Since \hyperlink{APGAI}{APGAI} has arguably the best theoretical guarantees when $\set{\THRESHOLD}(\mu) = \emptyset$, we expect it to have superior empirical performance on the instances \textsc{NoA1} and \textsc{NoA2}.
Figure~\ref{fig:supp_PoE_no_good_arms} validates empirically that \hyperlink{APGAI}{APGAI} significantly outperform all the other anytime algorithms by a large margin.
While Unif has the ``worse'' theoretical guarantees in Table~\ref{tab:summary_anytimeGAI}, the empirical study shows that it outperforms both DSR-G and DSH-G-WR.
This phenomenon is mainly due to the doubling trick.
Converting a fixed-budget algorithm to an anytime algorithm forces the algorithm to forget past observations, hence considerably impacting the empirical performance.

\begin{figure}
    \centering
    \includegraphics[width=0.32\linewidth]{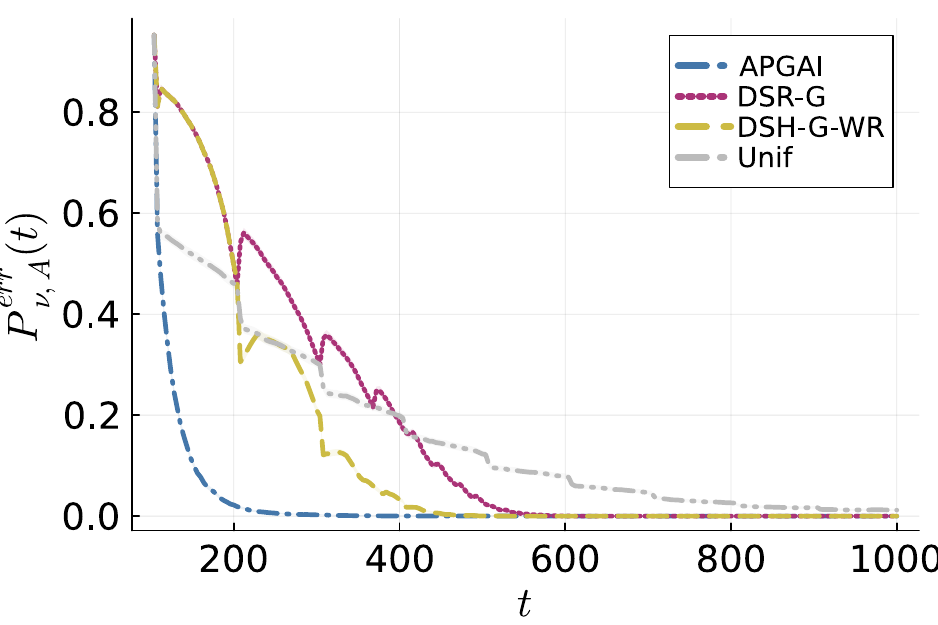}
    \includegraphics[width=0.32\linewidth]{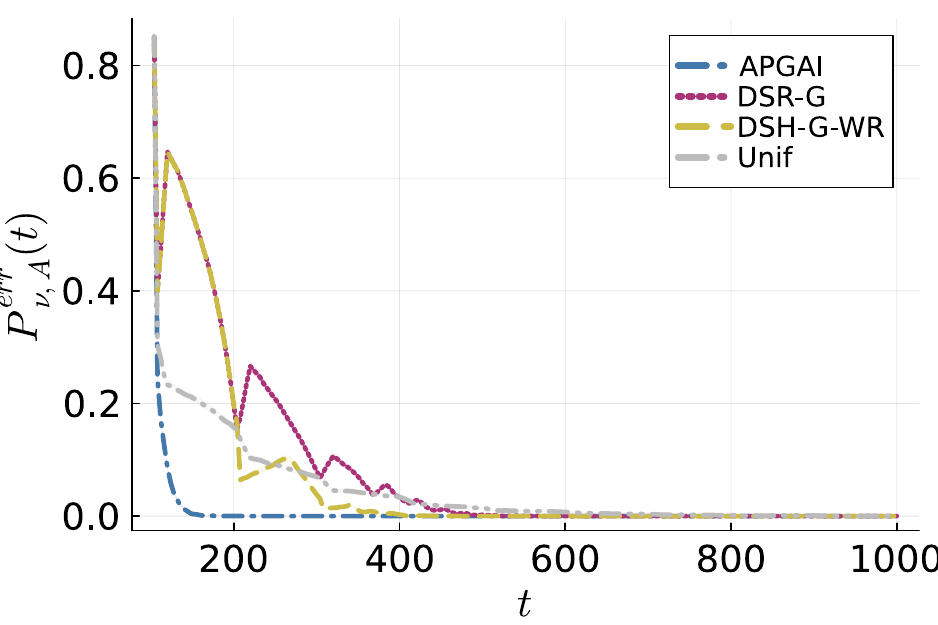}
    \includegraphics[width=0.32\linewidth]{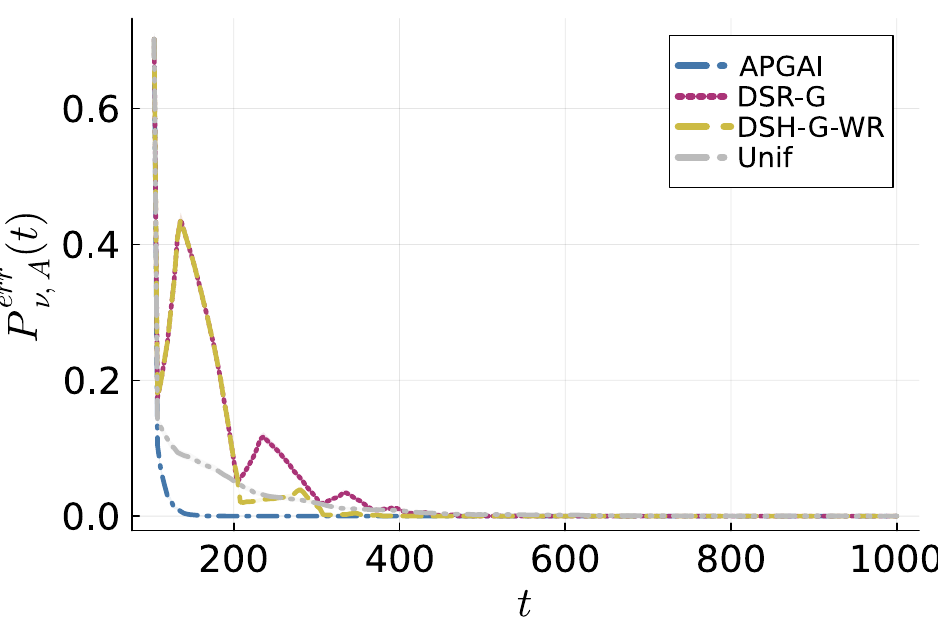}\\
    \includegraphics[width=0.32\linewidth]{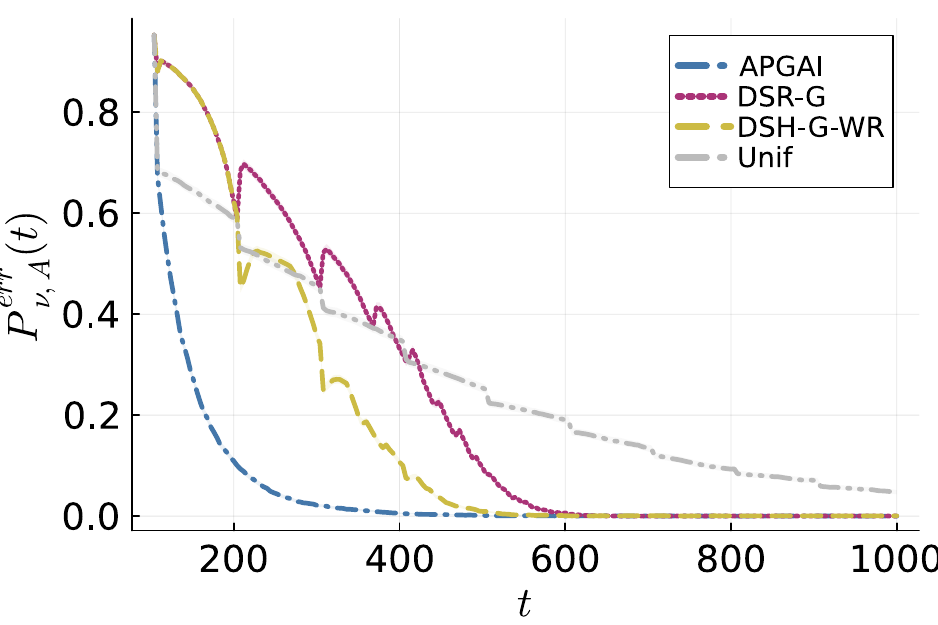}
    \includegraphics[width=0.32\linewidth]{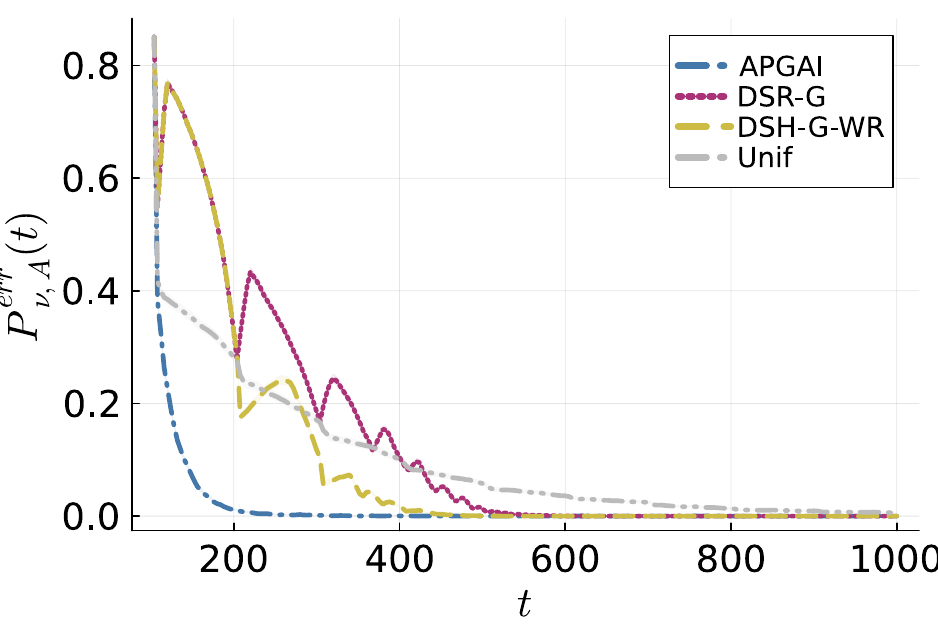}
    \includegraphics[width=0.32\linewidth]{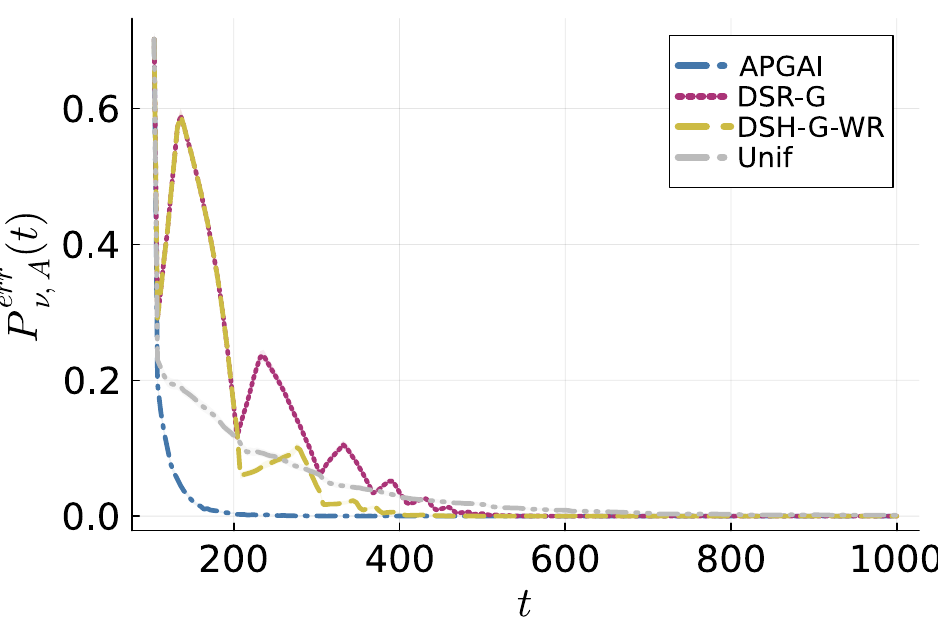}
    \caption{Empirical error for varying number of good arms $|\set{\THRESHOLD}(\mu)|  \in \{5, 15, 30\}$ (left to right) among $K=100$ arms on instances (top) \textsc{TwoG} and (bottom) \textsc{LinG}. ``-WR'' means that each SH instance keeps all its history instead of discarding it.}
    \label{fig:supp_PoE_varying_good_answers}
\end{figure}

\textit{Varying number of good arms.}
In Figure~\ref{fig:supp_PoE_varying_good_answers}, we study the impact of an increased number of good arms on the empirical error.
While Table~\ref{tab:summary_anytimeGAI} suggests that \hyperlink{APGAI}{APGAI} is not benefiting from increased $|\set{\THRESHOLD}(\mu)|$, we see that the empirical error is decreasing significantly as $|\set{\THRESHOLD}(\mu)|$ increases.
This suggests that better theoretical guarantees could be obtained when $\set{\THRESHOLD}(\mu) \ne \emptyset$.
It is an interesting direction for future research to show an asymptotic rate featuring a complexity inversly proportional to $|\set{\THRESHOLD}(\mu)|$.
In addition, we observe that \hyperlink{APGAI}{APGAI} outperforms all the other anytime algorithms by a large margin.
Intuitively, \hyperlink{APGAI}{APGAI} is greedy enough when $\set{\THRESHOLD}(\mu) \ne \emptyset$ to avoid sampling the arms which are not good.

\begin{figure}
    \centering
    \clemence{(a)} \includegraphics[width=0.45\linewidth]{images/Anyexpe/plot_allAny_PoE_exp_T3_KH19_K10_PoE_init1_partial4_N10000_delta01.pdf}
    \clemence{(b)} \includegraphics[width=0.45\linewidth]{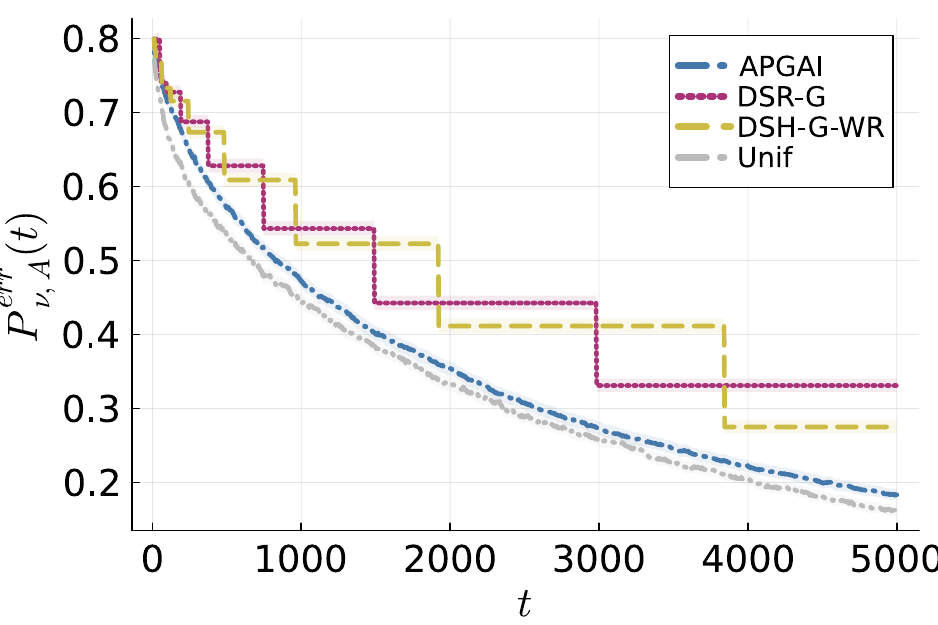}
    \caption{Empirical error on instances (a) \textsc{Thr3} and (b) \textsc{Med1}. ``-WR'' means that each SH instance keeps all its history instead of discarding it.}
    \label{fig:supp_PoE_similar_good_arms}
\end{figure}

\textit{Good arms with similar gaps.}
In light of Table~\ref{tab:summary_anytimeGAI}, one might expect that \hyperlink{APGAI}{APGAI} has worse empirical performance when $\set{\THRESHOLD}(\mu) \ne \emptyset$ compared to other anytime algorithms.
To assess this fact empirically, we first consider instances where the good arms have similar gaps, \eg \textsc{Thr3} and \textsc{Med1}.
In Figure~\ref{fig:supp_PoE_similar_good_arms}, we see that \hyperlink{APGAI}{APGAI} is better than Unif on \textsc{Thr3}, but worse on \textsc{Med1}.
In both cases, \hyperlink{APGAI}{APGAI} outperforms both DSR-G and DSH-G-WR.
Therefore, we see that \hyperlink{APGAI}{APGAI} has better empirical performance compared to the ones suggested by the theoretical guarantees summarized in Table~\ref{tab:summary_anytimeGAI}. 

\begin{figure}
    \centering
    \clemence{(a)} \includegraphics[width=0.45\linewidth]{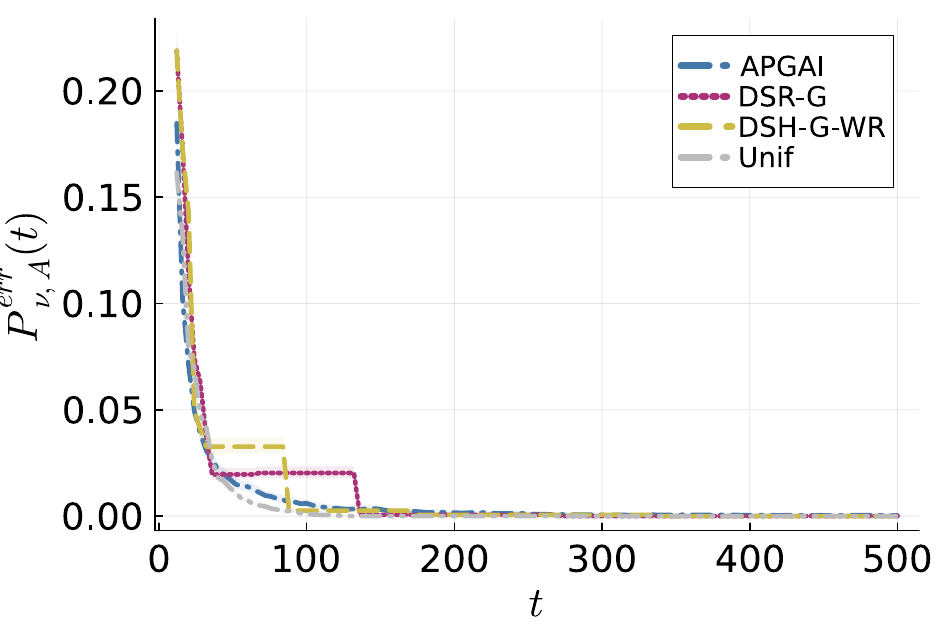}
    \clemence{(b)} \includegraphics[width=0.45\linewidth]{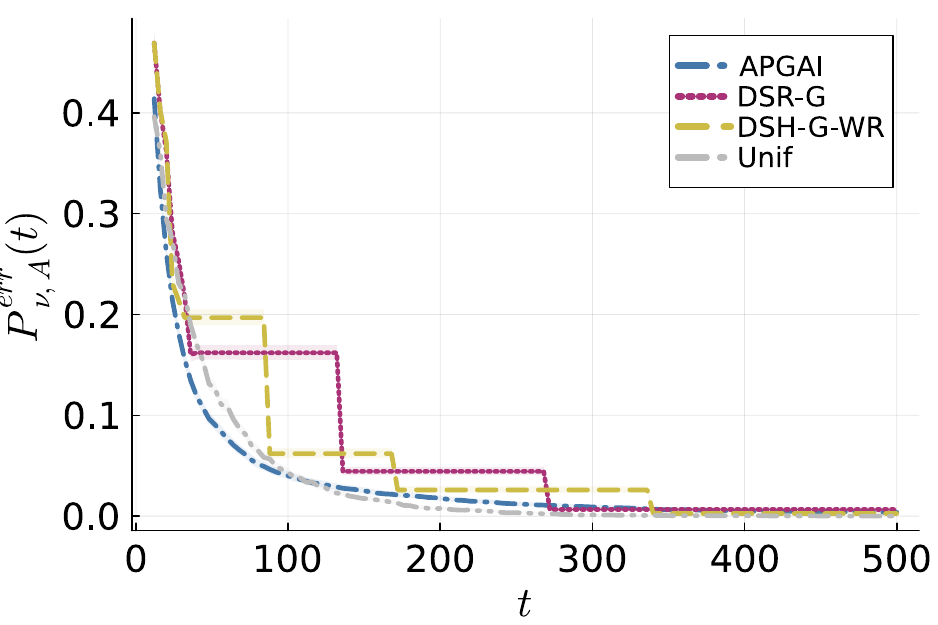}\\
    \clemence{(c)} \includegraphics[width=0.45\linewidth]{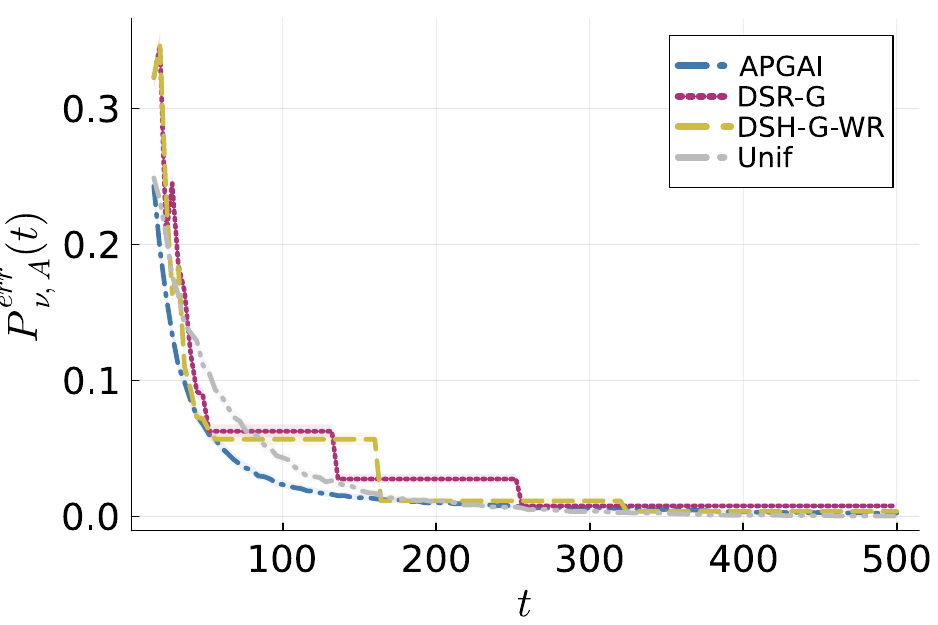}
    \clemence{(d)} \includegraphics[width=0.45\linewidth]{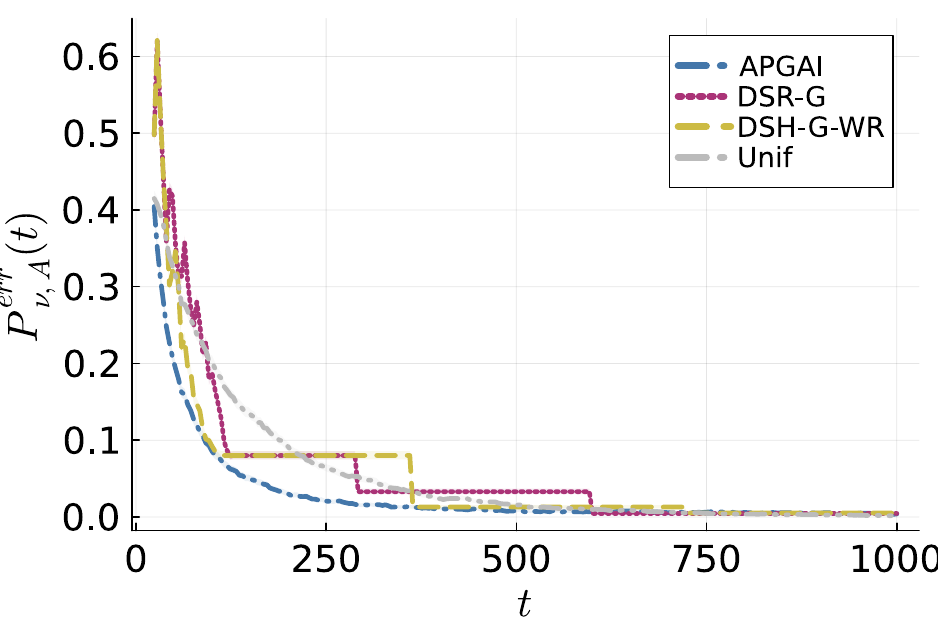}\\
    \clemence{(e)} \includegraphics[width=0.45\linewidth]{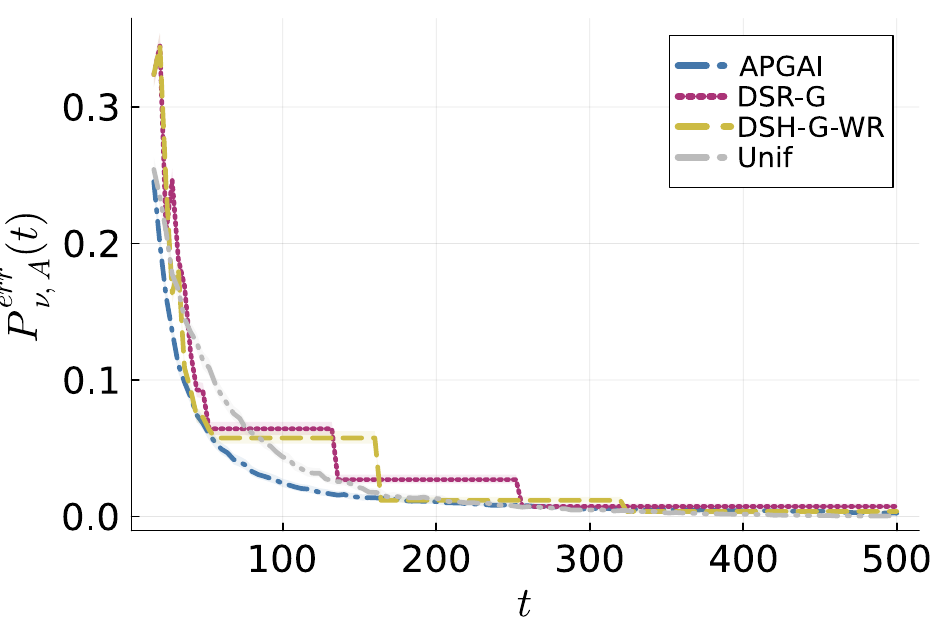}
    \clemence{(f)} \includegraphics[width=0.45\linewidth]{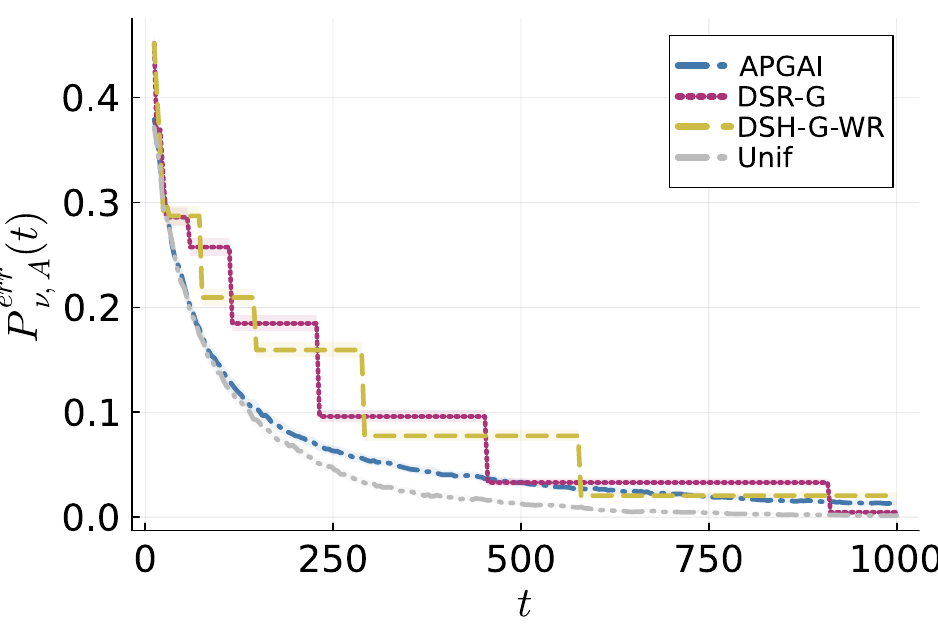}
    \caption{Empirical error on instances (a) \textsc{IsA2}, (b) \textsc{Med2}, (c) \textsc{IsA1}, (d) \textsc{RealL}, (e) \textsc{Thr1} and (f) \textsc{Thr2}. ``-WR'' means that each SH instance keeps all its history instead of discarding it.}
    \label{fig:supp_PoE_dissimilar_good_arms}
\end{figure}

\textit{Good arms with dissimilar gaps.}
In Figure~\ref{fig:supp_PoE_dissimilar_good_arms}, we consider instances where $\set{\THRESHOLD}(\mu) \ne \emptyset$ and good arms have dissimilar gaps.
Overall, \hyperlink{APGAI}{APGAI} always performs better than DSR-G and DSH-G-WR.
While Unif seems to outperform \hyperlink{APGAI}{APGAI} on some instances (\eg \textsc{Thr2} and \textsc{Med2}), it has worse performance on other instances (\eg \textsc{RealL} and \textsc{Thr1}).

\subsection{Supplementary Results on Empirical Stopping Time}
\label{app:ssec_supp_emp_stop_time}

While \hyperlink{APGAI}{APGAI} is designed to tackle anytime GAI, it also enjoys theoretical guarantees in the fixed-confidence setting when combined with the GLR stopping rule~Eq.~\eqref{eq:stopping_rule} with stopping threshold~Eq.~\eqref{eq:stopping_threshold}.
According to Table~\ref{tab:summary_FCGAI}, we expect that \hyperlink{APGAI}{APGAI} has good empirical performance when $\set{\THRESHOLD}(\mu) = \emptyset$, and sub-optimal ones when $\set{\THRESHOLD}(\mu) \ne \emptyset$.
Since we are interested in the empirical performance for moderate regime of confidence, we take $\delta = 0.01$ in the following.
We repeat our experiments over $1000$ runs.
We either display the boxplots or the mean with standard deviation as shaded area.

In summary, our experiments show that \hyperlink{APGAI}{APGAI} performs on par with all the other fixed-confidence algorithms when $\set{\THRESHOLD}(\mu) = \emptyset$.
When $\set{\THRESHOLD}(\mu) \ne \emptyset$, \hyperlink{APGAI}{APGAI} has good performance only when the good arms have similar gaps.
Importantly, its performance does not scale linearly with $|\set{\THRESHOLD}(\mu)|$ as suggested by Table~\ref{tab:summary_FCGAI}.
When good arms have dissimilar gaps, \hyperlink{APGAI}{APGAI} can suffer from large outliers due to the greedyness of it sampling rule.
Finally, we shows a simple way to circumvent this limitation by adding forced exloration on top of \hyperlink{APGAI}{APGAI}.

\begin{figure}[p]
    \centering
    \clemence{(a)} \includegraphics[width=0.45\linewidth]{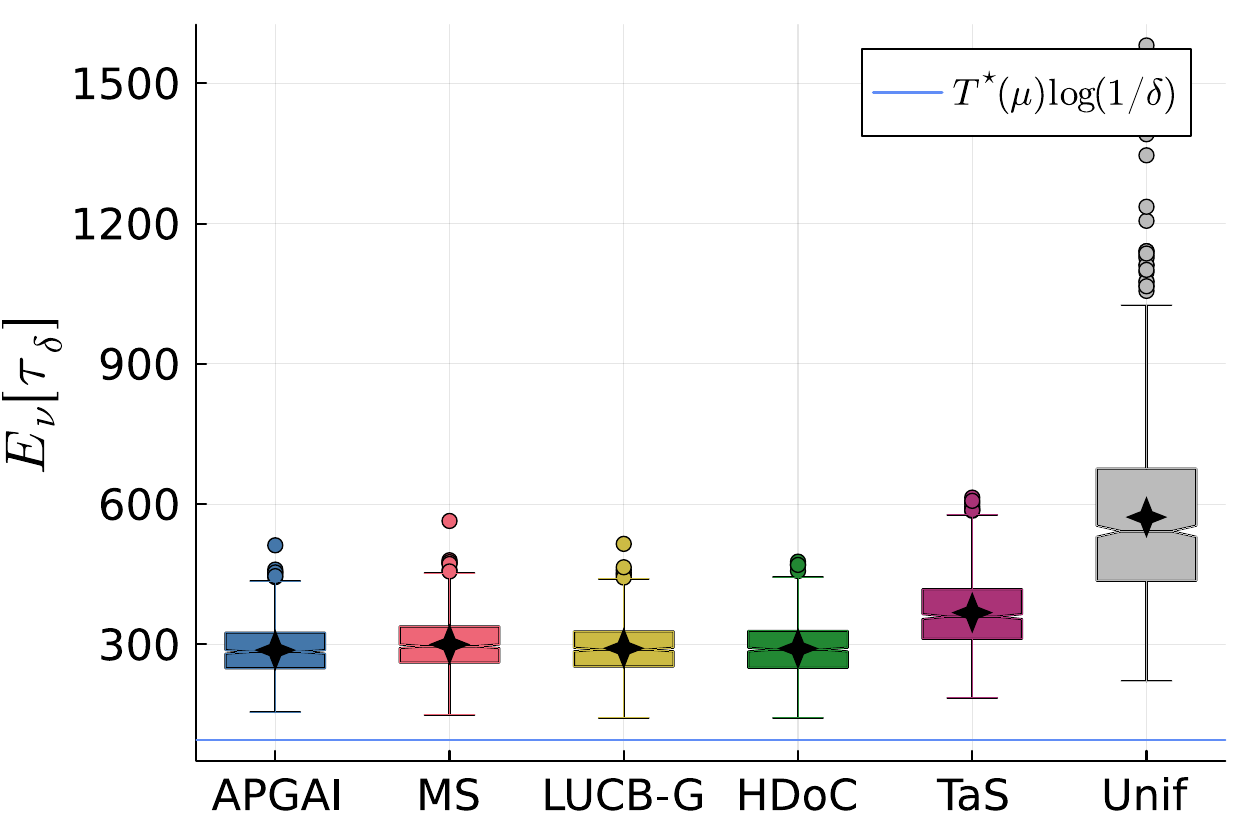}
    \clemence{(b)} \includegraphics[width=0.45\linewidth]{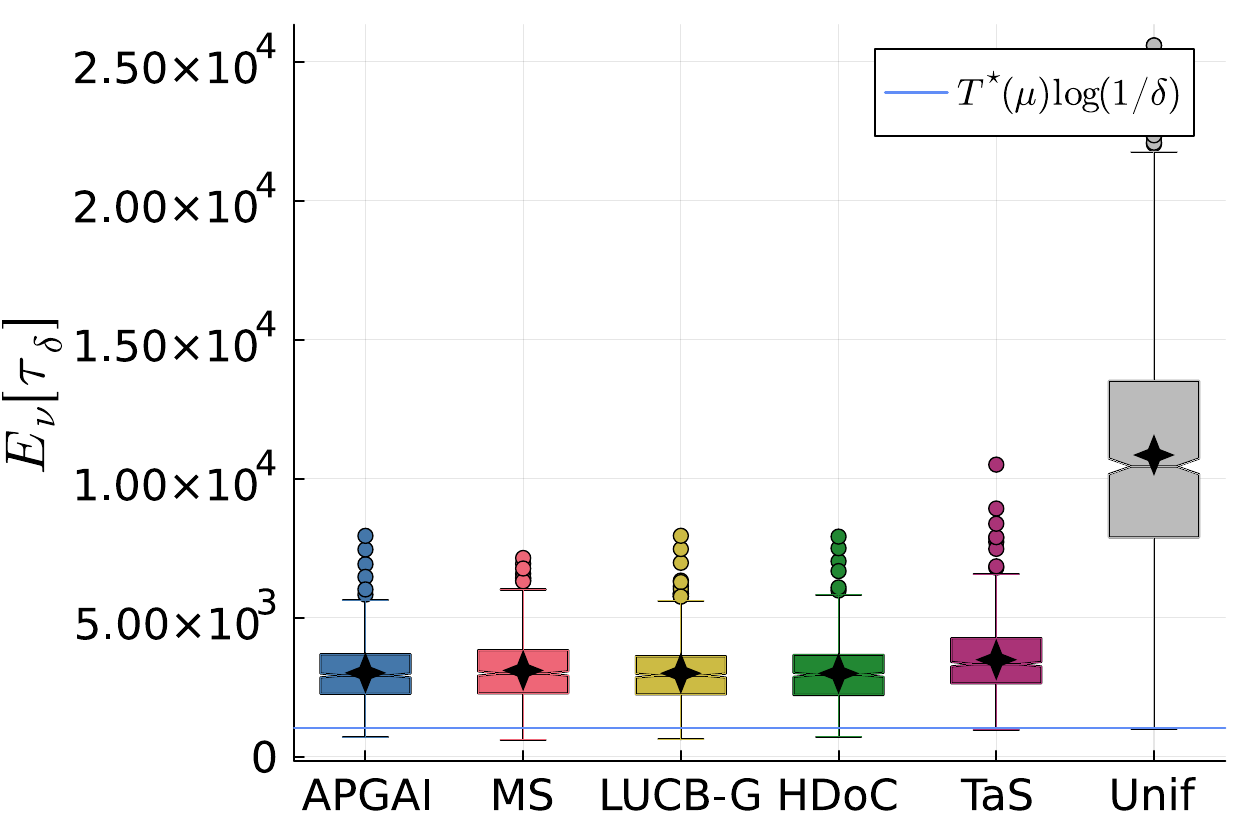}
    \caption{Empirical stopping time ($\delta = 0.01$) on instances (a) \textsc{NoA1} and (b) \textsc{NoA2}. ``MS'' is Murphy Sampling, ``TaS'' is Track-and-Stop and ``Unif'' is round-robin uniform sampling.}
    \label{fig:supp_no_good_arms}
\end{figure}

\textit{No good arms.}
Since \hyperlink{APGAI}{APGAI} is asymptotically optimal when $\set{\THRESHOLD}(\mu) = \emptyset$, we expect it to perform well on the instances \textsc{NoA1} and \textsc{NoA2}.
Figure~\ref{fig:supp_no_good_arms} shows that \hyperlink{APGAI}{APGAI} has comparable performance with existing fixed-confidence GAI algorithms on such instances, and that uniform sampling performs poorly.

\begin{figure}[p]
    \centering
    \clemence{(a)} \includegraphics[width=0.45\linewidth]{images/FCexpe/varying_good_answers_2G_KKG18_N1000.pdf}
    \clemence{(b)} \includegraphics[width=0.45\linewidth]{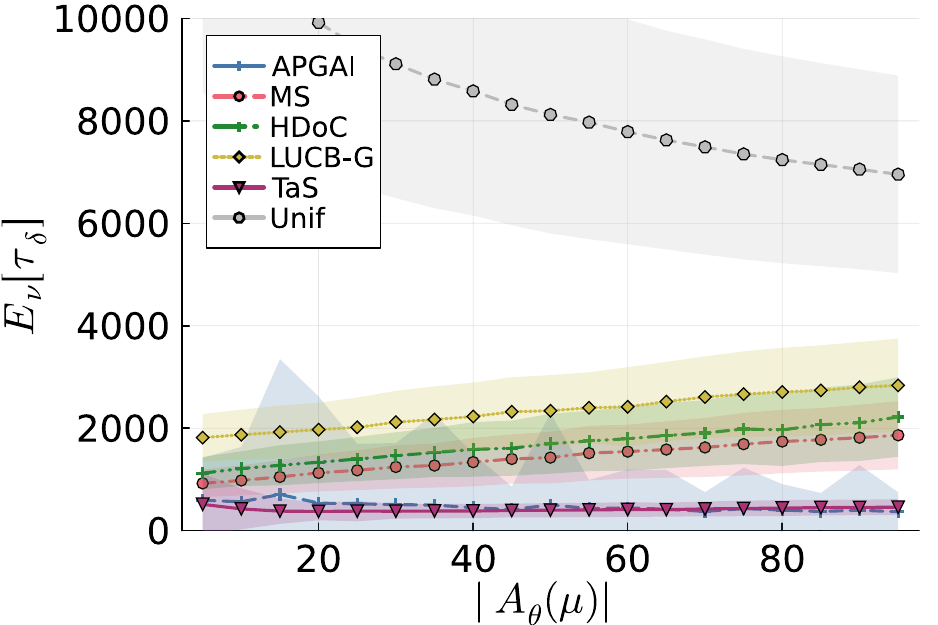}
    \caption{Empirical stopping time ($\delta = 0.01$) for varying number of good arms $|\set{\THRESHOLD}(\mu)|  \in \{5 k\}_{k \in [19]}$ among $K=100$ arms on instances (a) \textsc{TwoG} and (b) \textsc{LinG}. ``MS'' is Murphy Sampling, ``TaS'' is Track-and-Stop and ``Unif'' is round-robin uniform sampling.}
    \label{fig:supp_varying_good_answers}
\end{figure}

\textit{Varying number of good arms.}
In Figure~\ref{fig:supp_varying_good_answers}, we study the impact of an increased number of good arms on the empirical error.
While Table~\ref{tab:summary_FCGAI} suggests that \hyperlink{APGAI}{APGAI} is suffering from increased $|\set{\THRESHOLD}(\mu)|$ due to the dependency in $H_{\theta}(\mu)$, we see that the empirical stopping time remains the same when $|\set{\THRESHOLD}(\mu)|  \in \{5 k\}_{k \in [19]}$.
Therefore, Figure~\ref{fig:supp_varying_good_answers} empirically validate our theoretical intuition that \hyperlink{APGAI}{APGAI} can achieve an asymptotic upper bound of the order $2\max_{a \in \set{\THRESHOLD}(\mu)} \Delta_{a}^{-2} \log(1/\delta)$ as discussed in Appendix~\ref{app:sssec_discussion_suboptimality}.
On the \textsc{LinG}, we also observe that \hyperlink{APGAI}{APGAI} can have large outliers due to the good arms with small gaps (see below for more details).

\begin{figure}
    \centering
    \clemence{(a)} \includegraphics[width=0.45\linewidth]{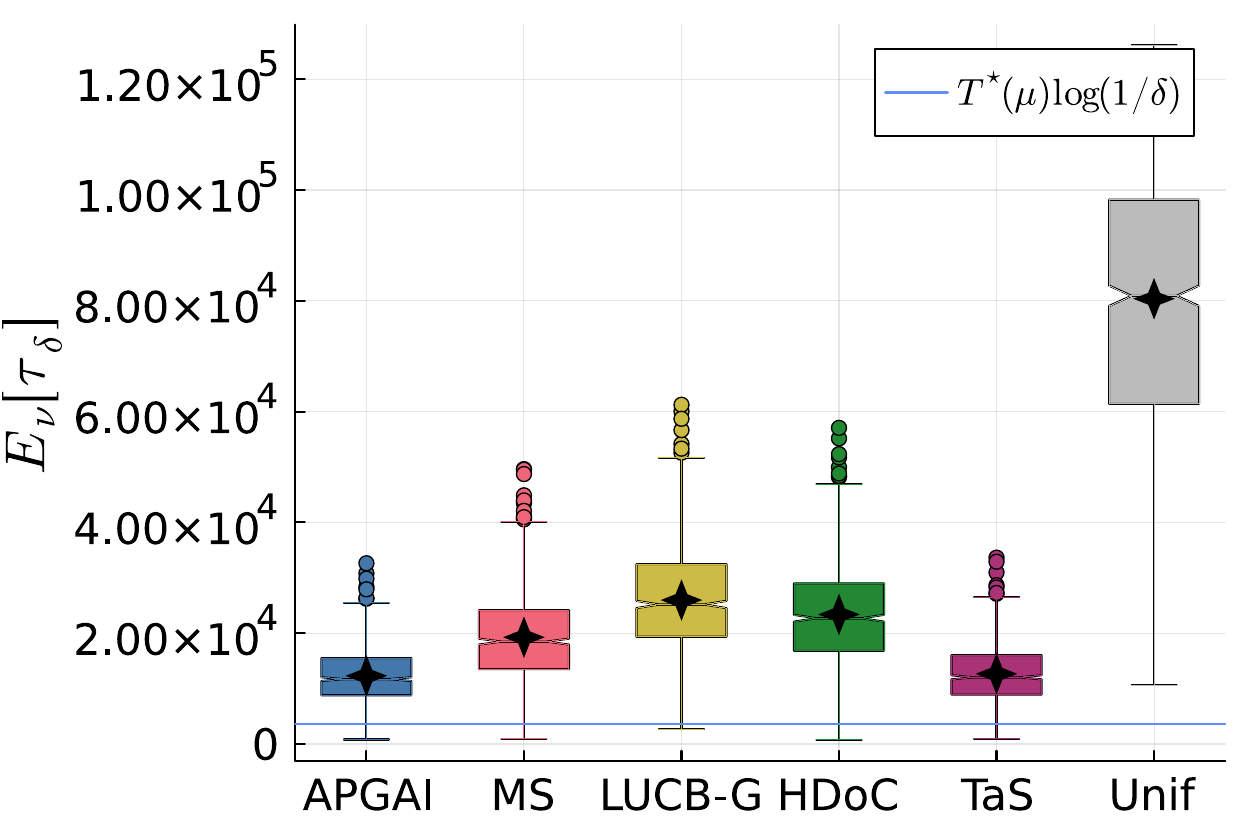}
    \clemence{(b)} \includegraphics[width=0.45\linewidth]{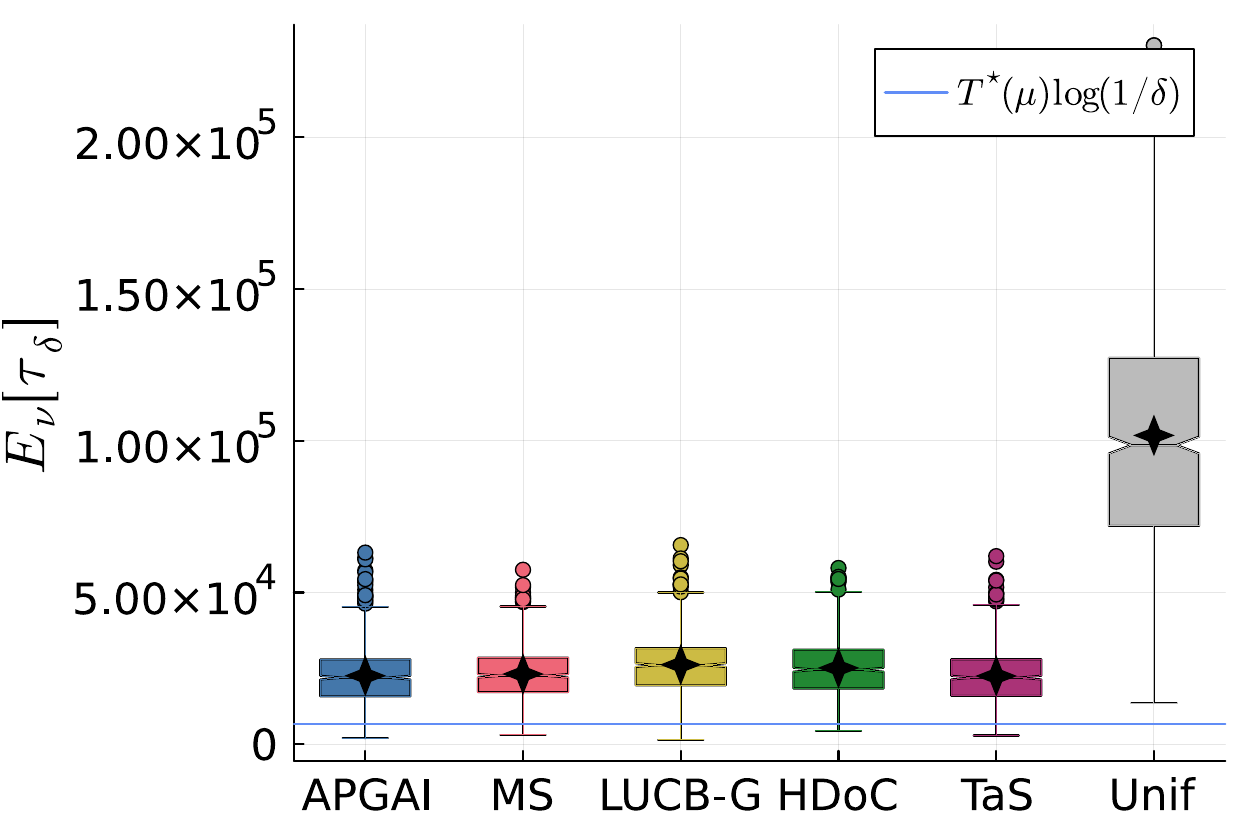}
    \caption{Empirical stopping time ($\delta = 0.01$) on instances (a) \textsc{Thr3} and (b) \textsc{Med1}. ``MS'' is Murphy Sampling, ``TaS'' is Track-and-Stop and ``Unif'' is round-robin uniform sampling.}
    \label{fig:supp_similar_good_arms}
\end{figure}

\textit{Good arms with similar gaps.}
When $\set{\THRESHOLD}(\mu) \ne \emptyset$ and good arms have similar means, Table~\ref{tab:summary_FCGAI} suggests that \hyperlink{APGAI}{APGAI} could be competitive with other algorithms.
Figure~\ref{fig:supp_similar_good_arms} validates this observation empirically.
On the \textsc{Thr3} instance, \hyperlink{APGAI}{APGAI} achieves better performance than the other fixed-confidence algorithms, except for Track-and-Stop which has similar performance.

\begin{figure}
    \centering
   \clemence{(a)} \includegraphics[width=0.45\linewidth]{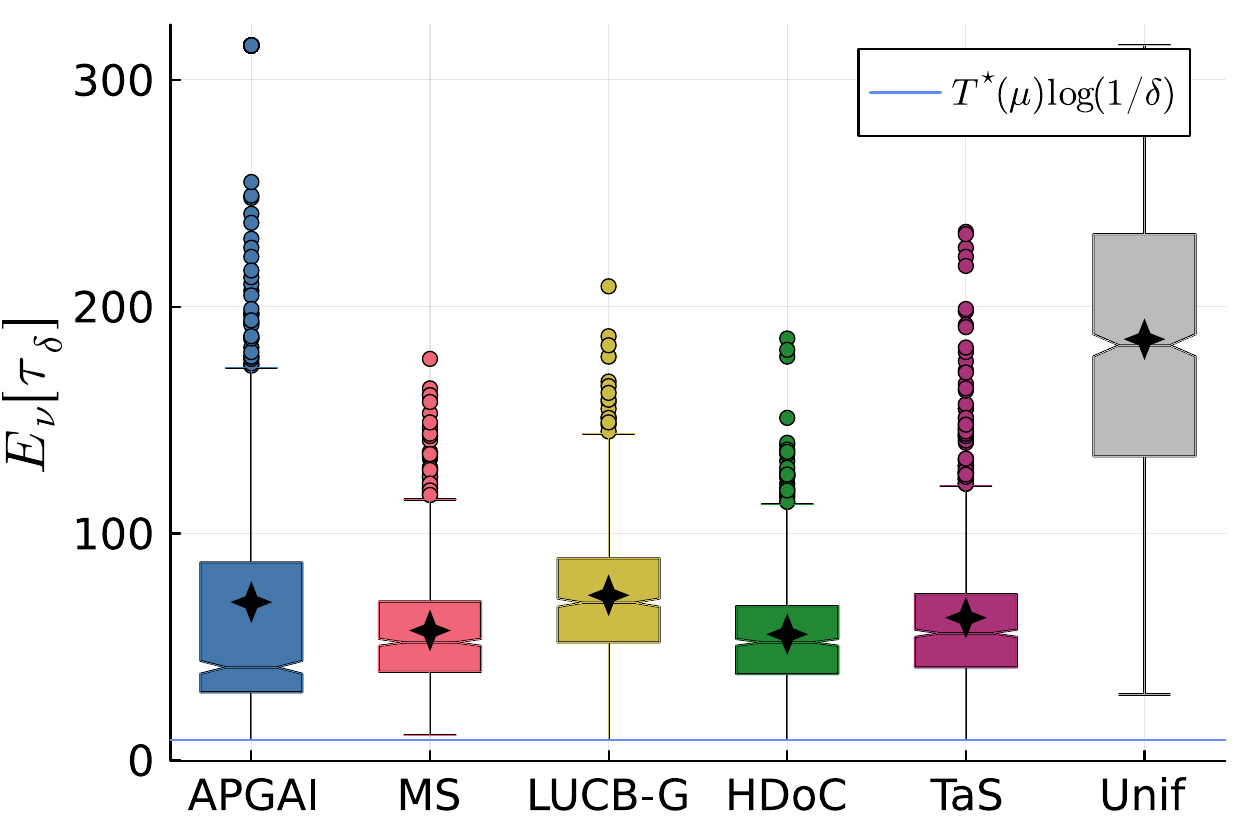}
    \clemence{(b)} \includegraphics[width=0.45\linewidth]{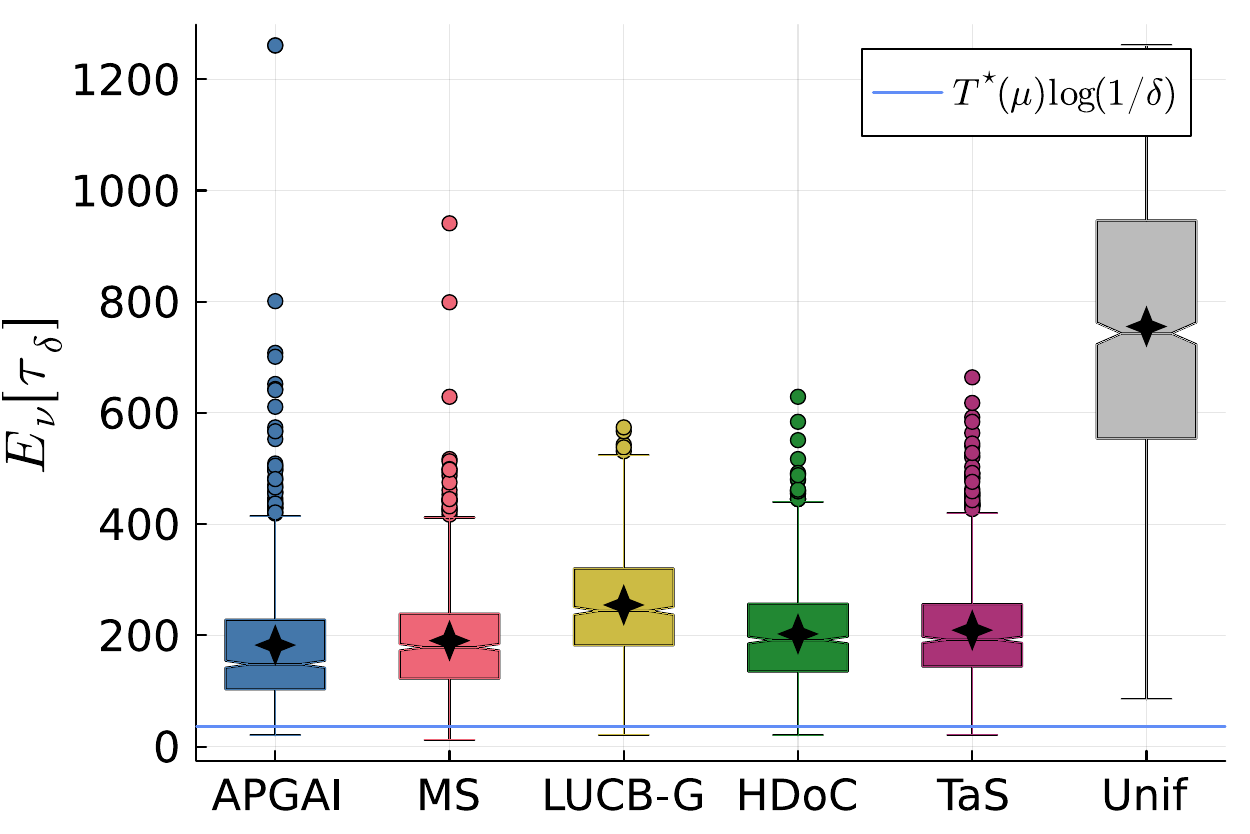}\\
   \clemence{(c)} \includegraphics[width=0.45\linewidth]{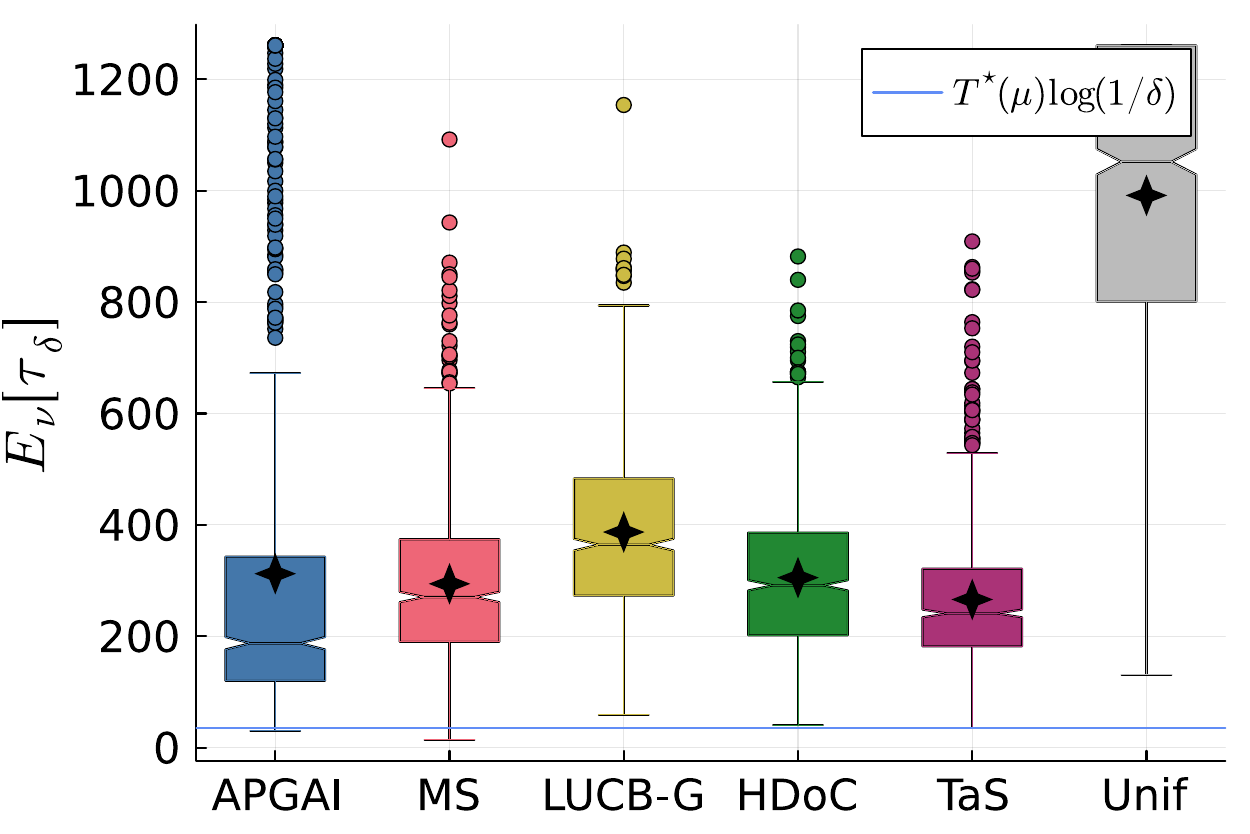} 
   \clemence{(d)} \includegraphics[width=0.45\linewidth]{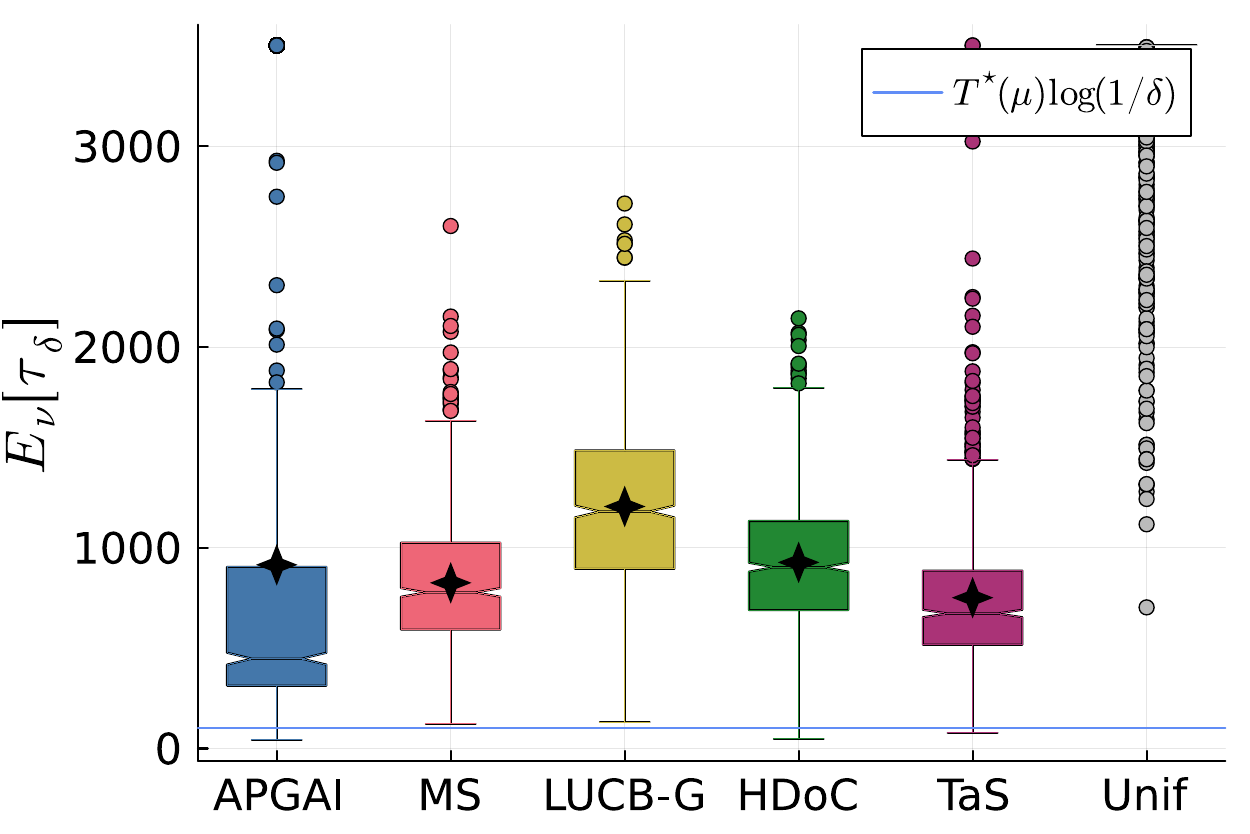}\\
   \clemence{(e)} \includegraphics[width=0.45\linewidth]{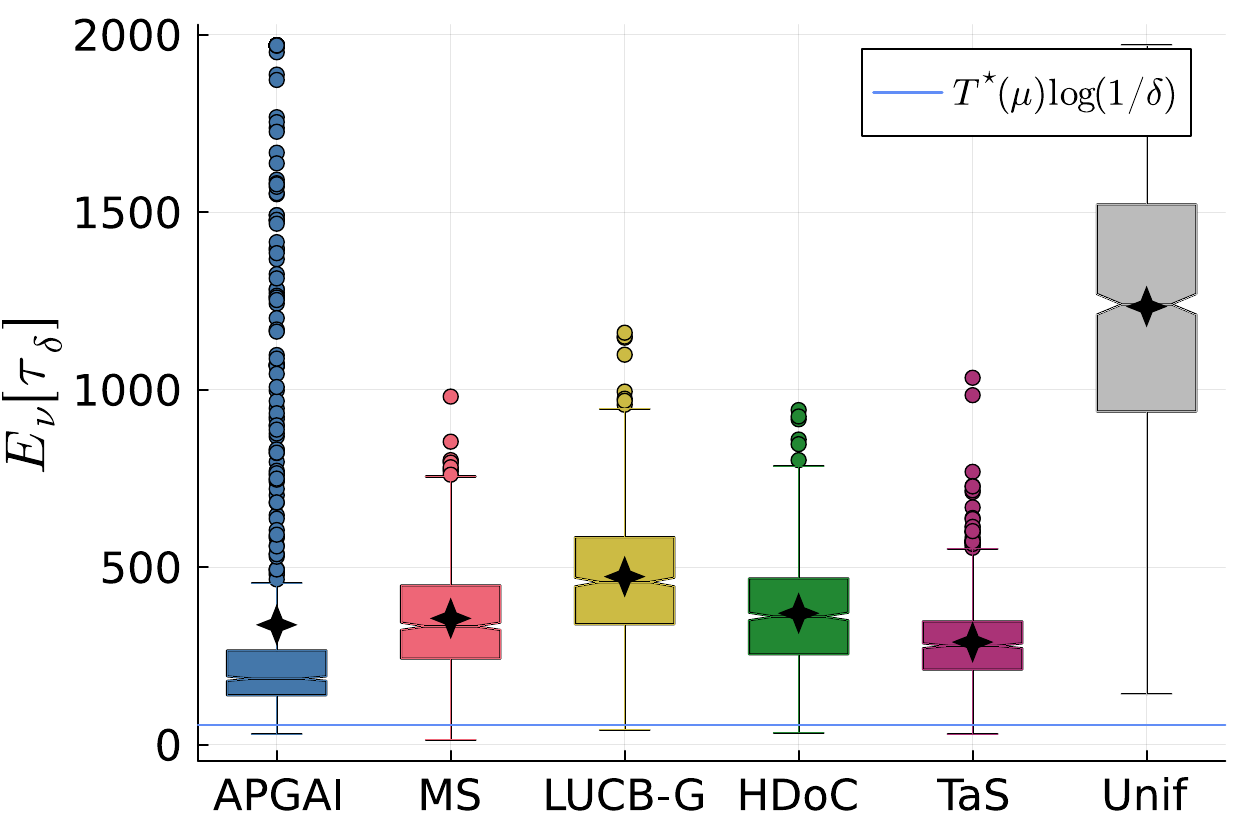}
  \clemence{(f)}  \includegraphics[width=0.45\linewidth]{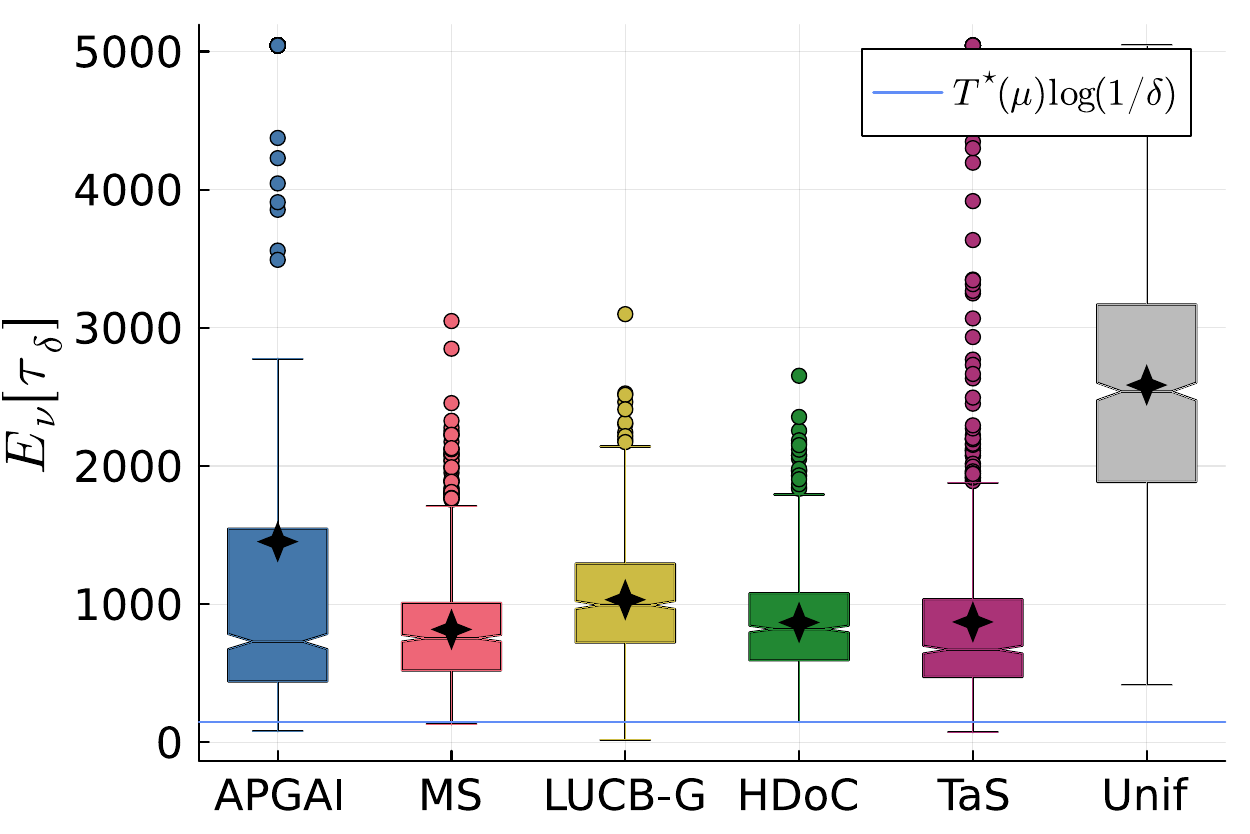}
    \caption{Empirical stopping time ($\delta = 0.01$) on instances (a) \textsc{IsA2}, (b) \textsc{Med2}, (c) \textsc{IsA1}, (d) \textsc{RealL}, (e) \textsc{Thr1} and (f) \textsc{Thr2}. ``MS'' is Murphy Sampling, ``TaS'' is Track-and-Stop and ``Unif'' is round-robin uniform sampling.}
    \label{fig:supp_dissimilar_good_arms}
\end{figure}

\textit{Good arms with dissimilar gaps.}
In Figure~\ref{fig:supp_dissimilar_good_arms}, we consider instances where $\set{\THRESHOLD}(\mu) \ne \emptyset$ and good arms have dissimilar gaps.
Table~\ref{tab:summary_FCGAI} suggests that \hyperlink{APGAI}{APGAI} can have poor empirical performance on such instances.
Empirically, we see that \hyperlink{APGAI}{APGAI} can suffer from very large outliers on such instances.
Depending on the initial draws, the greedy sampling rule of \hyperlink{APGAI}{APGAI} can focus on a good arm with small gap $\Delta_{a}$ instead of verifying a good arm with large gap $\Delta_{a}$.
Since those arms are significantly harder to verify, \hyperlink{APGAI}{APGAI} will incur a large empirical stopping time in that case.
This explains why the distribution of the empirical stopping time has a heavy tail with large outliers. 
\marc{A right-skewed stopping time distribution is not a desirable property in practical application, \hyperlink{APGAI}{APGAI} is not a good fixed-confidence GAI algorithm on instances with good arms have dissimilar gaps. }

\begin{figure}
    \centering
    \includegraphics[width=0.32\linewidth]{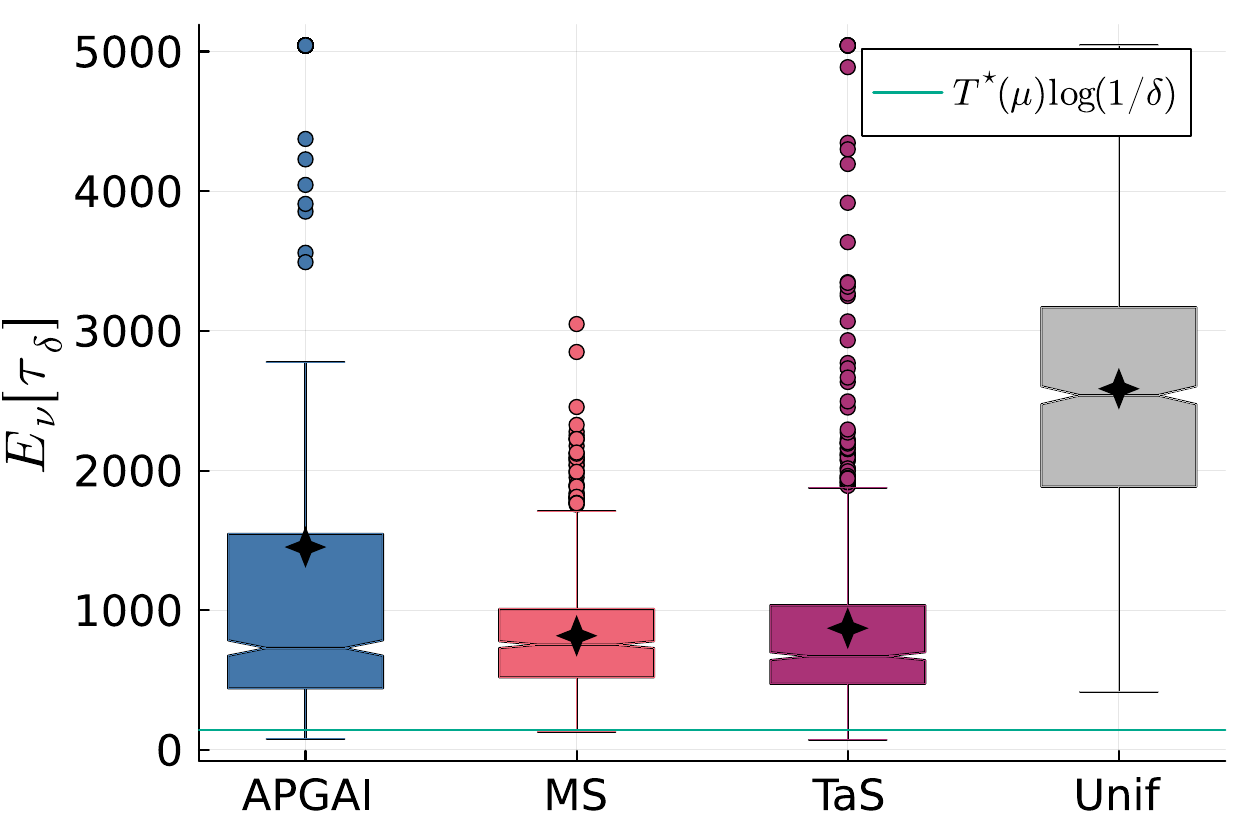}
    \includegraphics[width=0.32\linewidth]{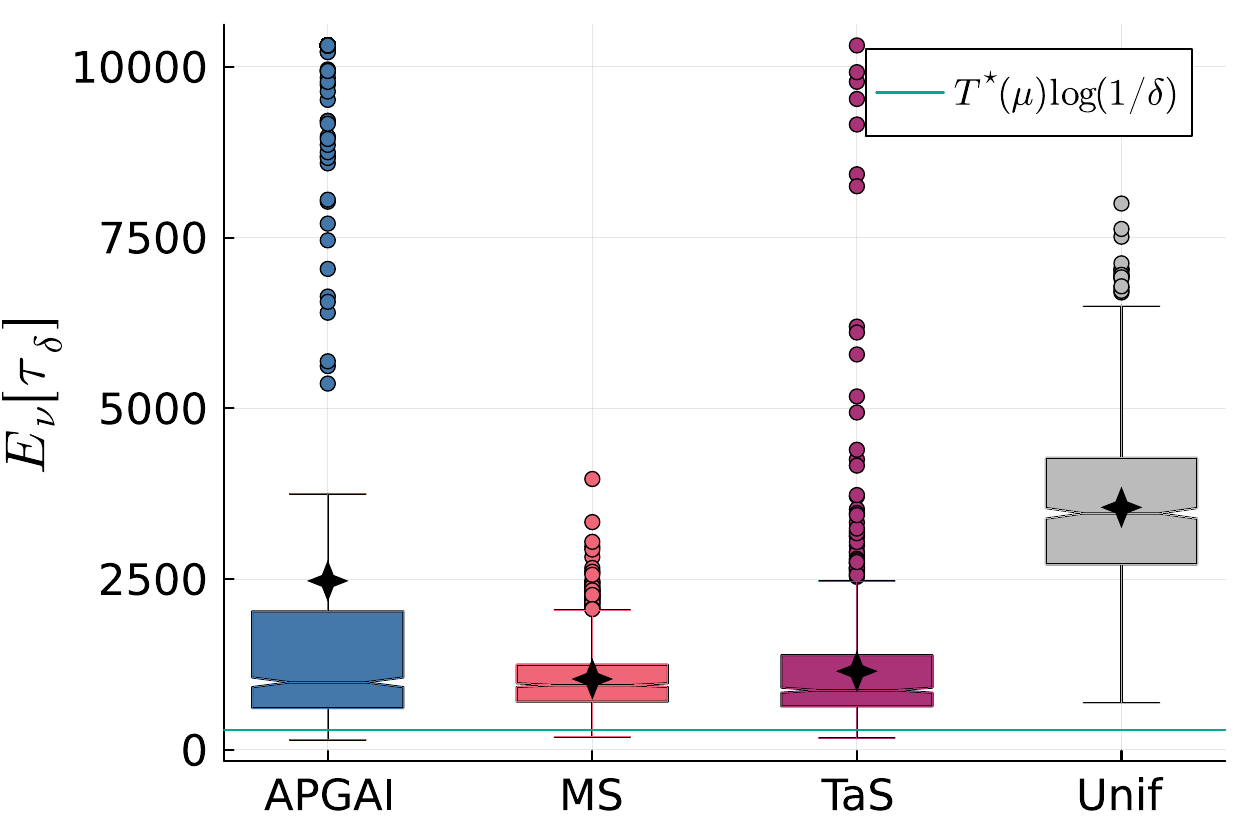}
    \includegraphics[width=0.32\linewidth]{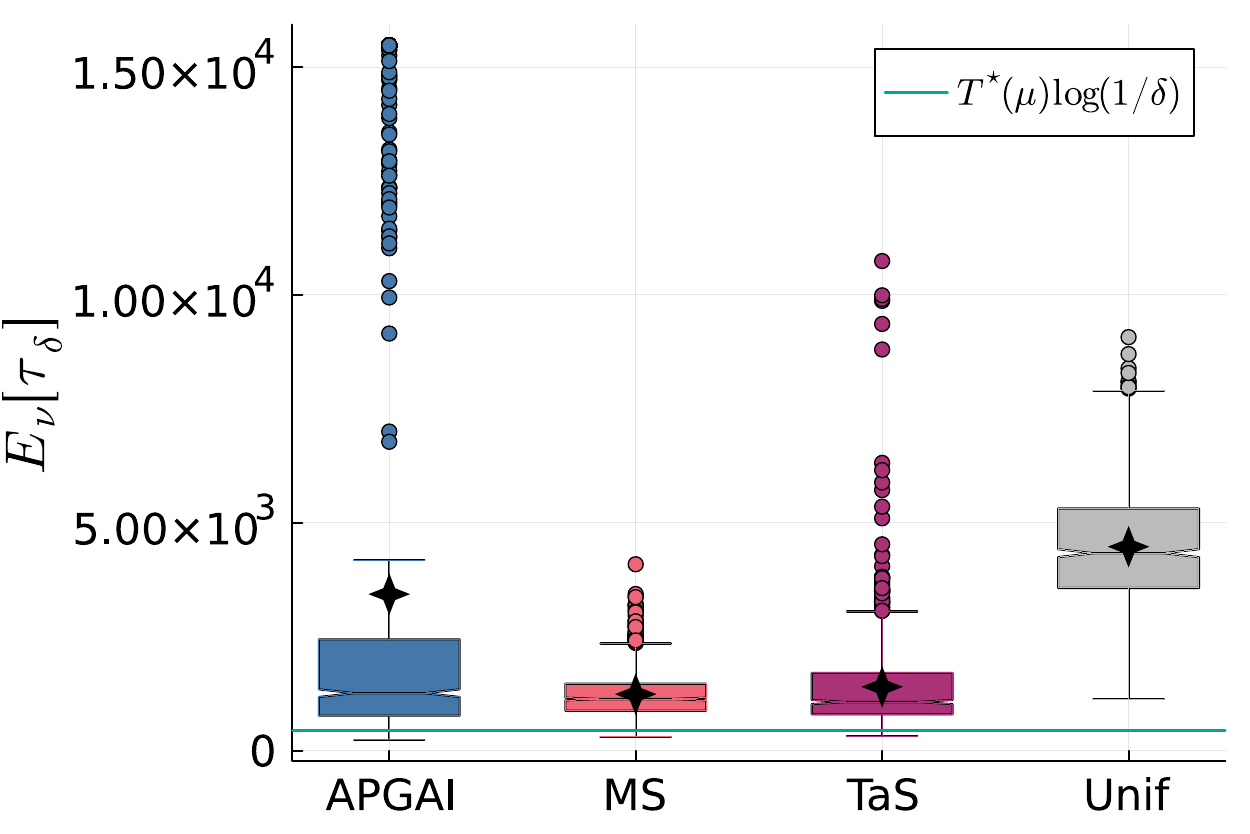}\\
    \includegraphics[width=0.32\linewidth]{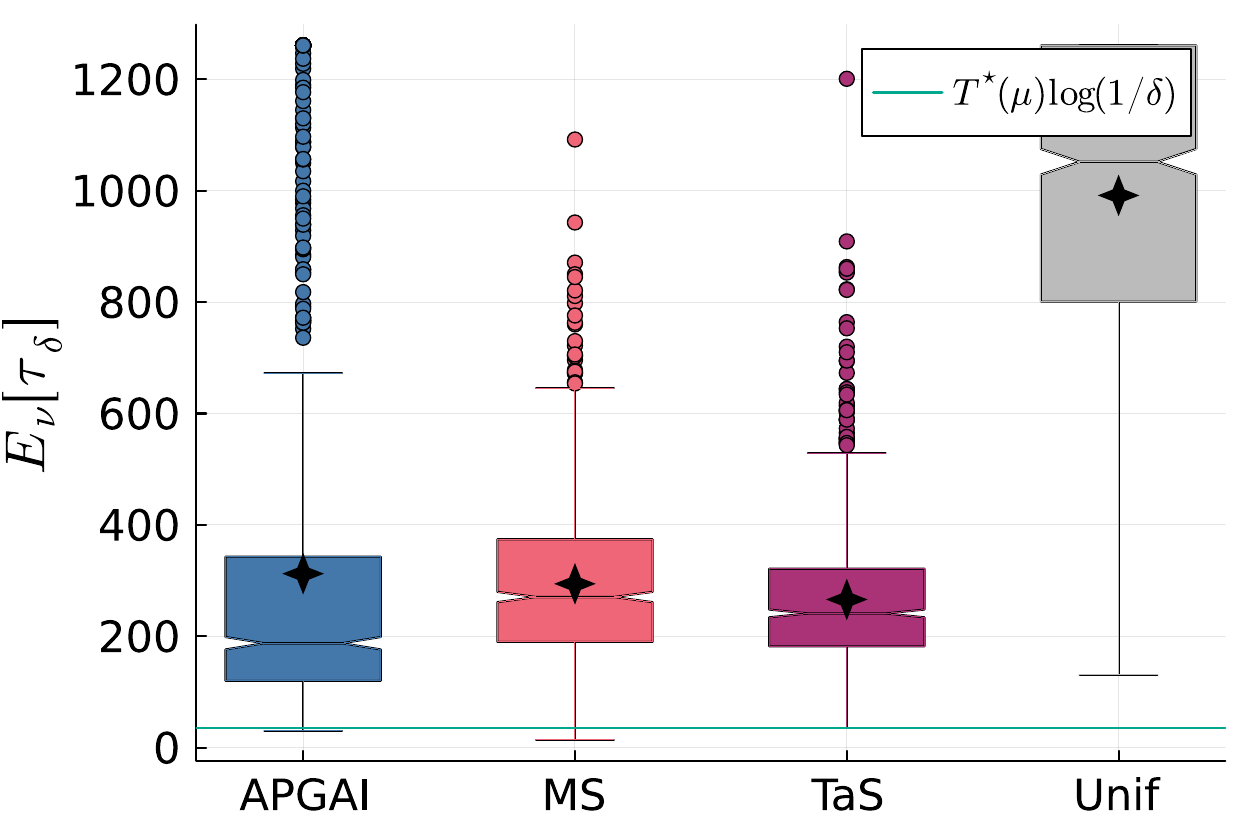}
    \includegraphics[width=0.32\linewidth]{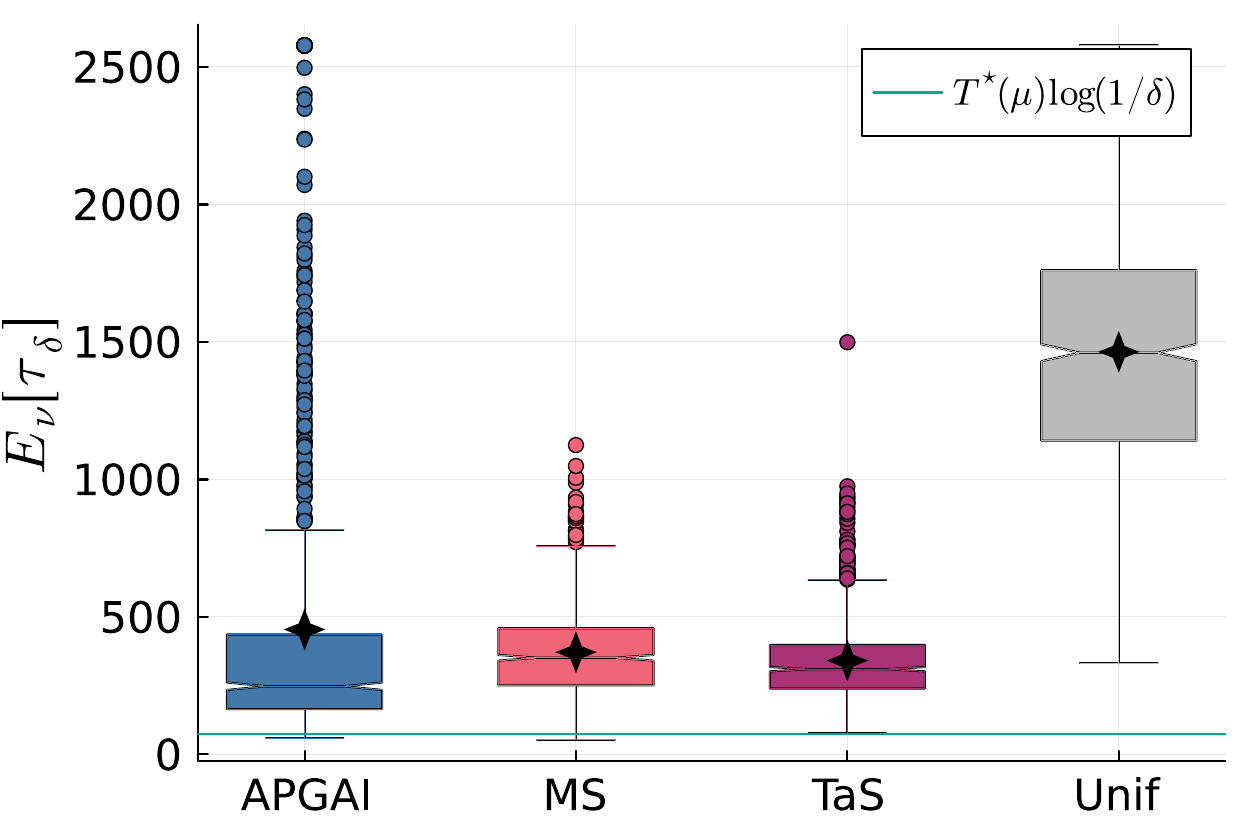}
    \includegraphics[width=0.32\linewidth]{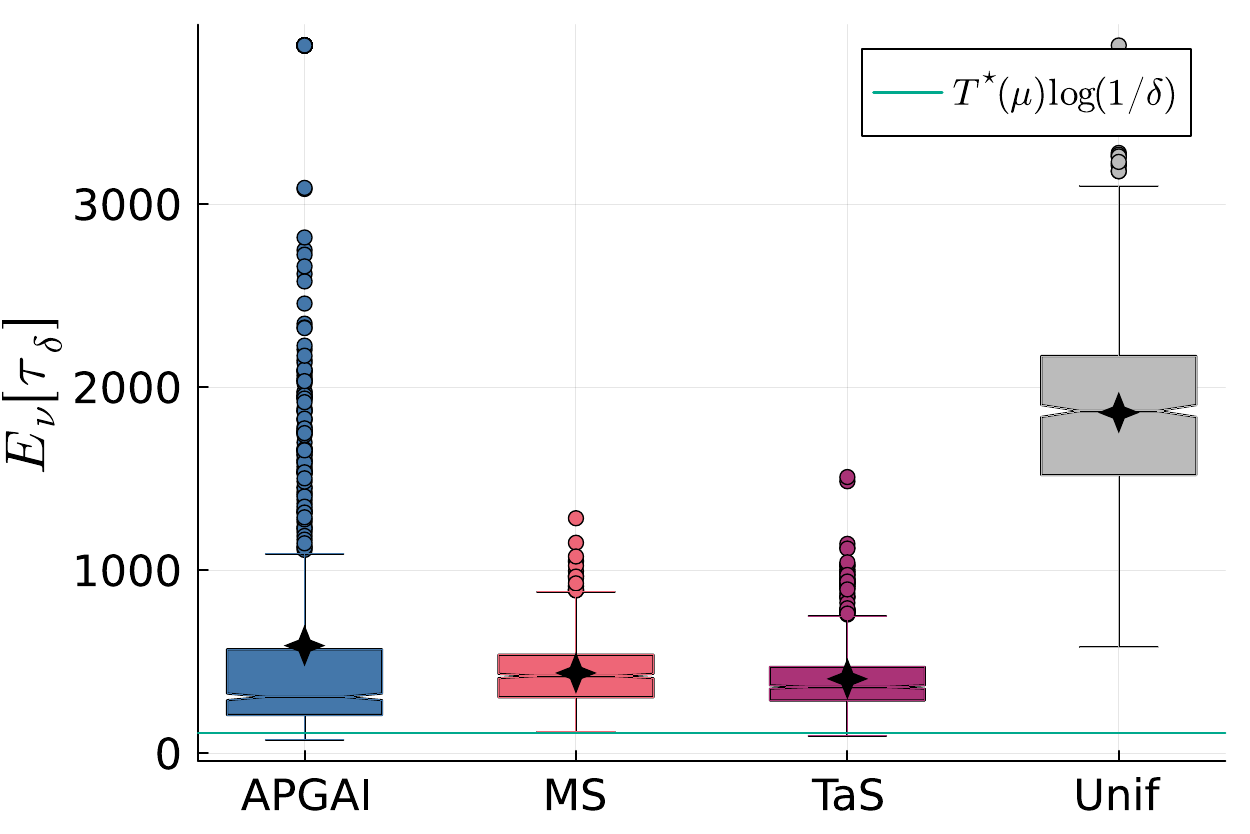}
    \caption{Empirical error for varying confidence level $\delta  \in \{10^{-2}, 10^{-4}, 10^{-6}\}$ (left to right) on instances (top) \textsc{Thr2} and (bottom) \textsc{IsA1}.}
    \label{fig:supp_varying_delta}
\end{figure}

In Figure~\ref{fig:supp_varying_delta}, we study the impact of a varying confidence level on instances where \hyperlink{APGAI}{APGAI} suffers from large outliers.
For a fair comparison, we only consider fixed-confidence algorithm whose sampling rule is independent of $\delta$ (\ie excluding LUCB-G and HDoC).
As expected, the large outliers phenomenon also increases when $\delta$ decreases.

 \begin{table}[t]
     \centering
     \scalebox{0.95}{
     \begin{tabular}{l r r r r r r r r r r}
    \toprule
         FE & \textsc{Thr1}  & \textsc{Thr2} & \textsc{Thr3}  & \textsc{Med1} & \textsc{Med2} & \textsc{IsA1} &  \textsc{IsA2} & \textsc{RealL} & \textsc{NoA1} & \textsc{NoA2}   \\
         \midrule
         No & $634$ & $2448$ & $12301$ & $22588$ & $184$ & $544$ & $159$ & $3721$ & $288$ & $3014$ \\
                      & $\pm 2091$ & $\pm 4269$ & $\pm 4755$ & $\pm 9204 $ & $\pm 147$ & $\pm 1591$ & $\pm 557$ & $\pm 12511$ & $\pm 56$ & $\pm 1031$ \\
         Yes & $341$ & $1466$ & $12584$ & $22394$ & $216$ & $341$ & $72$ & $921$ & $287$ & $3022$ \\
                       & $\pm 505$ & $\pm 2833$ & $\pm 4818$ & $\pm 8942$ & $\pm 106$ & $\pm 444$ & $\pm 49$ & $\pm 1389$ & $\pm 55$ & $\pm 1025$ \\
          \bottomrule
     \end{tabular}}
     \caption{Empirical stopping time ($\pm$ standard deviation) of \protect\hyperlink{APGAI}{APGAI} with or without forced exploration.}
     \label{tab:APGAI_forced_exploration}
 \end{table}

\textit{Fixing APGAI with forced exploration.}
In the fixed-confidence setting, \hyperlink{APGAI}{APGAI} can suffer from large outliers when good arms have dissimilar means since it can greedily focus on good arms with small gaps.
To fix this limitation, we propose to add forced exploration on top of \hyperlink{APGAI}{APGAI}, which we refer to as \hyperlink{APGAI}{APGAI}-FE.
Let $\mathcal U_{t} = \{ a \in \ARMS \mid \nsamples{a}{t} \le \sqrt{t} - K/2\}$.
When $\mathcal U_{t} \ne \emptyset$, we pull $a_{t+1} \in \argmin_{a \in \mathcal U_{t}} \nsamples{a}{t}$.
When $\mathcal U_{t} = \emptyset$, we pull according to \hyperlink{APGAI}{APGAI} sampling rule.

Table~\ref{tab:APGAI_forced_exploration} shows that adding forced exploration significantly reduce the mean and the variance of the stopping time on instances where \hyperlink{APGAI}{APGAI} was prone to large outliers.
For instances where \hyperlink{APGAI}{APGAI} had no large outliers, \hyperlink{APGAI}{APGAI}-FE has the same empirical performance.
Therefore, adding forced exploration allows to circumvent the empirical shortcomings of \hyperlink{APGAI}{APGAI} in the fixed-confidence setting.
 
\bibliography{main}

\end{document}